\documentclass{article}

\usepackage[final]{neurips_2023}

\usepackage[utf8]{inputenc} 
\usepackage[T1]{fontenc}    
\usepackage{hyperref}       
\usepackage{url}            
\usepackage{booktabs}       
\usepackage{amsfonts}       
\usepackage{nicefrac}       
\usepackage{microtype}      
\usepackage{xcolor}         

\usepackage{graphicx}
\usepackage{caption}
\usepackage{subcaption}

\usepackage{xcolor}    
\usepackage{bbm}
\usepackage{soul}

\usepackage{arydshln}

\usepackage{placeins}

\usepackage{multibib}
\usepackage[toc,page]{appendix}

\usepackage{amsmath}
\usepackage{amssymb}
\usepackage{mathtools}
\usepackage{amsthm}

\newcommand{\crtodo}[1]{}

\newcites{app}{References}
\newcommand{\etal}{et~al.\ }

\newcommand{\bb}[1]{\mathbb{#1}}
\newcommand{\cl}[1]{\mathcal{#1}}

\newcommand{\Vst}[2]{\hat{V}^{#2}_{\textnormal{sft}}(#1)}
\newcommand{\Vsp}[2]{V^{#2}_{\textnormal{sft}}(#1)}
\newcommand{\Vss}[1]{V^*_{\textnormal{sft}}(#1)}
\newcommand{\Vt}[2]{\hat{V}^{#2}(#1)}
\newcommand{\Qst}[3]{\hat{Q}^{#3}_{\textnormal{sft}}(#1,#2)}
\newcommand{\Qsp}[3]{Q^{#3}_{\textnormal{sft}}(#1,#2)}
\newcommand{\Qss}[2]{Q^*_{\textnormal{sft}}(#1,#2)}

\newcommand{\Qt}[3]{\hat{Q}^{#3}(#1,#2)}

\newcommand{\pis}{\pi^*_{\textnormal{sft}}}

\newcommand{\suc}[1]{\cup_{a\in\cl{A}}\textnormal{Succ}(#1,a)}
\newcommand{\succc}[2]{\textnormal{Succ}(#1,#2)}
\newcommand{\reg}{\textnormal{reg}}

\newcommand{\one}{\mathbbm{1}}
\DeclareMathOperator*{\argmax}{arg\,max}
\DeclareMathOperator*{\argmin}{arg\,min}
\newcommand{\rap}{\overset{p}{\to}}

\theoremstyle{plain}
\newtheorem{theorem}{Theorem}[section]
\newtheorem{prop}{Proposition}[section]
\newtheorem{corollary}{Corollary}[theorem]
\newtheorem{lemma}[theorem]{Lemma}

\newenvironment{proofoutline}{\proof[Proof outline]}{\endproof}

\newtheorem{innercustomgeneric}{\customgenericname}
\providecommand{\customgenericname}{}
\newcommand{\newcustomtheorem}[2]{%
	\newenvironment{#1}[1]
	{%
		\renewcommand\customgenericname{#2}%
		\renewcommand\theinnercustomgeneric{##1}%
		\innercustomgeneric
	}
	{\endinnercustomgeneric}
}

\newcustomtheorem{customthm}{Theorem}
\newcustomtheorem{customcorollary}{Corollary}
\newcustomtheorem{customcounter}{Counterexample}
\newcustomtheorem{customremark}{Remark}
\newcustomtheorem{customprop}{Proposition}

\title{Monte Carlo Tree Search with Boltzmann Exploration}

\author{%
	Michael Painter, Mohamed Baioumy, Nick Hawes, Bruno Lacerda \\
	Oxford Robotics Institute \\
	University of Oxford \\
	\texttt{\{mpainter, mohamed, nickh, bruno\}@robots.ox.ac.uk} \\
}

\begin{document}

\maketitle

\begin{abstract}
    Monte-Carlo Tree Search (MCTS) methods, such as Upper Confidence Bound applied to Trees (UCT), are instrumental to automated planning techniques. However, UCT can be slow to explore an optimal action when it initially appears inferior to other actions. Maximum ENtropy Tree-Search (MENTS) incorporates the maximum entropy principle into an MCTS approach, utilising \textit{Boltzmann policies} to sample actions, naturally encouraging more exploration. In this paper, we highlight a major limitation of MENTS: optimal actions for the maximum entropy objective do not necessarily correspond to optimal actions for the original objective. We introduce two algorithms, Boltzmann Tree Search (BTS) and Decaying ENtropy Tree-Search (DENTS), that address these limitations and preserve the benefits of Boltzmann policies, such as allowing actions to be sampled faster by using the Alias method. Our empirical analysis shows that our algorithms show consistent high performance across several benchmark domains, including the game of Go. 
\end{abstract}

\section{Introduction} \label{sec:intro}
	Planning under uncertainty is a core problem in Artificial Intelligence, commonly modelled as a Markov Decision Process (MDP) or variant thereof. MDPs can be solved using dynamic programming techniques to obtain an optimal policy \cite{bellman1957markovian}. However, computing a full optimal policy does not scale to large state-spaces, necessitating the use of 
    heuristic solvers \cite{hansen2001lao, bonet2003labeled} and online, sampling-based, planners based on  Monte-Carlo Tree-Search (MCTS), such as the  Upper Confidence Bound applied to Trees (UCT) algorithm~\cite{kocsis2006uct}.
    
	The UCT search policy is designed to minimise \textit{cumulative regret}, so manages a trade-off between exploration and exploitation. To exploit, UCT often selects the same action on successive trials, which can result in it getting stuck in local optima. Conversely, Maximum ENtropy Tree Search (MENTS) places a greater emphasis on exploration by combining MCTS with techniques from maximum entropy policy optimisation~\cite{ziebart2008maximum, haarnoja2017reinforcement, haarnoja2018soft}. MENTS jointly maximises cumulative rewards and policy entropy, where a \textit{temperature} parameter controls the weight of the entropy objective. However, MENTS is sensitive to this temperature parameter, and may not converge to the reward maximising policy or require a prohibitively low temperature to do so.

    In this work, we consider scenarios where MCTS methods are used with a simulator to plan how an agent should act. We introduce two algorithms for this scenario that address the above limitations. First, we present Boltzmann Tree Search (BTS) which uses a Boltzmann search policy like MENTS, but optimises for reward maximisation only. Secondly, we introduce Decaying ENtropy Tree Search (DENTS), which adds entropy backups to BTS, but is still \textit{consistent} (i.e. it converges to the reward maximising policy in the limit). 
    
    The main contributions of this paper are:~(1)~Demonstrating that the maximum entropy objective used in MENTS can be misaligned with reward maximisation, thus preventing it from converging to the optimal policy; (2)~Introducing two new algorithms, BTS and DENTS, which preserve the benefits of using Boltzmann search policies while being as simple to implement as UCT and MENTS, but converge to the reward maximising policy; (3)~Analysing MENTS, BTS and DENTS through the lens of \textit{simple regret} to provide theoretical convergence results; 
    (4)~Highlighting and demonstrating that the Alias method \cite{alias1,alias2} can be used with stochastic action selection to improve the asymptotic complexity of running a fixed number of trials over existing MCTS algorithms; and (5)~Demonstrating the performance improvements of Boltzmann search policies used in BTS and DENTS in benchmark gridworld environments and the game of Go.

\section{Background}
    \subsection{Markov Decision Processes} \label{sec:mdp}
        We define a (finite-horizon) MDP as a tuple $\cl{M}=(\cl{S}, \cl{A}, p, R, H)$, where $\cl{S}$ is the set of states; $\cl{A}$ is the set of actions; $p: \cl{S} \times \cl{A} \times \cl{S} \rightarrow[0,1]$ is the transition function where $p(s'|s,a,)$ is the probability of moving to state $s'$ given that action $a$ was taken in state $s$; $R: \cl{S} \times \cl{A} \rightarrow \bb{R}$ is the reward function; and $H\in\bb{N}$ is the finite horizon. Let $\succc{s}{a}$ denote the set of successor states of $(s,a)$, i.e. $\succc{s}{a}=\{s' \in \cl{S}\mid p(s'|s,a)>0\}$.
        
        A policy $\pi$ maps a state and timestep to a distribution over actions, and we denote the probability of executing $a$  at state $s$ and timestep $t$ as $\pi(a | s,t)$.
        %
        Let $s_t$ denote the state after $t$ time-steps and $a_t$ the action selected at $s_t$, according to $\pi$. The expected value $V^{\pi}$ and expected state-action value $Q^{\pi}$ of  $\pi$ are defined as: 
        \begin{align}
            V^{\pi}(s, t) &= \bb{E}_{\pi} \left[\sum_{i=t}^{H} R\left(s_{i}, a_{i}\right) \Big\vert s_t=s \right], \\
            Q^{\pi}(s,a,t) &= R(s,a) + \bb{E}_{s'\sim p(\cdot | s,a)}[V^{\pi}(s',t+1)].
        \end{align}
        The goal is to find the \textit{optimal policy} $\pi^*$ with the maximum expected reward: $\pi^{*}=\argmax_{\pi} V^{\pi}$. The optimal value functions are then defined as $V^*=V^{\pi^*}, Q^*=Q^{\pi^*}$. 
        For an MDP, there always exists an optimal policy $\pi^{*}$ which is deterministic \cite{kolobov2012planning}.
        

    \subsection{Maximum entropy policy optimization}
        In planning and reinforcement learning, the agent usually aims to maximise the expected sum of rewards. In maximum entropy policy optimisation, the objective is augmented with the expected entropy of the policy \cite{haarnoja2017reinforcement, ziebart2008maximum}. Formally, this is expressed as:
        \begin{align}
            \Vsp{s,t}{\pi} = \bb{E}_{\pi} \left[\sum_{i=t}^{H} R\left(s_{i}, a_{i}\right)+ \alpha \mathcal{H}\left(\pi\left(\cdot|s_{i}, i\right)\right) \Big\vert s_t=s\right], \label{eq:vsft_def}
        \end{align}
        \noindent where $\alpha\geq 0$ is a temperature parameter, and $\cl{H}$ is the Shannon entropy function. The temperature determines the relative importance of the entropy against the reward and thus controls the stochasticity of the optimal policy. The conventional reward maximisation objective can be recovered by setting $\alpha = 0$. 
        
        An optimal value function for maximum entropy optimization is obtained using the \textit{soft} Bellman optimality equations \cite{haarnoja2018soft}:
        \begin{align}
            \Qss{s}{a,t} = R(s,a) 
            + \bb{E}_{s' \sim p(\cdot | s,a)} 
            \left[ \Vss{s',t} \right], \label{eq:v_soft_bellman} \\
            \Vss{s,t} = 
            \alpha \log \sum_{a\in\cl{A}} \exp(\Qss{s}{a,t}/\alpha), \label{eq:q_soft_bellman}
        \end{align}
        which corresponds to a standard Bellman backup, with the $\max$ replaced by a \textit{softmax}, shown in Equation~(\ref{eq:q_soft_bellman}). The optimal soft policy $\pis=\argmax_\pi V_{\text{sft}}^\pi$ can be computed directly \cite{nachum2017bridging} as follows:
        \begin{align}
            \pis(a|s,t) = \exp((\Qss{s}{a,t}-\Vss{s,t})/\alpha). \label{eq:optimal_soft_policy}
        \end{align}
         Note that the soft policy is always stochastic for any $\alpha > 0$. Henceforth, we will use \textit{soft value} to refer to value functions indexed with `sft', \textit{(optimal) soft policy} to refer to policies of the form given in Equation (\ref{eq:optimal_soft_policy}), and \textit{standard value} and \textit{(optimal) standard policy} for values and policies of the form given in Section \ref{sec:mdp}, unless it is clear from the context.
         
         For the remainder of this paper we will drop the timestep $t$ from policies and value functions to simplify notation.

    \subsection{Monte-Carlo tree search} \label{sec:mcts}
    		MCTS methods build a search tree $\cl{T}$ using Monte-Carlo trials. Each trial is split into two phases: starting from the root node, actions are chosen according to a \textit{search policy} and states sampled from the transition distribution until the first state not in $\cl{T}$ is reached. A new node is added to $\cl{T}$ and its value is initialised using some function $V^{\text{init}}$, often using a \textit{rollout policy} to select actions until the time horizon $H$ is reached. In the second phase, the return for the trial is back-propagated up (or `backed up') the tree to update the values of nodes in $\cl{T}$. For a reader unfamiliar with MCTS, we refer to \cite{browne2012survey} for a review of the MCTS literature, as many variants of MCTS exist and may vary from our description. 
        
        Two critical choices in designing an MCTS algorithm are the search policy (which needs to balance exploration and exploitation) and the backups (how values are updated). MCTS algorithms are often designed to achieve \textit{consistency} (i.e. convergence to the optimal action in the limit), which implies that running more trials will increase the probability that the optimal action is recommended.
        
        To simplify notation we assume that each node in the search tree corresponds to a unique state, so we may represent nodes using states. Our algorithms and results do not make use of this assumption, and generalise to when this assumption does not hold.
    
        \paragraph{UCT} 
            UCT~\cite{kocsis2006uct} applies the upper confidence bound (UCB) in its search policy to balance exploration and exploitation. The $n$th trial of UCT operates as follows: let $\cl{T}$ be the current search tree and let $\tau=(s_0,a_0,...,a_{h-1},s_{h})$ denote the trajectory of the $n$th trial, where $s_h\not\in\cl{T}$ or $h=H$. At each node $s_t$ the UCT search policy $\pi_{\text{UCT}}$ will select a random action that has not previously been selected, otherwise, it will select the action with maximum UCB value:
            %
            \begin{align}
                \pi_{\textnormal{UCT}}(s) &= \max_{a\in\cl{A}} \bar{Q}(s, a)+c \sqrt{\frac{\log N(s)}{N(s, a)}}, \label{eq:uct_distr}
            \end{align}
            \noindent where, $\bar{Q}(s,a)$ is the current empirical Q-value estimate, $N(s)$ (and $N(s,a)$) is how many times $s$ has been visited (and action~$a$ selected) and $c$ is an exploration parameter. Then, $s_h$ is added to the tree: $\cl{T}\leftarrow \{s_h\}\cup\cl{T}$. The backup consists of updating empirical estimates for $t=h-1,...,0$:
            \begin{align}
                \bar{Q}(s_t, a_t) &\leftarrow \bar{Q}(s_t, a_t) + \frac{\bar{R}(t) - \bar{Q}(s_t, a_t)}{N(s_t, a_t) + 1}, \label{eq:uct_qbar}
            \end{align}
            \noindent where $\bar{R}(t) = V^{\text{init}}(s_h) + \sum_{i=t}^{h-1} R(s_i,a_i)$, and $V^{\text{init}}(s_h)=\sum_{i=h}^{H} R(s_i,a_i)$ if using a rollout policy.

        \paragraph{MENTS} 
            MENTS \cite{xiao2019maximum} combines maximum entropy policy optimization \cite{haarnoja2017reinforcement, ziebart2008maximum} with MCTS. Algorithmically, it is similar to UCT. The two differences are: (1) the search policy follows a stochastic Boltzmann policy, and (2) it uses soft values that are updated with dynamic programming backups. The MENTS search policy $\pi_{\textnormal{MENTS}}$ is given by:
            \begin{align}
                \pi_{\textnormal{MENTS}}(a|s) &= (1-\lambda_s)\rho_{\textnormal{MENTS}}(a|s) + \frac{\lambda_s}{|\cl{A}|}, \\
                \rho_{\textnormal{MENTS}}(a|s) &= \exp\left(\frac{1}{\alpha}\left(\Qst{s}{a}{}-\Vst{s}{}\right)\right) \label{eq:rhosft}
            \end{align}
            \noindent where $\lambda_s=\min(1,\epsilon/\log(e+N(s))),$ $\epsilon \in (0,\infty)$ is an exploration parameter and $\Vst{s}{}$ (and $\Qst{s}{a}{}$) are the current soft (Q-)value estimates. 
            The soft value of the new node is initialised $\Vst{s_h}{}\leftarrow V^{\text{init}}(s_h)$ and the soft values are updated with backups for $t=h-1,...,0$:
            \begin{align}
                \Qst{s_t}{a_t}{} &\leftarrow R(s_t,a_t) + \sum_{s'\in\succc{s}{a}} \left( \frac{N(s')}{N(s_t,a_t)} \Vst{s'}{} \right), \\ 
                \Vst{s_t}{} &\leftarrow \alpha \log \sum_{a\in\cl{A}} \exp \left(\frac{1}{\alpha}\Qst{s_t}{a}{} \right). \label{eq:soft_v_backup}
            \end{align}
            Each $\Qst{s}{a}{}$ is initialised using another function $Q^{\text{init}}_{\text{sft}}(s,a)$ (but is typically zero).

    \subsection{Simple regret}
    		UCB \cite{auer2002finite} is frequently used in MCTS methods to minimise \textit{cumulative regret} during the tree search. Cumulative regret is most appropriate in scenarios where the actions taken during tree search have an associated real-world cost. However, MCTS methods often use a simulator during the tree search, where the only significant real-world cost is associated with taking the recommended action after the tree search. In such scenarios, simple regret \cite{simple_regret_short,simple_regret_long} is more appropriate for analysing the performance of algorithms, as it only considers the cost of the actions that are actually executed. Under simple regret, algorithms are not penalised for under-exploiting during the search, thus can explore more, which leads to better recommendations by allowing algorithms to confirm that bad actions are indeed of lower value.

        We consider the problem of MDP planning as a sequential decision problem, where for each round $n$:
        \begin{enumerate}
            \item the forecaster algorithm produces a search policy $\pi^n$ and samples a trajectory $\tau\sim\pi^n$,
            \item the environment returns the rewards $R(s_t,a_t)$ for each $s_t,a_t$ pair in $\tau$,
            \item the forecaster algorithm produces a recommendation policy $\psi^n$,
            \item if environment sends stop signal, then end, else return to step 1.
        \end{enumerate}
        
        The \textit{simple regret} of the forecaster on it's $n$th round is then:
        \begin{align}
            \reg(s,\psi^n) = V^*(s)-V^{\psi^n}(s).
        \end{align}
        %
        %
        In MENTS the 
        recommendation policy suggested in \cite{xiao2019maximum} can be written:
        \begin{align}
            \psi_{\text{MENTS}}(s)=\argmax_{a\in\cl{A}}\Qst{s}{a}{}. 
        \end{align}
        We can now formally define \textit{consistency}: an algorithm is \textit{consistent} if and only if its recommendation policy $\psi^n$ converges to an expected simple regret of zero: $\bb{E}[\reg(s,\psi^n)]\rightarrow 0$ as $n\rightarrow \infty$. Note that because we are considering randomised algorithms, there is a distribution over the possible recommendation policies that could have been produced. If a policy has a simple regret of zero then it implies it is an optimal policy.

\section{Limitations of prior MCTS methods} \label{sec:limitations}
    In this section we use the \textit{D-chain problem} introduced in \cite{coquelin2007_uct} (Figure \ref{fig:dchain_illustration}) to highlight the limitations of UCT and MENTS. In the D-chain problem, when an agent chooses action~$a_L$, from some state~$d$, it moves to an absorbing state and receives a reward of $(D-d)/D$. In state~$D$, action~$a_R$ corresponds to an absorbing state with reward $R_f = 1$. The optimal standard policy always selects action~$a_R$. 
    
    \begin{figure}
        \centering
        \includegraphics[width=0.7\textwidth]{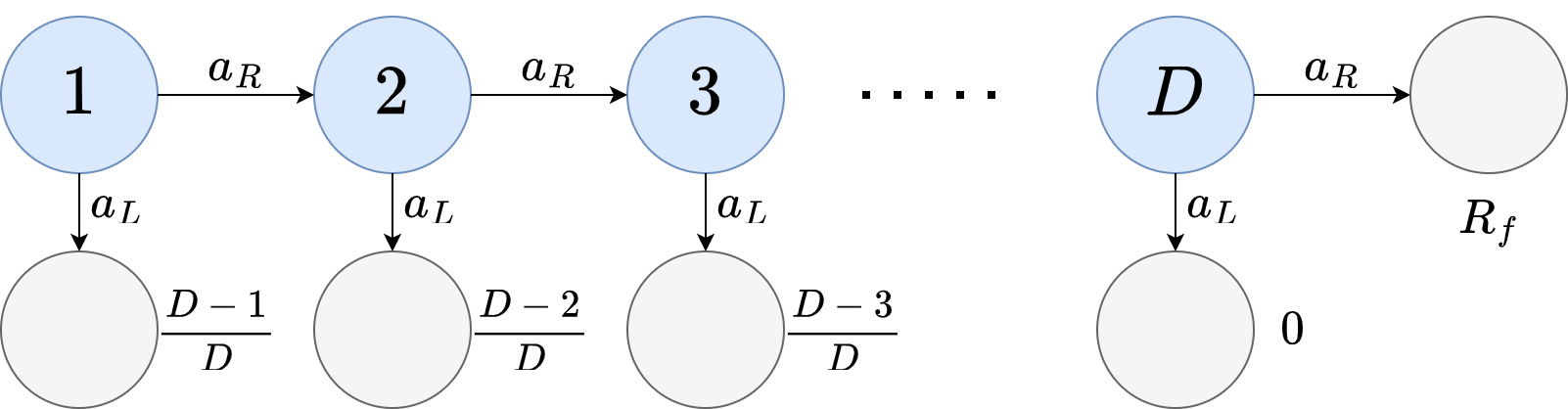}
        \caption{An illustration of the \textit{(modified) D-chain problem}, where 1 is the starting state, 
            transitions are deterministic and values next to states represent 
            rewards for arriving in that state.}
        \label{fig:dchain_illustration}
    \end{figure}
	
	In the 10-chain problem ($D=10$), UCT will recommend action~$a_L$ from state~$1$ (Figure \ref{fig:10_chain_results}). UCT requires $\Omega(\exp(...\exp(1)...))$ many trials ($D$ composed exponential functions) to recommend the optimal policy that reaches the reward of $R_f=1$ \cite{coquelin2007_uct}. This highlights the first limitation mentioned in Section \ref{sec:intro}: UCT quickly disregards action~$a_R$ at the initial state, to exploit the reward of $0.9$.

    When MENTS is run on the 10-chain problem, with the help of the entropy term it quickly finds the final reward of $R_f=1$ 
    (Figure \ref{fig:10_chain_results}). However, consider the \textit{modified 10-chain} with $R_f=1/2$ instead. Repeated applications of Equations (\ref{eq:v_soft_bellman}) and (\ref{eq:q_soft_bellman}) for $\alpha=1$ gives the optimal soft values of $\Qss{1}{a_R}=\log(\exp(1/2)+\sum_{i=0}^8\exp(i/10))\approx 2.74$ and $\Qss{1}{a_L}=0.9$. So in the modified 10-chain problem, we have $\Qss{1}{a_R}>\Qss{1}{a_L}$ with $\alpha=1$, whereas $Q^*(1,a_R)<Q^*(1,a_L)$. Thus, when MENTS converges, 
    it will recommend the wrong action with respect to the standard 
    objective (Figure \ref{fig:modified_10_chain_results}), i.e. it is not consistent. The modified 10-chain is an example of Proposition \ref{prop:ments_bad}, which states that MENTS will not always converge to the standard optimal policy.
    
    \begin{prop}\label{prop:ments_bad}
        There exists an MDP $\cl{M}$ and temperature $\alpha$ such that $\bb{E}[\reg(s_0,\psi^n_{\textnormal{MENTS}})] \not\to 0$ as $n\to\infty$. That is, MENTS is not consistent.
    \end{prop}
    \begin{proof}
    		Proof is by example with $\alpha=1$ in the modified 10-chain (Figure \ref{fig:dchain_illustration}).
    \end{proof}
    We can reduce the value of $\alpha$ to decrease the importance of entropy in the soft objective $\bb{E}\left[R\left(s_{t}, a_{t}\right) + \alpha \mathcal{H}\left(\pi\left(\cdot|s_t\right)\right)\right]$.
    If $\alpha$ is small enough, then MENTS recommendations can converge to the optimal standard policy (Theorem \ref{thrm:ments}). Hence, MENTS with a low temperature can solve the modified 10-chain problem (Figure \ref{fig:modified_10_chain_results}). 
    However, in practice, a low temperature will often cause MENTS to not sufficiently explore, as demonstrated in the original D-chain (Figure \ref{fig:10_chain_results}).
    %
    
    \begin{customthm}{3.2}\label{thrm:ments}
        For any MDP $\cl{M}$, after running $n$ trials of the MENTS algorithm with $\alpha \leq \Delta_{\cl{M}}/3H\log|\cl{A}|$, there exists constants $C,k>0$ such that: $\bb{E}[\reg(s_0,\psi_{\textnormal{MENTS}})] \leq C\exp(-kn)$, where $\Delta_{\cl{M}}=\min \{Q^*(s,a)-Q^*(s,a')\vert Q^*(s,a) \neq Q^*(s,a'),s\in\cl{S}, a,a'\in\cl{A},t\in\bb{N}\}$.
    \end{customthm}
    \begin{proofoutline}
    		We can show $\Qst{s}{a}{}\rap \Qss{s}{a}$ (Corollary \ref{cor:sftq_convg}) similarly to Thoerem \ref{thrm:dents}. The bound on $\alpha$ is required to ensure that $\pi^*(s_0)=\pis(s_0)$.
    \end{proofoutline}
    
    In conclusion, similar MDPs can require vastly different temperatures for MENTS to be effective. We discuss MENTS sensitivity to the temperature parameter further in Appendix \ref{app:param_sens}, and demonstrate this parameter sensitivity in the Frozen Lake environment (Section \ref{sec:gridworlds}) in Figure \ref{fig:fl_param_sens_ments} in the appendix. 
    %
    \begin{figure}
        \begin{subfigure}[b]{0.49\textwidth}
            \centering
            \includegraphics[width=\textwidth]{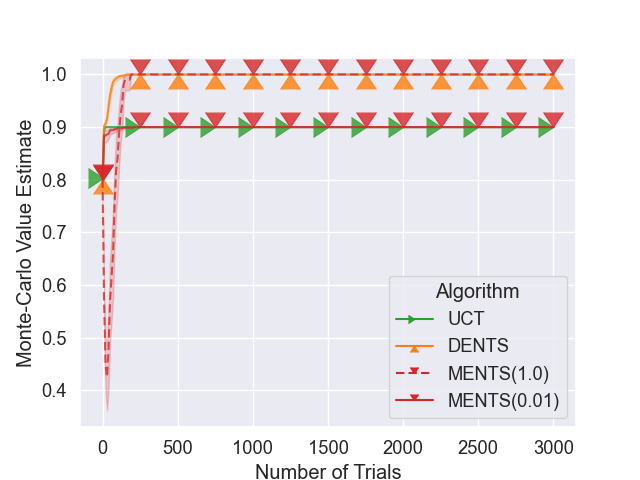}
            \caption{10-chain.}
            \label{fig:10_chain_results}
        \end{subfigure}
        \hfill
        \begin{subfigure}[b]{0.49\textwidth}
            \centering
            \includegraphics[width=\textwidth]{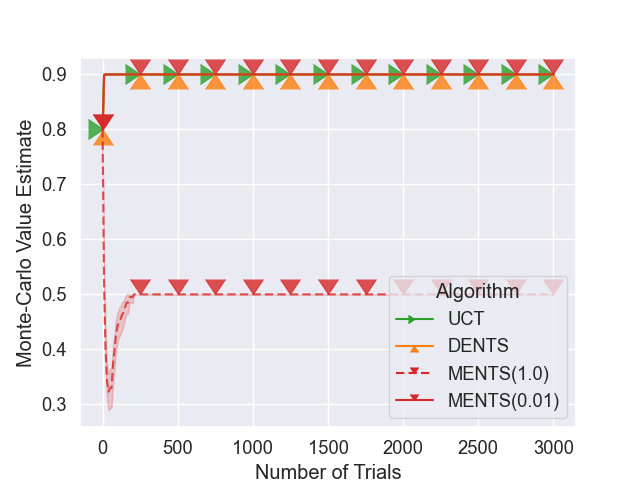}
            \caption{Modified 10-chain.}
            \label{fig:modified_10_chain_results}
        \end{subfigure}
        \caption{A comparison of MENTS, DENTS and UCT when run on the (modified) 10-chain.}
        \label{fig:dchain_graphs}
    \end{figure}

\section{Boltzmann search} \label{sec:boltzmann_search}
    We now introduce two algorithms that utilise Boltzmann search policies similar to MENTS and admit bounded simple regrets that converge to zero without restrictive constraints on parameters.
    Thus, they do not suffer from sensitivity to parameter selection that MENTS does. Both algorithms use action selection and value backups that are easy to implement and use.   
    We designed these algorithms with consistency in mind, which in practice, means that if we run more trials then we (with high probability) will recommend a better solution (note that Proposition \ref{prop:ments_bad} implies that this is not always the case for MENTS).

    \subsection{Boltzmann Tree Search} \label{sec:bts}
        Our first approach, put simply, replaces the use of soft values in MENTS with 
        \textit{Bellman} 
        values. We call this algorithm \textit{Boltzmann Tree Search} (BTS). BTS promotes exploration through the stochastic Boltzmann search policy, like MENTS, while using backups that optimise for the standard objective, like UCT. The search policy $\pi_{\textnormal{BTS}}$ and backups for the $n$th trial are given by:
        \begin{align}
            \pi_{\textnormal{BTS}}(a|s) &= (1-\lambda_s)\rho_{\textnormal{BTS}}(a|s) + \frac{\lambda_s}{|\cl{A}|}, \label{eq:bts_search_policy} \\ 
            \rho_{\textnormal{BTS}}(a|s) &\propto \exp\left(\frac{1}{\alpha}\left(\Qt{s}{a}{}\right)\right), \label{eq:bts_value_policy} \\ 
            \Qt{s_t}{a_t}{} &\leftarrow R(s_t,a_t) + \sum_{s' \in \succc{s_t}{a_t}} \left( \frac{N(s')}{N(s_t,a_t)} \Vt{s'}{} \right), \label{eq:dp_q_backup} \\ 
            \Vt{s_t}{} &\leftarrow\max_{a\in\cl{A}} \Qt{s_t}{a}{}, \label{eq:dp_v_backup} 
        \end{align}
        for $t=h-1,...,0$, where $\hat{V}$ and $\hat{Q}$ are 
        the current Bellman (Q-)value estimates,  
        $\lambda_s=\min(1,\epsilon/\log(e+N(s)))$, $\epsilon \in (0,\infty)$ is an exploration parameter and $\alpha$ is a search temperature (unrelated to entropy). 
        Each $\hat{V}(s)$ and $\hat{Q}(s,a)$ are initialised using $V^{\text{init}}$ and $Q^{\text{init}}$ functions similarly to MENTS. The 
        Bellman
        values are used for recommendations:
        \begin{align}
            \psi_{\text{BTS}}(s)=\argmax_{a\in\cl{A}}\Qt{s}{a}{}.
        \end{align}
        By using 
        Bellman
        backups, we can guarantee that the BTS recommendation policy converges to the optimal standard policy for any temperature $\alpha$, given enough time. In other words, BTS is consistent.
        \begin{theorem} \label{thrm:bts}
            For any MDP $\cl{M}$, after running $n$ trials of the BTS algorithm with a root node of $s_0$, there exists constants $C,k>0$ such that for all $\varepsilon>0$ we have $\bb{E}[\reg(s_0,\psi_{\textnormal{BTS}})] \leq C\exp(-kn)$, and also $\Vt{s_0}{} \rap V^*(s_0)$ as $n\rightarrow\infty$.
        \end{theorem}
        \begin{proofoutline}
        		This result is a special case of Theorem \ref{thrm:dents} by setting $\beta(m)=0$.
        \end{proofoutline}

    \subsection{Decaying Entropy Tree Search} \label{sec:dents}
        Secondly, we present Decaying ENtropy Tree Search (DENTS), which can effectively interpolate between the MENTS and BTS algorithms. DENTS also uses the dynamic programming backups from equations (\ref{eq:dp_q_backup}) and (\ref{eq:dp_v_backup}), but adds an \textit{entropy backup}. The \textit{entropy values} are weighted by a bounded non-negative function $\beta(N(s))$ in the DENTS search policy $\pi_{\textnormal{DENTS}}$:
        \begin{align}
            \pi_{\textnormal{DENTS}}(a|s) &= (1-\lambda_s)\rho_{\textnormal{DENTS}}(a|s) + \frac{\lambda_s}{|\cl{A}|}, \label{eq:dents_search_policy} \\
            \rho_{\textnormal{DENTS}}(a|s) &\propto \exp\left(\frac{1}{\alpha}\left(\Qt{s}{a}{} + \beta(N(s))\cl{H}_Q^{}(s,a) \right)\right), \label{eq:dents_policy} \\
            \cl{H}_V(s_t) &\leftarrow \cl{H}(\pi_{\textnormal{DENTS}}(\cdot | s_t)) + \sum_{a\in\cl{A}} \pi_{\textnormal{DENTS}}(a|s_t)\cl{H}_Q(s_t,a), \label{eq:entropy_v} \\
            \cl{H}_Q(s_t,a_t) &\leftarrow \sum_{s'\in \succc{s_t}{a_t}} \frac{N(s')}{N(s_t,a_t)} \cl{H}_V(s'),  
        \end{align}
        for $t=h-1,...,0$, where $\cl{H}_V(s)$ and $\cl{H}_Q(s,a)$ are the \textit{entropy values} of the search policy rooted at $s$ and $(s,a)$ respectively, and $\alpha,\lambda_s$ are the same as for BTS, as described in Section \ref{sec:bts}. Initial values are the same as Section \ref{sec:bts}, and the entropy values are initialised to zero.  In DENTS we can  view $\Qt{s}{a}{} + \beta(N(s))\cl{H}_Q(s,a)$ as a soft value for $(s,a)$. Hence, by setting $\beta(m)=\alpha$, the DENTS search will mimic the MENTS search (demonstrated in Appendix \ref{app:dents_mimic_ments}), and if $\beta(m)=0$ then the algorithm reduces to the BTS algorithm. By using a decaying function for $\beta$ we amplify values using entropy as an exploration bonus early in the search while allowing for more exploitation later. Recommendations still use 
        Bellman
        values:
        \begin{align}
            \psi_{\text{DENTS}}(s)=\argmax_{a\in\cl{A}}\Qt{s}{a}{}.
        \end{align}
        Because the recommendation policy $\psi_{\text{DENTS}}$ uses the 
        Bellman
        values, we can guarantee that it will converge to the optimal standard policy, and is consistent for any $\beta$.
        \begin{theorem} \label{thrm:dents}
            For any MDP $\cl{M}$, after running $n$ trials of the DENTS algorithm with a root node of $s_0$, if $\beta$ is a bounded function, then there exists constants $C,k>0$ such that for all $\varepsilon>0$ we have 
            $\bb{E}[\reg(s_0,\psi_{\textnormal{DENTS}})] \leq C\exp(-kn)$, and also $\Vt{s_0}{} \rap V^*(s_0)$ as $n\rightarrow\infty$.
        \end{theorem}
        \begin{proofoutline}
        		Let $0 < \eta < \pi_{\text{DENTS}}(a|s)$ for all $s,a$, which exists because $\exp(\cdot)>0$ and $1 / |\cl{A}| > 0$ in Equation \ref{eq:dents_search_policy} (Lemma \ref{lem:min_prob}),
        		which can be used with Hoeffding bounds to show for appropriate constants and any event $E$ that $\{\Pr(E)\leq C_0\exp(-k_0\varepsilon^2N(s_t))\}$ iff $\{\Pr(E)\leq C_1\exp(-k_1\varepsilon^2N(s_t,a_t))\}$ iff $\{\Pr(E)\leq C_2\exp(-k_2\varepsilon^2N(s_{t+1}))\}$.
			We can then show $\Pr(|\Qt{s_0}{a}{}-Q^*(s_0,a)|>\varepsilon) < C_3\exp\left(-k_3\varepsilon^2n\right)$ by induction, where the base case $\Vt{s_{H+1}}{}=V^*(s_{H+1})=0$ holds vacuously (Lemmas \ref{lem:stochastic_step}, \ref{lem:dents_val_induction_step} and Theorem \ref{thrm:dents_val_converge}).
			Let $\Delta_{\cl{M}}$ be a small constant (Equation (\ref{appeq:delta_diff})) such that $\forall s,a. | \Qt{s_0}{a}{}-Q^*(s_0,a) | \leq \Delta_{\cl{M}} / 2 \implies \argmax_a \Qt{s}{a}{} = \argmax_a Q^*(s,a)$. 
			Setting $\varepsilon = \Delta_{\cl{M}}/2$ gives a bound on $\Pr(\psi_{\text{DENTS}}\neq\pi^*)$, which can then be used in the definition of simple regret to give the result.
        \end{proofoutline}

    \subsection{Using the Alias method} \label{sec:alias}
    		The Alias method \cite{alias1,alias2} can be used to sample from a categorical distribution with $m$ categories in $O(1)$ time, with a preprocessing step of $O(m)$ time. Given any stochastic search policy, we can sample actions in amortised $O(1)$ time, by computing an alias table every $|\cl{A}|$ visits to a node, and then sampling from that table. Note that when using the Alias method we are making a trade off between using the most up to date policy and the speed of sampling actions.
    		
    		In Appendix \ref{app:alias} we discuss this idea in more detail, and give an informal analysis of the complexity to run $n$ trials. BTS, DENTS and MENTS can run $n$ trials in $O(n(H\log(|\cl{A}|)+|\cl{A}|))$ time when using the Alias method, as opposed to the typical complexity of $O(nH|\cl{A}|)$.

    \subsection{Limitations and benefits}
    		The main limitations of BTS and DENTS are as follows: (1)~the DENTS decay function $\beta$ can be non-trivial to set and tune; (2)~the focus on simple regret and exploration means they are not appropriate to use when actions taken during the tree search/planning phase have a real-world cost; (3)~the backups implemented directly as presented above are computationally more costly than computing the average returns that UCT uses; (4)~when it is desirable for an agent to follow the maximum entropy policy, then MENTS would be preferable, for example if the agent needs to explore to learn and discover an unknown environment.
    		
    		The main benefits of using a stochastic policy for action selection are: (1)~they allow the Alias method (Section \ref{sec:alias}) to be used to speed up trials; (2)~they naturally encourage exploration as actions are sampled randomly, which is useful for discovering sparse or delayed rewards and for confirming that actions with low values do in fact have a low value; and (3)~the entropy of a stochastic policy can be computed and used as an exploration bonus. In Appendix \ref{app:comparison_table} we summarise and compare the differences between the algorithms considered in this work in more detail.
    		

\section{Results} \label{sec:results}
    This section compares the proposed BTS and DENTS against MENTS and UCT on a set of goal-based MDPs and in the game of Go. For additional baselines, we also compare with the RENTS and TENTS algorithms \cite{rents_your_tents}, which use \textit{relative} and \textit{Tsalis} entropy in place of Shannon entropy respectively, and the H-MCTS algorithm \cite{karnin2013almost} which combines UCT and Sequential Halving.
    
    \subsection{Gridworlds} \label{sec:gridworlds}
    		To evaluate an algorithm with search tree $\cl{T}$, we complete the partial recommendation policy $\psi_{\text{alg}}$ as follows:
        \begin{align}
            \psi(a|s) =
            \begin{cases}
                1 & \text{ if } s\in\cl{T} \text{ and } a=\psi_{\text{alg}}(s), \\
                0 & \text{ if } s\in\cl{T} \text{ and } a\neq\psi_{\text{alg}}(s), \\
                \frac{1}{|\cl{A}|} & \text{ otherwise.}
            \end{cases} \label{eq:full_Recommend}
        \end{align}
        We sample a number of trajectories from $\psi$, and take the average return to estimate $V^{\psi}$. Although we are evaluating the algorithms in an \textit{offline planning} setting, it still indicates how the algorithms perform in an \textit{online} setting where we interleave planning in simulation with letting the agent act. 

    \subsubsection{Domains} \label{sec:gridworld_doms}
        To validate our approach, we use the Frozen Lake environment \cite{brockman2016openai}, and the Sailing Problem \cite{peret2004line}, commonly used to evaluate tree search algorithms \cite{peret2004line,kocsis2006uct,mcts_simple_regret,brue1}. We chose these environments to compare our algorithms in a domain with a sparse and dense reward respectively.
        
        The \emph{(Deterministic) Frozen Lake} is a grid world environment with one goal state. The agent can move in any cardinal direction at each time step, and walking into a wall leaves the agent in the same location. Trap states exist where the agent falls into a hole and the trial ends. If the agent arrives at the goal state after $t$ timesteps, then a reward of $0.99^t$ is received.
        
        The \emph{Sailing Problem} is a grid world environment with one goal state, at the opposite corner to the starting location of the agent. The agent has 8 different actions to travel each of the 8 adjacent states. In each state, the wind is blowing in a given direction and will stochastically change after every transition. The agent cannot sail directly into the wind. The cost of each action depends on the \textit{tack}, the angle between the direction of the agent's travel and the wind. 

    \subsubsection{Results}
        We used an 8x12 Frozen Lake environment and a 6x6 Sailing Problem for evaluation, more environment details are given in Appendix \ref{app:env_deets}. Parameters were selected using a hyper-parameter search (Appendix \ref{app:hps}). 
        Each algorithm is run 25 times on each environment and evaluated every 250 trials using 250 trajectories. 
        A horizon of 100 was used for Frozen Lake and 50 for the Sailing Problem.

        In Frozen Lake (Figure \ref{fig:fl}), entropy proved to be a useful exploration bonus for the \textit{sparse reward}. Values in UCT and BTS remain at zero until a trial successfully reaches the goal. However, entropy guides agents to avoid trap states, where the entropy is zero. DENTS was able to perform similarly to MENTS, and BTS was able to improve its policy over time more than UCT.

        In the Sailing Problem (Figure \ref{fig:sp}) UCT performs well due to the dense reward. BTS and DENTS also manage to keep up with UCT. MENTS and TENTS appear to be slightly hindered by entropy in this environment. The relative entropy encourages RENTS to pick the same actions over time, so it tends to pick a direction and stick with it regardless of cost.
        
        Finally, BTS and DENTS were able to perform well in both domains with a sparse and dense reward structure, whereas the existing methods performed better on one than the other, hence making BTS and DENTS good candidates for a general purpose MCTS algorithm.
        
        \begin{figure*}
            \centering
            \begin{subfigure}[b]{0.49\textwidth}
                \centering
                \includegraphics[width=\textwidth]{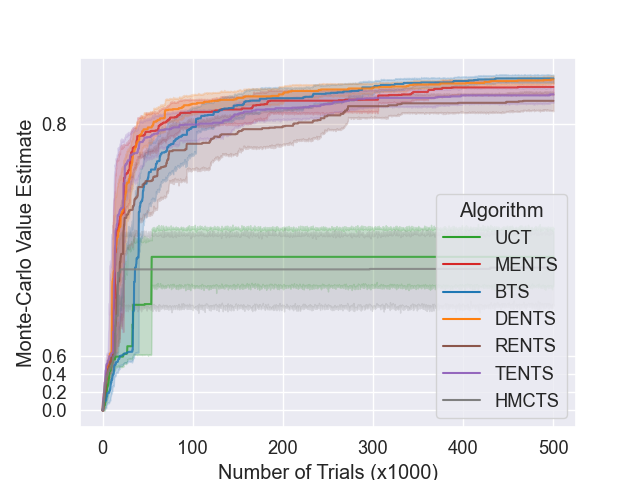}
                \caption{8x12 Frozen Lake.}
                \label{fig:fl}
            \end{subfigure}
            \begin{subfigure}[b]{0.49\textwidth}
                \centering
                \includegraphics[width=\textwidth]{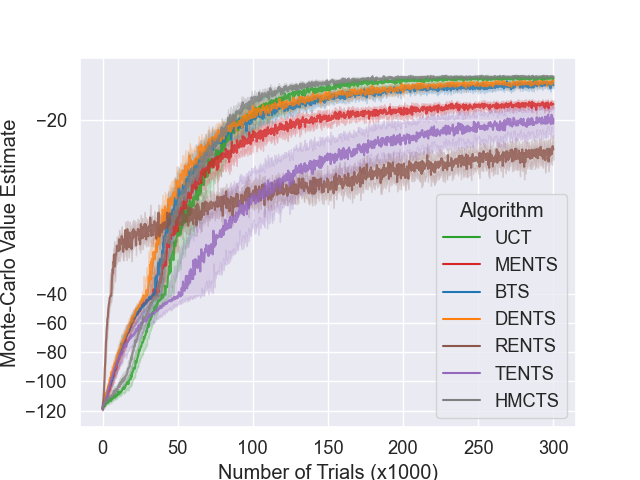}
                \caption{6x6 Sailing Problem.}
                \label{fig:sp}
            \end{subfigure}
            \caption{Results for gridworld environments. Further results are given in Appendix \ref{app:hps}.}
            \label{fig:gridworld_results}
        \end{figure*}

    \subsection{Go} \label{sec:go}
        For a more challenging domain we ran a round-robin tournament using the game of Go, which has widely motivated the development of MCTS methods \cite{gelly2007combining,silver2016mastering,silver2017mastering}. 
        In each match, each algorithm played $50$ games as black and $50$ as white. \textit{Area scoring} is used to score the games, with a \textit{komi} 
        of $7.5$. 
        We used an openly available value network $\tilde{V}$ and policy network $\tilde{\pi}$ from KataGo \cite{katago}. Our baseline was the PUCT algorithm \cite{poly_uct2}, as described in Alpha Go Zero \cite{silver2017mastering} using prioritised UCB \cite{prioritised_ucb} to utilise the policy neural network. Each algorithm was limited to $5$ seconds of compute time per move, allowed to use $32$ search threads per move, and had access to 80 Intel Xeon E5-2698V4 CPUs clocked at 2.2GHz, and a single Nvidia V100 GPU on a shared compute cluster.
        
        To use Boltzmann search in Go, we adapted the algorithms to account for an opponent that wishes to minimise the value of a two-player game. This is achieved by appropriately negating values used in the search policy and backups, which is described precisely in Appendix \ref{app:adapt_for_games}. 
        
        Additionally, we found that adapting the algorithms to use average returns (recall Equation (\ref{eq:uct_qbar})) outperformed using Bellman backups for Go (Appendix \ref{app:go_hps}). The Bellman backups were sensitive to and propogated noise from the neural network evaluations. We use the prefix `AR' to denote the algorithms using average returns, such as AR-DENTS. Full details for these algorithms are given in Appendix \ref{app:average_returns}.
        
        \subsubsection{Using neural networks with Boltzmann search}
            This section describes how to use value and policy networks in BTS. Adapting MENTS and DENTS are similar (Appendix \ref{app:adapt_for_nets}). Values can be initialised with the neural networks as $\Qt{s}{a}{}\leftarrow\log \tilde{\pi}(a|s)+B$ and $\Vt{s}{}\leftarrow\tilde{V}(s)$, where $B$ is a constant (adapted from Xiao \etal \cite{xiao2019maximum}). With such an initialisation, the initial BTS policy is 
            $\rho_{\textnormal{BTS}}(a|s)\propto\tilde{\pi}(a|s)^{1/\alpha}$. For these experiments we set a value of $B=\frac{-1}{|\mathcal{A}|}\sum_{a\in\mathcal{A}} \log\tilde{\pi}(a|s)$. Additionally, the stochastic search policy naturally lends itself to mixing in a prior policy, so we can replace BTS search policy $\pi_{\textnormal{BTS}}$ (Equation (\ref{eq:bts_search_policy})) with $\pi_{\textnormal{BTS,mix}}$:
            \begin{align}
                \pi_{\textnormal{BTS,mix}}(a|s) 
                &= \lambda_{\tilde{\pi}}\tilde{\pi}(a|s) + (1-\lambda_{\tilde{\pi}}) \pi_{\textnormal{BTS}}(a|s) \\
                &= \lambda_{\tilde{\pi}}\tilde{\pi}(a|s) + (1-\lambda_{\tilde{\pi}})(1-\lambda_s)\rho_{\textnormal{BTS}}(a|s) + \frac{(1-\lambda_{\tilde{\pi}})\lambda_s}{|\cl{A}|}, 
            \end{align}
            where $\lambda_{\tilde{\pi}}=\min(1,\epsilon_{\tilde{\pi}}/\log(e+N(s)))$, and $\epsilon_{\tilde{\pi}} \in (0,\infty)$ controls the weighting for the prior policy. 

        \subsubsection{Results}
            Results of the round-robin are summarised in Table \ref{table:go_results}, and we discuss how parameters were selected in Appendix \ref{app:go_hps}. BTS was able to run the most trials per move and beat all of the other algorithms other than DENTS which it drew. We used the optimisations outlined in Appendix \ref{app:alias} which allowed the Boltzmann search algorithms to run significantly more trials per move than PUCT. 
            BTS and DENTS were able to beat PUCT with results of 57-43 and 58-42 respectively. Using entropy did not seem to have much benefit in these experiments, as can be witnessed by MENTS only beating TENTS, and DENTS drawing 50-50 with BTS. This is likely because the additional exploration provided by entropy is vastly outweighed by utilising the information contained in the neural networks $\tilde{V}$ and $\tilde{\pi}$. 
            Interestingly RENTS had the best performance out of the prior works, losing 43-57 to PUCT, and the use of relative entropy appears to take advantage of a heuristic for Go that the RAVE \cite{rave} algorithm used: the value of a move is typically unaffected by other moves on the board.
            
            To validate the strength of our PUCT agent, we also compared it directly with KataGo \cite{katago}, limiting each algorithm to 1600~trials per move. Our PUCT agent won 61-39 in 9x9 Go, and lost 35-65 in 19x19 Go, suggesting that our PUCT agent is strong enough to provide a meaningful comparison for our other general purpose algorithms. Finally, note that we did not fine-tune the neural networks, so the Boltzmann search algorithms directly used the networks that were trained for use in PUCT.

        \begin{table*}[]
        \centering   
		    \begin{tabular}{l|cccccc|c} 
		        \textbf{Black \textbackslash White}     & PUCT  & AR-M  & AR-R  & AR-T  & AR-B  & AR-D   & Trials/move\\ 
		        \hline
		                                PUCT            &   -   & 33-17 & 27-23 & 42-8  & 17-33 & 15-35  & 1054 \\
		                                AR-MENTS        & 12-48 &   -   & 13-37 & 38-12 & 10-40 & 12-38  & 4851\\
		                                AR-RENTS        & 20-30 & 24-26 &   -   & 39-11 & 18-32 & 14-36  & 3672 \\
		                                AR-TENTS        &  8-42 & 11-39 &  9-41 &   -   &  6-44 & 10-40  & 5206 \\
		                                AR-BTS          & 25-25 & 35-15 & 31-19 & 34-16 &   -   & 15-35  & 5375 \\
		                                AR-DENTS        & 23-27 & 36-14 & 29-21 & 36-14 & 15-35 &   -    & 4677 \\         
		    \end{tabular}
            \caption{Results for the Go round-robin tournament. The first column gives the agent playing as black. The final column gives the average trials run per move across the entire round-robin. In the top row, we abbreviate the algorithm names for space.\label{table:go_results}}
        \end{table*}

\section{Related work}
    
    UCT \cite{kocsis2006uct,kocsis2006improved} is a widely used variant of MCTS. Polynomial UCT \cite{poly_uct2}, replaces the logarithmic term in UCB with a polynomial one (such as a square root), which has been further 
    popularised by its use in 
    Alpha Go and Alpha Zero \cite{silver2016mastering,silver2017mastering,silver2018general}. Coquelin and Munos introduce the Flat-UCB and BAST algorithms to adapt UCT for the D-chain problem \cite{coquelin2007_uct}. 
    However, we consider an alternative approach for search in MCTS rather than adapting UCB.
    
    Maximum entropy policy optimization methods are well-known in the reinforcement learning literature  \cite{haarnoja2017reinforcement, haarnoja2018soft, ziebart2008maximum}. 
    MENTS \cite{xiao2019maximum} is the first method to combine the principle of maximum entropy and MCTS. Kozakowski \etal \cite{ants} extend MENTS to arrive at Adaptive Entropy Tree Search (ANTS), 
    adapting parameters throughout the search to match a prescribed entropy value. Dam \etal \cite{rents_your_tents} also extend MENTS using \textit{Relative} and \textit{Tsallis entropy} to arrive at the RENTS and TENTS algorithms. 
    Our work is closely related to MENTS, however, we focus on reward maximisation and consider how entropy can be used in MCTS without altering our planning objective.
    
    Bubeck \etal \cite{simple_regret_short, simple_regret_long} introduce 
    simple regret in the context of multi-armed bandit problems (MABs). They alternate between pulling arms for exploration and outputting a \textit{recommendation}. They show for MABs that a uniform exploration produces an exponential bound on the simple regret of recommendations. 
    We use simple regret, but in the context of sequential decision-making, to analyse the convergence of MCTS algorithms.
    
    Tolpin and Shimony \cite{mcts_simple_regret} extend simple regret to MDP settings, showing an $O(\exp(-\sqrt{n}))$ bound on simple regret after $n$ trials by adapting UCT. The subsequent work of Hay \etal \cite{hay2014selecting} extends \cite{mcts_simple_regret} to consider a \textit{metalevel decision problem}, incorporating computation costs into the objective. Pepels \etal \cite{mcts_simple_regret_two} introduce a Hybrid MCTS (H-MCTS) motivated by the notion of simple regret. H-MCTS uses a mixture of Sequential Halving \cite{karnin2013almost}, and UCT. Feldman \etal \cite{brue1,brue2,brue3} introduce the Best Recommendation with Uniform Exploration (BRUE) algorithm. BRUE splits trials up to explicitly focus on exploration and value estimation one at a time. BRUE achieves an exponential bound $O(\exp(-n))$ on simple regret after $n$ trials \cite{brue1}. 
    Prior work that considers simple regret in MCTS has focused on adaptations to UCT, whereas this work focuses on algorithms that sample actions from Boltzmann distributions, rather than using UCB for action seclection.

\section{Conclusion}
    We considered the recently introduced MENTS algorithm, compared and contrasted it to UCT, and discussed the limitations of both. We introduced two new algorithms, BTS and DENTS, that are consistent, converge to the optimal standard policy, while preserving the benefits that come with using a stochastic Boltzmann search policy. Finally, we compared our algorithms in gridworld environments and Go, demonstrating the performance benefits of utilising the Alias method, that entropy can be a useful exploration bonus with sparse rewards, and more generally, 
    demonstrating the advantage of prioritising exploration in planning, by using Boltzmann search policies. 
    
    An interesting area of future work may include investigating good heuristics for setting parameters in BTS and DENTS. We noticed that the best value for the search temperature tended to be the same order of magnitude as the optimal value at the root node $V^*(s_0)$, which suggests a heuristic similar to \cite{prst} may be reasonable.

{
    \small
    
    \bibliographystyle{plain}
    \bibliography{mybib}
}

\newpage

\newpage
\begin{appendices}
\FloatBarrier

These appendices are structured as follows:
\begin{itemize}
	\item Appendix \ref{app:average_returns} gives details on how to adapt the BTS, DENTS, MENTS, RENTS and TENTS algorithms to use average returns backups, rather than dynamic programming backups, to arrive at the AR-BTS, AR-DENTS, AR-MENTS, AR-RENTS and AR-TENTS algorithms respectively.
	\item Appendix \ref{app:additional_algo_deets} covers algorithm details that we could not fit into the main paper. Specifically:
	\begin{itemize}
		\item in Appendix \ref{app:alias} we discuss the computational complexity of BTS, DENTS and MENTS in more detail, including what can be achieved by using the alias method \citeapp{alias1,alias2};
		\item in Appendix \ref{app:adapt_for_games} we give the specific details of how we adapted BTS, DENTS, MENTS, RENTS and TENTS for two-player games and to utilise neural networks;
		\item and in Appendix \ref{app:comparison_table} we summarise the differences between the MCTS algorithms considered in this work in a table.
	\end{itemize}
	\item Appendix \ref{app:additional_results} generally covers additional details about experiments and results. Specifically:
	\begin{itemize}
		\item in Appendix \ref{app:env_deets} we give more specific details about the gridworld envrionments used in our experiments;
		\item in Appendix \ref{app:param_sens} we have a more detailed discussion about parameter sensitivity, with additional results and discussion using the D-chain environments in Appendix \ref{app:param_sens_dchain} and in the Frozen Lake environment in Appendix \ref{app:param_sens_fl};
		\item in Appendix \ref{app:hps} we give details for the hyperparameter search used to select parameters in the gridworld domains;
		\item in Appendix \ref{app:dents_mimic_ments} we show empirically that DENTS can follow a search policy very similar to the MENTS search policy;
		\item and in Appendix \ref{app:go_results} we give some additional details about our Go experiments, including how we selected hyperparameters.
	\end{itemize}
	\item Appendix \ref{app:proofs} constitutes the theoretical section of this work, beginning by revisiting the MCTS \textit{stochastic process}, introducing notation needed to reason about convergence properties and then providing proofs for the results quoted in the main paper and Appendix \ref{app:average_returns}. Appendix \ref{app:proofs} has its own introduction/contents to give an overview of how the theory is built up.
\end{itemize}

\section{Code} \label{app:code}
    The code for this paper can be found at \url{https://github.com/MWPainter/thts-plus-plus/tree/xpr_go}.

\newpage
\section{Using average returns with Boltzmann search} \label{app:average_returns}

    \FloatBarrier
    
    In this section we consider how to adapt each of the algorithms considered in this paper to use average returns. We start by describing AR-BTS and AR-DENTS and then subsequently how AR-DENTS can be adapted to give AR-MENTS, AR-RENTS and AR-TENTS. We also give an example of why the adaptations are necessary.

    \subsection{AR-BTS}
        We can replace the Bellman values used in BTS with average returns similar to UCT, resulting in the subsequent AR-BTS algorithm. Recall that $\bar{Q}(s,a)$, the average return at $(s,a)$, is defined by:
        \begin{align}
            \bar{Q}(s_t, a_t) &\leftarrow \bar{Q}(s_t, a_t) + \frac{\bar{R}(s_t,a_t) - \bar{Q}(s_t, a_t)}{N(s_t, a_t) + 1},  \label{appeq:average_return}
        \end{align}
        where $\bar{R}(s_t,a_t) = \sum_{i=t}^H R(s_i,a_i)$. An additional issue that needs consideration when using average returns is that the resulting empirical values will not converge to the optimal values when using a stochastic policy, as we argue below. Hence, in AR-BTS we also decay the search temperature, such that the Boltzmann policy tends towards a greedy policy. So, let $\alpha$ be a non-negative function with $\alpha(m)\rightarrow 0$ as $m\rightarrow \infty$. The search policy used for AR-BTS is then:
        \begin{align}
            \pi_{\textnormal{AR-BTS}}(a|s) &= (1-\lambda_s)\rho_{\textnormal{AR-BTS}}(a|s) + \frac{\lambda_s}{|\cl{A}|}, \\
            \rho_{\textnormal{AR-BTS}}(a|s) &\propto \exp\left(\frac{1}{\alpha(N(s))}\left(\bar{Q}(s,a)\right)\right). 
        \end{align}
        The recommendation policy for AR-BTS then uses the $\bar{Q}$ average return values:
        \begin{align}
            \psi_{\text{AR-BTS}}(s)&=\argmax_{a\in\cl{A}}\bar{Q}(s,a).
        \end{align}
        To see why the decaying search temperature is necessary, consider the MDP in Figure \ref{fig:ar_eg_mdp}, with a temperature of $\alpha(m)=1$. Consider the average return $\bar{V}(2)$ from state $2$ that has been visited $N(2)$ times. The probability of selecting action $a_1$ from state 2 will be $1/(1+e^2)$, and the probability of selecting action $a_2$ from state 2 will be $e^2/(1+e^2)$. So as $N(2)\rightarrow\infty$ we expect the average return to converge to slightly less than two: $\bar{V}(2)\rightarrow 2e^2/(1+e^2) < 2$. Taking advantage of this, we can show that for any fixed temperature $\alpha(m)=\alpha_{\text{fix}}$, there is an MDP that AR-BTS would not be consistent. We summarise this result in Proposition \ref{prop:ar_bts_fails}. Thus, necessitating the use of the decaying search temperature. 
        
        \begin{figure}
            \centering
            \includegraphics[width=0.3\textwidth]{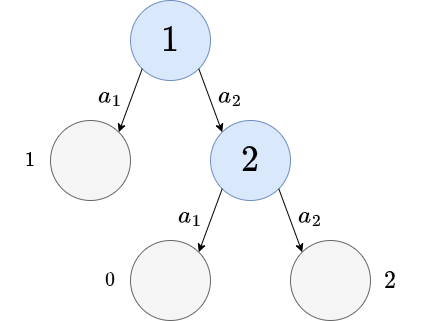}
            \caption{An example MDP to demonstrate the necessity for a decaying search temperature when using average returns.}
            \label{fig:ar_eg_mdp}
        \end{figure}

        \begin{prop}\label{prop:ar_bts_fails}
            For any $\alpha_{\textnormal{fix}}>0$, there is an MDP $\cl{M}$ such that AR-BTS with $\alpha(m)=\alpha_{\textnormal{fix}}$ is not consistent: $\bb{E}[\reg(s_0,\psi_{\textnormal{AR-BTS}})] \not\to 0$ as $n\to\infty$. 
        \end{prop}
        \begin{proofoutline}
            We give a proof outline in Section \ref{sec:ar_proofs}, using an MDP similar to the one shown in Figure \ref{fig:ar_eg_mdp}.
        \end{proofoutline}

        Additionally, provided the search temperature decays to zero, AR-BTS will be consistent (Theorem \ref{thrm:ar_bts}).
        
        \begin{customthm}{B.2} \label{thrm:ar_bts}
            For any MDP $\cl{M}$, if $\alpha(m)\rightarrow 0$ as $m\rightarrow\infty$ then $\bb{E}[\reg(s_0,\psi_{\textnormal{AR-BTS}})]\rightarrow 0$ as $n\rightarrow\infty$, where $n$ is the number of trials.
        \end{customthm}
        \begin{proofoutline}
            We give a proof outline in Section \ref{sec:ar_proofs}.
        \end{proofoutline}

    \subsection{AR-DENTS}
        Using average returns in DENTS is similar to BTS. To compute the average returns we use Equation (\ref{appeq:average_return}) again. To compute the entropy values $\cl{H}_V$ and $\cl{H}_Q$ to use in AR-DENTS, we use the same entropy backups as in DENTS:
        \begin{align}
            \cl{H}_V(s_t) &\leftarrow \cl{H}(\pi_{\textnormal{AR-DENTS}}(\cdot | s_t)) + \sum_{a\in\cl{A}} \pi_{\textnormal{AR-DENTS}}(a_t|s_t)\cl{H}_Q(s_t,a_t), \label{appeq:entropy_v} \\
            \cl{H}_Q(s_t,a_t) &\leftarrow \sum_{s'\in \succc{s_t}{a_t}} \frac{N(s')}{N(s_t,a_t)} \cl{H}_V(s'). \label{appeq:entropy_q}
        \end{align}
        The search policy in AR-DENTS is the same as in DENTS, but with the Bellman value $\hat{Q}$ replaced by the average return value $\bar{Q}$:
        \begin{align}
            \pi_{\textnormal{AR-DENTS}}(a|s) &= (1-\lambda_s)\rho_{\textnormal{AR-DENTS}}(a|s) + \frac{\lambda_s}{|\cl{A}|}, \\
            \rho_{\textnormal{AR-DENTS}}(a|s) &\propto \exp\left(\frac{1}{\alpha(N(s))}\left(\bar{Q}(s,a) + \beta(N(s))\cl{H}_Q(s_t,a_t) \right)\right).
        \end{align}
        The recommendation policy for AR-DENTS also use the $\bar{Q}$ average return values:
        \begin{align*}
            \psi_{\text{AR-DENTS}}(s)&=\argmax_{a\in\cl{A}}\bar{Q}(s,a).
        \end{align*}
        Similarly to the AR-BTS result, we show that if $\alpha(m)\rightarrow 0$ and $\beta(m)\rightarrow 0$, then AR-DENTS is consistent. Again, because we are considering the AR versions of our algorithms from a practical viewpoint rather than a theoretical one, we only show consistency without any specific (simple) regret bounds.

        \begin{customthm}{B.3} \label{thrm:ar_dents}
            For any MDP $\cl{M}$, if $\alpha(m)\rightarrow 0$ and $\beta(m)\rightarrow 0$ as $m\rightarrow\infty$ then $\bb{E}[\reg(s_0,\psi_{\textnormal{AR-DENTS}})]\rightarrow 0$ as $n\rightarrow\infty$, where $n$ is the number of trials.
        \end{customthm}
        \begin{proofoutline}
            We give a proof outline in Section \ref{sec:ar_proofs}.
        \end{proofoutline}

    \subsection{AR-MENTS, AR-RENTS and AR-TENTS}
        MENTS uses \textit{soft values}, $\hat{Q}_{\textnormal{sft}}$, which are not obvious how to replace with average returns. So to produce the AR variants of MENTS, RENTS and TENTS we use AR-DENTS as a starting point.

        \paragraph{AR-MENTS.} For AR-MENTS we use AR-DENTS, but set $\beta(m)=\alpha(m)=\alpha_{\text{fix}}$. This algorithm resembles MENTS, as the weighting used for entropy in soft values is the same as the Boltzmann policy search temperature.

        \paragraph{AR-RENTS.} To arrive at AR-RENTS, we replace any use of $\hat{Q}_{\textnormal{sft}}(s,a)$ with $\bar{Q}(s,a)+\beta(m)\cl{H}_Q(s,a)$. So we use Equations (\ref{appeq:average_return}), (\ref{appeq:entropy_v}) and (\ref{appeq:entropy_q}) for backups, but replace the Shannon entropy function $\cl{H}$, with a relative entropy function $\cl{H}_{\textnormal{relative}}$ in Equation (\ref{appeq:entropy_v}). The relative entropy function $\cl{H}_{\textnormal{relative}}$ uses the \textit{Kullback-Leibler divergence} between the search policy and the search policy of the parent decision node. The search policy used is the same as in RENTS, with the aforementioned substitution for soft values. See \cite{rents_your_tents} for full details on computing relative entropy and the search policy used in RENTS.

        \paragraph{AR-TENTS.} Similarly, for AR-TENTS, we replace any use of $\hat{Q}^m_{\textnormal{sft}}(s,a)$ with $\bar{Q}^m(s,a)+\beta(m)\cl{H}_Q(s,a)$, and use Equations (\ref{appeq:average_return}), (\ref{appeq:entropy_v}) and (\ref{appeq:entropy_q}) for backups. This time, we replace the Shannon entropy function $\cl{H}$, with a Tsallis entropy function $\cl{H}_{\textnormal{Tsallis}}$ in Equation (\ref{appeq:entropy_v}). Again, we use the same search policy used in TENTS, with the substitution for soft values. See \cite{rents_your_tents} for how Tsallis entropy is computed and the corresponding search policy for TENTS.

\newpage      
\section{Additional algorithm discussion} \label{app:additional_algo_deets}
	\subsection{The alias method, implementation details and optimisations} \label{app:alias}

    \FloatBarrier

		The Alias method \cite{alias1,alias2} can be used to sample from fixed categorical distribution in $O(1)$ time. If the distribution consists of $m$ categories, then a preprocessing step of $O(m)$ is needed to construct the table, details on how to construct the table can be found in \cite{alias2}. We provide a small example of an alias table in Figure \ref{fig:alias}.

        \begin{figure}
            \centering
            \includegraphics[scale=0.4]{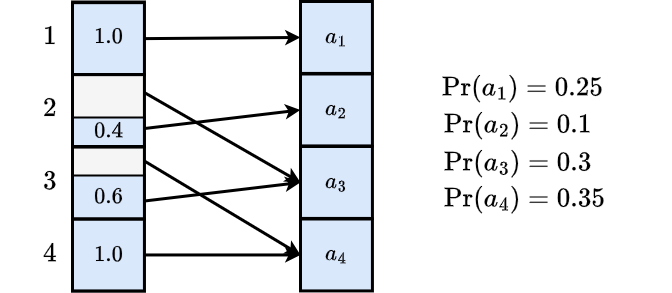}
            \caption{An example of an alias table for a categorical distribution over four actions. To sample from the table we can draw a random index from $I\sim U(\{1,2,3,4\})$, sample a uniformly random number from $x\sim U(0,1)$ and then follow the pointer depending on if $x>\text{thresholds}[I]$ or not. For example, sampling either $(I,x)=(2,0.5)$  or $(I,x)=(3,0.1)$ would lead to the action $a_3$.}
            \label{fig:alias}
        \end{figure}

        Any MCTS algorithm that samples actions from a stochastic distribution can make use of the Alias method, for example all of the Boltzmann search algorithms considered in this work. For a node with $A=|\cl{A}|$ actions/children, an MCTS algorithm that samples actions from a distribution would typically on each trial: construct the distribution in $O(A)$ time, and then sample from the distribution using a single uniform random number in $O(A)$ time. However, we can instead construct a distribution every $A$ times we visit a search node, which gives an amortised complexity of $O(1)$, and then sampling from an alias table only takes $O(1)$ time.
        
        It is worth noting that when we use the Alias method, we are not sampling from the most up to date distributions possible, so we are making a trade off between being able to sample actions more efficiently and using the most up to date distributions. Moreover, when we use this method, we are making use of the stochastic distribution that will still allow a variety of actions to be sampled on different trials with a fixed distribution. In contrast, UCT uses a deterministic distribution (recall Equation (\ref{eq:uct_distr})) which selects the action with maximum UCB value and different actions are selected on different trials by changing the distribution (i.e. a new action is selected typically when the confidence term becomes large enough, which requires the UCB values to be computed each time).
        
        We now consider the runtime complexity of running $n$ trials for different algorithms. We summarise these complexities as part of Table \ref{table:compare_algs} and we assume that successor states can be sampled in $O(1)$ time, noting that if we are given a tabular successor state distribution then we can use the Alias method again to sample from it in $O(1)$ time. 
        
        \subsubsection{BTS backup complexity}
        
        		Firstly, recall the backups used by BTS for a trajectory $\tau=(s_0,a_0,...,a_{h-1},s_h)$:
	        \begin{align}
	            \Qt{s_t}{a_t}{} &\leftarrow R(s_t,a_t) + \sum_{s' \in \succc{s_t}{a_t}} \left( \frac{N(s')}{N(s_t,a_t)} \Vt{s'}{} \right), \label{appeqc:dp_q_backup} \\ 
	            \Vt{s_t}{} &\leftarrow\max_{a\in\cl{A}} \Qt{s_t}{a}{}. \label{appeqc:dp_v_backup} 
	        \end{align}
	        
	        In Equation (\ref{appeqc:dp_q_backup}) observe that only the $\Vt{s_t}{}$ term in the sum will change, that is, for each $s' \in \{s''\in\succc{s_t}{a_t}|s''\neq s_{t+1}\}$ that the value of $\Vt{s'}{}$ will be the same. Hence, we can implement the update in $O(1)$ as:
        		
	        \begin{align}
	            \Qt{s_t}{a_t}{} &\leftarrow 
				\begin{cases}
					R(s_t,a_t) + \Vt{s'}{} & \text{if } N(s_t,a_t) = 1, \\
					\Qt{s_t}{a_t}{} - \frac{N(s_{t+1})-1}{N(s_t,a_t)-1} \Vt{s_{t+1}}{\text{old}} + \frac{N(s_{t+1})}{N(s_t,a_t)} \Vt{s_{t+1}}{} & \text{if } N(s_t,a_t) > 1,
				\end{cases}	\label{appeqc:opt_q_backup}
	        \end{align}
	        
	        where $\Vt{s_{t+1}}{\text{old}}$ is the previous value of $\Vt{s_{t+1}}{}$. To more efficiently implement Equation (\ref{appeqc:dp_v_backup}) we can use a \textit{max heap}, which will take $O(\log(A))$ to update the value of $\Qt{s_t}{a_t}{}$ in the heap, and $O(1)$ to read out the maximum value. 
	        
	        Hence, an optimised version of the BTS backups will take $O(\log(A))$. So when running BTS using the Alias method, we need a total of $O(\log(A))$ time to perform the sampling and backups of each node visited, and additionally we will initialise a new Alias table on every trial for the new node added. Hence the runtime complexity of running $n$ trials of BTS with the Alias method will take $O(n(H\log(A)+A))$.
	        
	    \subsubsection{DENTS backup complexity}
	    
	    		When using the Alias method with DENTS, we also need to consider the entropy backups too:
	        \begin{align}
	            \cl{H}_V(s_t) &\leftarrow \cl{H}(\pi_{\textnormal{DENTS}}(\cdot | s_t)) + \sum_{a\in\cl{A}} \pi_{\textnormal{DENTS}}(a|s_t)\cl{H}_Q(s_t,a), \label{appeqc:entropy_v} \\
	            \cl{H}_Q(s_t,a_t) &\leftarrow \sum_{s'\in \succc{s_t}{a_t}} \frac{N(s')}{N(s_t,a_t)} \cl{H}_V(s').  \label{appeqc:entropy_q}
	        \end{align}
	        
	        The backup for $\cl{H}_Q$ can be performed in $O(1)$ time, similarly to Equation (\ref{appeqc:opt_q_backup}):
	        \begin{align}
	            \cl{H}_Q(s_t,a_t) &\leftarrow 
				\begin{cases}
					\cl{H}_V(s_{t+1}) & \text{if } N(s_t,a_t) = 1, \\
					\cl{H}_Q(s_t,a_t) - \frac{N(s_{t+1})-1}{N(s_t,a_t)-1} \cl{H}_V^{\text{old}}(s_{t+1}) \Vt{s_{t+1}} + \frac{N(s_{t+1})}{N(s_t,a_t)} \cl{H}_V(s_{t+1}) & \text{if } N(s_t,a_t) > 1,
				\end{cases}	\label{appeqc:opt_entq_backup}
	        \end{align}
	        
	        where $\cl{H}_V^{\text{old}}(s_{t+1})$ is the previous value of $\cl{H}_V(s_{t+1})$.
	        
	        The backup for $\cl{H}_V$ can also be implemented in amortised $O(1)$ time, by noting that we are only updating $\pi_{\textnormal{DENTS}}$ every $A$ visits. Let $\pi^{\text{old}}_{\textnormal{DENTS}}$ be the policy on the previous trial. Then, we have to perform the full $O(A)$ backup when the policy is updated every $A$ visits, but otherwise we can get away with an $O(1)$ backup:
	        \begin{align}
	        \cl{H}_V(s_t) &\leftarrow
	        		\begin{cases}
	        			\cl{H}(\pi_{\textnormal{DENTS}}(\cdot | s_t)) + \sum_{a\in\cl{A}} \pi_{\textnormal{DENTS}}(a|s_t)\cl{H}_Q(s_t,a) & \text{if } \pi_{\textnormal{DENTS}}\neq \pi^{\text{old}}_{\textnormal{DENTS}}, \\
	        			\cl{H}_V(s_t) + \pi_{\textnormal{DENTS}}(a|s_t) \left( \cl{H}_Q(s_t,a_t) - \cl{H}^{\text{old}}_Q(s_t,a_t)\right) & \text{if } \pi_{\textnormal{DENTS}}= \pi^{\text{old}}_{\textnormal{DENTS}},
	        		\end{cases}
	        \end{align}
	    		
	    		where $\cl{H}^{\text{old}}_Q(s_t,a_t)$ is the previous value of $\cl{H}_Q(s_t,a_t)$.
	    		
	    		As the entropy backups used in DENTS can be implemented in amortised $O(1)$ time, it follows that the complexity of running $n$ trials of DENTS with the Alias method will also take $O(n(H\log(A)+A))$.
	    	\subsubsection{MENTS backup complexity}
	    		
	    		Recall the backups used for MENTS:
            \begin{align}
                \Qst{s_t}{a_t}{} &\leftarrow R(s_t,a_t) + \sum_{s'\in\succc{s}{a}} \left( \frac{N(s')}{N(s_t,a_t)} \Vst{s'}{} \right), \label{appeqc:soft_q_backup} \\
                \Vst{s_t}{} &\leftarrow \alpha \log \sum_{a\in\cl{A}} \exp \left(\frac{1}{\alpha}\Qst{s_t}{a}{} \right). \label{appeqc:soft_v_backup}
            \end{align}
            
            The backup for $\Qst{s_t}{a_t}{}$ can be implemented in $O(1)$ similarly to Equation (\ref{appeqc:opt_q_backup}). However, considering how to implement Equation (\ref{appeqc:soft_v_backup}) is more complex. Firstly, we point out the implementing the equation exactly as written is infact numerically unstable, and in practise we need to make use of the equation:
            \begin{align}
            		\alpha \log \sum_{a\in\cl{A}} \exp \left(\frac{1}{\alpha}\Qst{s_t}{a}{} \right)
            			= \alpha \log \sum_{a\in\cl{A}} \exp \left(\frac{1}{\alpha}\left(\Qst{s_t}{a}{} - C \right)\right) + C,
            \end{align}
            
            where $C$ is an arbitrary constant. The value of $C$ needs to be chosen carefully to avoid numerical underflow or overflow, and is typically set to $C=\max_a \Qst{s_t}{a}{}$. To efficiently compute this backup, we can keep two auxilary variables:
            \begin{align}
            		M(s_t) &= \max_a \Qst{s_t}{a}{} \\
            		E(s_t) &= \sum_{a\in\cl{A}} \exp \left(\frac{1}{\alpha}\left(\Qst{s_t}{a}{} - M(s_t) \right)\right).
            \end{align}
            
            Now we can perform the backup for $\Qst{s_t}{a_t}{}$ as follows:
            \begin{align}
            		M(s_t) &\leftarrow \max_a \Qst{s_t}{a}{} \\
            		E(s_t) &\leftarrow
            			\begin{cases}
            				\exp(0)	& \text{ if} N(s_t)=1 \\
            				\left(
            					E(s_t) - \exp 
            					\left(
            						\frac{1}{\alpha}
            						\left(
            							\Qst{s_t}{a_t}{\text{old}} - M^{\text{old}}(s_t) 
            						\right)
            					\right)
            				\right)
            				\exp
            				\left(
            					\frac{1}{\alpha} (M^{\text{old}}(s_t)-M(s_t))
            				\right) \\
            				\ \ \ \ \ + \exp 
            				\left(
            					\frac{1}{\alpha}
            					\left(
            						\Qst{s_t}{a_t}{} - M(s_t) 
            					\right) 
            				\right)
            				& \text{ if} N(s_t) > 1.
            			\end{cases} \\
            		\Vst{s_t}{} &\leftarrow  \alpha \log(E(s_t)) + M(s_t), 
            \end{align}
            
            where $M^{\text{old}}(s_t), \Qst{s_t}{a_t}{\text{old}}$ are the previous values of $M(s_t), \Qst{s_t}{a_t}{}$ respectively. Similarly to the BTS backups, the requirement of computing a $\max$ operation means that these backups can be computed in $O(\log(A))$ time, so the complexity of running $n$ trials of MENTS with the Alias method takes $O(n(H\log(A)+A))$ time. 
            
            It may be possible to implement MENTS faster with a runtime of $O(n(H+A))$ if we could set $M(s_t)$ to some constant value, however this will likely depend on the MDP and the scale of the rewards. Additionally, we found even running this version of MENTS to be quite unstable, and had to resort to using $O(A)$ backups for MENTS in our experiments.
        
        \subsubsection{RENTS and TENTS}
        
       		Both RENTS and TENTS can utilise the Alias method too. We did not consider how to optimise the backups for these algorithms.

	    	\subsubsection{Average return complexities}
	    		For the AR versions of the algorithms, we note that the backup from Equation (\ref{appeq:average_return}) can be implemented in $O(1)$ time. Following similar reasoning to the previous algorithms, this means that running $n$ trials with the Alias method for one of the average return algorithms will take $O(n(H+A))$ time.

	    	\subsubsection{Decaying the search temperature}
	    		Although theoretically using a fixed search temperature $\alpha$ is sufficient for BTS and DENTS to converge, in practise algorithms may perform better with a decaying search temperature $\alpha(m)$ as described in AR-BTS and AR-DENTS. Note that the proofs of convergence for AR-BTS and AR-DENTS also hold for BTS with a decaying temperature. A decaying search temperature would allow BTS and DENTS to focus on deeper parts of the search tree over time.

	\subsection{Adapting Boltzmann search algorithms for two-player games and neural nets} \label{app:adapt_for_games} \label{app:adapt_for_nets}
            Now we will detail the adaptions required for each of BTS, DENTS, and MENTS to be used for games. For completeness, we reiterate the adaptions for BTS.
            
            To run our algorithms on games with two players we need to account for an \textit{opponent} that is trying to minimise the value of the game. The agent playing as black (or the \textit{player}) in Go will act on odd timesteps, and the opponent playing as white will act on even timesteps. Fairly informally, this means that the maximum entropy objective needs to be updated:
        \begin{align}
            \Vsp{s,t}{\pi} = \bb{E}_{\pi} \left[\sum_{i=t}^{H} R\left(s_{i}, a_{i}\right)+ (-1)^{t-1}\alpha \mathcal{H}\left(\pi\left(\cdot|s_{i}\right)\right) \Big\vert s_t=s\right]. \label{appeq:altered_max_ent_objective}
        \end{align}
        This means that each agent (the player and the opponent) will be simultaneously trying to maximise their own entropy, whilst minimising the others.
            
            \paragraph{BTS}
            Values can be initialised with the neural networks as $Q^{\text{init}}(s,a)=\log \tilde{\pi}(a|s)+B$ and $V^{\text{init}}(s)=\tilde{V}(s)$, where $B$ is a constant (adapted from Xiao \etal \cite{xiao2019maximum}). With such an initialisation, the initial BTS policy is $\rho_{\textnormal{BTS}}(a|s)\propto\tilde{\pi}(a|s)^{1/\alpha}$.
            
            The stochastic search policy naturally lends itself to mixing in a prior policy, so we can replace BTS search policy $\pi_{\textnormal{BTS}}^{n}$ (Equation (\ref{eq:bts_search_policy})) with $\pi_{\textnormal{BTS,mix}}$:
            \begin{align}
                \pi_{\textnormal{BTS,mix}}(a|s) 
                &= \lambda_{\tilde{\pi}}\tilde{\pi}(a|s) + (1-\lambda_{\tilde{\pi}}) \pi_{\textnormal{BTS}}(a|s) \\
                &= \lambda_{\tilde{\pi}}\tilde{\pi}(a|s) + (1-\lambda_{\tilde{\pi}})(1-\lambda_s)\rho_{\textnormal{BTS}}(a|s) + \frac{(1-\lambda_{\tilde{\pi}})\lambda_s}{|\cl{A}|}, \label{appeq:ments_mixed_policy}
            \end{align}
            where $\lambda_{\tilde{\pi}}=\min(1,\epsilon_{\tilde{\pi}}/\log(e+N(s)))$, and $\epsilon_{\tilde{\pi}} \in (0,\infty)$ controls the weighting for the prior policy. To run BTS on a two-player game we need to account the opponent trying to minimise the value of the game. In BTS we can negate the values used in the search policy, and replace the max operation with a min in Bellman backups. That is, we replace equations (\ref{eq:bts_value_policy}) and (\ref{eq:dp_v_backup}) with:
            \begin{align}
                \rho_{\textnormal{BTS}}(a|s) &\propto \exp\left(\frac{-1}{\alpha}\Qt{s}{a}{}\right), \\
                \Vt{s_t}{} &=\min_{a\in\cl{A}} \Qt{s_t}{a}{}. \label{appeq:bellman_v_oponent}
            \end{align}
            Finally, when making a recommendation as the opponent, we use the replace the recommendation policy with:
            \begin{align}
                \psi_{\text{BTS}}(s)=\argmin_{a\in\cl{A}}\Qt{s}{a}{}.
            \end{align}

            \paragraph{MENTS}
            For MENTS we can use the same initialisations as in BTS, that is $Q^{\text{init}}_{\text{sft}}(s,a)=\log \tilde{\pi}(a|s)+B$ and $V^{\text{init}}(s)=\tilde{V}(s)$, with the same value for $B$. The prior policy $\tilde{\pi}$ is mixed into the MENTS search policy in the same way as for BTS (Equation \ref{appeq:ments_mixed_policy}). 

            At opponent nodes, we can replace any use of the temperature $\alpha$ by $-\alpha$, which effectively turns the softmax into a softmin and gives the highest density to the lowest values in the search policy. So at opponent nodes, we replace Equations (\ref{eq:rhosft}) and (\ref{eq:soft_v_backup}) by:
            \begin{align}
                \rho_{\textnormal{sft}}(a|s) &= \exp\left(\frac{-1}{\alpha}\left(\Qst{s}{a}{}-\Vst{s}{}\right)\right), \\
                \Vst{s_t}{} &= -\alpha \log \sum_{a\in\cl{A}} \exp \left(\frac{-1}{\alpha}\Qst{s_t}{a}{} \right). 
            \end{align}
            Finally, to make a recommendation as the opponent, we replace the recommendation policy with:
            \begin{align}
                \psi_{\text{MENTS}}(s)=\argmin_{a\in\cl{A}}\Qst{s}{a}{}.
            \end{align}

            \paragraph{DENTS}
            For DENTS we again use the same initialisations as BTS, setting $Q^{\text{init}}(s,a)=\log \tilde{\pi}(a|s)+B$ and $V^{\text{init}}(s)=\tilde{V}(s)$, using the same value for $B$. And again, the prior policy $\tilde{\pi}$ is mixed into the DENTS search policy in the same way as for BTS (Equation \ref{appeq:ments_mixed_policy}). 

            At opponent nodes, we update the search policy and backups by replacing Equations (\ref{eq:dents_policy}) and (\ref{eq:entropy_v}) by the following:
            \begin{align}
                \rho_{\textnormal{DENTS}}(a|s) &\propto \exp\left(\frac{-1}{\alpha}\left(\Qt{s}{a}{} + \beta(N(s))\cl{H}_Q(s_t,a_t) \right)\right), \\ 
                \cl{H}_V(s_t) &= -\cl{H}(\pi_{\textnormal{DENTS}}(\cdot | s_t)) + \sum_{a\in\cl{A}} \pi_{\textnormal{DENTS}}(a_t|s_t)\cl{H}_Q(s_t,a_t), 
            \end{align}
            and we use replacement for Bellman backups as BTS does in Equation (\ref{appeq:bellman_v_oponent}). Finally, the recommendation policy for an opponent is:
            \begin{align}
                \psi_{\text{DENTS}}(s)=\argmin_{a\in\cl{A}}\Qt{s}{a}{}.
            \end{align}

	\subsection{Comparison of MCTS algorithms} \label{app:comparison_table}
		
		In Table \ref{table:compare_algs} we summarise the properties of UCT, MENTS, BTS and DENTS so that they can be easily compared.
		
		We also summarise here the benefits that using a stochastic (Boltzmann) policy for action selection can provide:
		\begin{itemize}
			\item Using a stochastic policy allows the Alias method (as described in Section \ref{app:alias}) to be used, which can significantly increase the number of trials that can be run in a fixed time period;
			\item Using a stochastic policy naturally encourages exploration as it will still always have some probability of sampling each action;
			\item And the entropy can be computed of a stochastic distribution, which can then be used as an exploration bonus.
		\end{itemize}
		
		The benefits of additional exploration from entropy and the stochastic distribution include: helping to find delayed/sparse rewards in the environment; confirming that bad actions are in fact bad; and, greater exploration leads to less contention between threads in a multi-threaded implementation (discussed in Section \ref{app:multi_thread}).

        \begin{table*}[]
            \centering
            \begin{tabular}{p{8cm}|c|c|c|c}
                						& UCT      		& MENTS          & BTS               & DENTS \\ 
                \hline
                {Is consistent for any setting of hyperparameters	 \newline (Simple regret tends to 0 as $n\rightarrow\infty$)}
                						& $\checkmark$ 	& X 				& $\checkmark$		& $\checkmark$ 	\\
                \hline
                {Utilises entropy for exploration \newline (E.g. Helpful for sparser rewards)}
                						& X				& $\checkmark$	& X					& $\checkmark$	\\
                \hline
                {Samples actions stochastically \newline(from a Boltzmann distribution)}
                						& X 				& $\checkmark$ 	& $\checkmark$ 		& $\checkmark$	\\
                \hline
                {Optimises for cumulative regret  \newline (penalises suboptimal actions during planning)} 
                						& $\checkmark$	& X 				& X					& X 	\\
                \hline
                	{Optimises for simple regret  \newline (does not penalise exploration during planning)}
                						& X 				& X 				& $\checkmark$ 		& $\checkmark$	\\
                \hline
                	{Complexity to run $n$ trials}	
                						& \multicolumn{4}{c}{$O(nHA)$} 							\\
                \hline
                	{Complexity to run $n$ trials using the Alias method  \newline (with backups implemented as efficiently as possible)}	
                						& N/A 		& \multicolumn{3}{c}{$O(n(H\log(A)+A))$} 			\\
            \end{tabular}
            \caption{Summary of complexities for different algorithms considered in this work, considering average return variants and alias table optimisations. \label{table:compare_algs}}
        \end{table*}

		\subsubsection{When to use each MCTS algorithm}
			Although the best way to know for sure which MCTS algorithm will be the best for a given problem is to directly try it and compare performance, there are some properties of problems that give a good indication of which algorithm would be a good fit.
			
			Two properties of environments that would make them more amenable to using entropy for exploration are large delayed rewards, and, containing trap states (i.e. states where the agent cannot move to another state for the rest of the trial). The Frozen Lake environment (Section \ref{sec:gridworld_doms}) demonstrates both of these properties, there is a large reward that the agent only gets when it reaches the goal, and there are holes that the agent falls into. Hence, if a problem contains large delayed rewards or trap states, then DENTS would likely be the best fit.
			
			When the reward signal is dense or engineered so that optimal policy can be found without needing additional exploration, then using entropy for additional exploration may cause unnecessary additional exploration, in which case using UCT or BTS would be more suitable.
			
			Finally, we note that there may be more properties of problems that we have not considered in this work that make them ameniable to different algorithms. For example, as mentioned in Section \ref{sec:limitations} that sometimes it may be desirable for an agent to follow a maximum entropy policy to explore and learn in an unknown environment, in which case MENTS, RENTS or TENTS may be more preferrable.

	    \subsubsection{Discussion on multi-threading in MCTS} \label{app:multi_thread}
	        In this section we give a brief discussion about why Boltzmann search policies naturally utilise multi-threading more than UCT. Consider a two-armed bandit, with actions $a_1$ and $a_2$, and corresponding rewards of $0$ and $1$ respectively. Intuitively, an emphasis on exploration will lead to less contention between threads as they will explore different parts of the search tree.
	        
	        When running UCB on this multi-armed bandit, with a bias of $2.0$, it will pull $a_1$ a total of $20$ times in the first $1000$ pulls, and then it will only pull $a_1$ a total of $9$ times in the next $4000$ pulls, as it begins to exploit. If we had multiple threads trying to pull arms simultaneously, but they could not pull arm $a_2$ simultaneously, then it would lead to lots of contention between the threads.
	
	        In stark contrast, when using a Boltzmann policy, with a temperature of $1.0$ on this multi-armed bandit, the expected number of pulls for $a_1$ will be $268$ pull for every $1000$ total pulls. In a multi-threaded environment, having around a quarter of the threads pulling arm $a_1$ rather than all threads pulling arm $a_2$ naturally leads to less contention and waiting.
	
	        Although this is an unrealistic example, it highlights the problem at hand: the exploitation of UCB leads to contention between threads in a multi-threaded setting. Moreover, this can be witnessed in Section \ref{app:go_results}, Table \ref{table:go_trials_per_move} where UCT is able to run fewer trials in the earlier moves of a Go game.
	        
	        Aditionally, it is worth observing that both UCT and BTS have parameters that control the amount of exploration (the bias $c$ and search temperature $\alpha$ for UCT and BTS respectively). However, because UCT is designed with cumulative regret in mind, it will always towards picking the same action on successive trials, as can be seen in our simple example.
	

\newpage
\section{Additional results and discussion} \label{app:additional_results}

    \FloatBarrier

    \subsection{Gridworld environment details} \label{app:env_deets}
        
        For space, the specific maps for the gridworlds are omitted from the results section in the main paper. In the gridworld maps, \texttt{S} denotes the starting location of the agent, \texttt{F} denotes spaces the agent can move to, \texttt{H} denote holes that end the agents trial and \texttt{G} is the goal location. 
        
        In Figure \ref{fig:fl8} we give an 8x8 Frozen Lake environment that is used in Section \ref{app:param_sens} to demonstrate how the different algorithms perform with a variety of temperatures. Figure \ref{fig:fl12} gives the 8x12 Frozen Lake Environment that is used for hyperparameter selection in Section \ref{app:hps}. And in Figure \ref{fig:fl12test} we give the 8x12 Frozen Lake Environment that is used to test the algorithms in Section \ref{sec:results}. Each of these maps was randomly generated, with each location having a probability of $1/5$ of being a hole, and the maps were checked to have a viable path from the starting location to the goal location.

        Aditionally, we give the map used in the 6x6 Sailing Problem, and the wind transition probabilities in Figure \ref{fig:sailing_deets}. In the Sailing domain, actions and wind directions can take values in $\{0,1,...,7\}$, with a value of $0$ representing North/up, $2$ representing East/right, $4$ representing South/down and $6$ representing West/right. The remaining numbers represent the inter-cardinal directions. In Section \ref{app:hps} the wind direction was set to North (or $0$) in the initial state, and for testing in Section \ref{sec:results}, the initial wind direction was set to South-East (or $3$).
    
        \begin{figure}
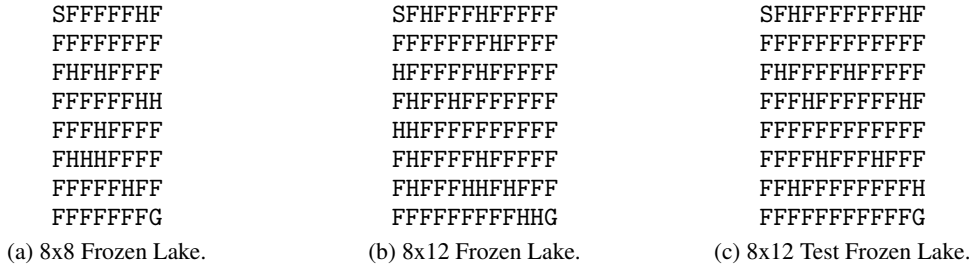

             \centering
             \begin{subfigure}[b]{0.3\textwidth}
                 \centering
                 \texttt{SFFFFFHF} \\
                 \texttt{FFFFFFFF} \\
                 \texttt{FHFHFFFF} \\
                 \texttt{FFFFFFHH} \\
                 \texttt{FFFHFFFF} \\
                 \texttt{FHHHFFFF} \\
                 \texttt{FFFFFHFF} \\
                 \texttt{FFFFFFFG} 
                 \caption{8x8 Frozen Lake.}
                 \label{fig:fl8}
             \end{subfigure}
             \hfill
             \begin{subfigure}[b]{0.3\textwidth}
                 \centering
                 \texttt{SFHFFFHFFFFF} \\
                 \texttt{FFFFFFFHFFFF} \\
                 \texttt{HFFFFFHFFFFF} \\
                 \texttt{FHFFHFFFFFFF} \\
                 \texttt{HHFFFFFFFFFF} \\
                 \texttt{FHFFFFHFFFFF} \\
                 \texttt{FHFFFHHFHFFF} \\
                 \texttt{FFFFFFFFFHHG} 
                 \caption{8x12 Frozen Lake.}
                 \label{fig:fl12}
             \end{subfigure}
             \hfill
             \begin{subfigure}[b]{0.3\textwidth}
                 \centering
                 \texttt{SFHFFFFFFFHF} \\
                 \texttt{FFFFFFFFFFFF} \\
                 \texttt{FHFFFFHFFFFF} \\
                 \texttt{FFFHFFFFFFHF} \\
                 \texttt{FFFFFFFFFFFF} \\
                 \texttt{FFFFHFFFHFFF} \\
                 \texttt{FFHFFFFFFFFH} \\
                 \texttt{FFFFFFFFFFFG} 
                 \caption{8x12 Test Frozen Lake.}
                 \label{fig:fl12test}
            \end{subfigure}
            \caption{Maps used for experiments using the Frozen Lake environment in Sections \ref{sec:results}, \ref{app:param_sens} and \ref{app:hps}. \texttt{S} is the starting location for the agent, \texttt{F} represents floor that the agent can move too, \texttt{H} are holes that end the agents trial and \texttt{G} is the goal location.}
                \label{fig:maps}
        \end{figure}
    
        \begin{figure}
             \centering
             \begin{subfigure}[b]{0.49\textwidth}
                 \centering
                 \texttt{FFFFFG} \\
                 \texttt{FFFFFF} \\
                 \texttt{FFFFFF} \\
                 \texttt{FFFFFF} \\
                 \texttt{FFFFFF} \\
                 \texttt{SFFFFF} 
                 \caption{6x6 Sailing Problem map.}
             \end{subfigure}
             \hfill
             \begin{subfigure}[b]{0.49\textwidth}
                 \centering
                 \begin{align*}
                     \begin{pmatrix}
                        0.4 & 0.3 & 0.0 & 0.0 & 0.0 & 0.0 & 0.0 & 0.3 \\
                        0.4 & 0.3 & 0.3 & 0.0 & 0.0 & 0.0 & 0.0 & 0.0 \\
                        0.0 & 0.4 & 0.3 & 0.3 & 0.0 & 0.0 & 0.0 & 0.0 \\
                        0.0 & 0.0 & 0.4 & 0.3 & 0.3 & 0.0 & 0.0 & 0.0 \\
                        0.0 & 0.0 & 0.0 & 0.4 & 0.2 & 0.4 & 0.0 & 0.0 \\
                        0.0 & 0.0 & 0.0 & 0.0 & 0.3 & 0.3 & 0.4 & 0.0 \\
                        0.0 & 0.0 & 0.0 & 0.0 & 0.0 & 0.3 & 0.3 & 0.4 \\
                        0.4 & 0.0 & 0.0 & 0.0 & 0.0 & 0.0 & 0.3 & 0.3 
                     \end{pmatrix}
                 \end{align*}
                 \caption{Wind transition probabilities.}
             \end{subfigure}
                \caption{The map used for the 6x6 Sailing Problem and the wind transition probabilities. For the wind transition probabilities, the $(i,j)th$ element of the matrix denotes the probability that the wind changes from direction $i$ to direction $j$, where $0$ denotes North/up, $1$ denotes North-East/up-right, and so on.}
                \label{fig:sailing_deets}
        \end{figure}

    \subsection{Futher discussion on parameter sensitivity} \label{app:param_sens}
        In this section we provide a more detailed discussion on sensitivity to the temperature parameter in MENTS \cite{xiao2019maximum}, RENTS and TENTS \cite{rents_your_tents}. We provide more thorough results on D-chain environments, with each algorithm, and varying both the temperature parameter and the $\epsilon$ exploration parameter. Additionally, we provide results using the 8x8 Frozen Lake environment (Figure \ref{fig:fl8}) using a variety of temperatures to demonstrate how each algorithm performs with different temperatures in that domain.

        \subsubsection{Parameter sensitivity in D-Chain} \label{app:param_sens_dchain}
            
            \begin{figure}
                \centering
                \includegraphics[width=0.7\textwidth]{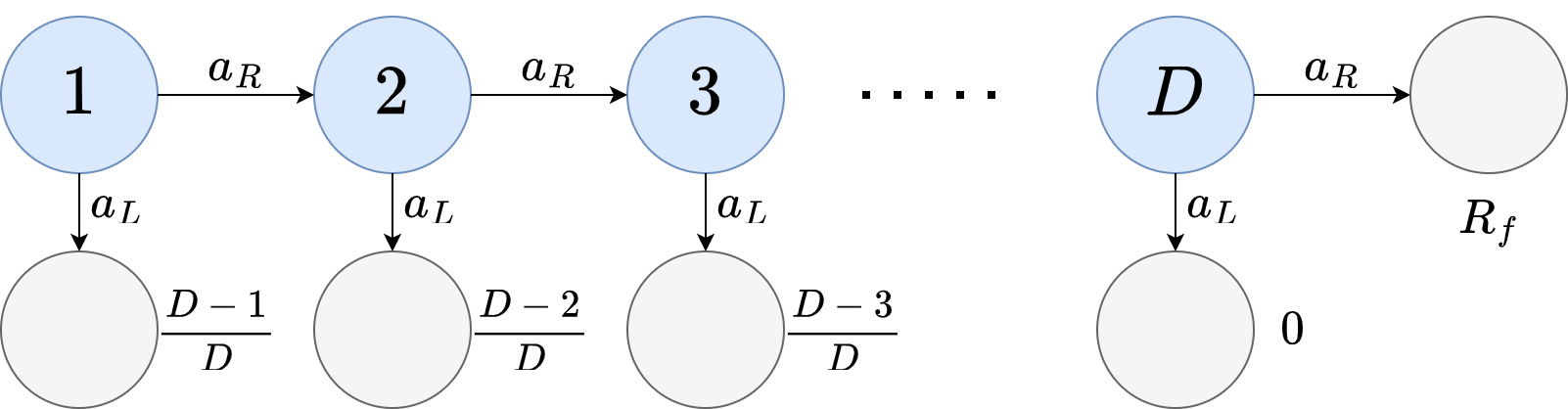}
                \caption{An illustration of the \textit{(modified) D-chain problem}, where 1 is the starting state, transitions are deterministic and values next to states represent rewards for arriving in that state.}
                \label{fig:dchain_illustration_2}
            \end{figure}

            We run each algorithm with a variety of $\alpha$ temperatures, and the $\epsilon$ exploration parameter on the 10-chain environments (Figure \ref{fig:dchain_illustration_2}). Additionally, we ran UCT with a variety of bias parameters. Figures \ref{fig:uct_10chain_hps}, \ref{fig:ments_10chain_hps}, \ref{fig:rents_10chain_hps}, \ref{fig:tents_10chain_hps}, \ref{fig:bts_10chain_hps} and \ref{fig:dents_10chain_hps} give results for the 10-chain environment, with algorithms UCT, MENTS, RENTS, TENTS, BTS and DENTS respectively. Figures \ref{fig:uct_10chain_half_hps}, \ref{fig:ments_10chain_half_hps}, \ref{fig:rents_10chain_half_hps}, \ref{fig:tents_10chain_half_hps}, \ref{fig:bts_10chain_half_hps} and \ref{fig:dents_10chain_half_hps} give results for the modified 10-chain environment, with algorithms UCT, MENTS, RENTS, TENTS, BTS and DENTS respectively. 

            As expected with UCT, regardless of how the bias parameter is set, in both the 10-chain ($D=10$, $R_f=1.0$) and modified 10-chain ($D=10$, $R_f=0.5$) environments, it only achieves a value of $0.9$. See Figures \ref{fig:uct_10chain_hps} and \ref{fig:uct_10chain_half_hps} for plots.

            As discussed in Section \ref{sec:limitations}, for higher temperatures in MENTS it will find the reward of $R_f$ in both the 10-chain and modified 10-chain environments. At a temperature of $\alpha=0.15$ MENTS is able to find the reward of $R_f=1$ on the 10-chain (Figure \ref{fig:ments_10chain_hps}), but will still recommend a policy that gives the reward of $R_f=0.5$ on the modified 10-chain (Figure \ref{fig:ments_10chain_half_hps}). At a temperature of $\alpha=0.1$ MENTS will struggle to find the reward of $R_f=1$ in the 10-chain, without the help of the exploration parameter, but this is the first temperature we tried that was able to recommend the optimal policy in the modified 10-chain (Figure \ref{fig:ments_10chain_half_hps}). For low temperatures, such as $\alpha=0.01$, MENTS was able to find the optimal policy, but in the case of the 10-chain with $R_f=1$ it can only do so with the help of a higher exploration parameter.

            When we ran TENTS on the (modified) 10-chain, we see results that parallel MENTS, see Figures \ref{fig:tents_10chain_hps} and \ref{fig:tents_10chain_half_hps}. Interestingly, RENTS was only able to find the reward of $R_f=1$ on the 10-chain environment if we used a low temperature, $\alpha=0.01$ and a high exploration parameter, $\epsilon=10$. Otherwise, RENTS tended to behave similarly to UCT on these environments, see Figures \ref{fig:rents_10chain_hps} and \ref{fig:rents_10chain_half_hps}.

            In contrast, BTS was able to find the reward of $R_f=1.0$ in the 10-chain when a high search temperature or high exploration parameter was used (Figure \ref{fig:bts_10chain_hps}). And, in the modified 10-chain, BTS always achieves a reward of $0.9$ regardless of how the parameters are set (Figure \ref{fig:bts_10chain_half_hps}). DENTS performance on the 10-chain (Figure \ref{fig:dents_10chain_hps}) and modified 10-chain (Figure \ref{fig:dents_10chain_half_hps}) was similar to BTS, but tended to find the reward of $R_f=1$ in the 10-chain marginally faster. For the decay function $\beta$ in DENTS, we always set $\beta(m)=\alpha/\log(e+m)$ for these experiments.

            To demonstrate that the $\epsilon$ exploration parameter is insufficient to make up for a low temperature, we also consider the 20-chain ($D=20$, $R_f=1$) and modified 20-chain ($D=20$, $R_f=0.5$) problems. We don't give plots for all algorithms on both of the 20-chain environments like we do for 10-chain environments, but opt for the plots that demonstrate something interesting. 
            
            In Figure \ref{fig:ments_20chain_hps} we see MENTS on the 20-chain is able to find the reward of $R_f=1$ for higher temperatures. However, this time, the exploration parameter does not make much of an impact when using lower temperatures. Moreover, a large exploration parameter appears to negatively impact MENTS ability to find $R_f=1$. This makes sense considering that a uniformly random policy will find the reward at the end of the chain once every $2^{10}$ trials in the 10-chain, but only once every $2^{20}$ in the 20-chain. Again, on the modified 20-chain, MENTS is only able to recommend the optimal policy for low temperatures (see Figure \ref{fig:ments_20chain_half_hps}). 

            When we ran BTS on the 20-chain, it was unsuccessful at finding the final reward of $R_f=1$, which makes sense as it is not using entropy for exploration, and it is unlikely to follow a random policy to the end of the chain (Figure \ref{fig:bts_20chain_hps}). For DENTS, we again used a decay function of $\beta(m)=\alpha/\log(e+m)$ for simplicity, and unfortunately it was only able to make slow progress towards finding the final reward of $R_f=1$ for high temperatures. However, if we independently set the values of $\alpha$ and 
            
            However, DENTS on the 20-chain begins to make slow progress towards finding the final reward of $R_f=1$, but requires a higher temperature to be used, as we decay the weighting of entropy over time (Figure \ref{fig:dents_20chain_hps}). Again we used a decay function of $\beta(m)=\alpha/\log(e+m)$ here for simplicity, and if we properly select them DENTS is more than capable of solving the 20-chain. For example we show that using DENTS with $\alpha=0.5$, $\beta(m)=10/\log(e+m)$ and $\epsilon=0.01$ in Figure \ref{fig:dents_20chain_tuned}, where $\alpha$ is set low enough that there is still a high probability of following the chain to the end, $\beta$ is set to be large initially to encourage exploring with the entropy reward and $\epsilon$ is set low to avoid random exploration ending trials before reaching the end of the chain. If we were to run DENTS and BTS on the modified 20-chain they would recommend the optimal policy giving a value of $0.95$ for all of the parameters we searched over (not shown).
            
            Finally, in Figure \ref{fig:dbments_20chain_hps} we also consider running DENTS, but instead setting $\beta(m)=\alpha$ to replicate MENTS. The main difference between DENTS in this case and MENTS is the recommendation policy, where DENTS uses the Bellman values for recommendations, rather than soft values. So even in cases where the MENTS search is more desirable, we can replicate it with DENTS while providing recommendations for the standard objective. Moreover, running DENTS with $\beta(m)=\alpha$ on the modified 20-chain would always yield the optimal value of $0.95$ because of the use of Bellman values for recommendations (not shown).

            \FloatBarrier

            \begin{figure}
                \centering
                
                \begin{subfigure}[b]{0.32\textwidth}
                    \centering
                    \includegraphics[width=\textwidth]{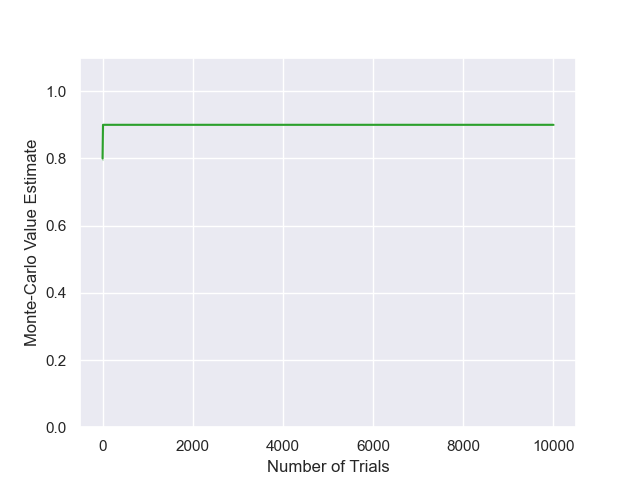}
                    \caption*{Bias set using \cite{prst}.}
                \end{subfigure}
                \begin{subfigure}[b]{0.32\textwidth}
                    \centering
                    \includegraphics[width=\textwidth]{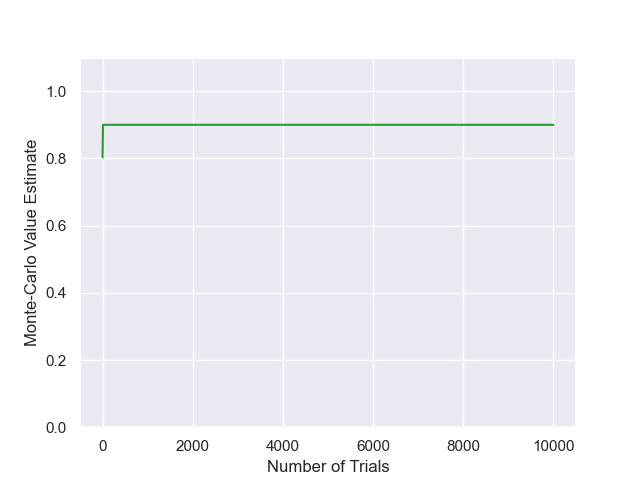}
                    \caption*{Bias $=0.1$.}
                \end{subfigure}
                \begin{subfigure}[b]{0.32\textwidth}
                    \centering
                    \includegraphics[width=\textwidth]{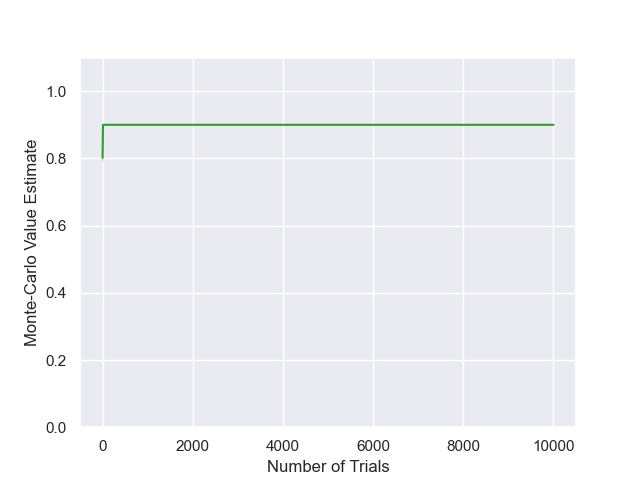}
                    \caption*{Bias $=1$.}
                \end{subfigure}
                
                \begin{subfigure}[b]{0.32\textwidth}
                    \centering
                    \includegraphics[width=\textwidth]{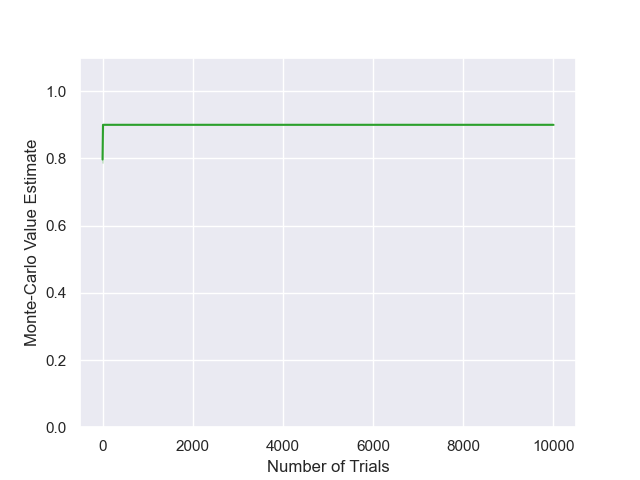}
                    \caption*{Bias $=2$.}
                \end{subfigure}
                \begin{subfigure}[b]{0.32\textwidth}
                    \centering
                    \includegraphics[width=\textwidth]{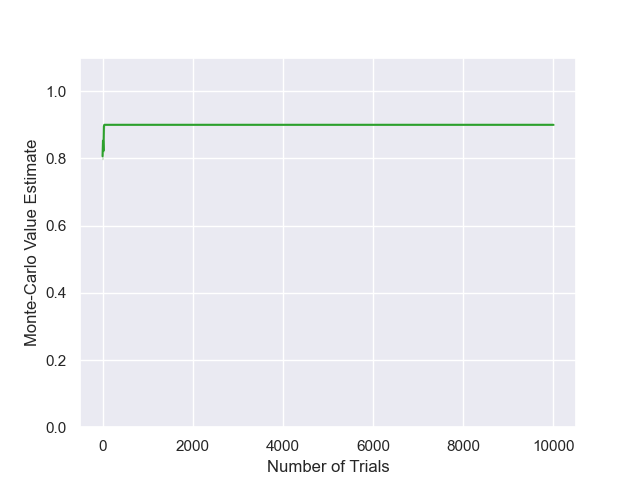}
                    \caption*{Bias $=10$.}
                \end{subfigure}
                \begin{subfigure}[b]{0.32\textwidth}
                    \centering
                    \includegraphics[width=\textwidth]{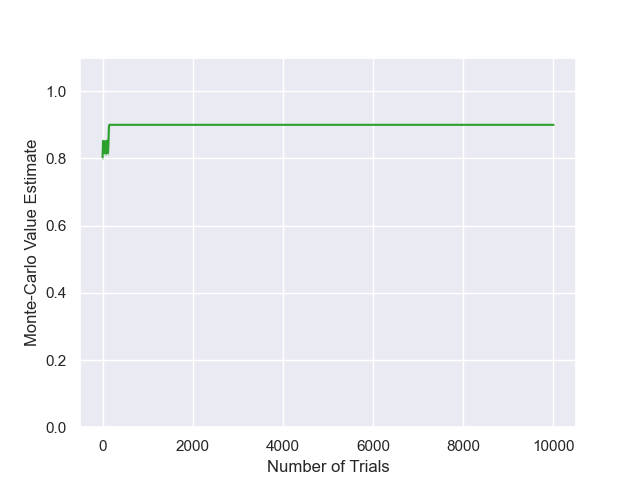}
                    \caption*{Bias $=100$.}
                \end{subfigure}
                
                \caption{Results for UCT on the 10-chain ($D=10$, $R_f=1.0$), for varying bias parameters.}
                \label{fig:uct_10chain_hps}
            \end{figure}

            \begin{figure}
                \centering
                
                \begin{subfigure}[b]{0.24\textwidth}
                    \centering
                    \includegraphics[width=\textwidth]{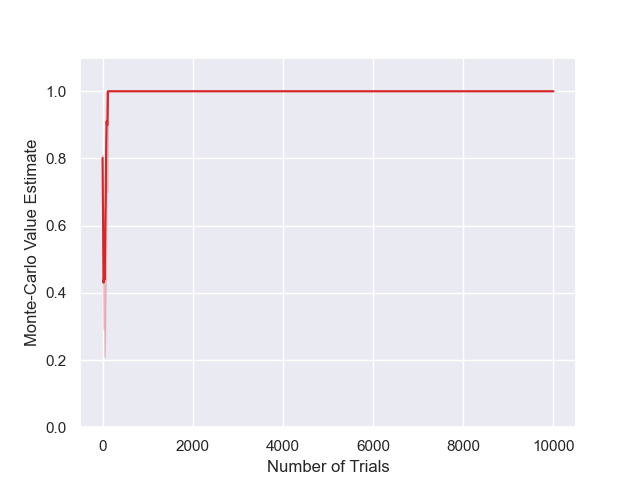}
                    \caption*{$\alpha=1,\epsilon=0.01$}
                \end{subfigure}
                \begin{subfigure}[b]{0.24\textwidth}
                    \centering
                    \includegraphics[width=\textwidth]{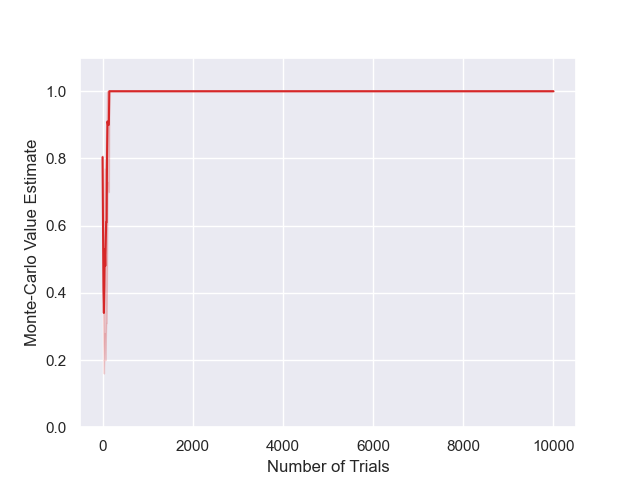}
                    \caption*{$\alpha=1,\epsilon=0.1$}
                \end{subfigure}
                \begin{subfigure}[b]{0.24\textwidth}
                    \centering
                    \includegraphics[width=\textwidth]{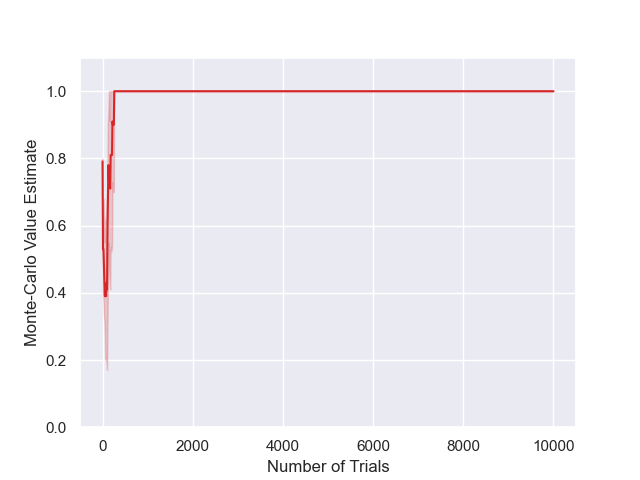}
                    \caption*{$\alpha=1,\epsilon=1$}
                \end{subfigure}
                \begin{subfigure}[b]{0.24\textwidth}
                    \centering
                    \includegraphics[width=\textwidth]{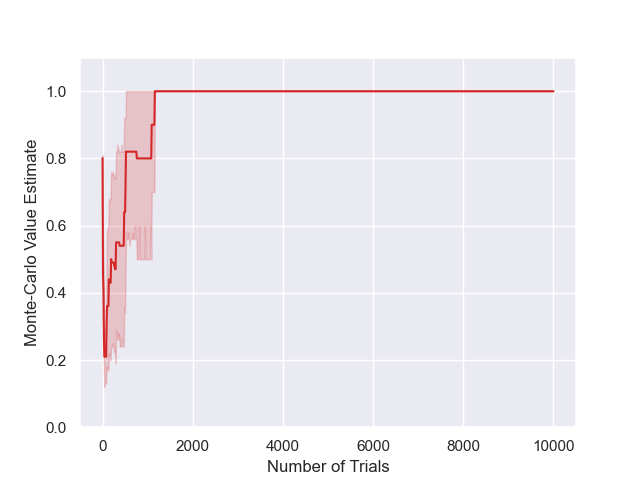}
                    \caption*{$\alpha=1,\epsilon=10$}
                \end{subfigure}
                
                \begin{subfigure}[b]{0.24\textwidth}
                    \centering
                    \includegraphics[width=\textwidth]{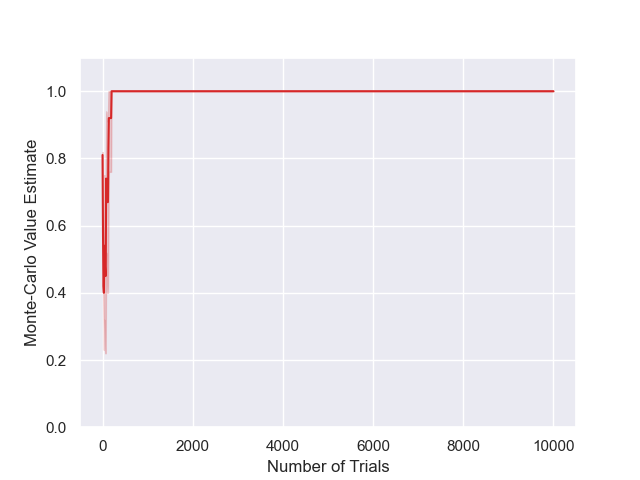}
                    \caption*{$\alpha=0.5,\epsilon=0.01$}
                \end{subfigure}
                \begin{subfigure}[b]{0.24\textwidth}
                    \centering
                    \includegraphics[width=\textwidth]{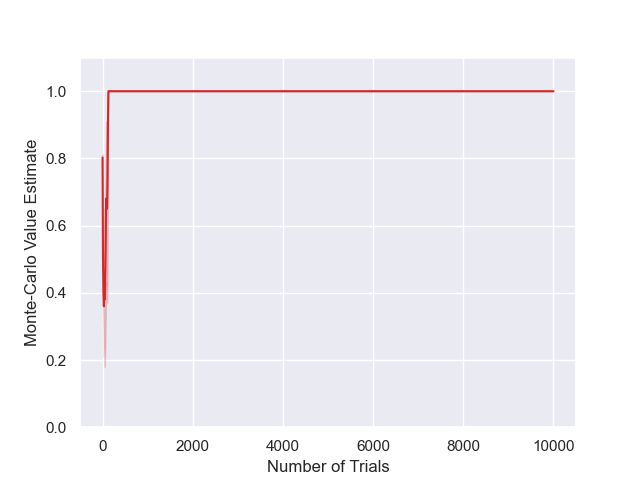}
                    \caption*{$\alpha=0.5,\epsilon=0.1$}
                \end{subfigure}
                \begin{subfigure}[b]{0.24\textwidth}
                    \centering
                    \includegraphics[width=\textwidth]{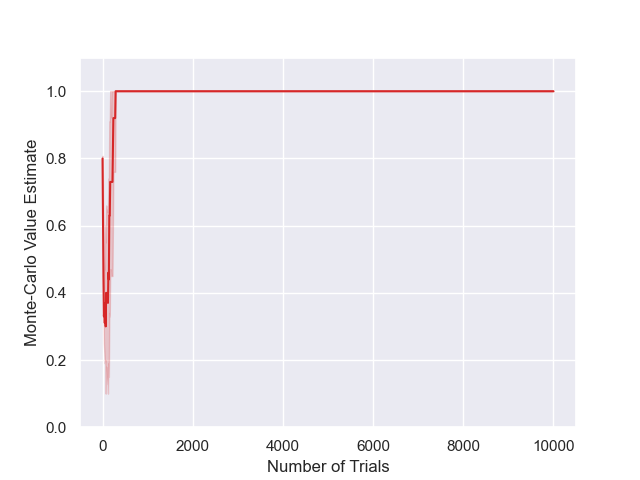}
                    \caption*{$\alpha=0.5,\epsilon=1$}
                \end{subfigure}
                \begin{subfigure}[b]{0.24\textwidth}
                    \centering
                    \includegraphics[width=\textwidth]{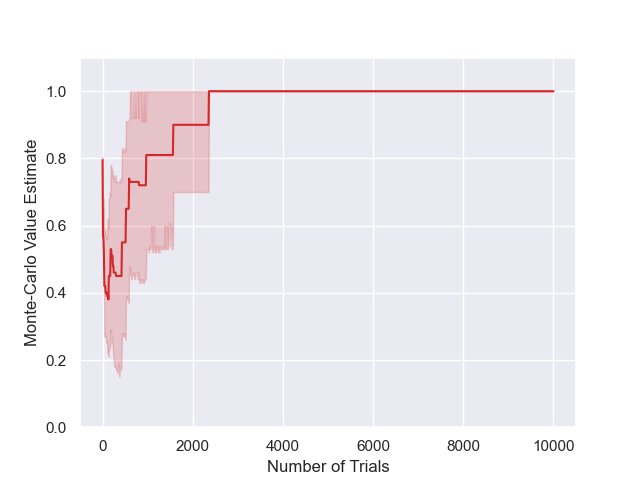}
                    \caption*{$\alpha=0.5,\epsilon=10$}
                \end{subfigure}
                
                \begin{subfigure}[b]{0.24\textwidth}
                    \centering
                    \includegraphics[width=\textwidth]{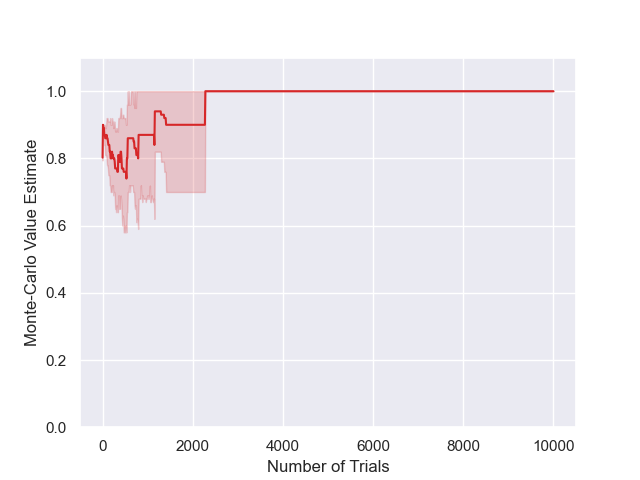}
                    \caption*{$\alpha=0.2,\epsilon=0.01$}
                \end{subfigure}
                \begin{subfigure}[b]{0.24\textwidth}
                    \centering
                    \includegraphics[width=\textwidth]{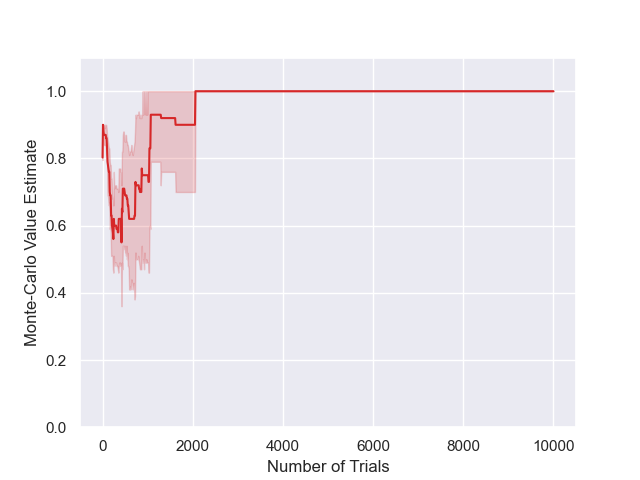}
                    \caption*{$\alpha=0.2,\epsilon=0.1$}
                \end{subfigure}
                \begin{subfigure}[b]{0.24\textwidth}
                    \centering
                    \includegraphics[width=\textwidth]{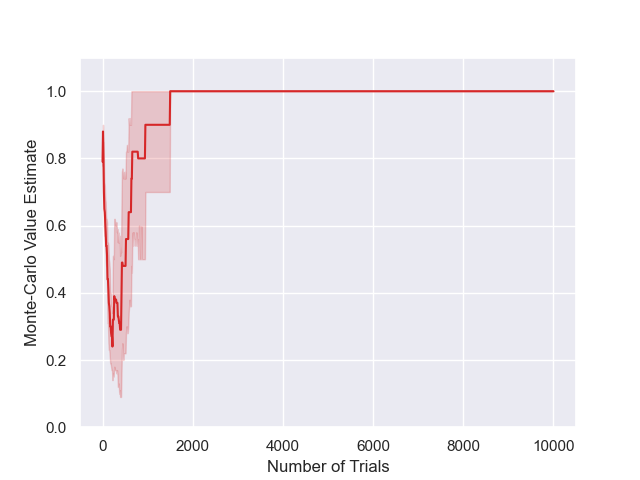}
                    \caption*{$\alpha=0.2,\epsilon=1$}
                \end{subfigure}
                \begin{subfigure}[b]{0.24\textwidth}
                    \centering
                    \includegraphics[width=\textwidth]{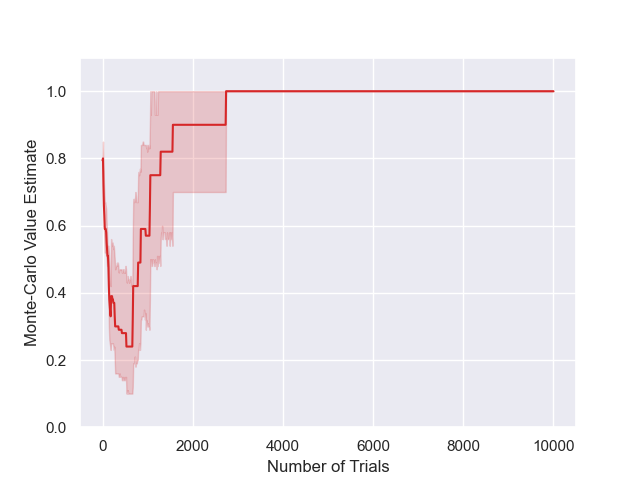}
                    \caption*{$\alpha=0.2,\epsilon=10$}
                \end{subfigure}
                
                \begin{subfigure}[b]{0.24\textwidth}
                    \centering
                    \includegraphics[width=\textwidth]{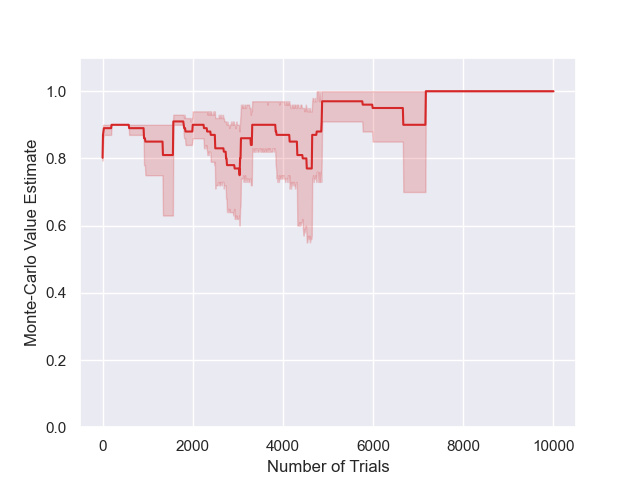}
                    \caption*{$\alpha=0.15,\epsilon=0.01$}
                \end{subfigure}
                \begin{subfigure}[b]{0.24\textwidth}
                    \centering
                    \includegraphics[width=\textwidth]{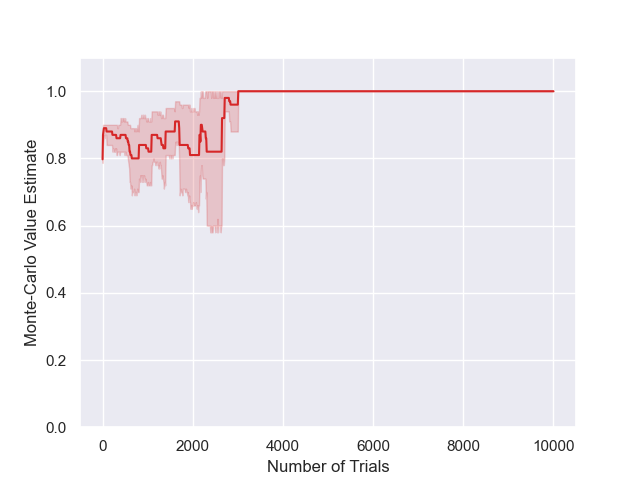}
                    \caption*{$\alpha=0.15,\epsilon=0.1$}
                \end{subfigure}
                \begin{subfigure}[b]{0.24\textwidth}
                    \centering
                    \includegraphics[width=\textwidth]{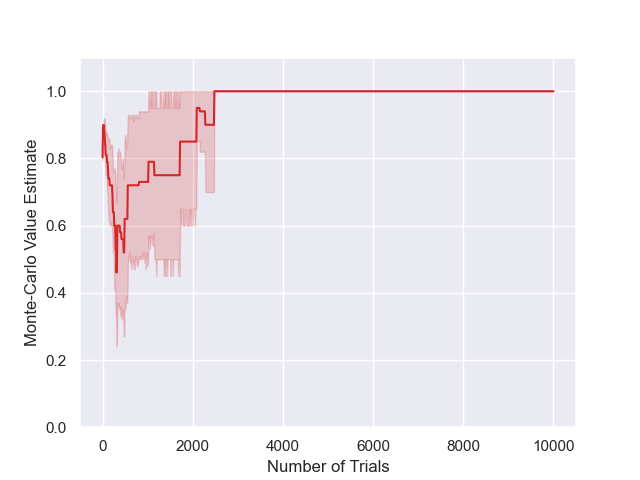}
                    \caption*{$\alpha=0.15,\epsilon=1$}
                \end{subfigure}
                \begin{subfigure}[b]{0.24\textwidth}
                    \centering
                    \includegraphics[width=\textwidth]{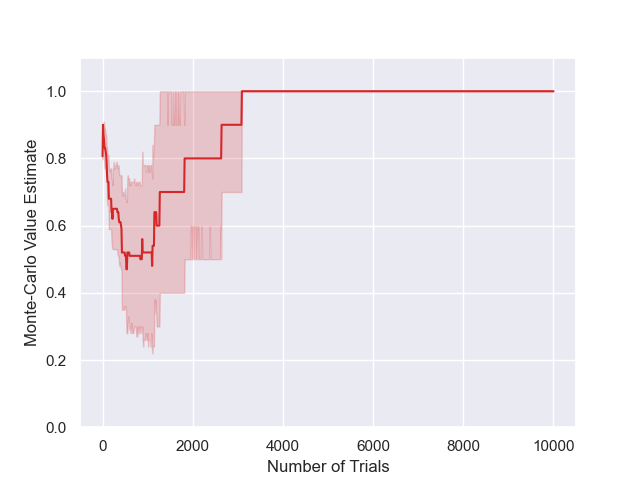}
                    \caption*{$\alpha=0.15,\epsilon=10$}
                \end{subfigure}
                
                \begin{subfigure}[b]{0.24\textwidth}
                    \centering
                    \includegraphics[width=\textwidth]{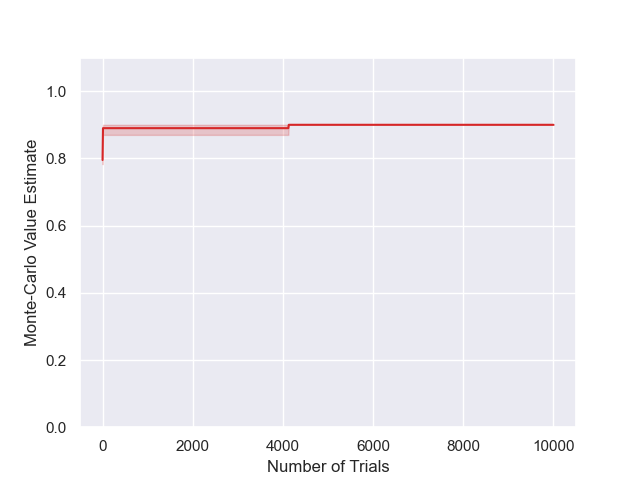}
                    \caption*{$\alpha=0.1,\epsilon=0.01$}
                \end{subfigure}
                \begin{subfigure}[b]{0.24\textwidth}
                    \centering
                    \includegraphics[width=\textwidth]{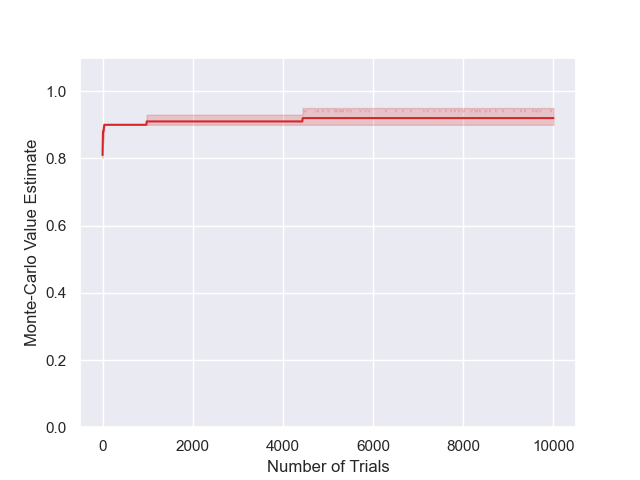}
                    \caption*{$\alpha=0.1,\epsilon=0.1$}
                \end{subfigure}
                \begin{subfigure}[b]{0.24\textwidth}
                    \centering
                    \includegraphics[width=\textwidth]{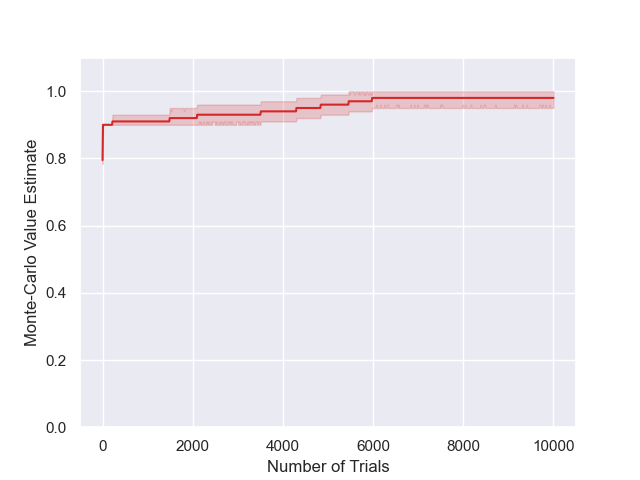}
                    \caption*{$\alpha=0.1,\epsilon=1$}
                \end{subfigure}
                \begin{subfigure}[b]{0.24\textwidth}
                    \centering
                    \includegraphics[width=\textwidth]{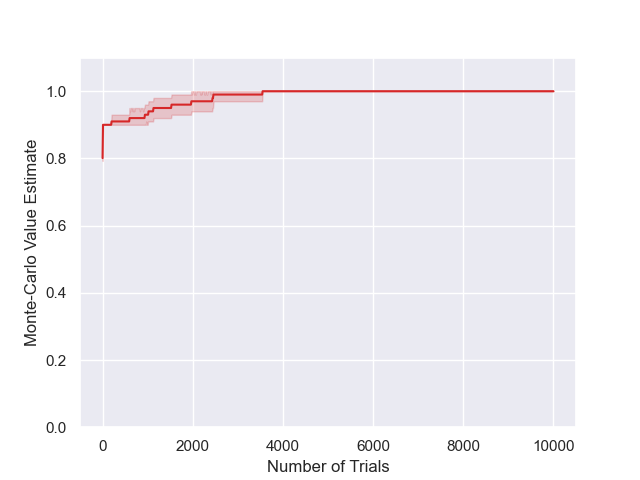}
                    \caption*{$\alpha=0.1,\epsilon=10$}
                \end{subfigure}
                
                \begin{subfigure}[b]{0.24\textwidth}
                    \centering
                    \includegraphics[width=\textwidth]{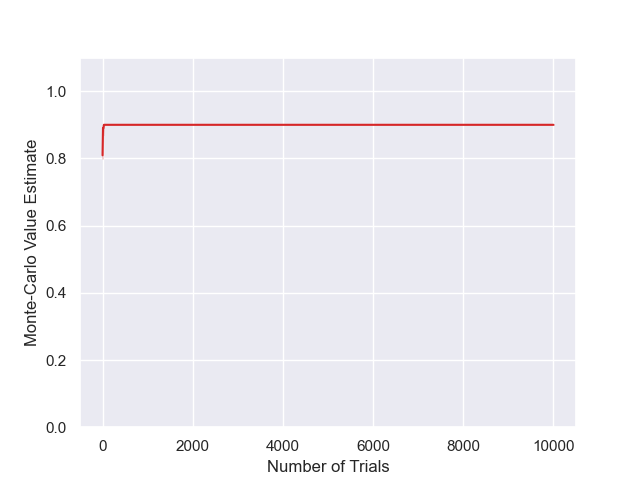}
                    \caption*{$\alpha=0.05,\epsilon=0.01$}
                \end{subfigure}
                \begin{subfigure}[b]{0.24\textwidth}
                    \centering
                    \includegraphics[width=\textwidth]{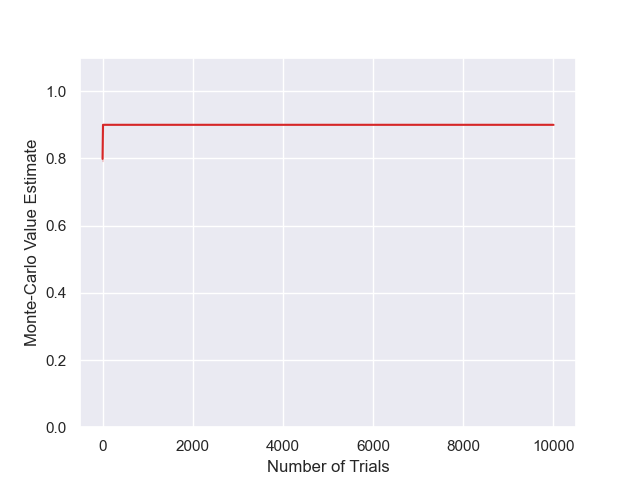}
                    \caption*{$\alpha=0.05,\epsilon=0.1$}
                \end{subfigure}
                \begin{subfigure}[b]{0.24\textwidth}
                    \centering
                    \includegraphics[width=\textwidth]{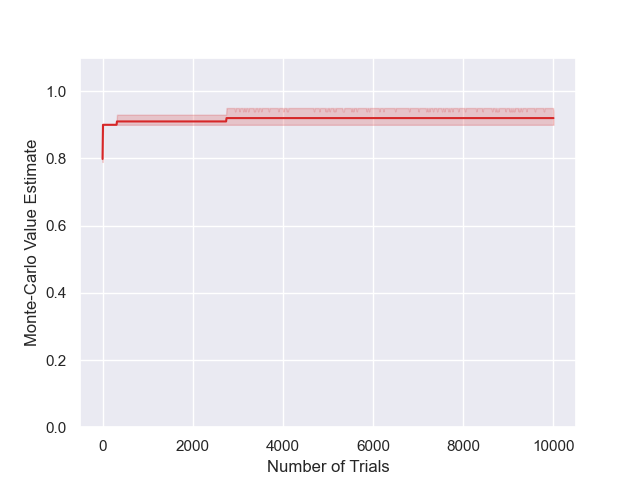}
                    \caption*{$\alpha=0.05,\epsilon=1$}
                \end{subfigure}
                \begin{subfigure}[b]{0.24\textwidth}
                    \centering
                    \includegraphics[width=\textwidth]{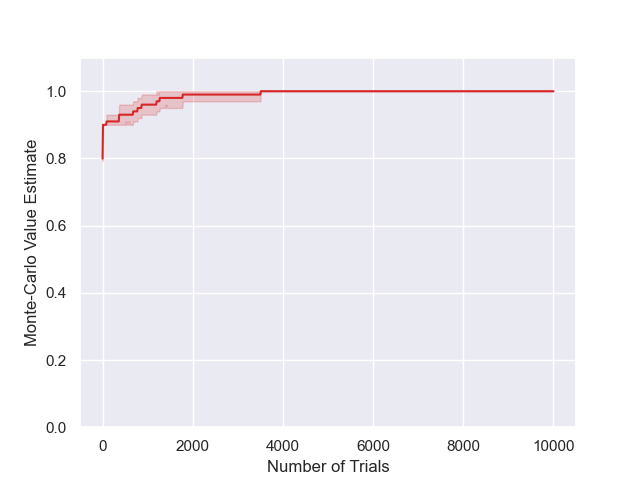}
                    \caption*{$\alpha=0.05,\epsilon=10$}
                \end{subfigure}
                
                \begin{subfigure}[b]{0.24\textwidth}
                    \centering
                    \includegraphics[width=\textwidth]{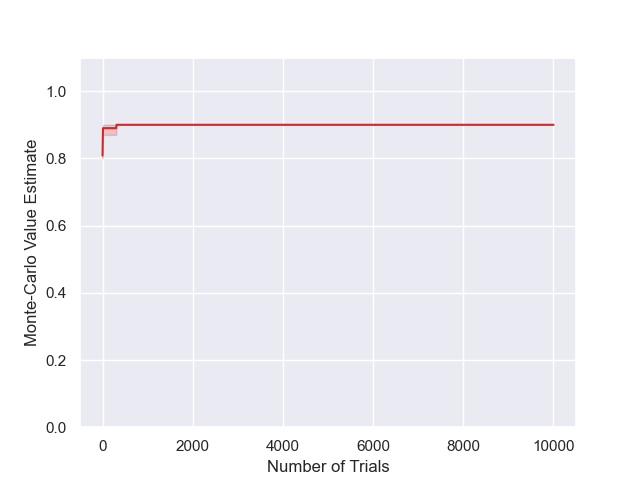}
                    \caption*{$\alpha=0.01,\epsilon=0.01$}
                \end{subfigure}
                \begin{subfigure}[b]{0.24\textwidth}
                    \centering
                    \includegraphics[width=\textwidth]{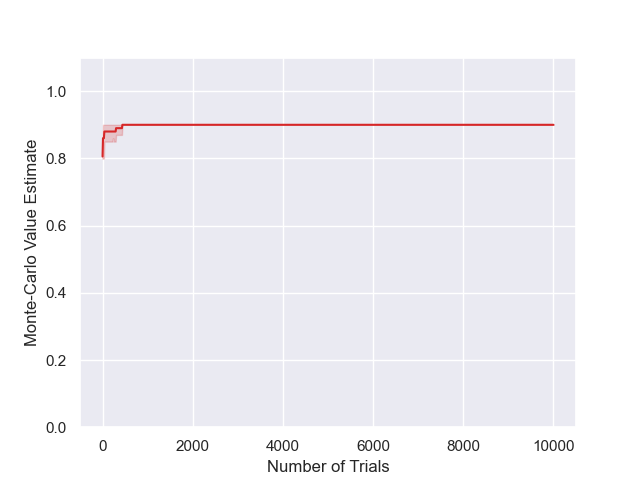}
                    \caption*{$\alpha=0.01,\epsilon=0.1$}
                \end{subfigure}
                \begin{subfigure}[b]{0.24\textwidth}
                    \centering
                    \includegraphics[width=\textwidth]{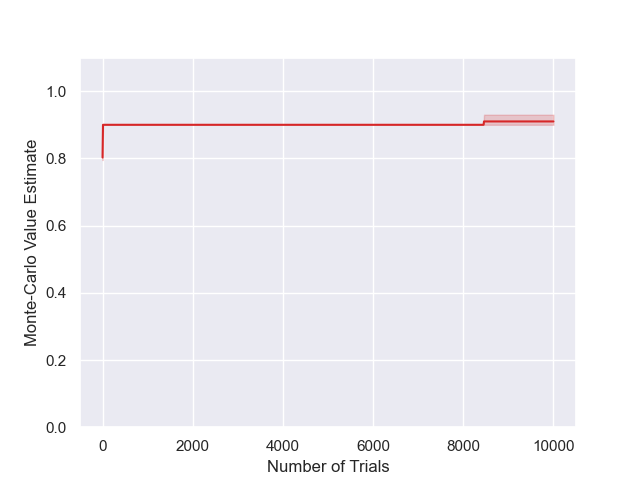}
                    \caption*{$\alpha=0.01,\epsilon=1$}
                \end{subfigure}
                \begin{subfigure}[b]{0.24\textwidth}
                    \centering
                    \includegraphics[width=\textwidth]{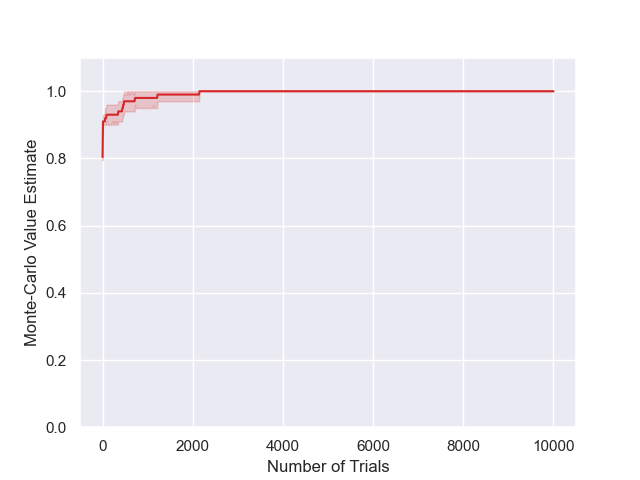}
                    \caption*{$\alpha=0.01,\epsilon=10$}
                \end{subfigure}
                
                \caption{Results for MENTS on the 10-chain ($D=10$, $R_f=1.0$), for varying temperatures and exploration parameters.}
                \label{fig:ments_10chain_hps}
            \end{figure}

            \begin{figure}
                \centering
                
                \begin{subfigure}[b]{0.24\textwidth}
                    \centering
                    \includegraphics[width=\textwidth]{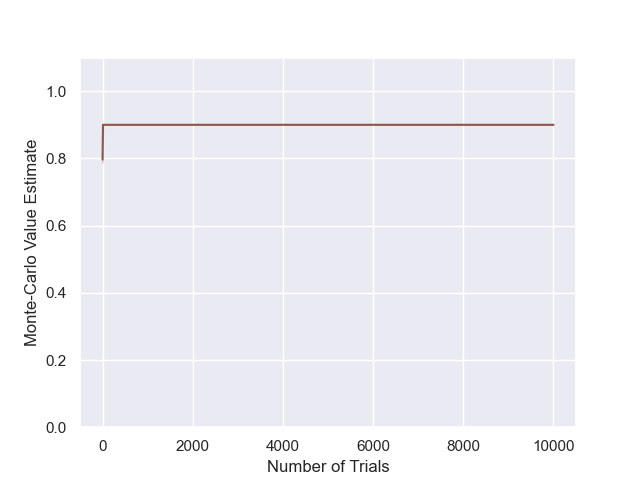}
                    \caption*{$\alpha=1000,\epsilon=0.01$}
                \end{subfigure}
                \begin{subfigure}[b]{0.24\textwidth}
                    \centering
                    \includegraphics[width=\textwidth]{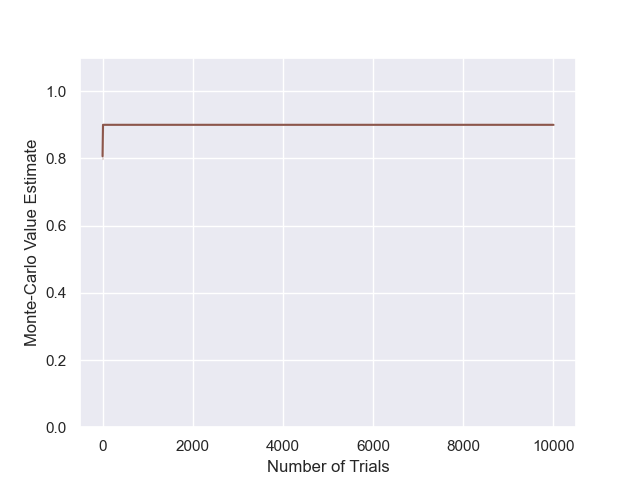}
                    \caption*{$\alpha=1000,\epsilon=0.1$}
                \end{subfigure}
                \begin{subfigure}[b]{0.24\textwidth}
                    \centering
                    \includegraphics[width=\textwidth]{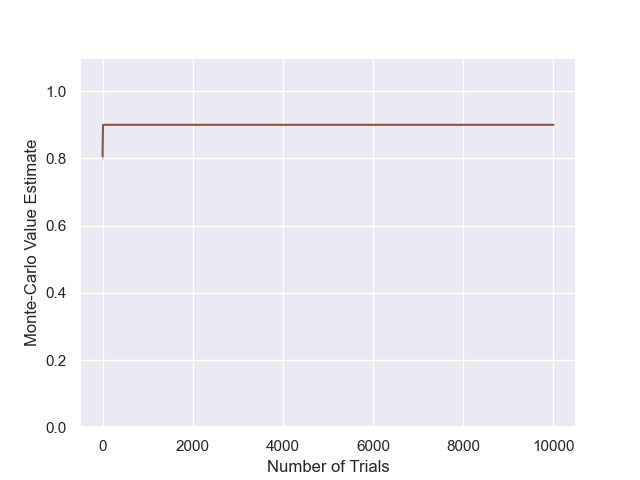}
                    \caption*{$\alpha=1000,\epsilon=1$}
                \end{subfigure}
                \begin{subfigure}[b]{0.24\textwidth}
                    \centering
                    \includegraphics[width=\textwidth]{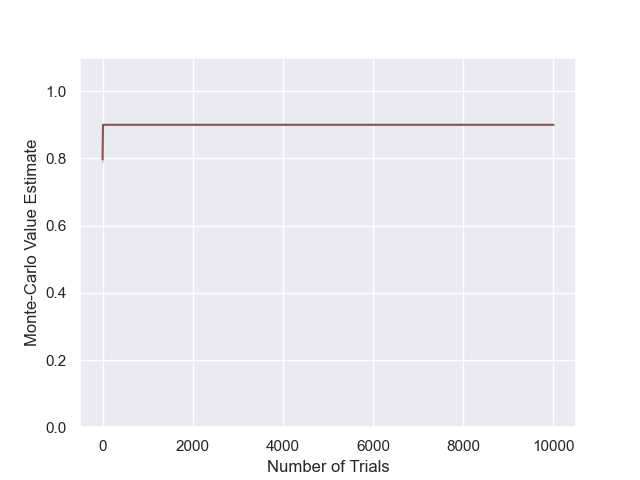}
                    \caption*{$\alpha=1000,\epsilon=10$}
                \end{subfigure}
                
                \begin{subfigure}[b]{0.24\textwidth}
                    \centering
                    \includegraphics[width=\textwidth]{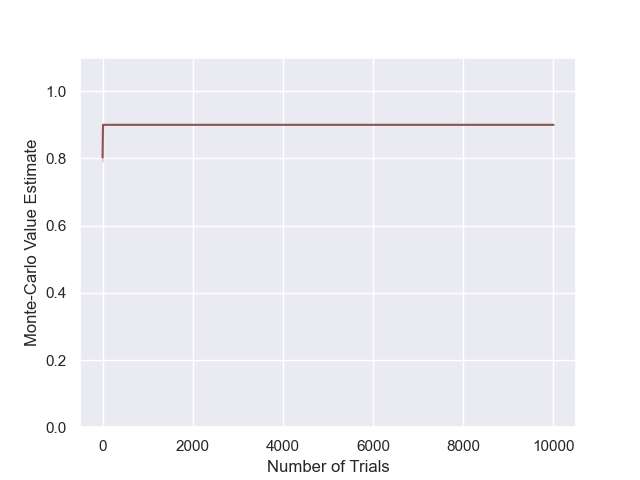}
                    \caption*{$\alpha=100,\epsilon=0.01$}
                \end{subfigure}
                \begin{subfigure}[b]{0.24\textwidth}
                    \centering
                    \includegraphics[width=\textwidth]{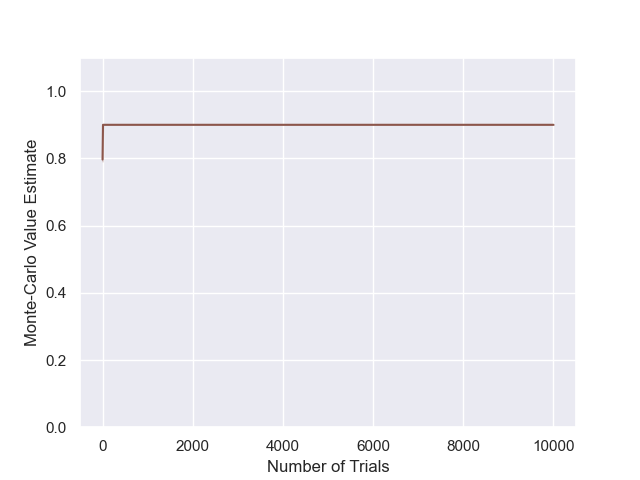}
                    \caption*{$\alpha=100,\epsilon=0.1$}
                \end{subfigure}
                \begin{subfigure}[b]{0.24\textwidth}
                    \centering
                    \includegraphics[width=\textwidth]{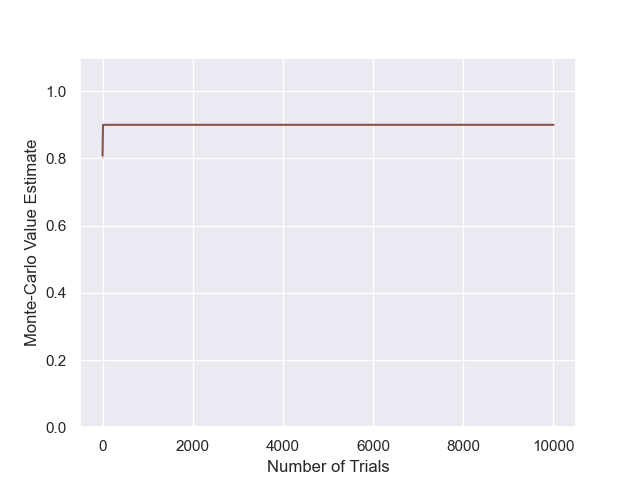}
                    \caption*{$\alpha=100,\epsilon=1$}
                \end{subfigure}
                \begin{subfigure}[b]{0.24\textwidth}
                    \centering
                    \includegraphics[width=\textwidth]{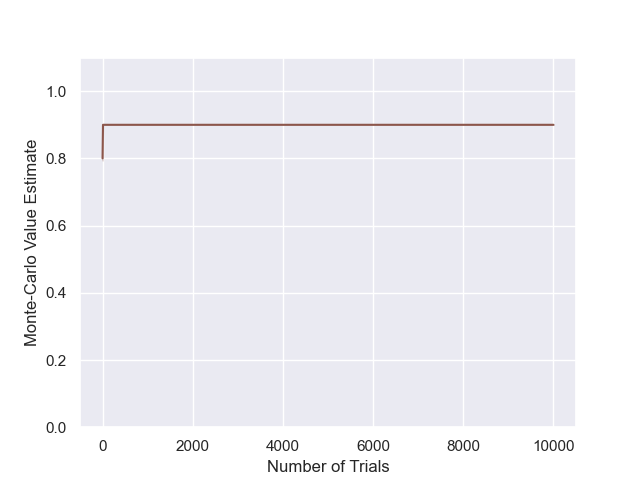}
                    \caption*{$\alpha=100,\epsilon=10$}
                \end{subfigure}
                
                \begin{subfigure}[b]{0.24\textwidth}
                    \centering
                    \includegraphics[width=\textwidth]{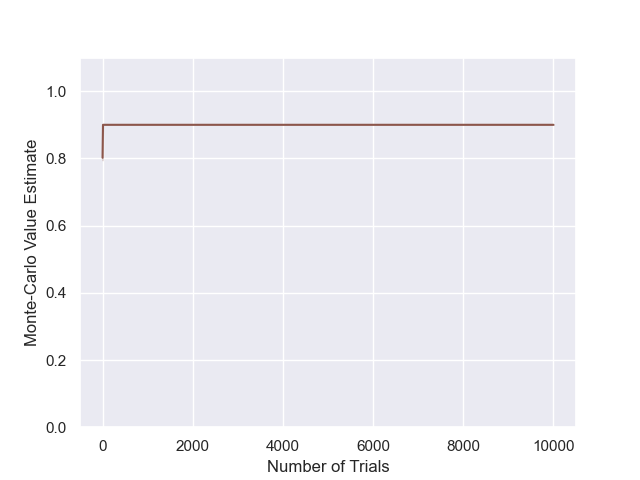}
                    \caption*{$\alpha=10,\epsilon=0.01$}
                \end{subfigure}
                \begin{subfigure}[b]{0.24\textwidth}
                    \centering
                    \includegraphics[width=\textwidth]{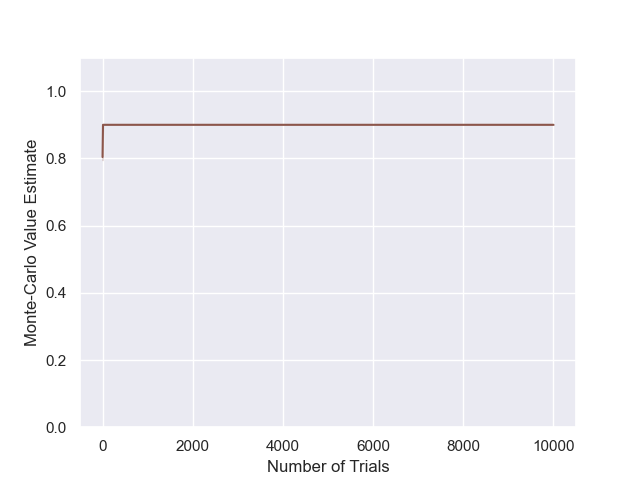}
                    \caption*{$\alpha=10,\epsilon=0.1$}
                \end{subfigure}
                \begin{subfigure}[b]{0.24\textwidth}
                    \centering
                    \includegraphics[width=\textwidth]{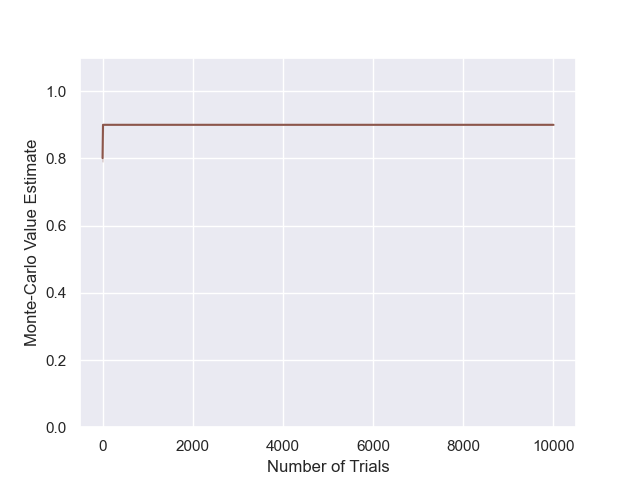}
                    \caption*{$\alpha=10,\epsilon=1$}
                \end{subfigure}
                \begin{subfigure}[b]{0.24\textwidth}
                    \centering
                    \includegraphics[width=\textwidth]{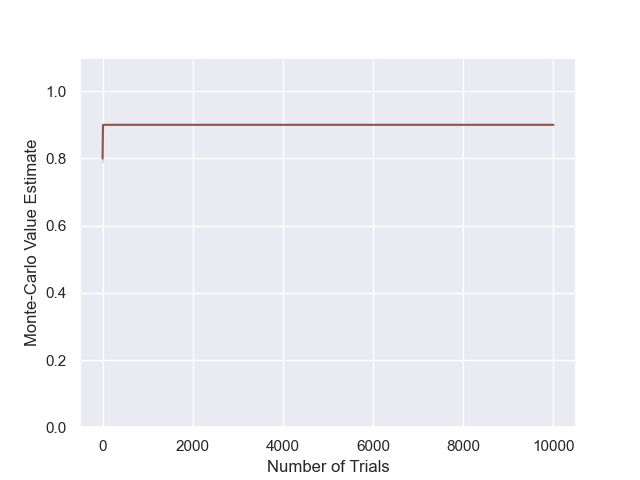}
                    \caption*{$\alpha=10,\epsilon=10$}
                \end{subfigure}
                
                \begin{subfigure}[b]{0.24\textwidth}
                    \centering
                    \includegraphics[width=\textwidth]{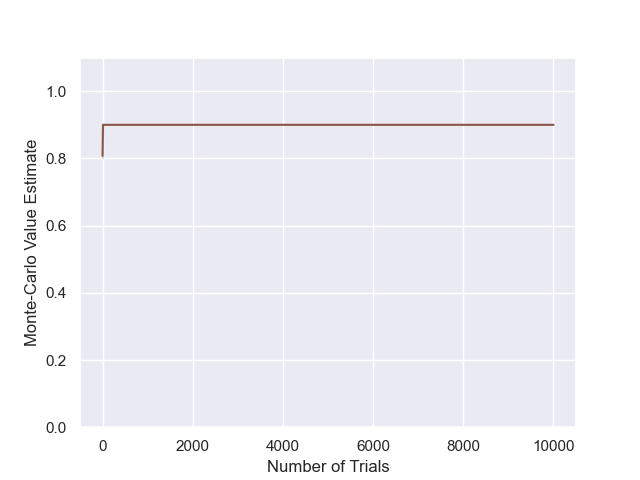}
                    \caption*{$\alpha=1,\epsilon=0.01$}
                \end{subfigure}
                \begin{subfigure}[b]{0.24\textwidth}
                    \centering
                    \includegraphics[width=\textwidth]{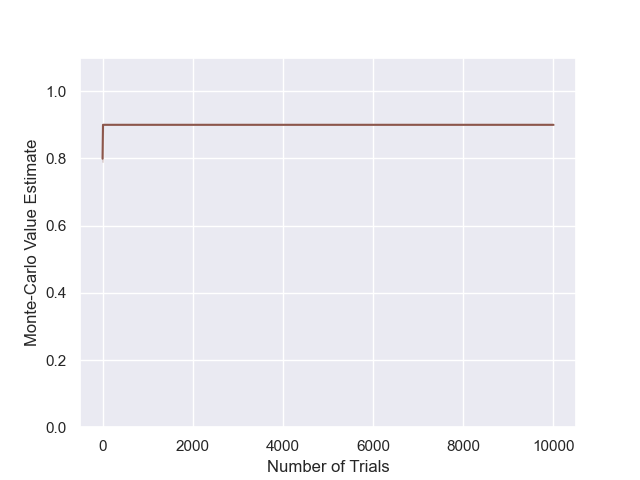}
                    \caption*{$\alpha=1,\epsilon=0.1$}
                \end{subfigure}
                \begin{subfigure}[b]{0.24\textwidth}
                    \centering
                    \includegraphics[width=\textwidth]{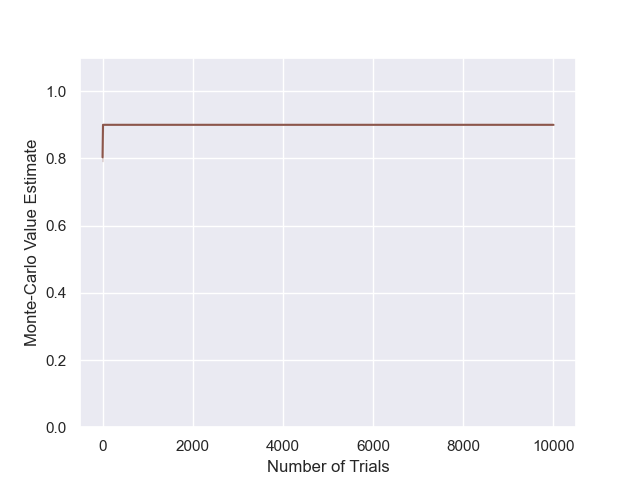}
                    \caption*{$\alpha=1,\epsilon=1$}
                \end{subfigure}
                \begin{subfigure}[b]{0.24\textwidth}
                    \centering
                    \includegraphics[width=\textwidth]{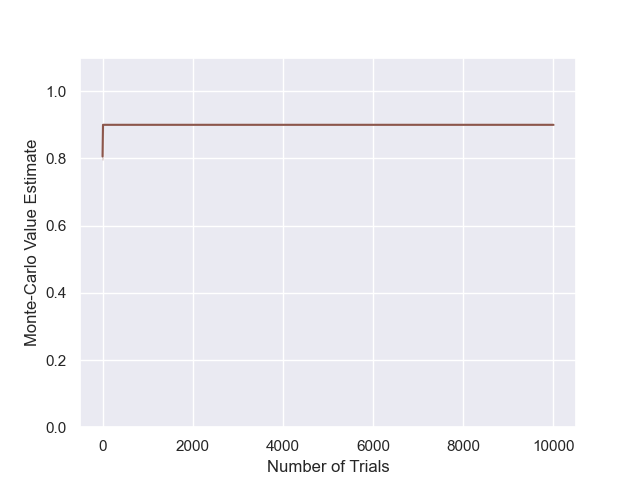}
                    \caption*{$\alpha=1,\epsilon=10$}
                \end{subfigure}
                
                \begin{subfigure}[b]{0.24\textwidth}
                    \centering
                    \includegraphics[width=\textwidth]{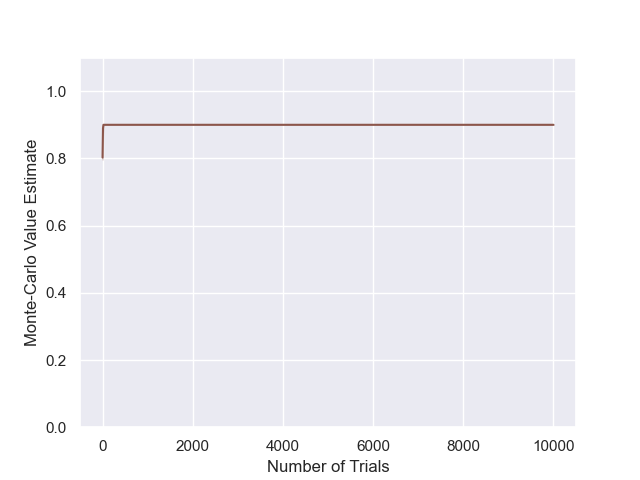}
                    \caption*{$\alpha=0.1,\epsilon=0.01$}
                \end{subfigure}
                \begin{subfigure}[b]{0.24\textwidth}
                    \centering
                    \includegraphics[width=\textwidth]{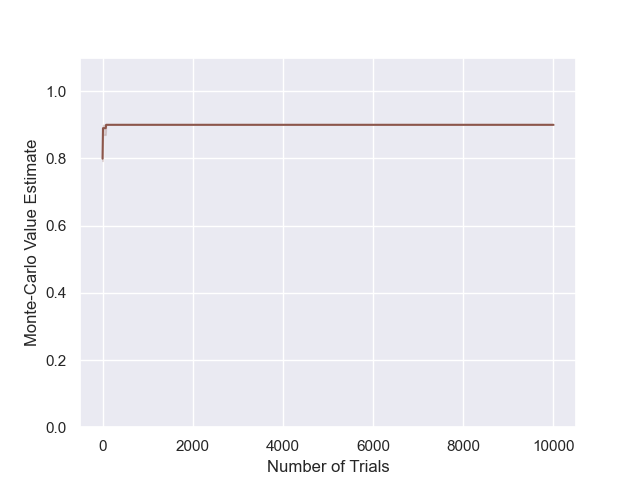}
                    \caption*{$\alpha=0.1,\epsilon=0.1$}
                \end{subfigure}
                \begin{subfigure}[b]{0.24\textwidth}
                    \centering
                    \includegraphics[width=\textwidth]{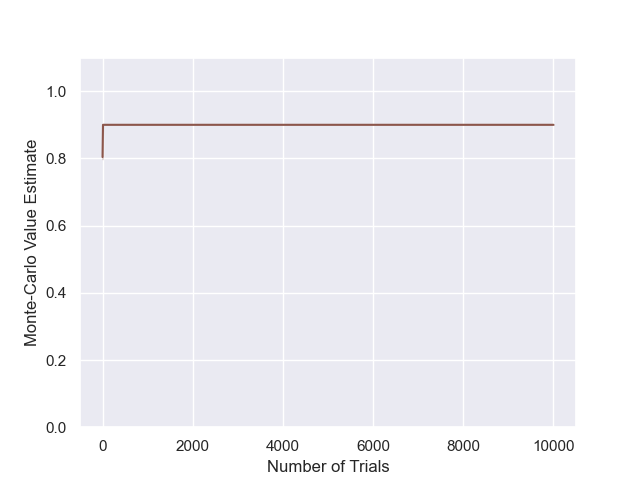}
                    \caption*{$\alpha=0.1,\epsilon=1$}
                \end{subfigure}
                \begin{subfigure}[b]{0.24\textwidth}
                    \centering
                    \includegraphics[width=\textwidth]{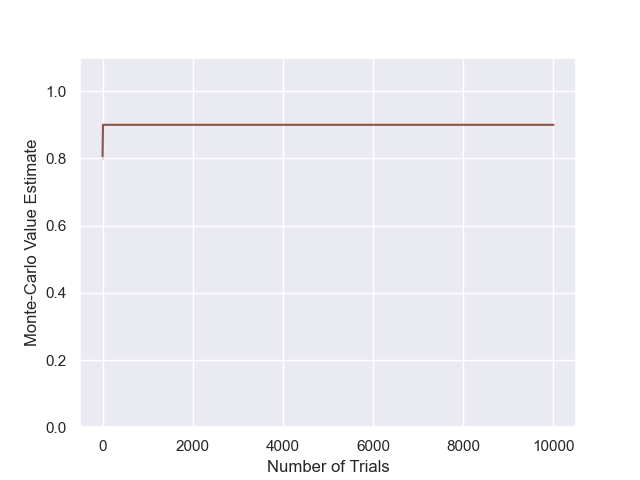}
                    \caption*{$\alpha=0.01,\epsilon=10$}
                \end{subfigure}
                
                \begin{subfigure}[b]{0.24\textwidth}
                    \centering
                    \includegraphics[width=\textwidth]{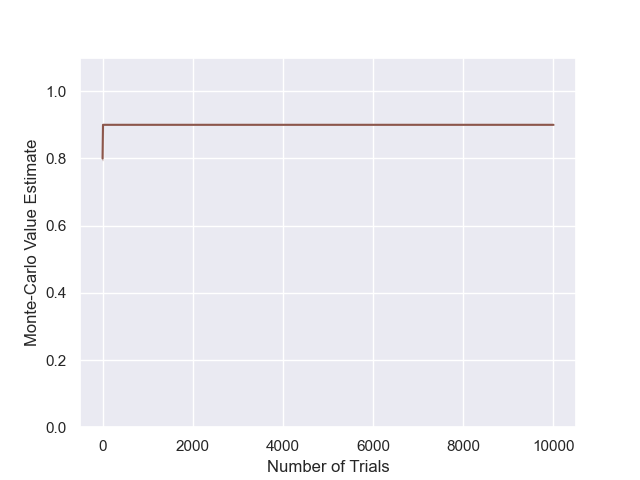}
                    \caption*{$\alpha=0.01,\epsilon=0.01$}
                \end{subfigure}
                \begin{subfigure}[b]{0.24\textwidth}
                    \centering
                    \includegraphics[width=\textwidth]{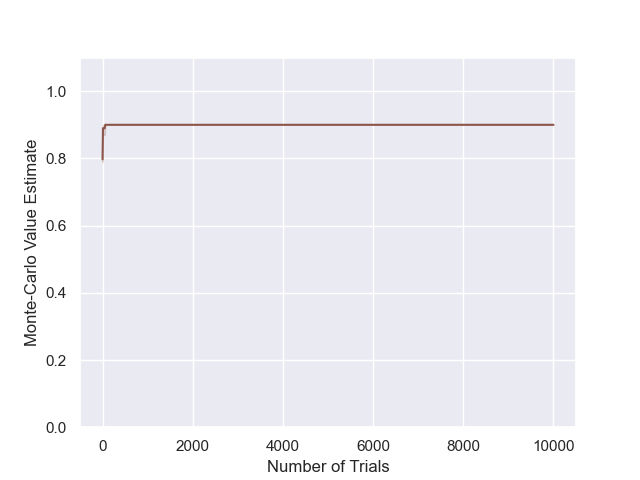}
                    \caption*{$\alpha=0.01,\epsilon=0.1$}
                \end{subfigure}
                \begin{subfigure}[b]{0.24\textwidth}
                    \centering
                    \includegraphics[width=\textwidth]{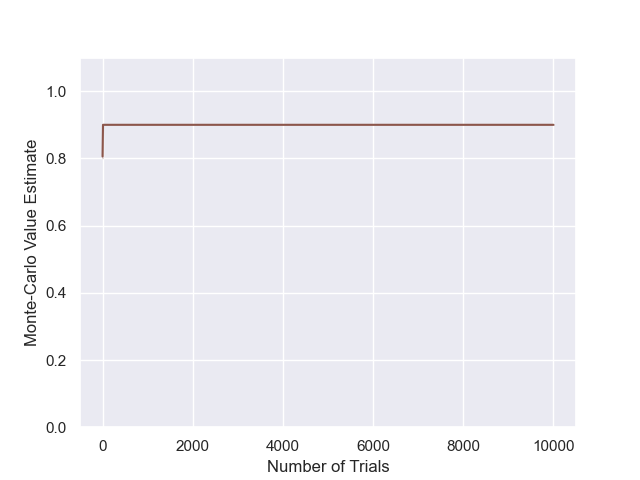}
                    \caption*{$\alpha=0.01,\epsilon=1$}
                \end{subfigure}
                \begin{subfigure}[b]{0.24\textwidth}
                    \centering
                    \includegraphics[width=\textwidth]{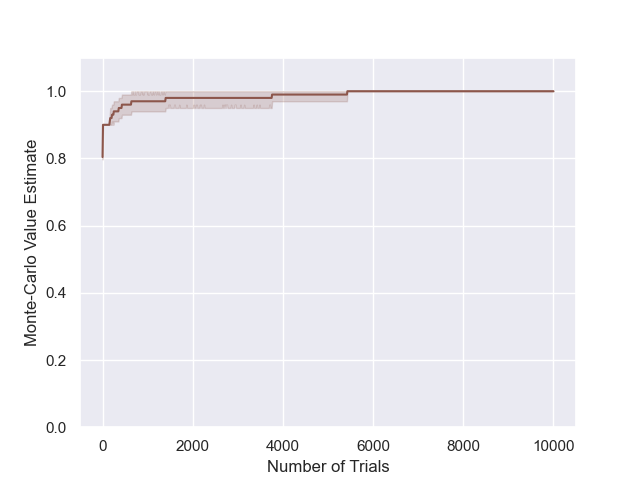}
                    \caption*{$\alpha=0.01,\epsilon=10$}
                \end{subfigure}
                
                \caption{Results for RENTS on the 10-chain ($D=10$, $R_f=1.0$), for varying temperatures and exploration parameters.}
                \label{fig:rents_10chain_hps}
            \end{figure}

            \begin{figure}
                \centering
                
                \begin{subfigure}[b]{0.24\textwidth}
                    \centering
                    \includegraphics[width=\textwidth]{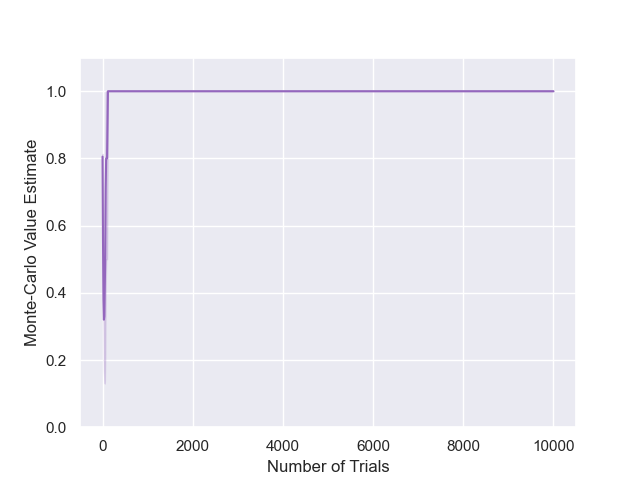}
                    \caption*{$\alpha=10,\epsilon=0.01$}
                \end{subfigure}
                \begin{subfigure}[b]{0.24\textwidth}
                    \centering
                    \includegraphics[width=\textwidth]{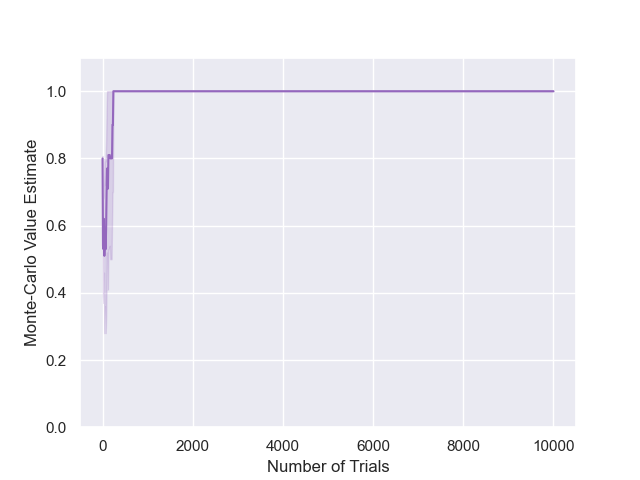}
                    \caption*{$\alpha=10,\epsilon=0.1$}
                \end{subfigure}
                \begin{subfigure}[b]{0.24\textwidth}
                    \centering
                    \includegraphics[width=\textwidth]{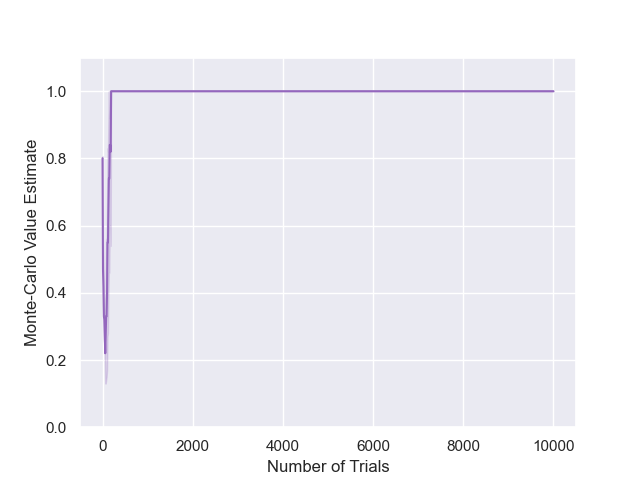}
                    \caption*{$\alpha=10,\epsilon=1$}
                \end{subfigure}
                \begin{subfigure}[b]{0.24\textwidth}
                    \centering
                    \includegraphics[width=\textwidth]{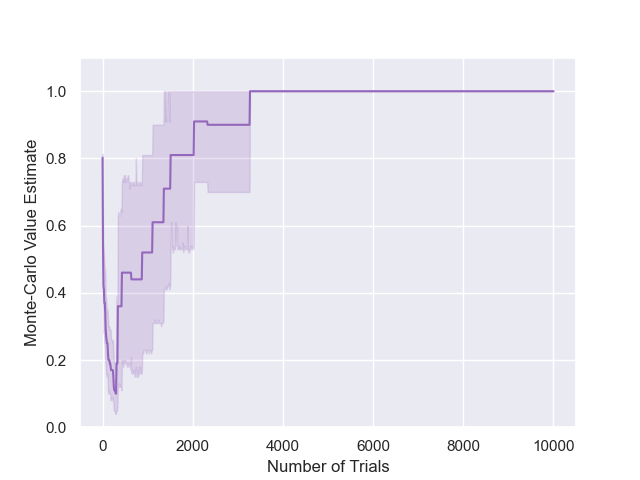}
                    \caption*{$\alpha=10,\epsilon=10$}
                \end{subfigure}
                
                \begin{subfigure}[b]{0.24\textwidth}
                    \centering
                    \includegraphics[width=\textwidth]{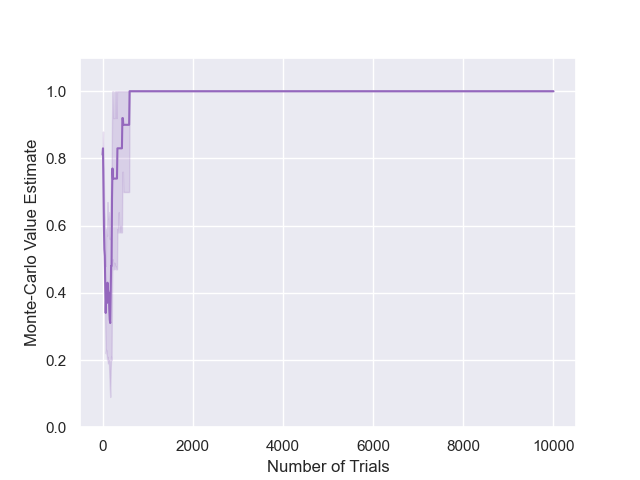}
                    \caption*{$\alpha=1,\epsilon=0.01$}
                \end{subfigure}
                \begin{subfigure}[b]{0.24\textwidth}
                    \centering
                    \includegraphics[width=\textwidth]{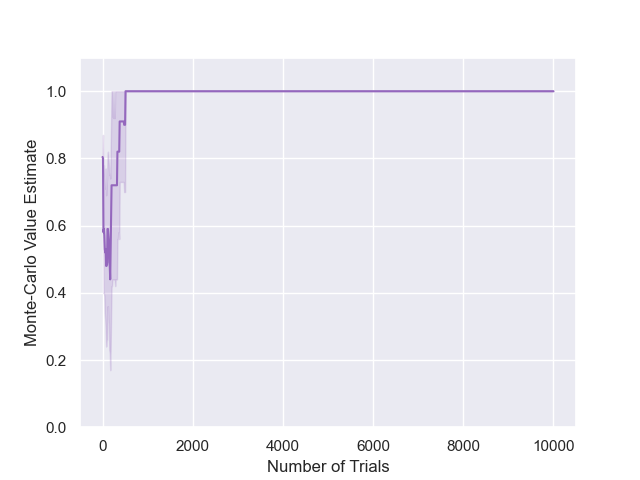}
                    \caption*{$\alpha=1,\epsilon=0.1$}
                \end{subfigure}
                \begin{subfigure}[b]{0.24\textwidth}
                    \centering
                    \includegraphics[width=\textwidth]{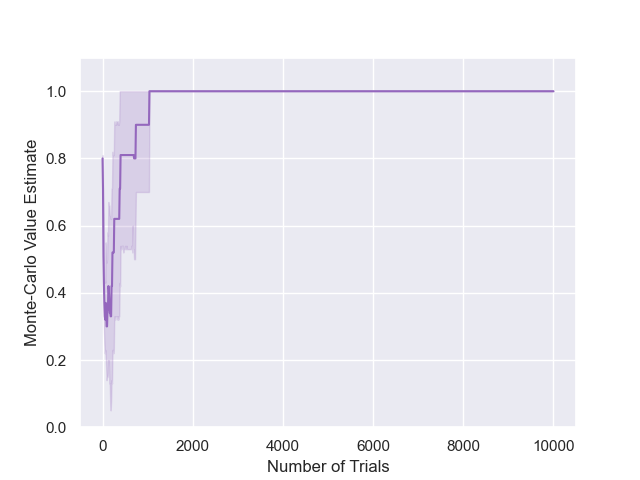}
                    \caption*{$\alpha=1,\epsilon=1$}
                \end{subfigure}
                \begin{subfigure}[b]{0.24\textwidth}
                    \centering
                    \includegraphics[width=\textwidth]{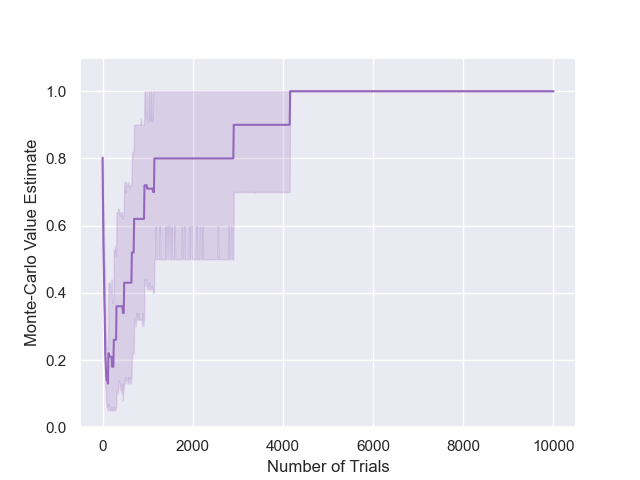}
                    \caption*{$\alpha=1,\epsilon=10$}
                \end{subfigure}
                
                \begin{subfigure}[b]{0.24\textwidth}
                    \centering
                    \includegraphics[width=\textwidth]{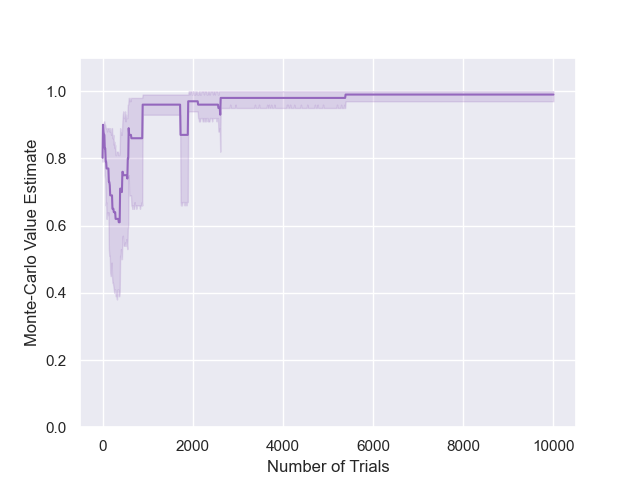}
                    \caption*{$\alpha=0.7,\epsilon=0.01$}
                \end{subfigure}
                \begin{subfigure}[b]{0.24\textwidth}
                    \centering
                    \includegraphics[width=\textwidth]{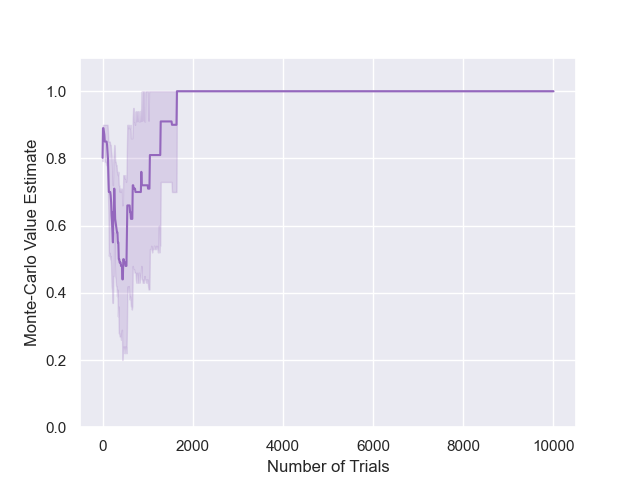}
                    \caption*{$\alpha=0.7,\epsilon=0.1$}
                \end{subfigure}
                \begin{subfigure}[b]{0.24\textwidth}
                    \centering
                    \includegraphics[width=\textwidth]{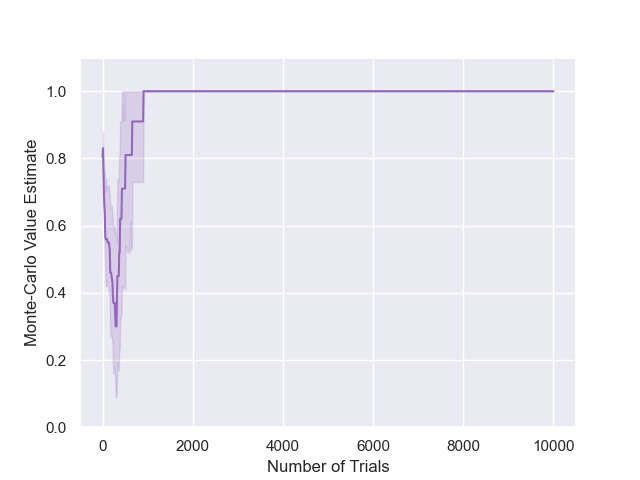}
                    \caption*{$\alpha=0.7,\epsilon=1$}
                \end{subfigure}
                \begin{subfigure}[b]{0.24\textwidth}
                    \centering
                    \includegraphics[width=\textwidth]{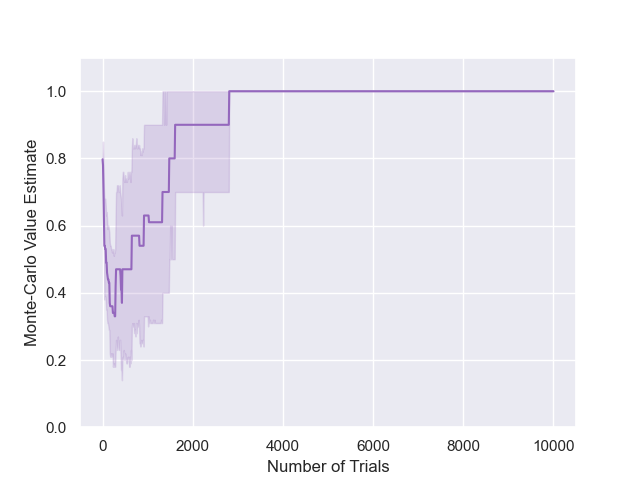}
                    \caption*{$\alpha=0.7,\epsilon=10$}
                \end{subfigure}
                
                \begin{subfigure}[b]{0.24\textwidth}
                    \centering
                    \includegraphics[width=\textwidth]{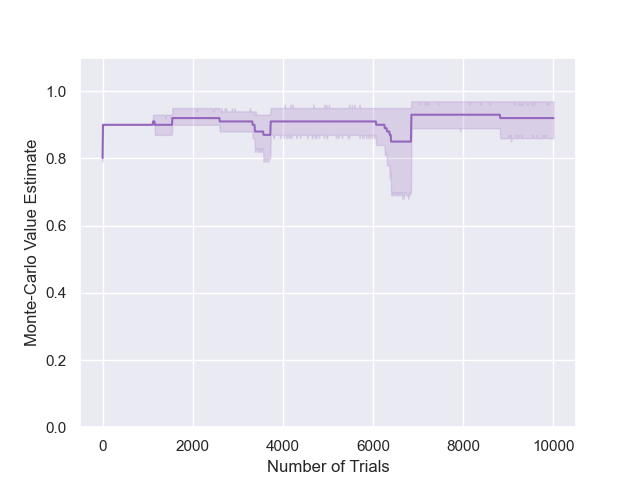}
                    \caption*{$\alpha=0.5,\epsilon=0.01$}
                \end{subfigure}
                \begin{subfigure}[b]{0.24\textwidth}
                    \centering
                    \includegraphics[width=\textwidth]{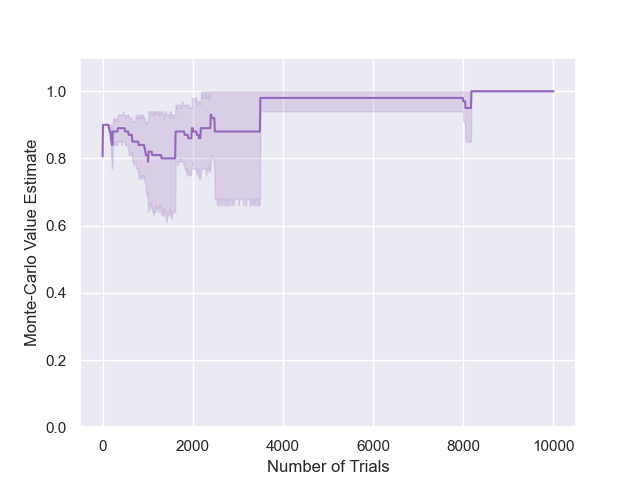}
                    \caption*{$\alpha=0.5,\epsilon=0.1$}
                \end{subfigure}
                \begin{subfigure}[b]{0.24\textwidth}
                    \centering
                    \includegraphics[width=\textwidth]{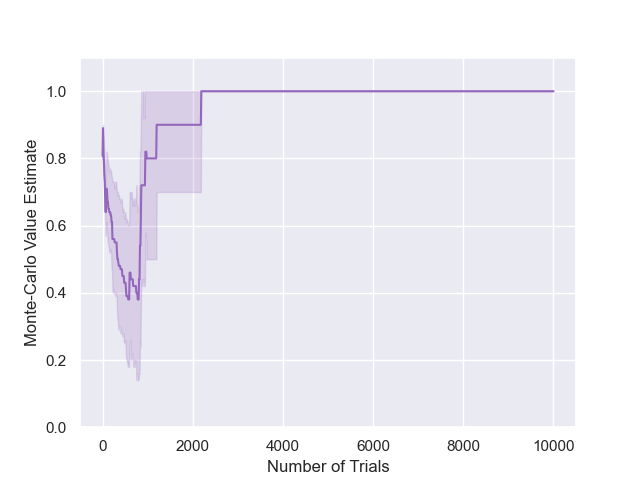}
                    \caption*{$\alpha=0.5,\epsilon=1$}
                \end{subfigure}
                \begin{subfigure}[b]{0.24\textwidth}
                    \centering
                    \includegraphics[width=\textwidth]{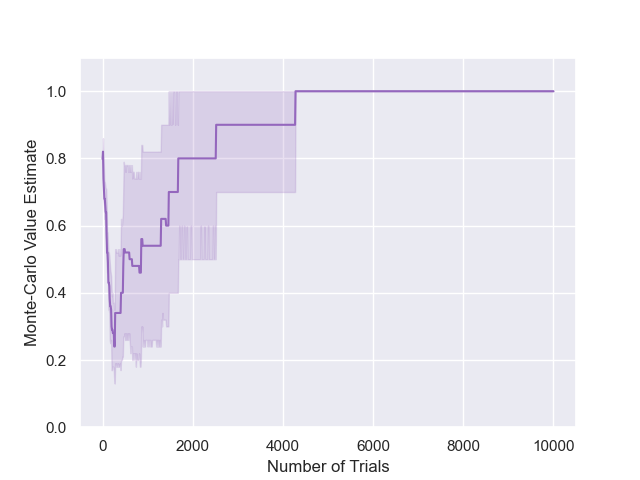}
                    \caption*{$\alpha=0.5,\epsilon=10$}
                \end{subfigure}
                
                \begin{subfigure}[b]{0.24\textwidth}
                    \centering
                    \includegraphics[width=\textwidth]{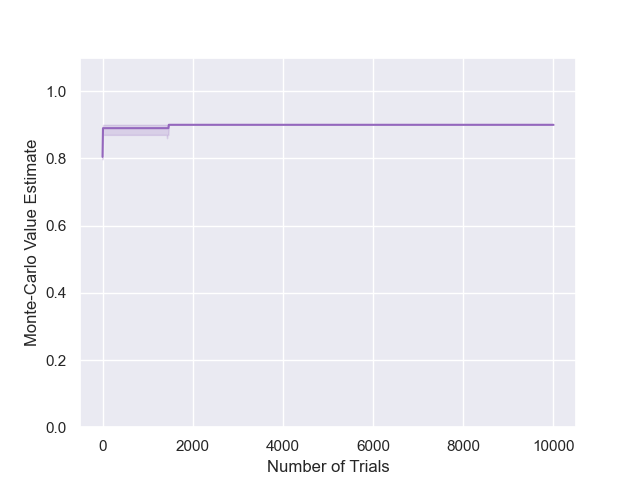}
                    \caption*{$\alpha=0.3,\epsilon=0.01$}
                \end{subfigure}
                \begin{subfigure}[b]{0.24\textwidth}
                    \centering
                    \includegraphics[width=\textwidth]{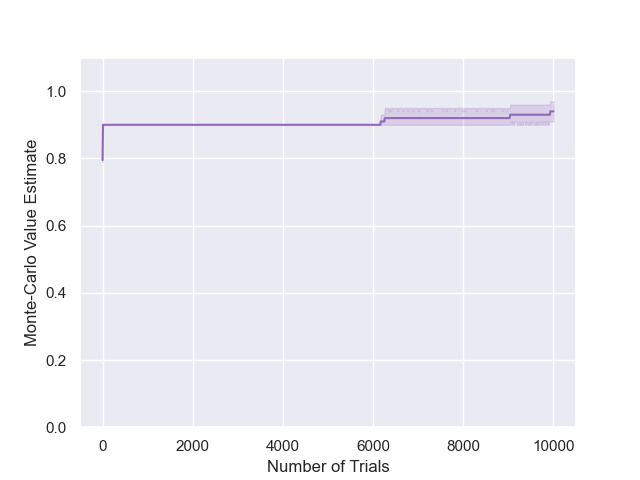}
                    \caption*{$\alpha=0.3,\epsilon=0.1$}
                \end{subfigure}
                \begin{subfigure}[b]{0.24\textwidth}
                    \centering
                    \includegraphics[width=\textwidth]{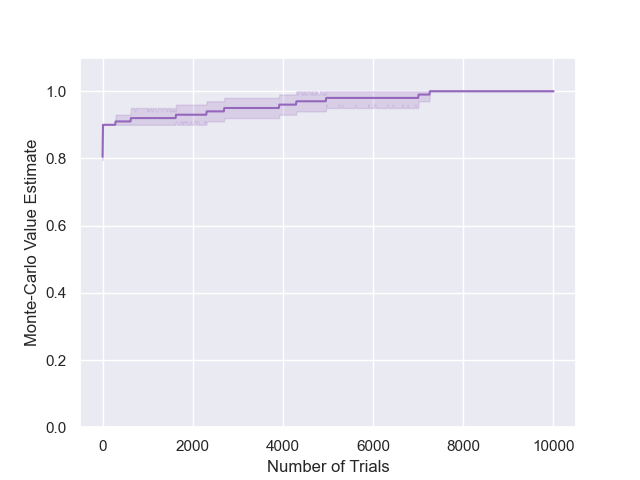}
                    \caption*{$\alpha=0.3,\epsilon=1$}
                \end{subfigure}
                \begin{subfigure}[b]{0.24\textwidth}
                    \centering
                    \includegraphics[width=\textwidth]{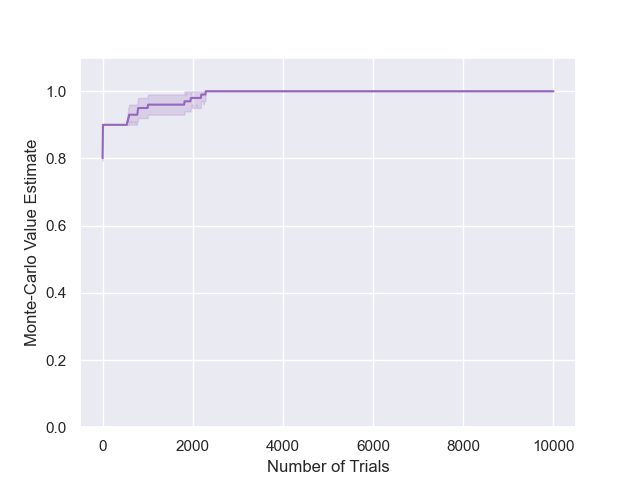}
                    \caption*{$\alpha=0.3,\epsilon=10$}
                \end{subfigure}
                
                \begin{subfigure}[b]{0.24\textwidth}
                    \centering
                    \includegraphics[width=\textwidth]{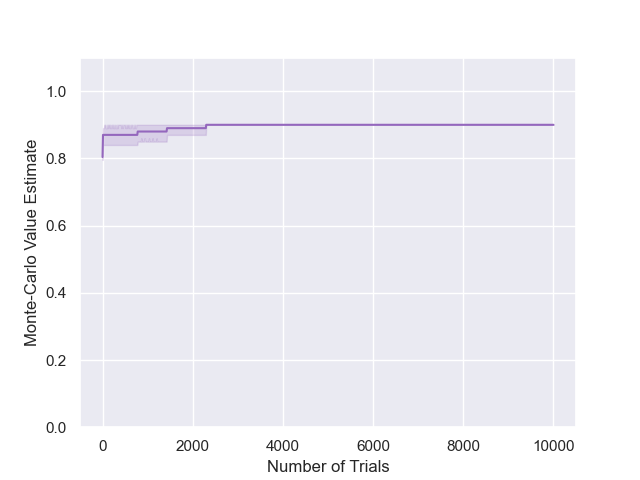}
                    \caption*{$\alpha=0.1,\epsilon=0.01$}
                \end{subfigure}
                \begin{subfigure}[b]{0.24\textwidth}
                    \centering
                    \includegraphics[width=\textwidth]{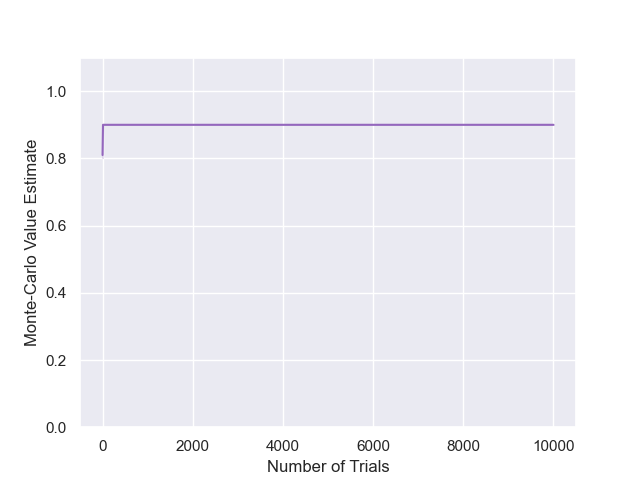}
                    \caption*{$\alpha=0.1,\epsilon=0.1$}
                \end{subfigure}
                \begin{subfigure}[b]{0.24\textwidth}
                    \centering
                    \includegraphics[width=\textwidth]{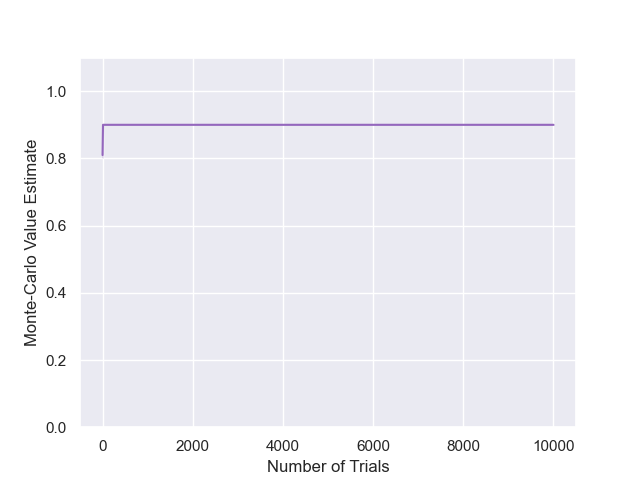}
                    \caption*{$\alpha=0.1,\epsilon=1$}
                \end{subfigure}
                \begin{subfigure}[b]{0.24\textwidth}
                    \centering
                    \includegraphics[width=\textwidth]{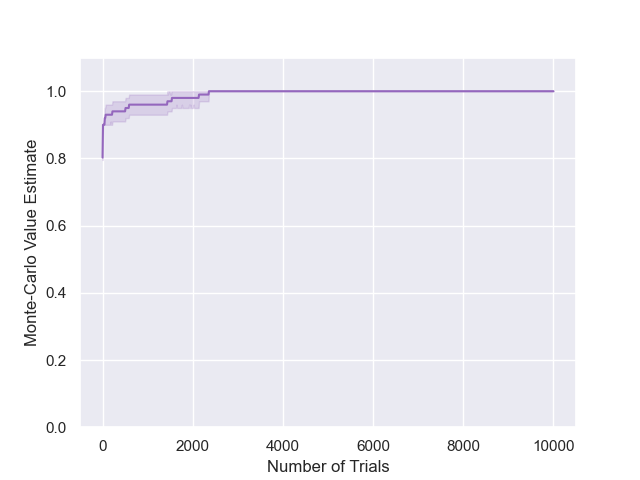}
                    \caption*{$\alpha=0.1,\epsilon=10$}
                \end{subfigure}
                
                \begin{subfigure}[b]{0.24\textwidth}
                    \centering
                    \includegraphics[width=\textwidth]{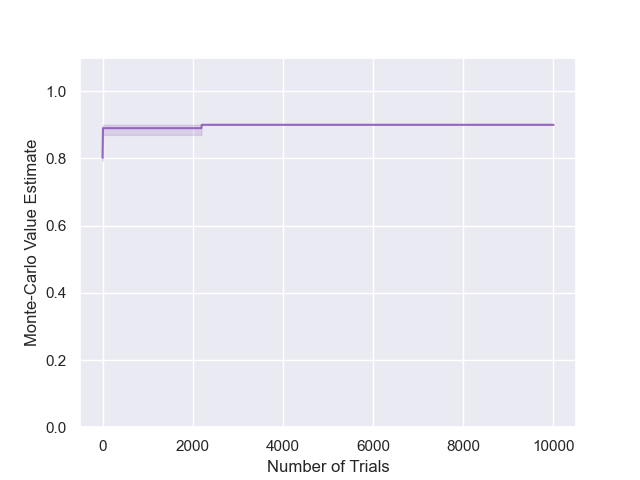}
                    \caption*{$\alpha=0.01,\epsilon=0.01$}
                \end{subfigure}
                \begin{subfigure}[b]{0.24\textwidth}
                    \centering
                    \includegraphics[width=\textwidth]{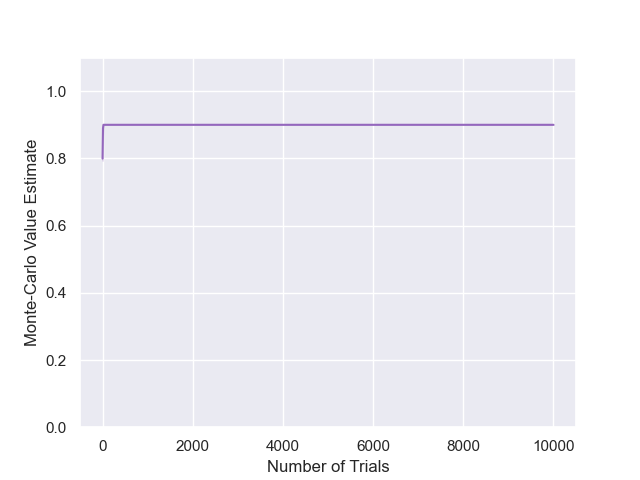}
                    \caption*{$\alpha=0.01,\epsilon=0.1$}
                \end{subfigure}
                \begin{subfigure}[b]{0.24\textwidth}
                    \centering
                    \includegraphics[width=\textwidth]{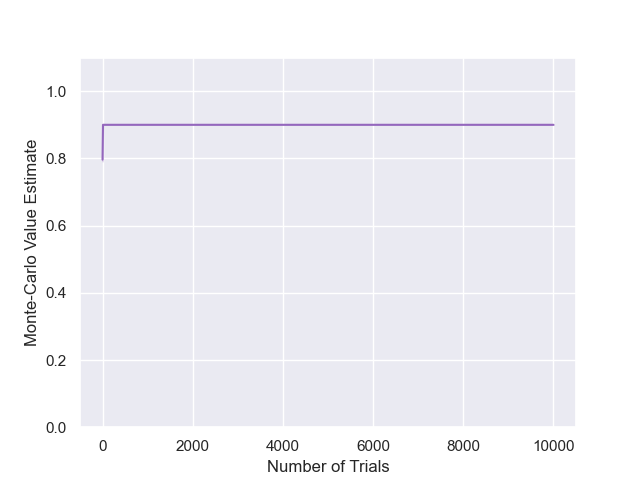}
                    \caption*{$\alpha=0.01,\epsilon=1$}
                \end{subfigure}
                \begin{subfigure}[b]{0.24\textwidth}
                    \centering
                    \includegraphics[width=\textwidth]{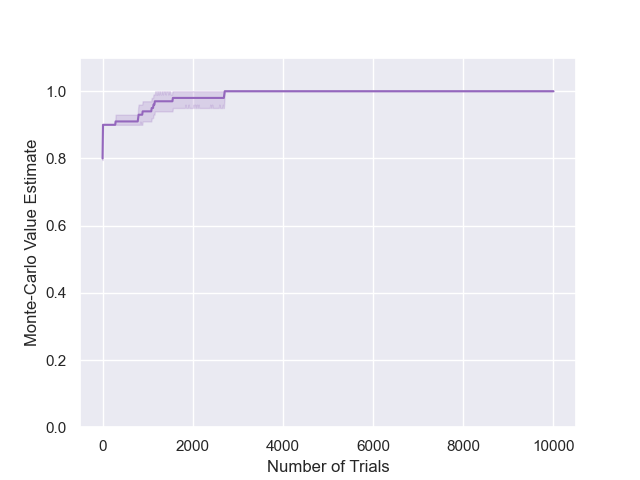}
                    \caption*{$\alpha=0.01,\epsilon=10$}
                \end{subfigure}
                
                \caption{Results for TENTS on the 10-chain ($D=10$, $R_f=1.0$), for varying temperatures and exploration parameters.}
                \label{fig:tents_10chain_hps}
            \end{figure}

            \begin{figure}
                \centering
                
                \begin{subfigure}[b]{0.24\textwidth}
                    \centering
                    \includegraphics[width=\textwidth]{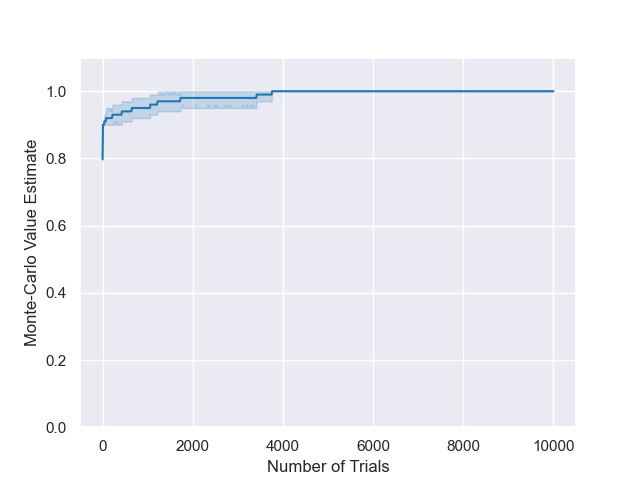}
                    \caption*{$\alpha=1,\epsilon=0.01$}
                \end{subfigure}
                \begin{subfigure}[b]{0.24\textwidth}
                    \centering
                    \includegraphics[width=\textwidth]{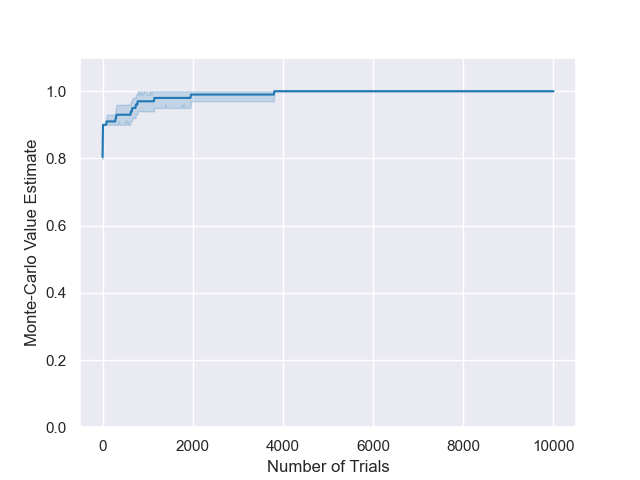}
                    \caption*{$\alpha=1,\epsilon=0.1$}
                \end{subfigure}
                \begin{subfigure}[b]{0.24\textwidth}
                    \centering
                    \includegraphics[width=\textwidth]{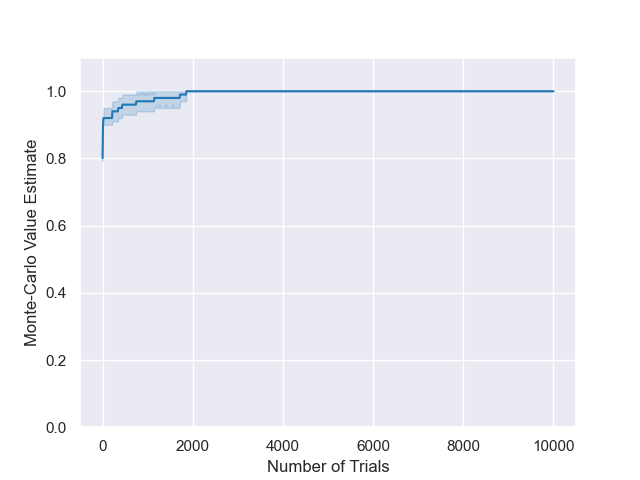}
                    \caption*{$\alpha=1,\epsilon=1$}
                \end{subfigure}
                \begin{subfigure}[b]{0.24\textwidth}
                    \centering
                    \includegraphics[width=\textwidth]{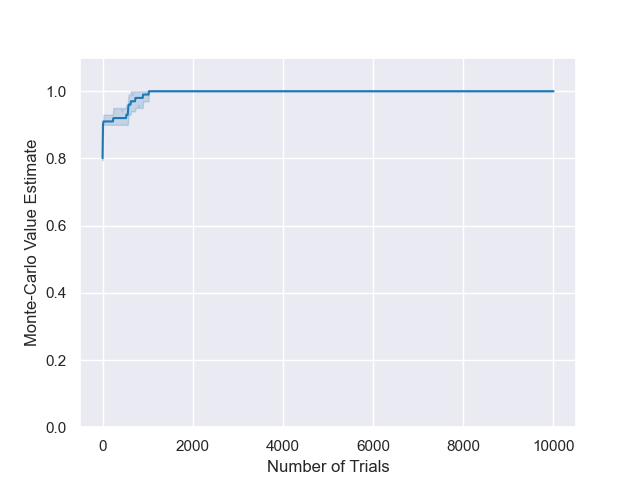}
                    \caption*{$\alpha=1,\epsilon=10$}
                \end{subfigure}
                
                \begin{subfigure}[b]{0.24\textwidth}
                    \centering
                    \includegraphics[width=\textwidth]{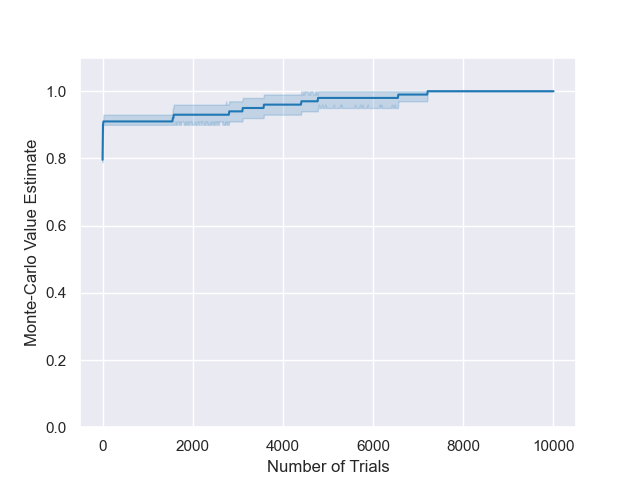}
                    \caption*{$\alpha=0.5,\epsilon=0.01$}
                \end{subfigure}
                \begin{subfigure}[b]{0.24\textwidth}
                    \centering
                    \includegraphics[width=\textwidth]{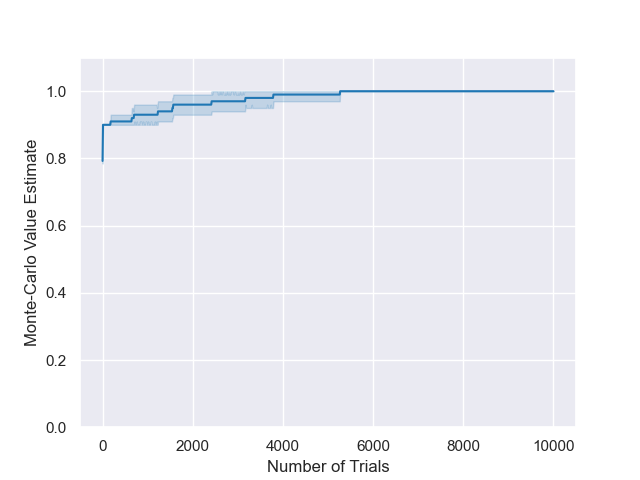}
                    \caption*{$\alpha=0.5,\epsilon=0.1$}
                \end{subfigure}
                \begin{subfigure}[b]{0.24\textwidth}
                    \centering
                    \includegraphics[width=\textwidth]{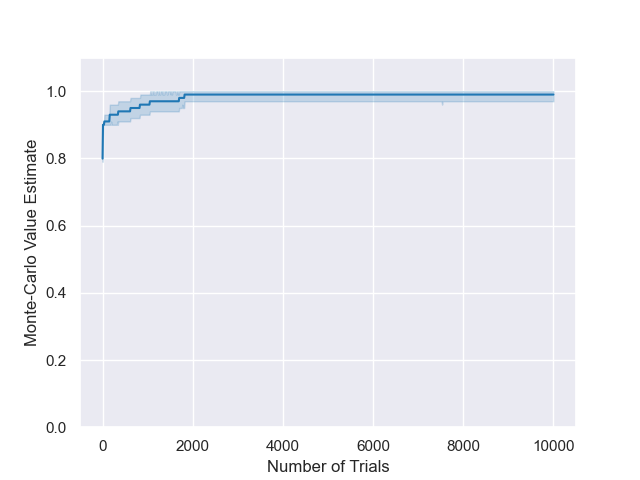}
                    \caption*{$\alpha=0.5,\epsilon=1$}
                \end{subfigure}
                \begin{subfigure}[b]{0.24\textwidth}
                    \centering
                    \includegraphics[width=\textwidth]{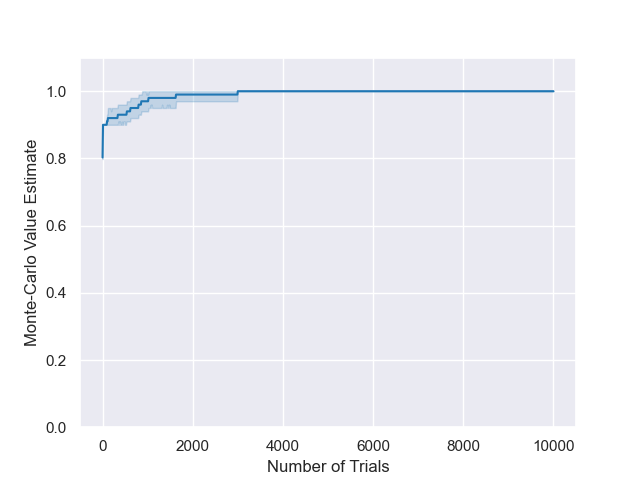}
                    \caption*{$\alpha=0.5,\epsilon=10$}
                \end{subfigure}
                
                \begin{subfigure}[b]{0.24\textwidth}
                    \centering
                    \includegraphics[width=\textwidth]{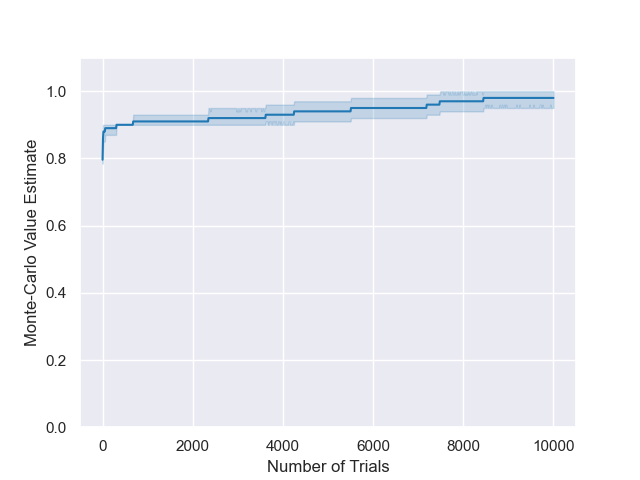}
                    \caption*{$\alpha=0.2,\epsilon=0.01$}
                \end{subfigure}
                \begin{subfigure}[b]{0.24\textwidth}
                    \centering
                    \includegraphics[width=\textwidth]{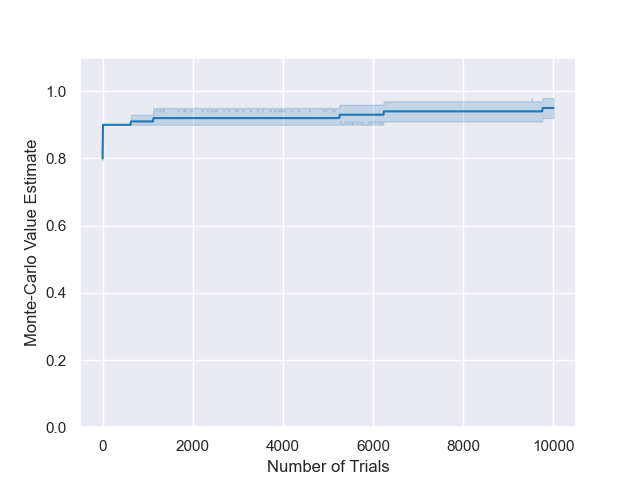}
                    \caption*{$\alpha=0.2,\epsilon=0.1$}
                \end{subfigure}
                \begin{subfigure}[b]{0.24\textwidth}
                    \centering
                    \includegraphics[width=\textwidth]{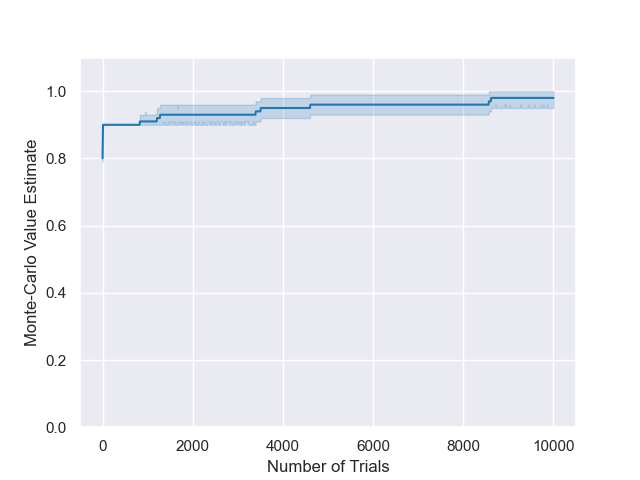}
                    \caption*{$\alpha=0.2,\epsilon=1$}
                \end{subfigure}
                \begin{subfigure}[b]{0.24\textwidth}
                    \centering
                    \includegraphics[width=\textwidth]{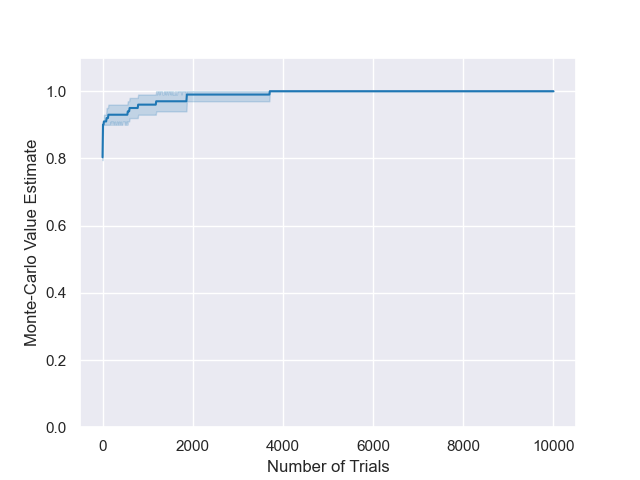}
                    \caption*{$\alpha=0.2,\epsilon=10$}
                \end{subfigure}
                
                \begin{subfigure}[b]{0.24\textwidth}
                    \centering
                    \includegraphics[width=\textwidth]{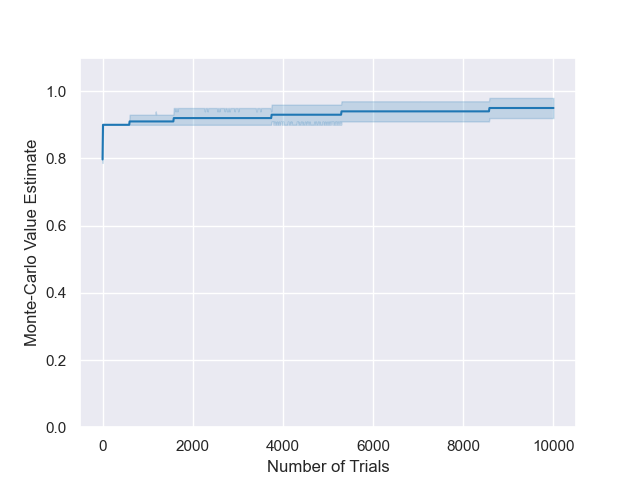}
                    \caption*{$\alpha=0.15,\epsilon=0.01$}
                \end{subfigure}
                \begin{subfigure}[b]{0.24\textwidth}
                    \centering
                    \includegraphics[width=\textwidth]{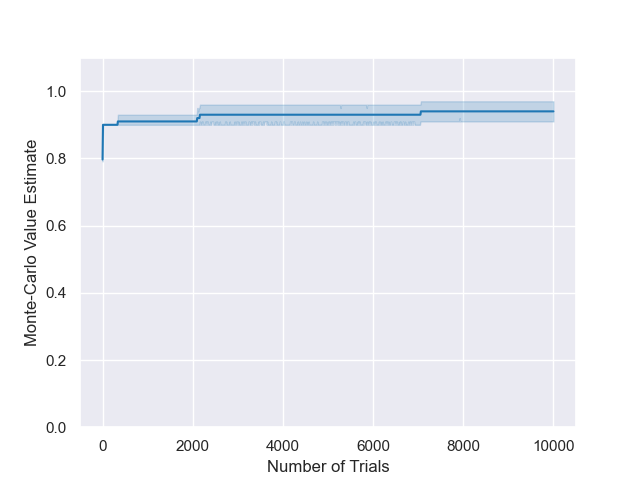}
                    \caption*{$\alpha=0.15,\epsilon=0.1$}
                \end{subfigure}
                \begin{subfigure}[b]{0.24\textwidth}
                    \centering
                    \includegraphics[width=\textwidth]{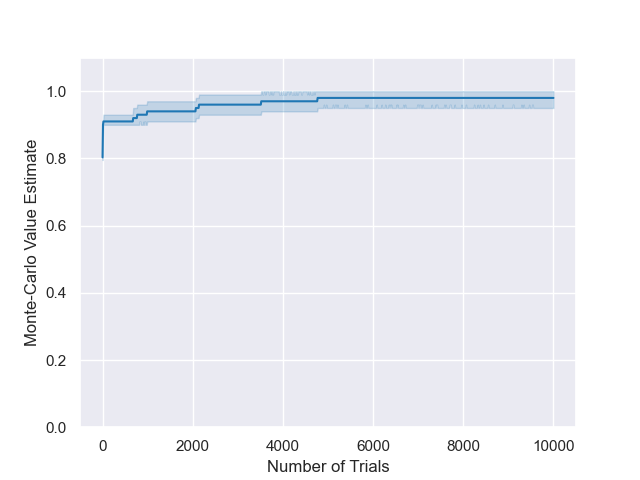}
                    \caption*{$\alpha=0.15,\epsilon=1$}
                \end{subfigure}
                \begin{subfigure}[b]{0.24\textwidth}
                    \centering
                    \includegraphics[width=\textwidth]{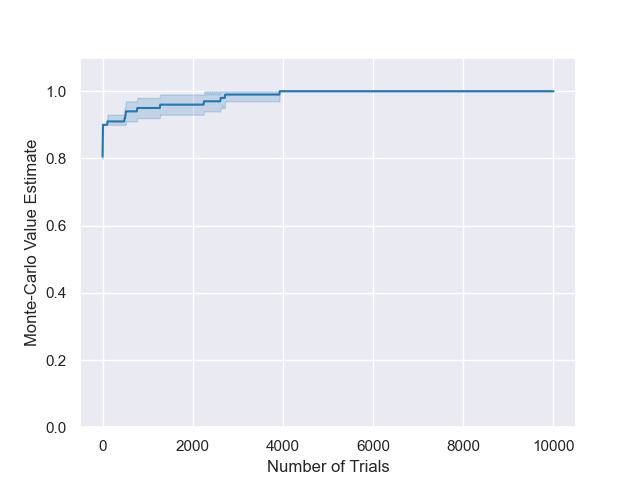}
                    \caption*{$\alpha=0.15,\epsilon=10$}
                \end{subfigure}
                
                \begin{subfigure}[b]{0.24\textwidth}
                    \centering
                    \includegraphics[width=\textwidth]{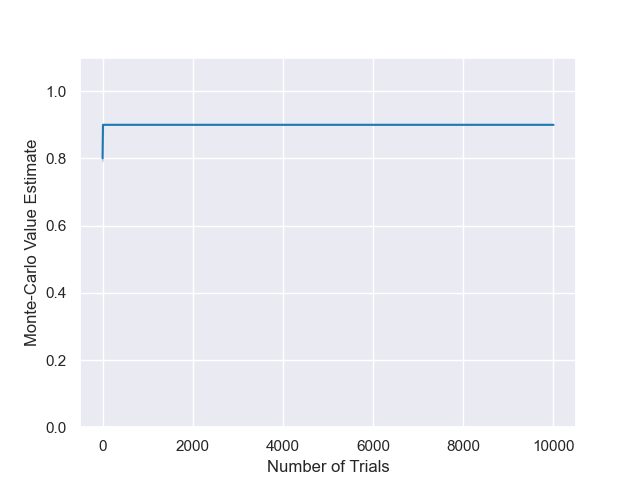}
                    \caption*{$\alpha=0.1,\epsilon=0.01$}
                \end{subfigure}
                \begin{subfigure}[b]{0.24\textwidth}
                    \centering
                    \includegraphics[width=\textwidth]{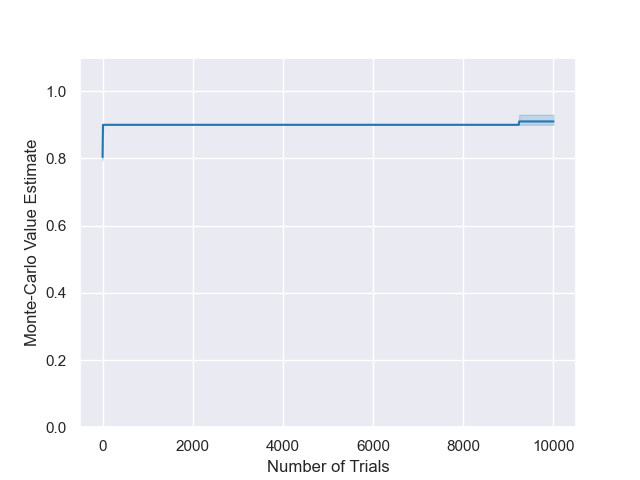}
                    \caption*{$\alpha=0.1,\epsilon=0.1$}
                \end{subfigure}
                \begin{subfigure}[b]{0.24\textwidth}
                    \centering
                    \includegraphics[width=\textwidth]{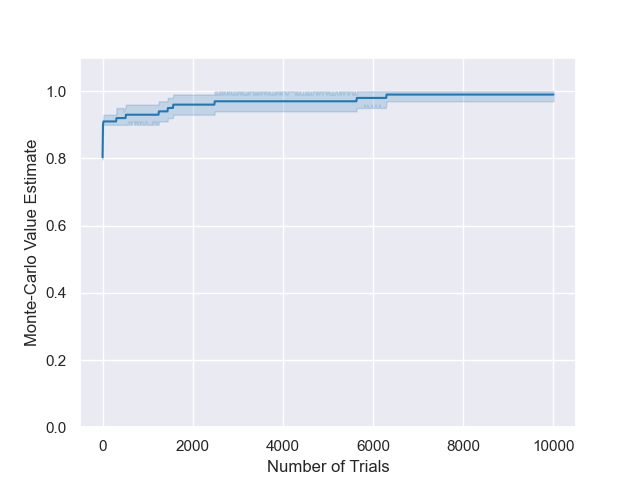}
                    \caption*{$\alpha=0.1,\epsilon=1$}
                \end{subfigure}
                \begin{subfigure}[b]{0.24\textwidth}
                    \centering
                    \includegraphics[width=\textwidth]{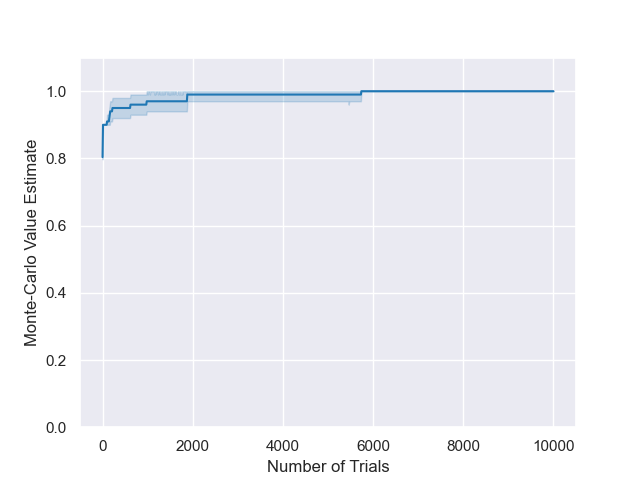}
                    \caption*{$\alpha=0.1,\epsilon=10$}
                \end{subfigure}
                
                \begin{subfigure}[b]{0.24\textwidth}
                    \centering
                    \includegraphics[width=\textwidth]{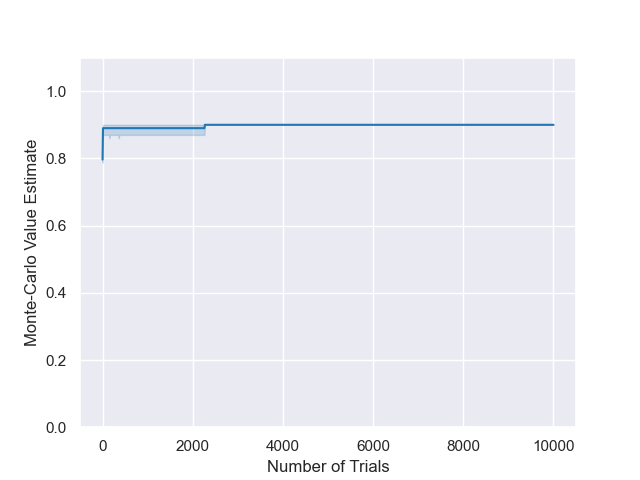}
                    \caption*{$\alpha=0.05,\epsilon=0.01$}
                \end{subfigure}
                \begin{subfigure}[b]{0.24\textwidth}
                    \centering
                    \includegraphics[width=\textwidth]{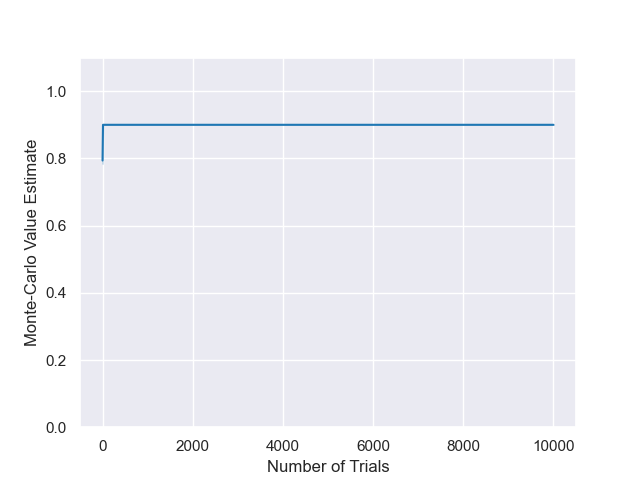}
                    \caption*{$\alpha=0.05,\epsilon=0.1$}
                \end{subfigure}
                \begin{subfigure}[b]{0.24\textwidth}
                    \centering
                    \includegraphics[width=\textwidth]{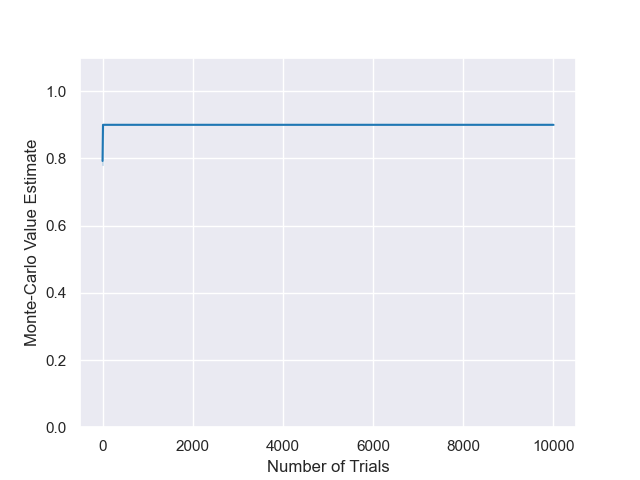}
                    \caption*{$\alpha=0.05,\epsilon=1$}
                \end{subfigure}
                \begin{subfigure}[b]{0.24\textwidth}
                    \centering
                    \includegraphics[width=\textwidth]{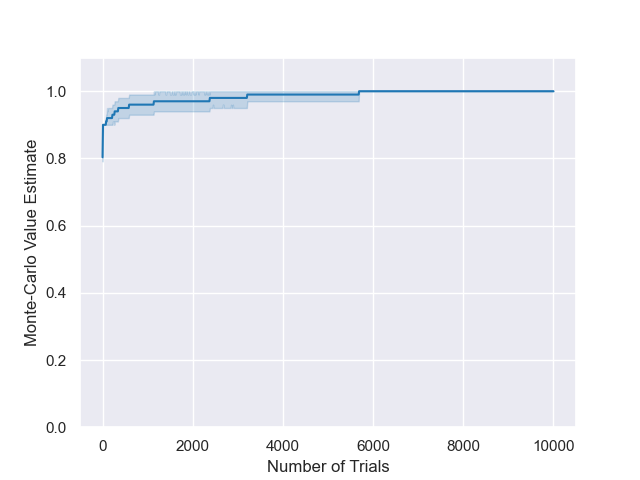}
                    \caption*{$\alpha=0.05,\epsilon=10$}
                \end{subfigure}
                
                \begin{subfigure}[b]{0.24\textwidth}
                    \centering
                    \includegraphics[width=\textwidth]{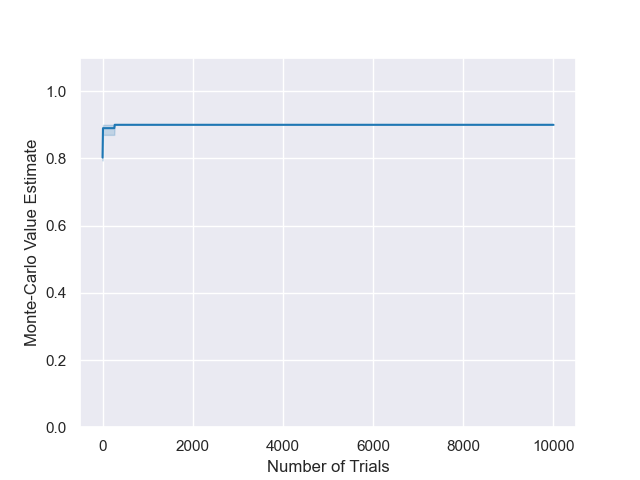}
                    \caption*{$\alpha=0.01,\epsilon=0.01$}
                \end{subfigure}
                \begin{subfigure}[b]{0.24\textwidth}
                    \centering
                    \includegraphics[width=\textwidth]{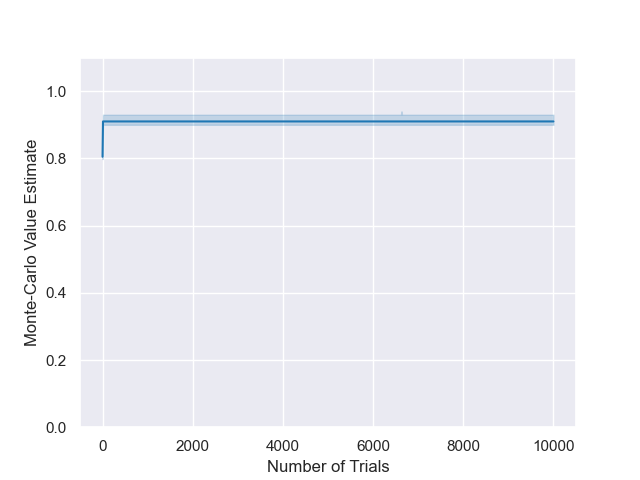}
                    \caption*{$\alpha=0.01,\epsilon=0.1$}
                \end{subfigure}
                \begin{subfigure}[b]{0.24\textwidth}
                    \centering
                    \includegraphics[width=\textwidth]{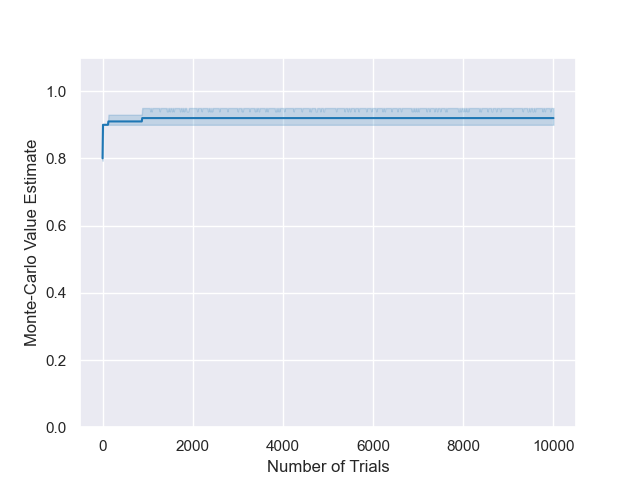}
                    \caption*{$\alpha=0.01,\epsilon=1$}
                \end{subfigure}
                \begin{subfigure}[b]{0.24\textwidth}
                    \centering
                    \includegraphics[width=\textwidth]{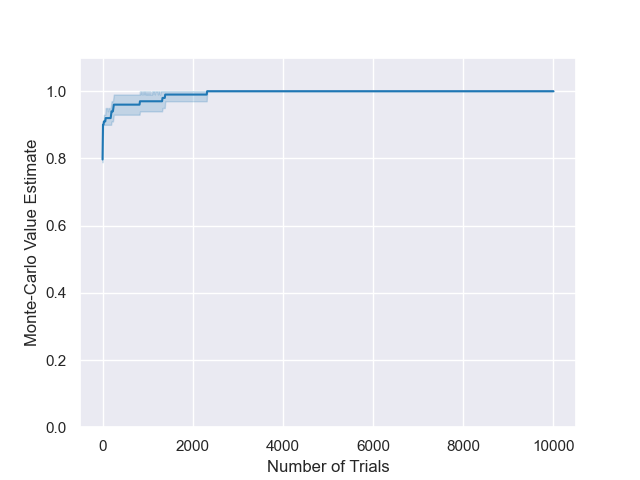}
                    \caption*{$\alpha=0.01,\epsilon=10$}
                \end{subfigure}
                
                \caption{Results for BTS on the 10-chain ($D=10$, $R_f=1.0$), for varying temperatures and exploration parameters.}
                \label{fig:bts_10chain_hps}
            \end{figure}

            \begin{figure}
                \centering
                
                \begin{subfigure}[b]{0.24\textwidth}
                    \centering
                    \includegraphics[width=\textwidth]{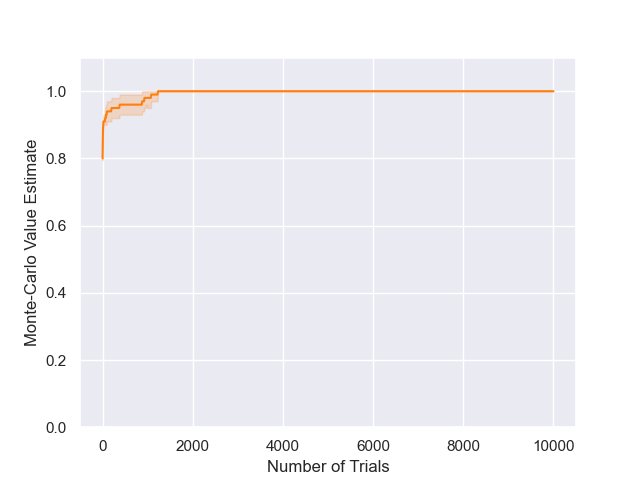}
                    \caption*{$\alpha=1,\epsilon=0.01$}
                \end{subfigure}
                \begin{subfigure}[b]{0.24\textwidth}
                    \centering
                    \includegraphics[width=\textwidth]{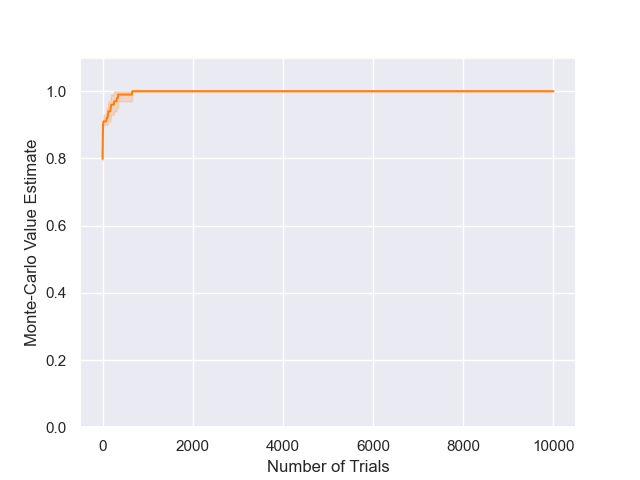}
                    \caption*{$\alpha=1,\epsilon=0.1$}
                \end{subfigure}
                \begin{subfigure}[b]{0.24\textwidth}
                    \centering
                    \includegraphics[width=\textwidth]{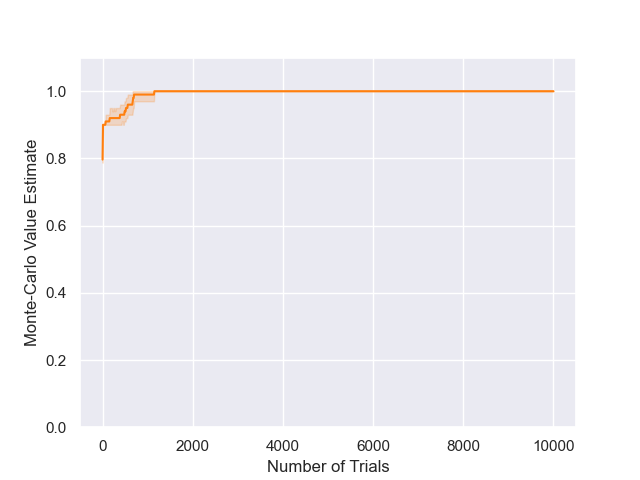}
                    \caption*{$\alpha=1,\epsilon=1$}
                \end{subfigure}
                \begin{subfigure}[b]{0.24\textwidth}
                    \centering
                    \includegraphics[width=\textwidth]{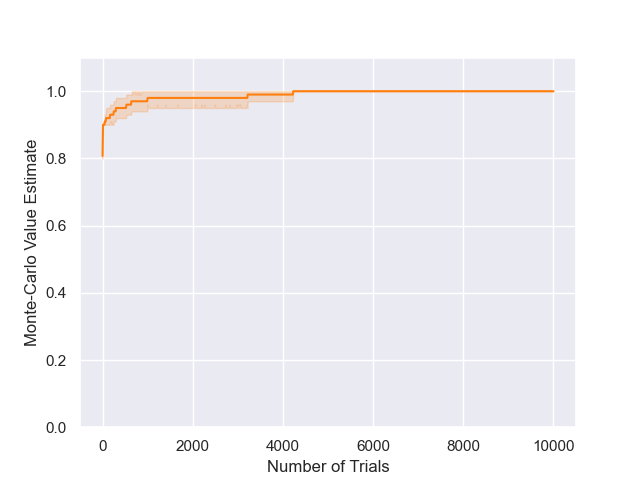}
                    \caption*{$\alpha=1,\epsilon=10$}
                \end{subfigure}
                
                \begin{subfigure}[b]{0.24\textwidth}
                    \centering
                    \includegraphics[width=\textwidth]{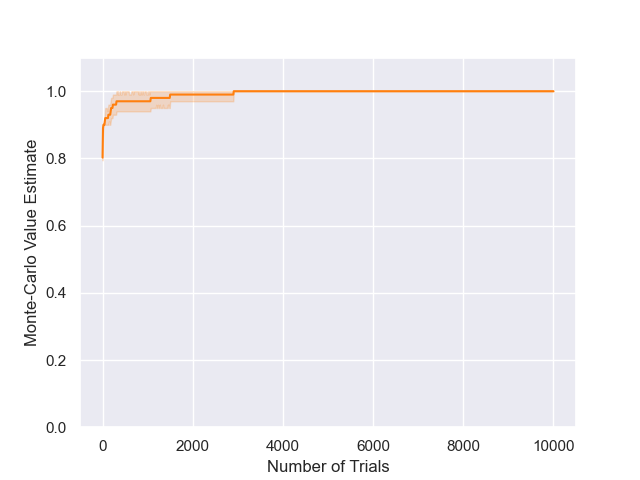}
                    \caption*{$\alpha=0.5,\epsilon=0.01$}
                \end{subfigure}
                \begin{subfigure}[b]{0.24\textwidth}
                    \centering
                    \includegraphics[width=\textwidth]{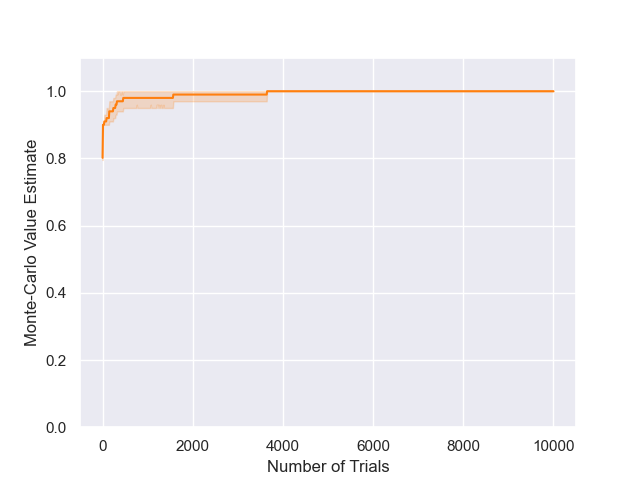}
                    \caption*{$\alpha=0.5,\epsilon=0.1$}
                \end{subfigure}
                \begin{subfigure}[b]{0.24\textwidth}
                    \centering
                    \includegraphics[width=\textwidth]{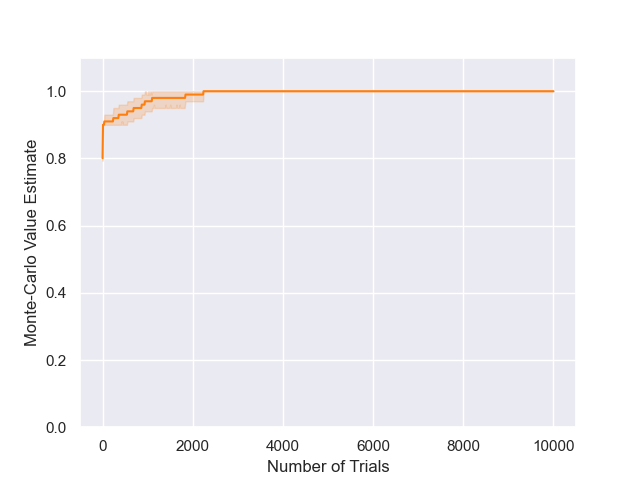}
                    \caption*{$\alpha=0.5,\epsilon=1$}
                \end{subfigure}
                \begin{subfigure}[b]{0.24\textwidth}
                    \centering
                    \includegraphics[width=\textwidth]{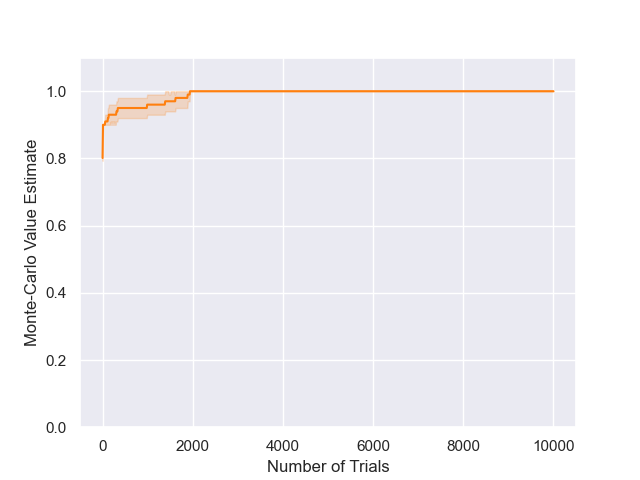}
                    \caption*{$\alpha=0.5,\epsilon=10$}
                \end{subfigure}
                
                \begin{subfigure}[b]{0.24\textwidth}
                    \centering
                    \includegraphics[width=\textwidth]{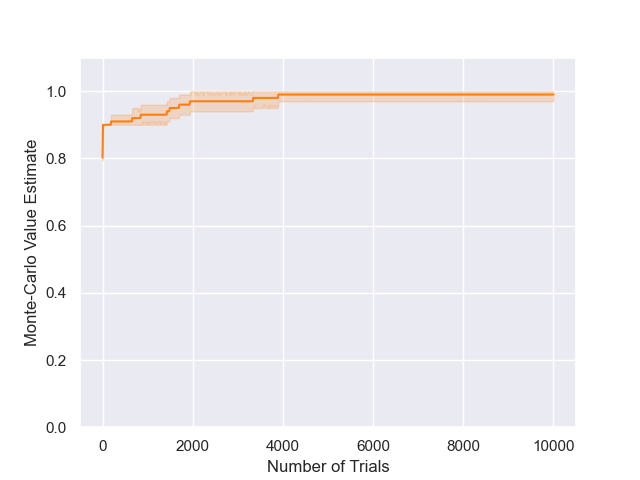}
                    \caption*{$\alpha=0.2,\epsilon=0.01$}
                \end{subfigure}
                \begin{subfigure}[b]{0.24\textwidth}
                    \centering
                    \includegraphics[width=\textwidth]{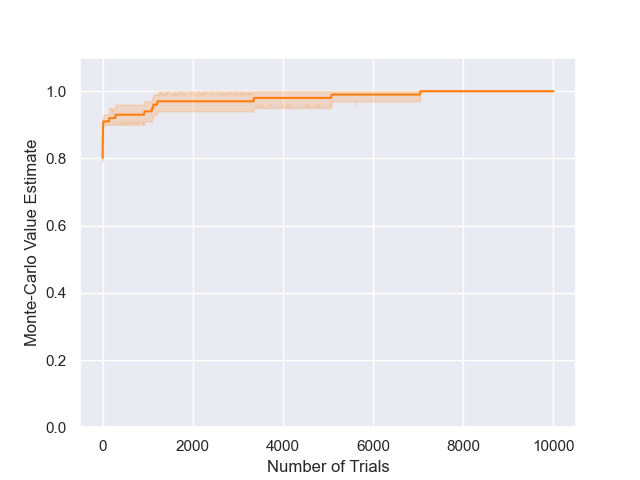}
                    \caption*{$\alpha=0.2,\epsilon=0.1$}
                \end{subfigure}
                \begin{subfigure}[b]{0.24\textwidth}
                    \centering
                    \includegraphics[width=\textwidth]{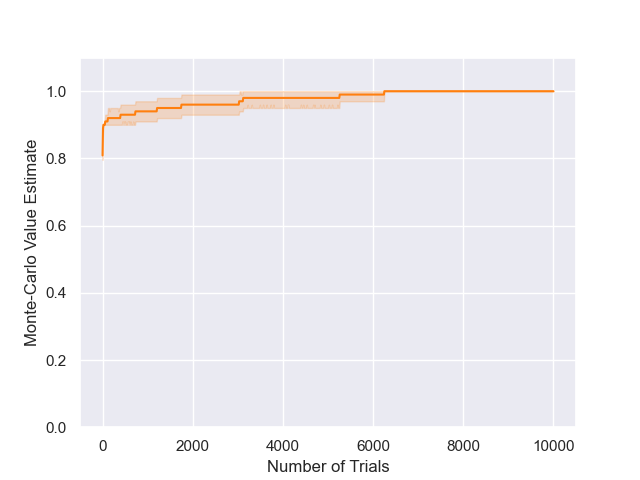}
                    \caption*{$\alpha=0.2,\epsilon=1$}
                \end{subfigure}
                \begin{subfigure}[b]{0.24\textwidth}
                    \centering
                    \includegraphics[width=\textwidth]{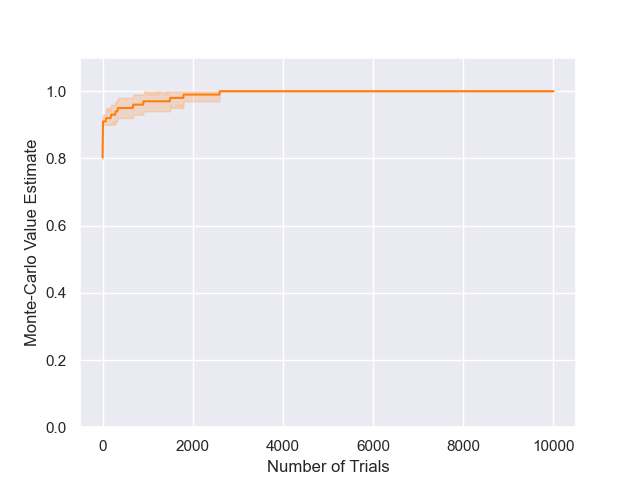}
                    \caption*{$\alpha=0.2,\epsilon=10$}
                \end{subfigure}
                
                \begin{subfigure}[b]{0.24\textwidth}
                    \centering
                    \includegraphics[width=\textwidth]{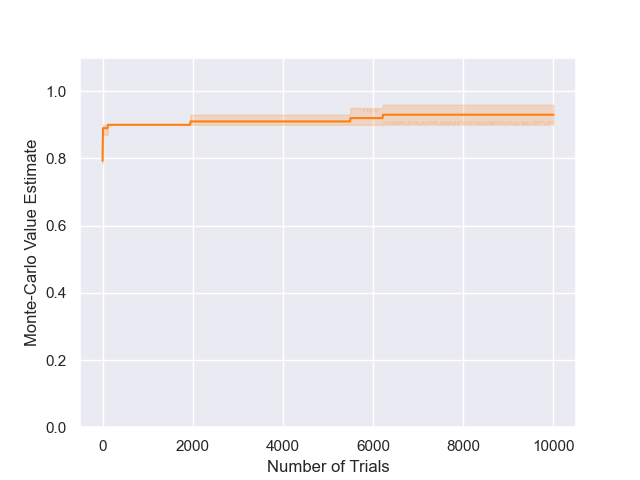}
                    \caption*{$\alpha=0.15,\epsilon=0.01$}
                \end{subfigure}
                \begin{subfigure}[b]{0.24\textwidth}
                    \centering
                    \includegraphics[width=\textwidth]{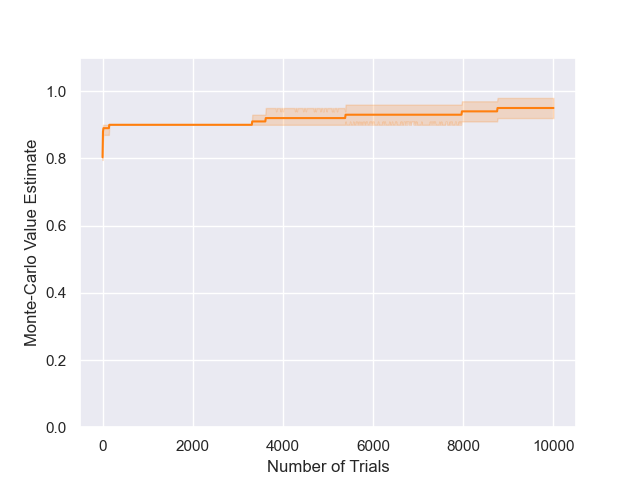}
                    \caption*{$\alpha=0.15,\epsilon=0.1$}
                \end{subfigure}
                \begin{subfigure}[b]{0.24\textwidth}
                    \centering
                    \includegraphics[width=\textwidth]{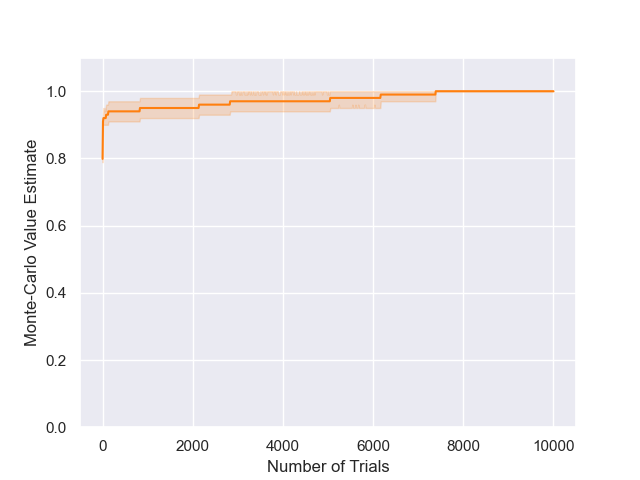}
                    \caption*{$\alpha=0.15,\epsilon=1$}
                \end{subfigure}
                \begin{subfigure}[b]{0.24\textwidth}
                    \centering
                    \includegraphics[width=\textwidth]{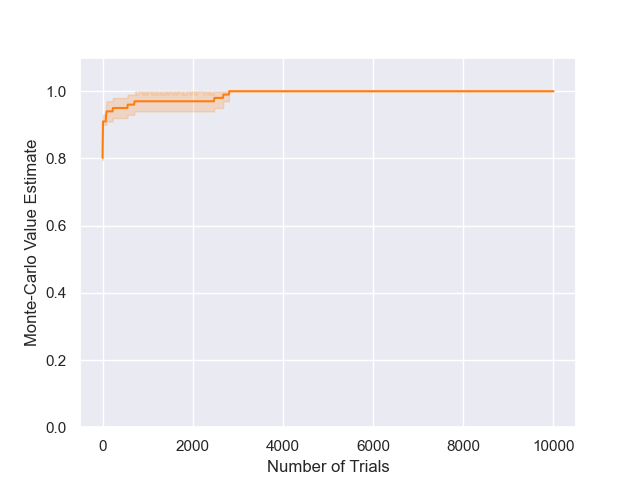}
                    \caption*{$\alpha=0.15,\epsilon=10$}
                \end{subfigure}
                
                \begin{subfigure}[b]{0.24\textwidth}
                    \centering
                    \includegraphics[width=\textwidth]{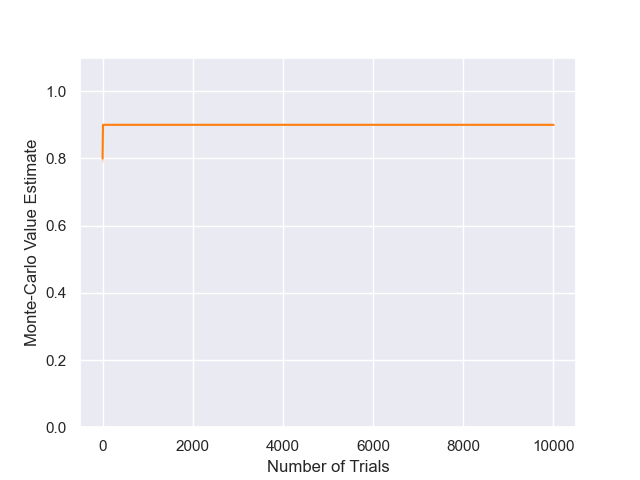}
                    \caption*{$\alpha=0.1,\epsilon=0.01$}
                \end{subfigure}
                \begin{subfigure}[b]{0.24\textwidth}
                    \centering
                    \includegraphics[width=\textwidth]{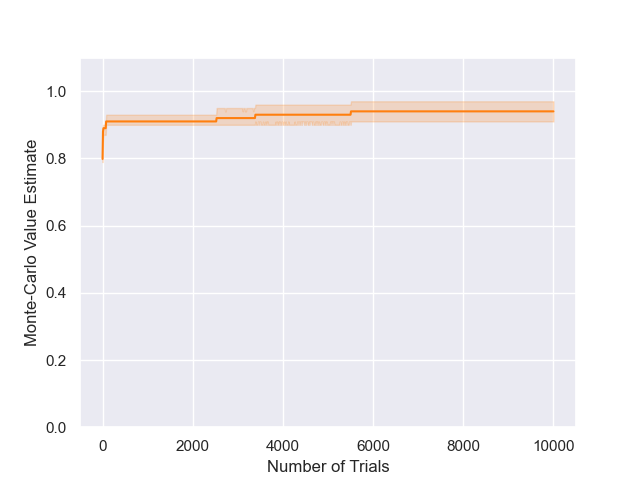}
                    \caption*{$\alpha=0.1,\epsilon=0.1$}
                \end{subfigure}
                \begin{subfigure}[b]{0.24\textwidth}
                    \centering
                    \includegraphics[width=\textwidth]{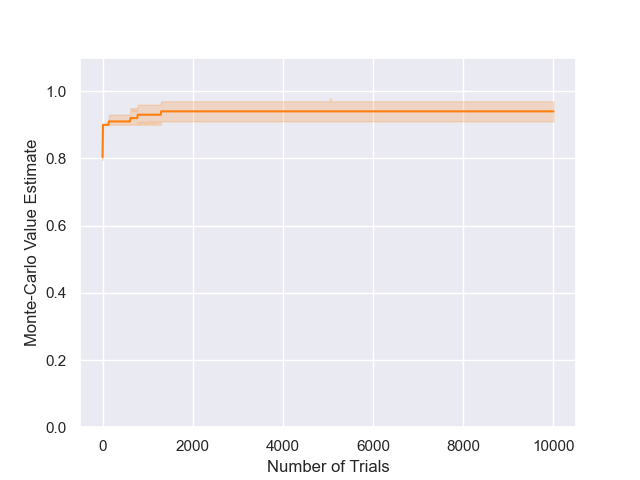}
                    \caption*{$\alpha=0.1,\epsilon=1$}
                \end{subfigure}
                \begin{subfigure}[b]{0.24\textwidth}
                    \centering
                    \includegraphics[width=\textwidth]{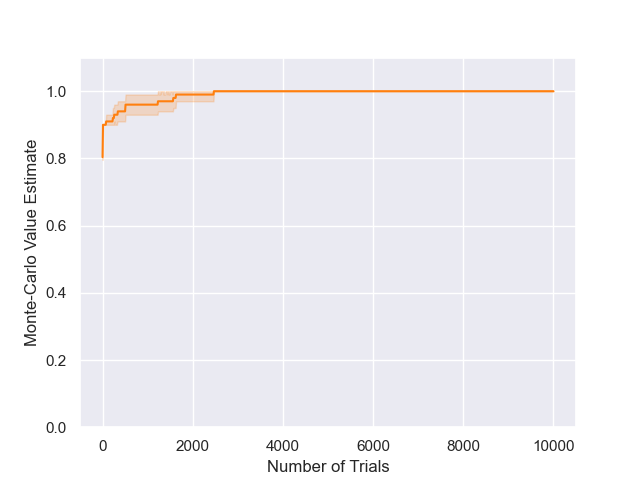}
                    \caption*{$\alpha=0.1,\epsilon=10$}
                \end{subfigure}
                
                \begin{subfigure}[b]{0.24\textwidth}
                    \centering
                    \includegraphics[width=\textwidth]{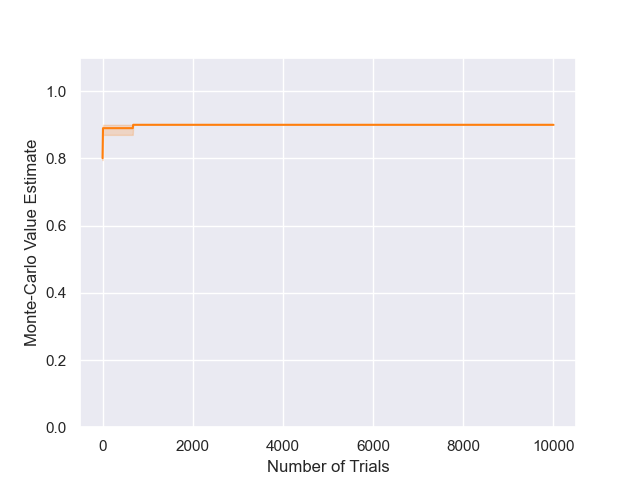}
                    \caption*{$\alpha=0.05,\epsilon=0.01$}
                \end{subfigure}
                \begin{subfigure}[b]{0.24\textwidth}
                    \centering
                    \includegraphics[width=\textwidth]{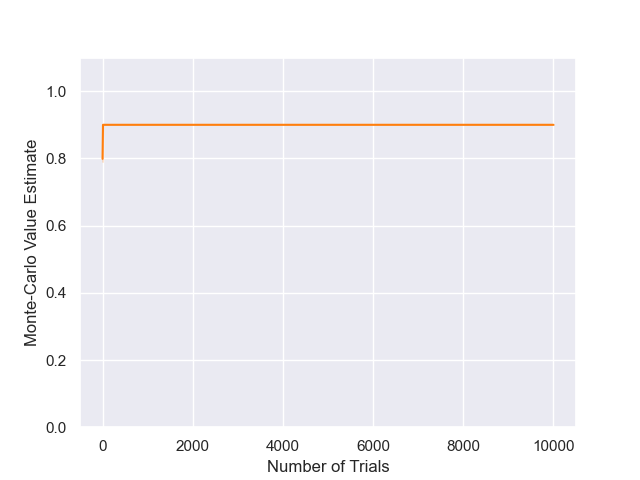}
                    \caption*{$\alpha=0.05,\epsilon=0.1$}
                \end{subfigure}
                \begin{subfigure}[b]{0.24\textwidth}
                    \centering
                    \includegraphics[width=\textwidth]{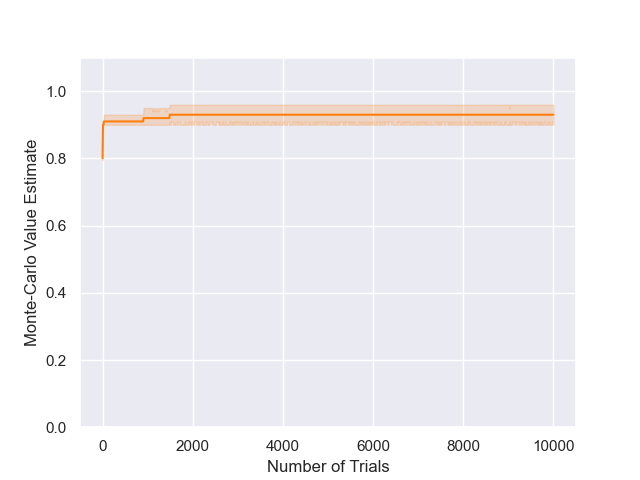}
                    \caption*{$\alpha=0.05,\epsilon=1$}
                \end{subfigure}
                \begin{subfigure}[b]{0.24\textwidth}
                    \centering
                    \includegraphics[width=\textwidth]{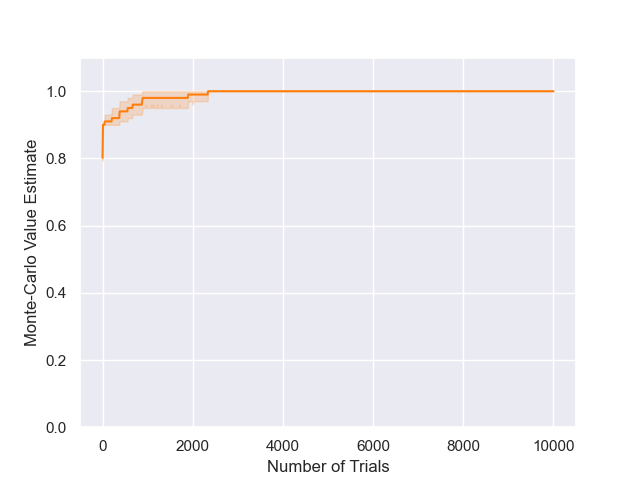}
                    \caption*{$\alpha=0.05,\epsilon=10$}
                \end{subfigure}
                
                \begin{subfigure}[b]{0.24\textwidth}
                    \centering
                    \includegraphics[width=\textwidth]{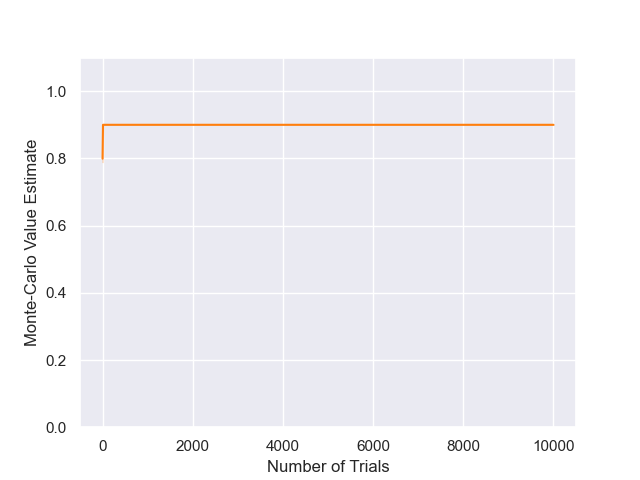}
                    \caption*{$\alpha=0.01,\epsilon=0.01$}
                \end{subfigure}
                \begin{subfigure}[b]{0.24\textwidth}
                    \centering
                    \includegraphics[width=\textwidth]{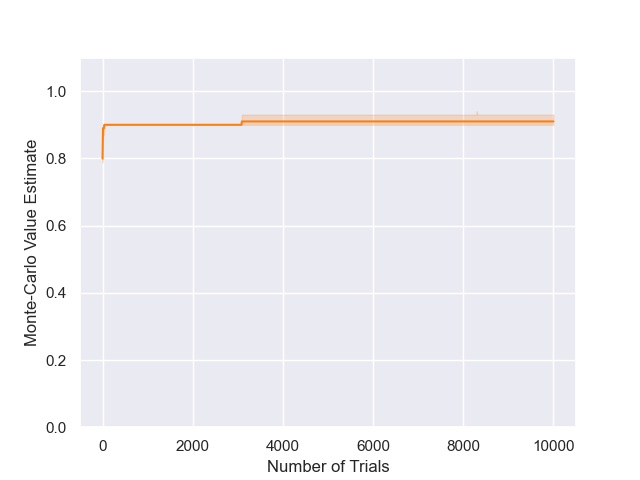}
                    \caption*{$\alpha=0.01,\epsilon=0.1$}
                \end{subfigure}
                \begin{subfigure}[b]{0.24\textwidth}
                    \centering
                    \includegraphics[width=\textwidth]{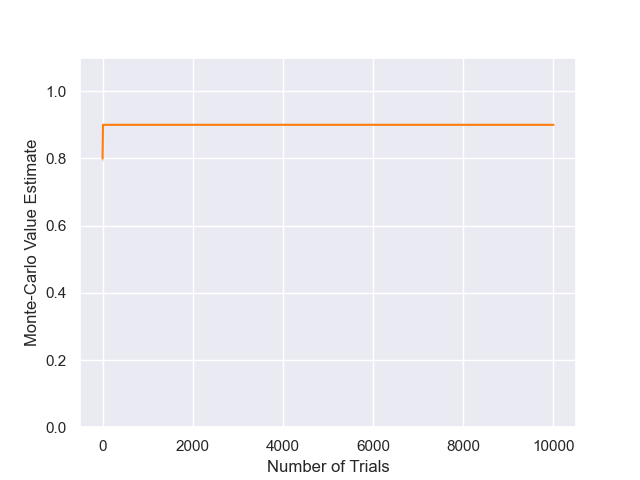}
                    \caption*{$\alpha=0.01,\epsilon=1$}
                \end{subfigure}
                \begin{subfigure}[b]{0.24\textwidth}
                    \centering
                    \includegraphics[width=\textwidth]{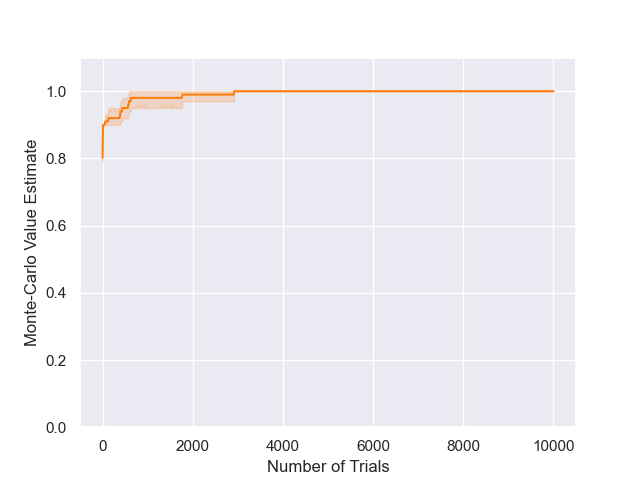}
                    \caption*{$\alpha=0.01,\epsilon=10$}
                \end{subfigure}
                
                \caption{Results for DENTS on the 10-chain ($D=10$, $R_f=1.0$), for varying temperatures and exploration parameters. The decay function was set to $\beta(m)=\alpha/\log(e+m)$.}
                \label{fig:dents_10chain_hps}
            \end{figure}

            \begin{figure}
                \centering
                
                \begin{subfigure}[b]{0.32\textwidth}
                    \centering
                    \includegraphics[width=\textwidth]{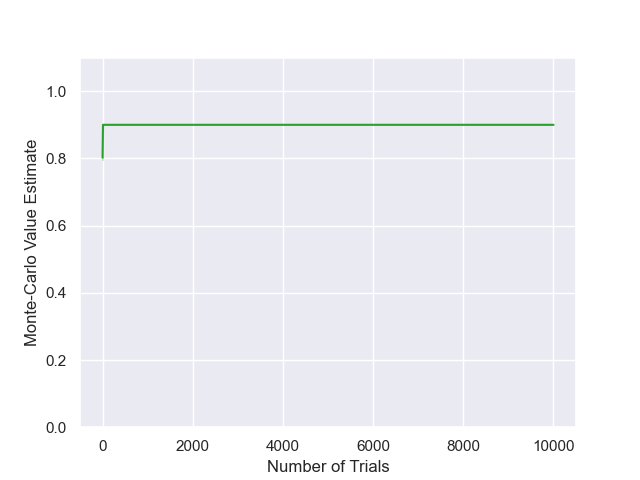}
                    \caption*{Bias set using \cite{prst}.}
                \end{subfigure}
                \begin{subfigure}[b]{0.32\textwidth}
                    \centering
                    \includegraphics[width=\textwidth]{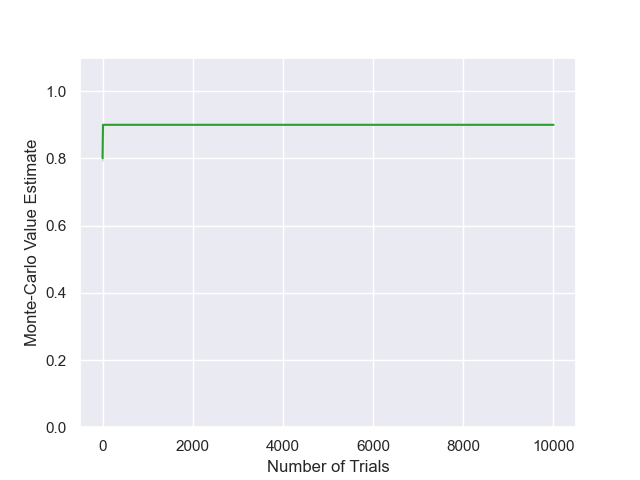}
                    \caption*{Bias $=0.1$.}
                \end{subfigure}
                \begin{subfigure}[b]{0.32\textwidth}
                    \centering
                    \includegraphics[width=\textwidth]{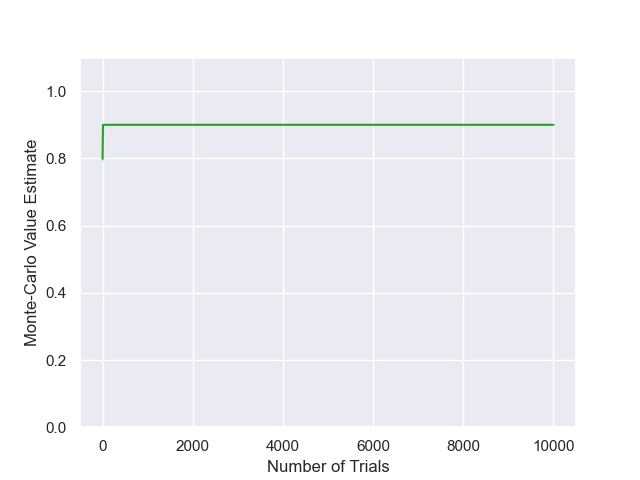}
                    \caption*{Bias $=1$.}
                \end{subfigure}
                
                \begin{subfigure}[b]{0.32\textwidth}
                    \centering
                    \includegraphics[width=\textwidth]{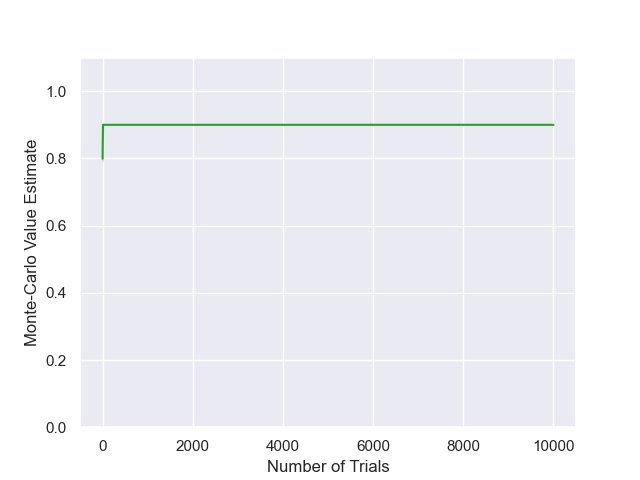}
                    \caption*{Bias $=2$.}
                \end{subfigure}
                \begin{subfigure}[b]{0.32\textwidth}
                    \centering
                    \includegraphics[width=\textwidth]{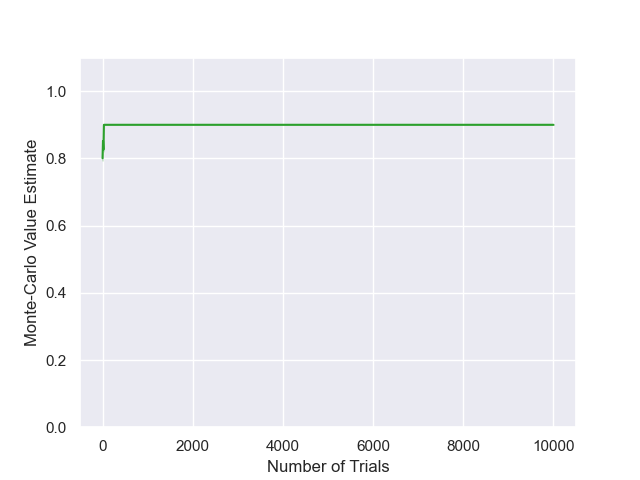}
                    \caption*{Bias $=10$.}
                \end{subfigure}
                \begin{subfigure}[b]{0.32\textwidth}
                    \centering
                    \includegraphics[width=\textwidth]{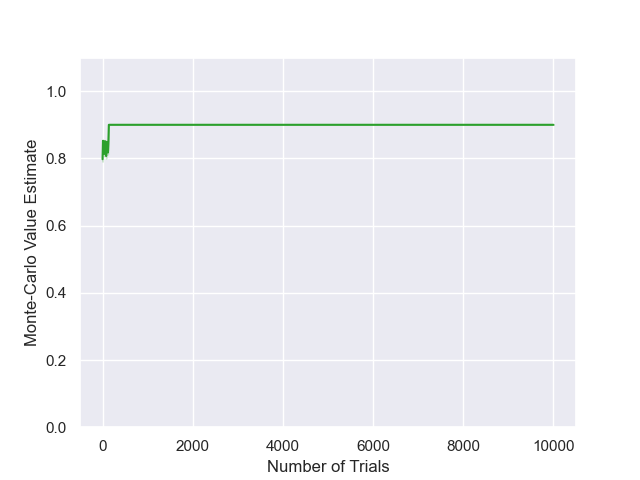}
                    \caption*{Bias $=100$.}
                \end{subfigure}
                
                \caption{Results for UCT on the modified 10-chain ($D=10$, $R_f=0.5$), for varying bias parameters.}
                \label{fig:uct_10chain_half_hps}
            \end{figure}

            \begin{figure}
                \centering
                
                \begin{subfigure}[b]{0.24\textwidth}
                    \centering
                    \includegraphics[width=\textwidth]{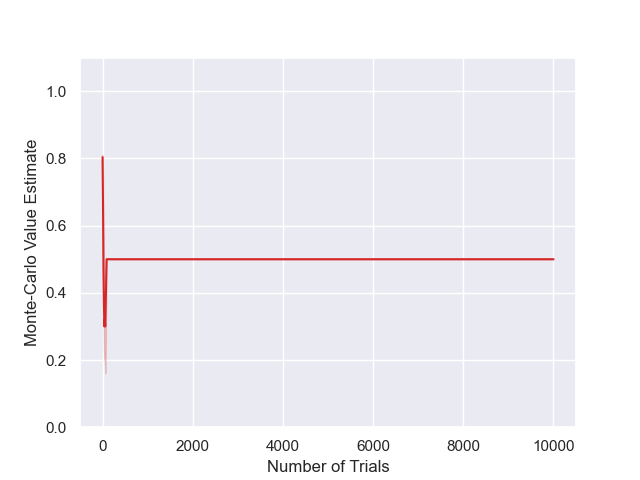}
                    \caption*{$\alpha=1,\epsilon=0.01$}
                \end{subfigure}
                \begin{subfigure}[b]{0.24\textwidth}
                    \centering
                    \includegraphics[width=\textwidth]{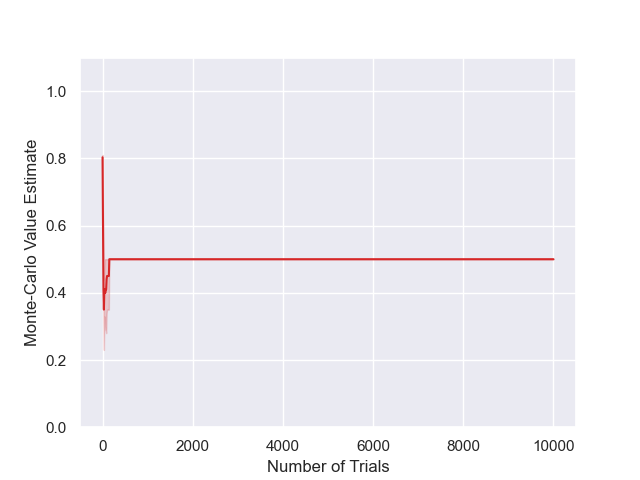}
                    \caption*{$\alpha=1,\epsilon=0.1$}
                \end{subfigure}
                \begin{subfigure}[b]{0.24\textwidth}
                    \centering
                    \includegraphics[width=\textwidth]{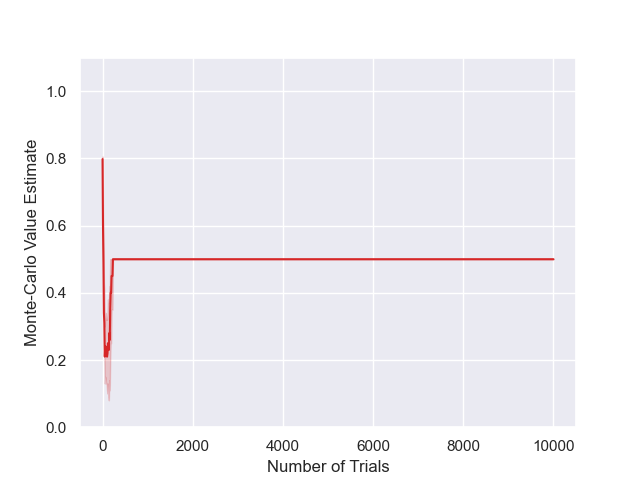}
                    \caption*{$\alpha=1,\epsilon=1$}
                \end{subfigure}
                \begin{subfigure}[b]{0.24\textwidth}
                    \centering
                    \includegraphics[width=\textwidth]{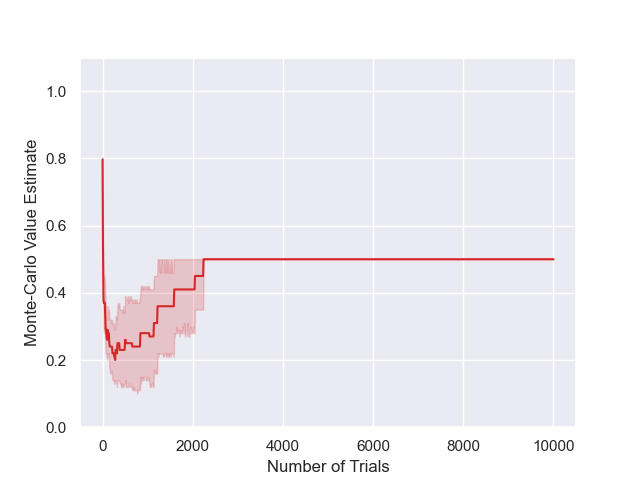}
                    \caption*{$\alpha=1,\epsilon=10$}
                \end{subfigure}
                
                \begin{subfigure}[b]{0.24\textwidth}
                    \centering
                    \includegraphics[width=\textwidth]{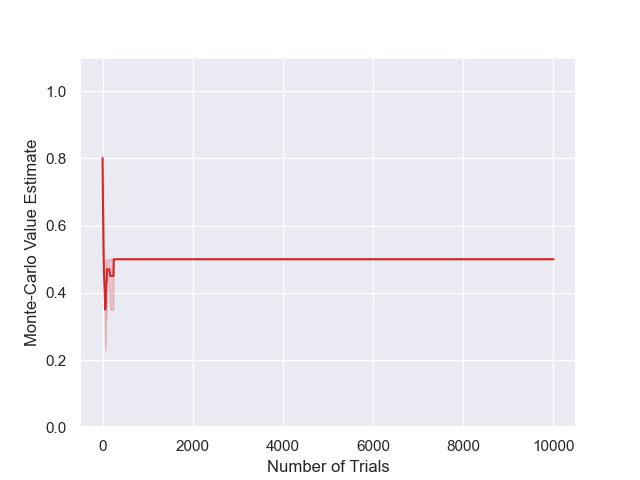}
                    \caption*{$\alpha=0.5,\epsilon=0.01$}
                \end{subfigure}
                \begin{subfigure}[b]{0.24\textwidth}
                    \centering
                    \includegraphics[width=\textwidth]{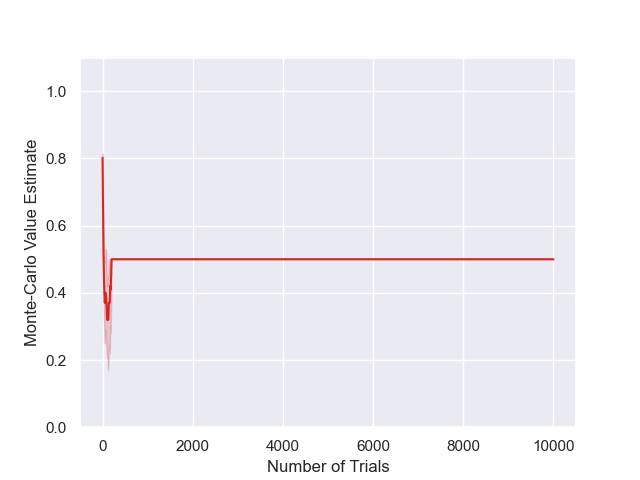}
                    \caption*{$\alpha=0.5,\epsilon=0.1$}
                \end{subfigure}
                \begin{subfigure}[b]{0.24\textwidth}
                    \centering
                    \includegraphics[width=\textwidth]{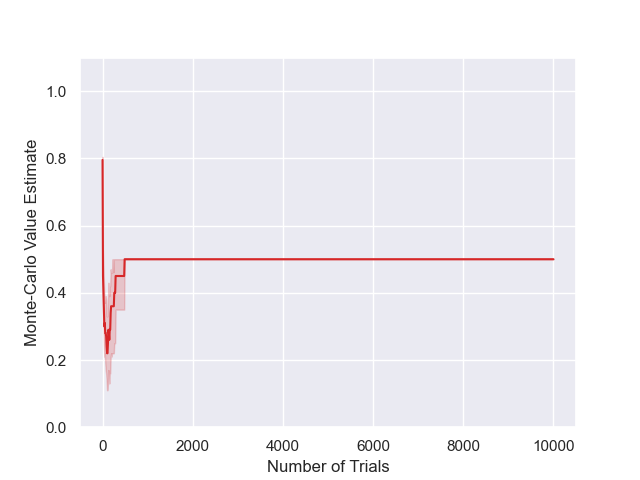}
                    \caption*{$\alpha=0.5,\epsilon=1$}
                \end{subfigure}
                \begin{subfigure}[b]{0.24\textwidth}
                    \centering
                    \includegraphics[width=\textwidth]{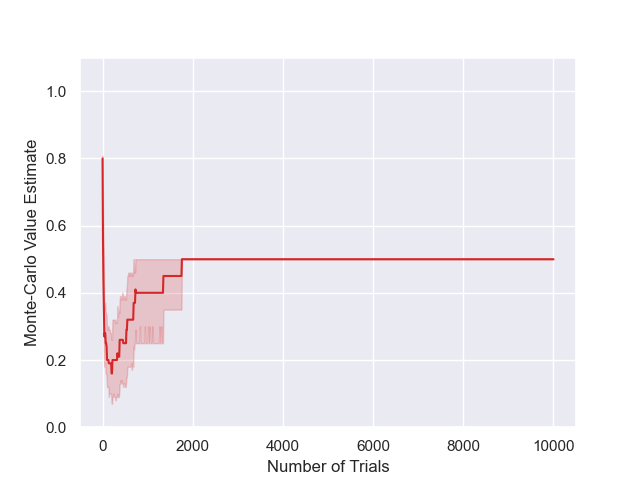}
                    \caption*{$\alpha=0.5,\epsilon=10$}
                \end{subfigure}
                
                \begin{subfigure}[b]{0.24\textwidth}
                    \centering
                    \includegraphics[width=\textwidth]{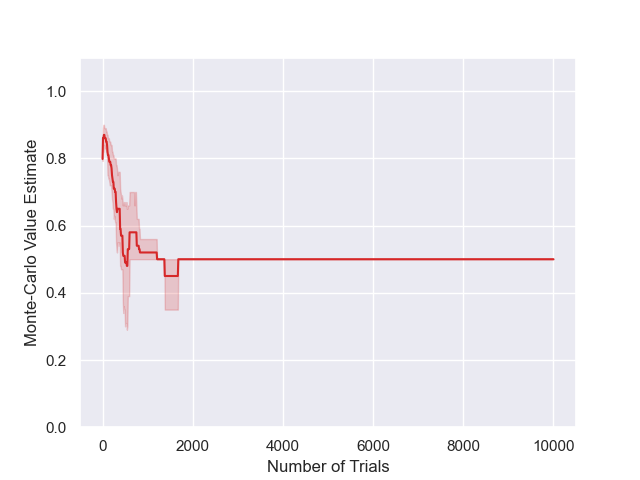}
                    \caption*{$\alpha=0.2,\epsilon=0.01$}
                \end{subfigure}
                \begin{subfigure}[b]{0.24\textwidth}
                    \centering
                    \includegraphics[width=\textwidth]{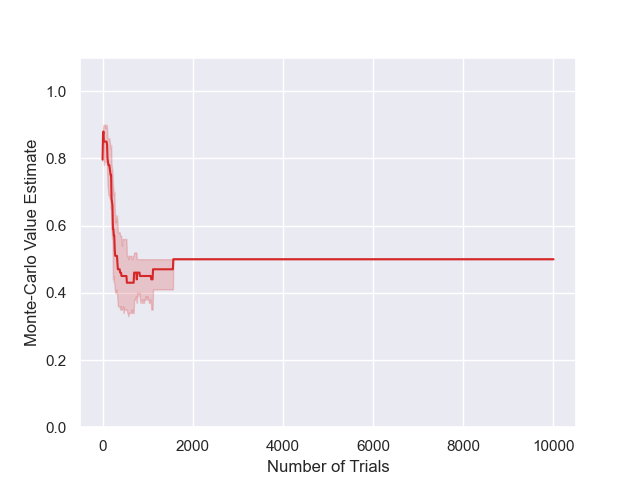}
                    \caption*{$\alpha=0.2,\epsilon=0.1$}
                \end{subfigure}
                \begin{subfigure}[b]{0.24\textwidth}
                    \centering
                    \includegraphics[width=\textwidth]{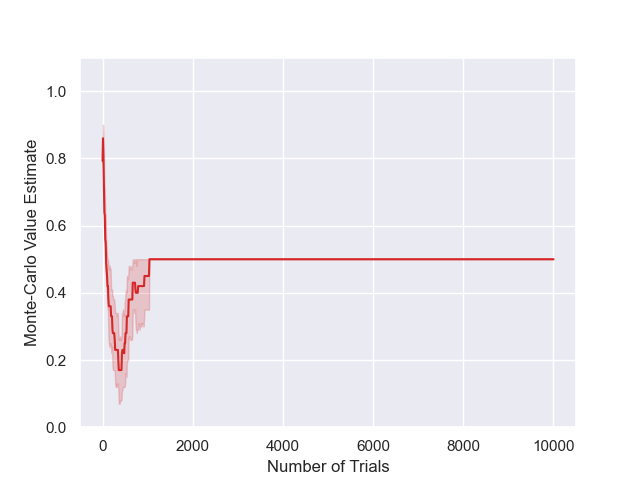}
                    \caption*{$\alpha=0.2,\epsilon=1$}
                \end{subfigure}
                \begin{subfigure}[b]{0.24\textwidth}
                    \centering
                    \includegraphics[width=\textwidth]{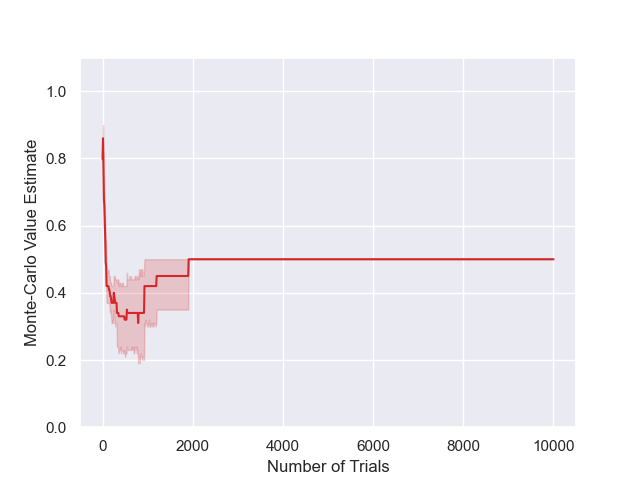}
                    \caption*{$\alpha=0.2,\epsilon=10$}
                \end{subfigure}
                
                \begin{subfigure}[b]{0.24\textwidth}
                    \centering
                    \includegraphics[width=\textwidth]{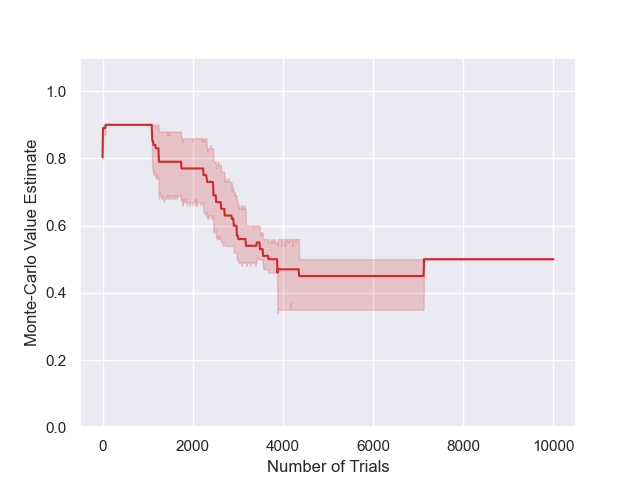}
                    \caption*{$\alpha=0.15,\epsilon=0.01$}
                \end{subfigure}
                \begin{subfigure}[b]{0.24\textwidth}
                    \centering
                    \includegraphics[width=\textwidth]{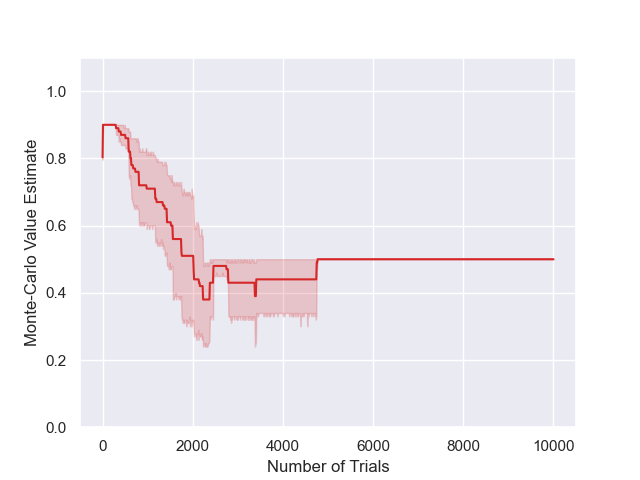}
                    \caption*{$\alpha=0.15,\epsilon=0.1$}
                \end{subfigure}
                \begin{subfigure}[b]{0.24\textwidth}
                    \centering
                    \includegraphics[width=\textwidth]{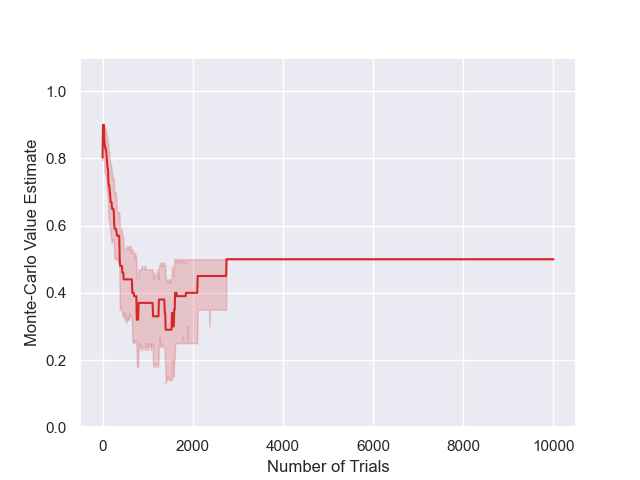}
                    \caption*{$\alpha=0.15,\epsilon=1$}
                \end{subfigure}
                \begin{subfigure}[b]{0.24\textwidth}
                    \centering
                    \includegraphics[width=\textwidth]{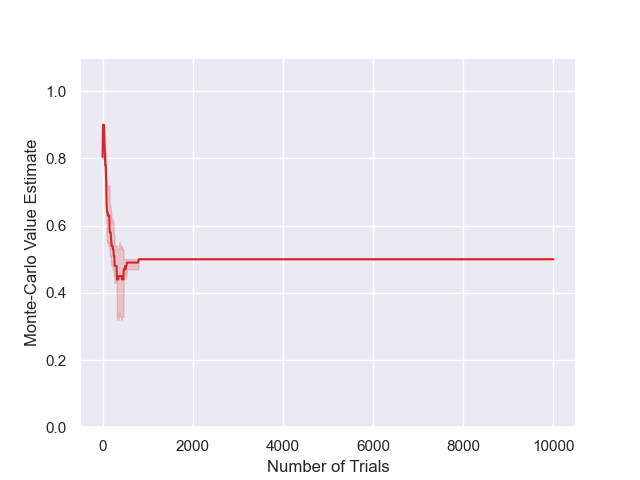}
                    \caption*{$\alpha=0.15,\epsilon=10$}
                \end{subfigure}
                
                \begin{subfigure}[b]{0.24\textwidth}
                    \centering
                    \includegraphics[width=\textwidth]{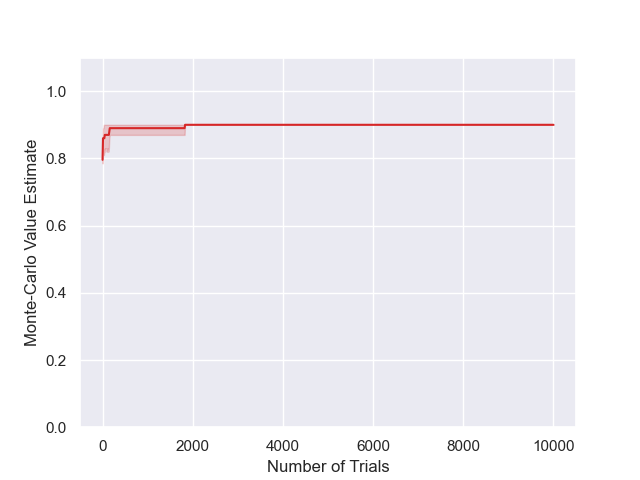}
                    \caption*{$\alpha=0.1,\epsilon=0.01$}
                \end{subfigure}
                \begin{subfigure}[b]{0.24\textwidth}
                    \centering
                    \includegraphics[width=\textwidth]{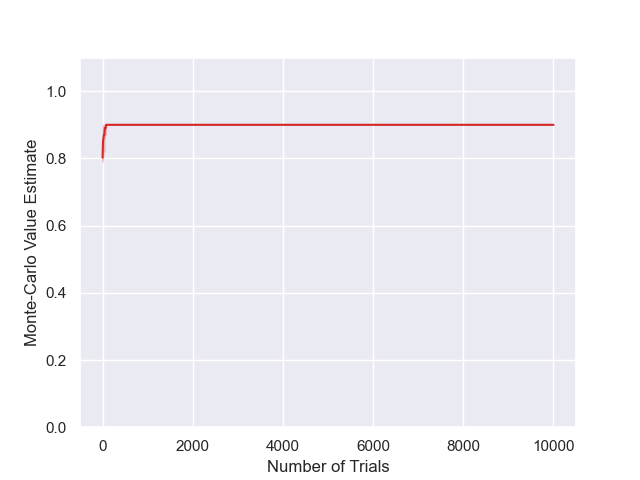}
                    \caption*{$\alpha=0.1,\epsilon=0.1$}
                \end{subfigure}
                \begin{subfigure}[b]{0.24\textwidth}
                    \centering
                    \includegraphics[width=\textwidth]{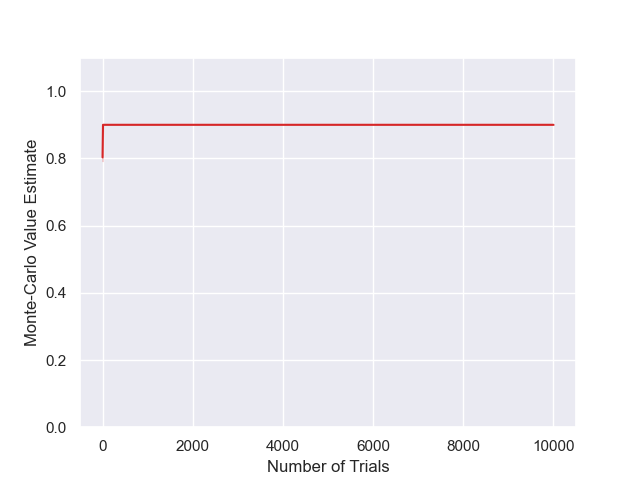}
                    \caption*{$\alpha=0.1,\epsilon=1$}
                \end{subfigure}
                \begin{subfigure}[b]{0.24\textwidth}
                    \centering
                    \includegraphics[width=\textwidth]{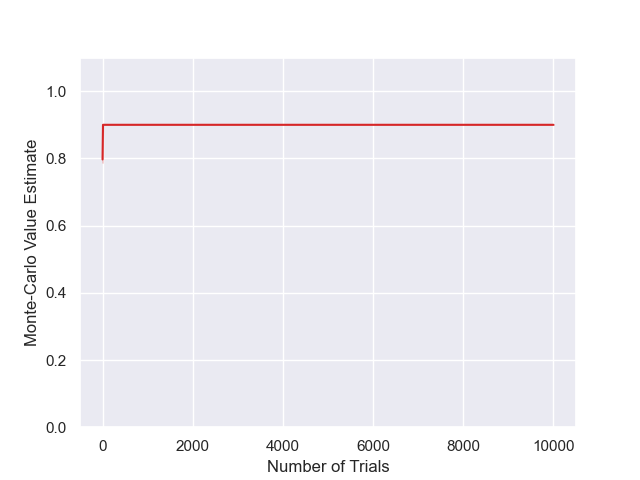}
                    \caption*{$\alpha=0.1,\epsilon=10$}
                \end{subfigure}
                
                \begin{subfigure}[b]{0.24\textwidth}
                    \centering
                    \includegraphics[width=\textwidth]{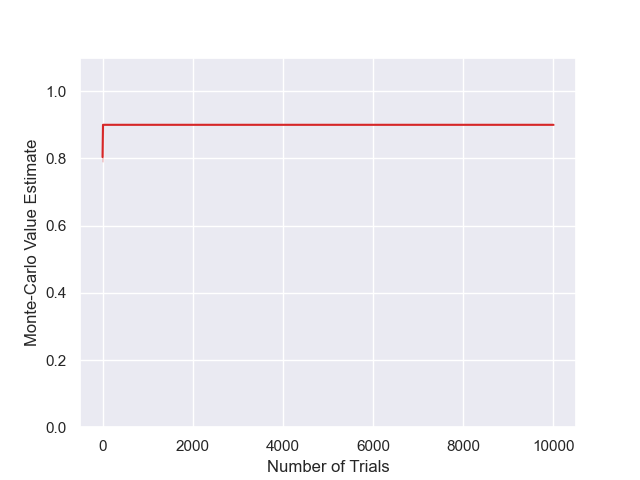}
                    \caption*{$\alpha=0.05,\epsilon=0.01$}
                \end{subfigure}
                \begin{subfigure}[b]{0.24\textwidth}
                    \centering
                    \includegraphics[width=\textwidth]{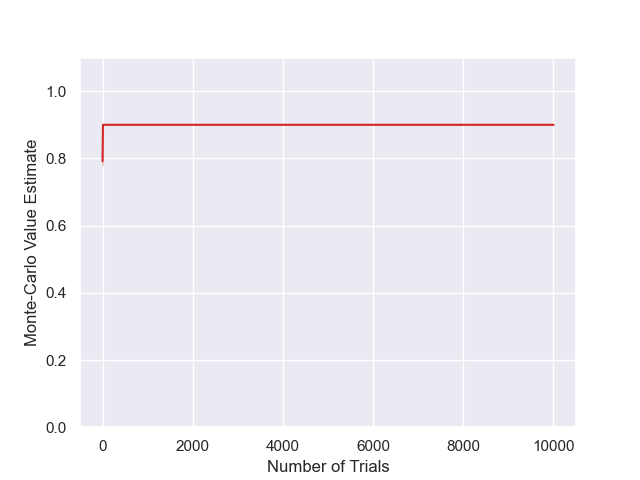}
                    \caption*{$\alpha=0.05,\epsilon=0.1$}
                \end{subfigure}
                \begin{subfigure}[b]{0.24\textwidth}
                    \centering
                    \includegraphics[width=\textwidth]{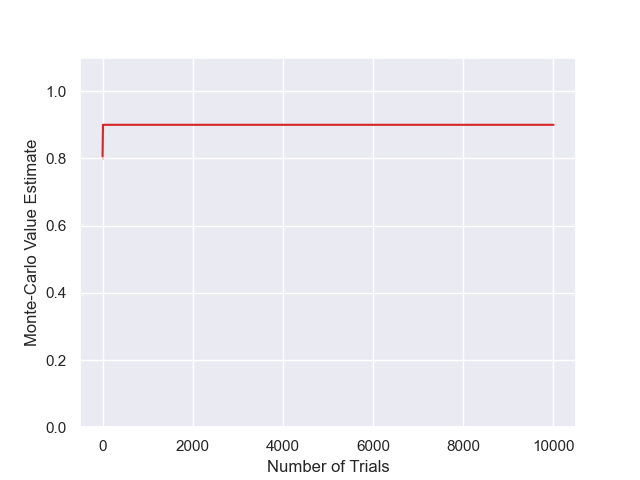}
                    \caption*{$\alpha=0.05,\epsilon=1$}
                \end{subfigure}
                \begin{subfigure}[b]{0.24\textwidth}
                    \centering
                    \includegraphics[width=\textwidth]{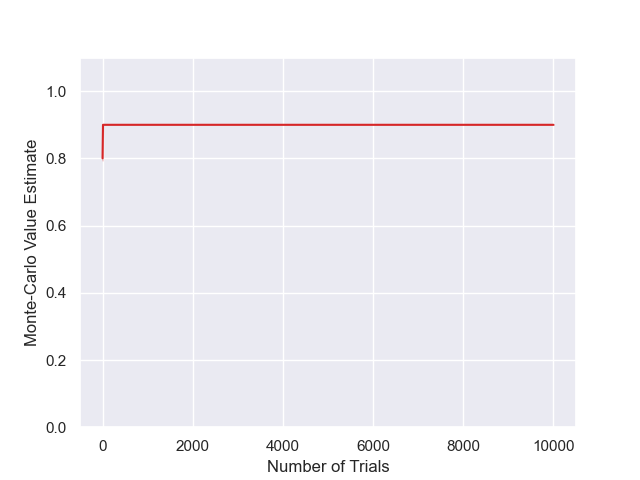}
                    \caption*{$\alpha=0.05,\epsilon=10$}
                \end{subfigure}
                
                \begin{subfigure}[b]{0.24\textwidth}
                    \centering
                    \includegraphics[width=\textwidth]{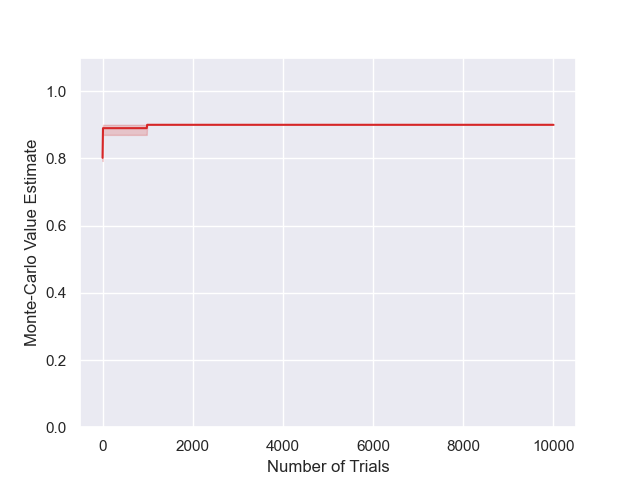}
                    \caption*{$\alpha=0.01,\epsilon=0.01$}
                \end{subfigure}
                \begin{subfigure}[b]{0.24\textwidth}
                    \centering
                    \includegraphics[width=\textwidth]{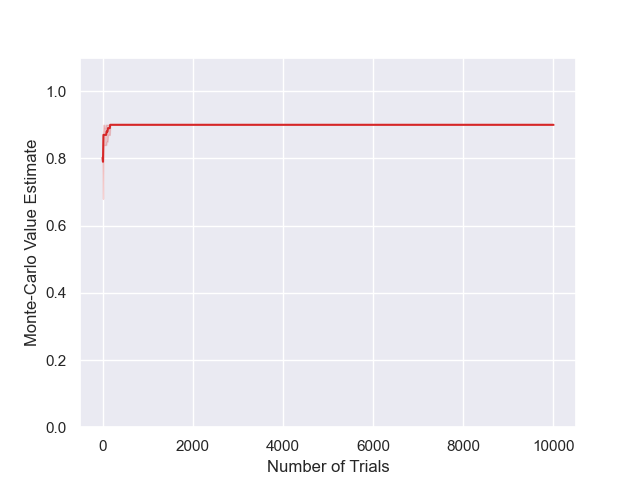}
                    \caption*{$\alpha=0.01,\epsilon=0.1$}
                \end{subfigure}
                \begin{subfigure}[b]{0.24\textwidth}
                    \centering
                    \includegraphics[width=\textwidth]{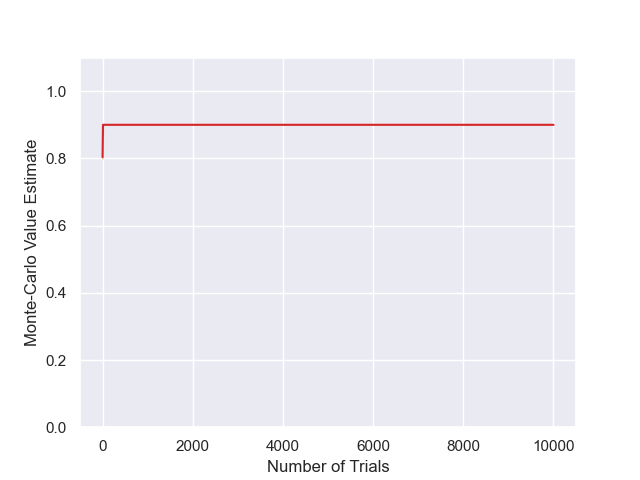}
                    \caption*{$\alpha=0.01,\epsilon=1$}
                \end{subfigure}
                \begin{subfigure}[b]{0.24\textwidth}
                    \centering
                    \includegraphics[width=\textwidth]{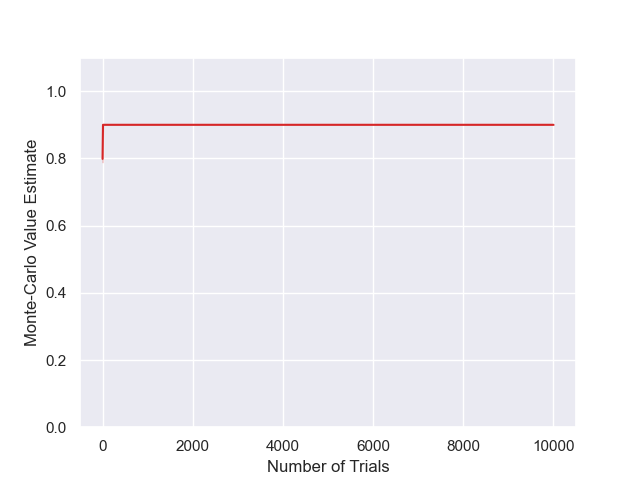}
                    \caption*{$\alpha=0.01,\epsilon=10$}
                \end{subfigure}
                
                \caption{Results for MENTS on the modified 10-chain ($D=10$, $R_f=0.5$), for varying temperatures and exploration parameters.}
                \label{fig:ments_10chain_half_hps}
            \end{figure}

            \begin{figure}
                \centering
                
                \begin{subfigure}[b]{0.24\textwidth}
                    \centering
                    \includegraphics[width=\textwidth]{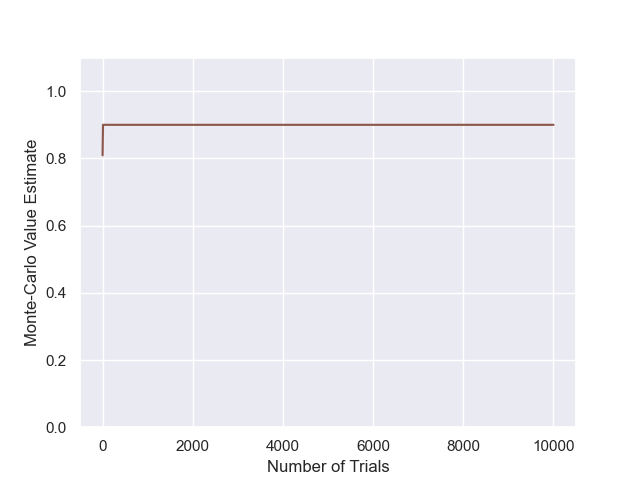}
                    \caption*{$\alpha=1000,\epsilon=0.01$}
                \end{subfigure}
                \begin{subfigure}[b]{0.24\textwidth}
                    \centering
                    \includegraphics[width=\textwidth]{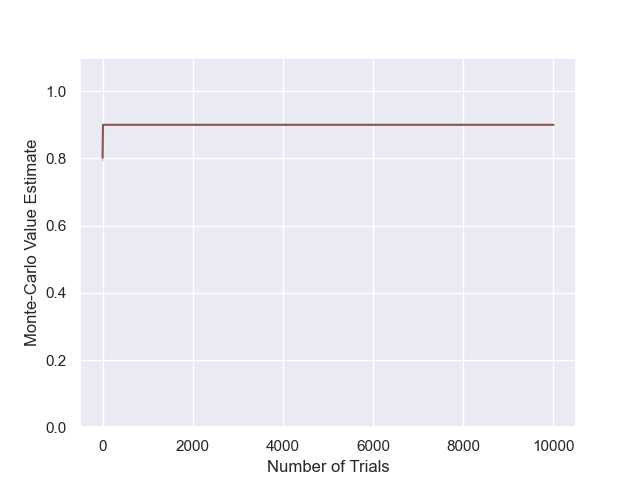}
                    \caption*{$\alpha=1000,\epsilon=0.1$}
                \end{subfigure}
                \begin{subfigure}[b]{0.24\textwidth}
                    \centering
                    \includegraphics[width=\textwidth]{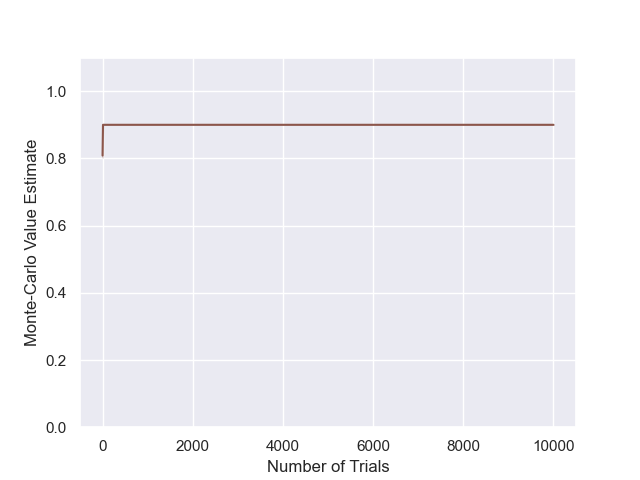}
                    \caption*{$\alpha=1000,\epsilon=1$}
                \end{subfigure}
                \begin{subfigure}[b]{0.24\textwidth}
                    \centering
                    \includegraphics[width=\textwidth]{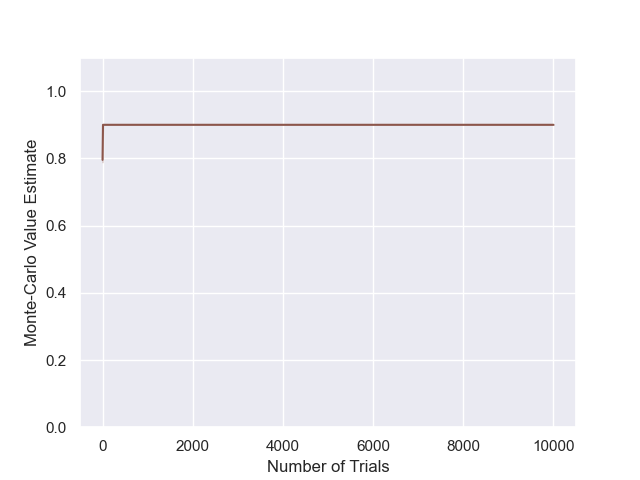}
                    \caption*{$\alpha=1000,\epsilon=10$}
                \end{subfigure}
                
                \begin{subfigure}[b]{0.24\textwidth}
                    \centering
                    \includegraphics[width=\textwidth]{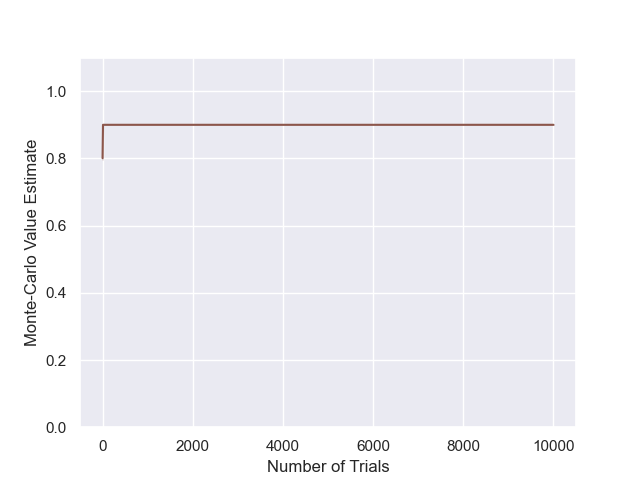}
                    \caption*{$\alpha=100,\epsilon=0.01$}
                \end{subfigure}
                \begin{subfigure}[b]{0.24\textwidth}
                    \centering
                    \includegraphics[width=\textwidth]{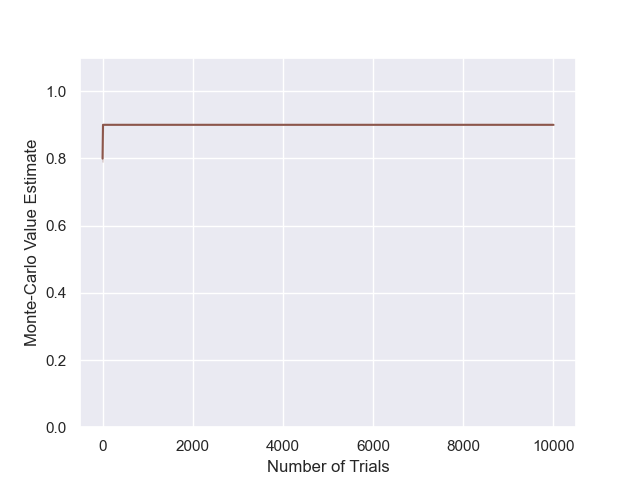}
                    \caption*{$\alpha=100,\epsilon=0.1$}
                \end{subfigure}
                \begin{subfigure}[b]{0.24\textwidth}
                    \centering
                    \includegraphics[width=\textwidth]{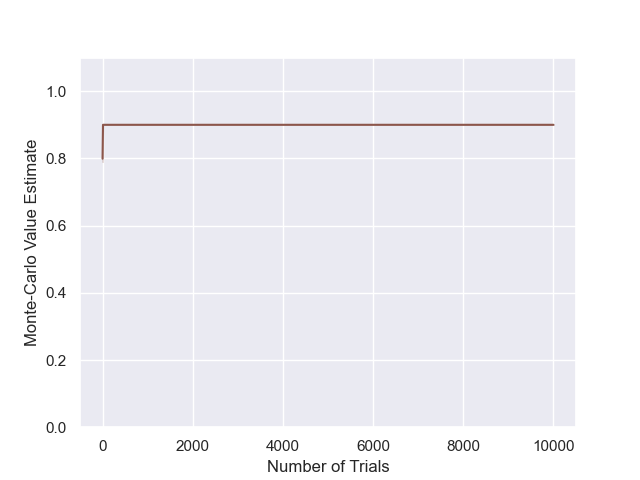}
                    \caption*{$\alpha=100,\epsilon=1$}
                \end{subfigure}
                \begin{subfigure}[b]{0.24\textwidth}
                    \centering
                    \includegraphics[width=\textwidth]{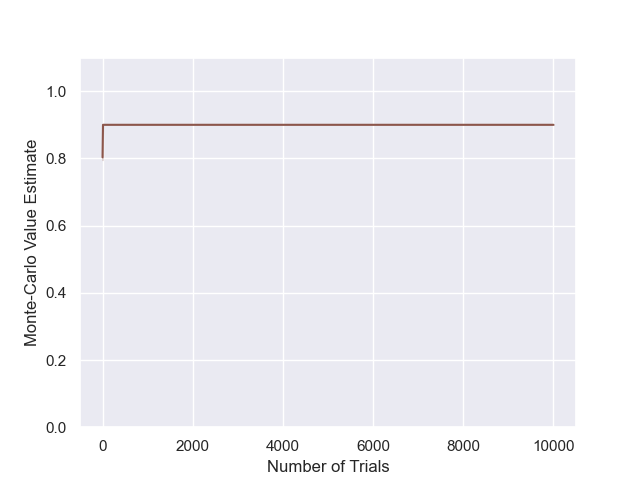}
                    \caption*{$\alpha=100,\epsilon=10$}
                \end{subfigure}
                
                \begin{subfigure}[b]{0.24\textwidth}
                    \centering
                    \includegraphics[width=\textwidth]{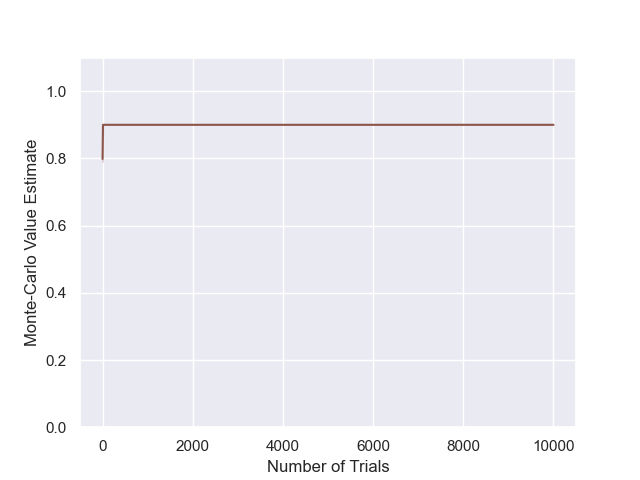}
                    \caption*{$\alpha=10,\epsilon=0.01$}
                \end{subfigure}
                \begin{subfigure}[b]{0.24\textwidth}
                    \centering
                    \includegraphics[width=\textwidth]{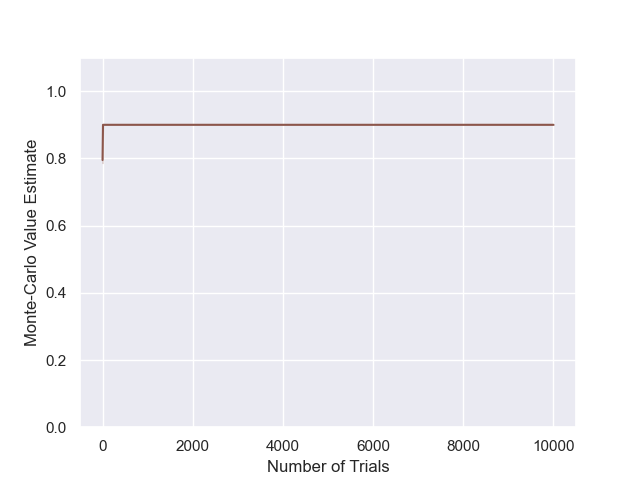}
                    \caption*{$\alpha=10,\epsilon=0.1$}
                \end{subfigure}
                \begin{subfigure}[b]{0.24\textwidth}
                    \centering
                    \includegraphics[width=\textwidth]{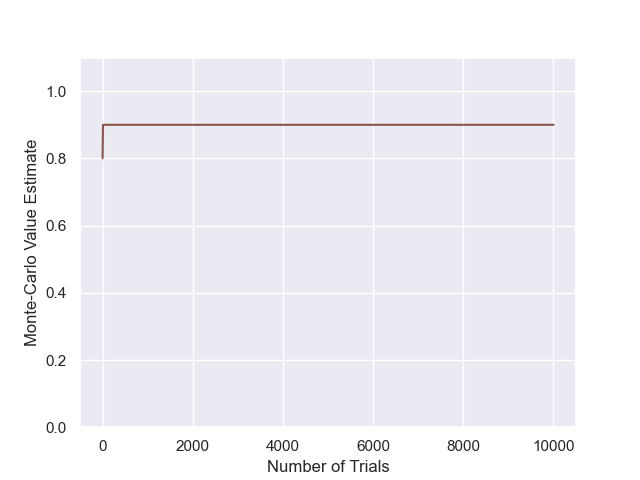}
                    \caption*{$\alpha=10,\epsilon=1$}
                \end{subfigure}
                \begin{subfigure}[b]{0.24\textwidth}
                    \centering
                    \includegraphics[width=\textwidth]{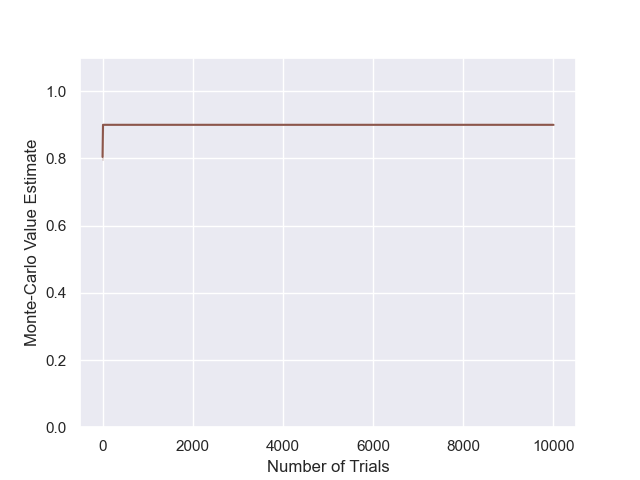}
                    \caption*{$\alpha=10,\epsilon=10$}
                \end{subfigure}
                
                \begin{subfigure}[b]{0.24\textwidth}
                    \centering
                    \includegraphics[width=\textwidth]{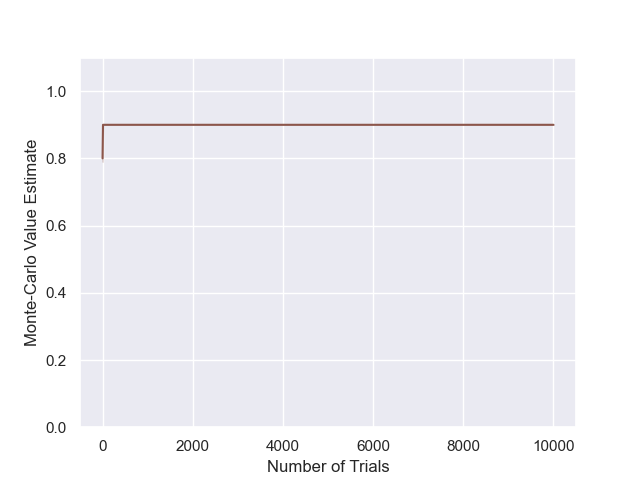}
                    \caption*{$\alpha=1,\epsilon=0.01$}
                \end{subfigure}
                \begin{subfigure}[b]{0.24\textwidth}
                    \centering
                    \includegraphics[width=\textwidth]{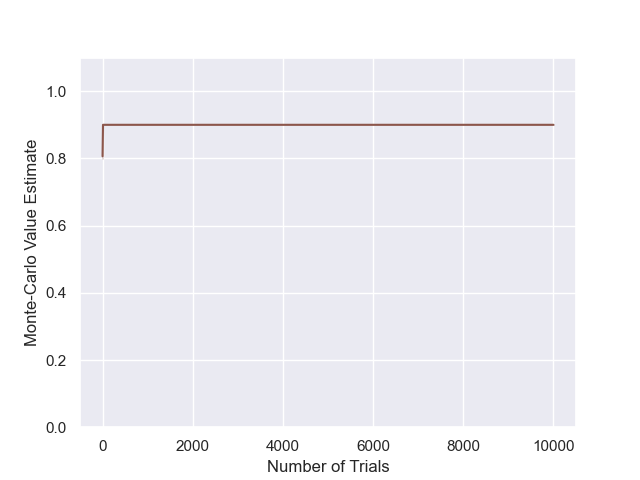}
                    \caption*{$\alpha=1,\epsilon=0.1$}
                \end{subfigure}
                \begin{subfigure}[b]{0.24\textwidth}
                    \centering
                    \includegraphics[width=\textwidth]{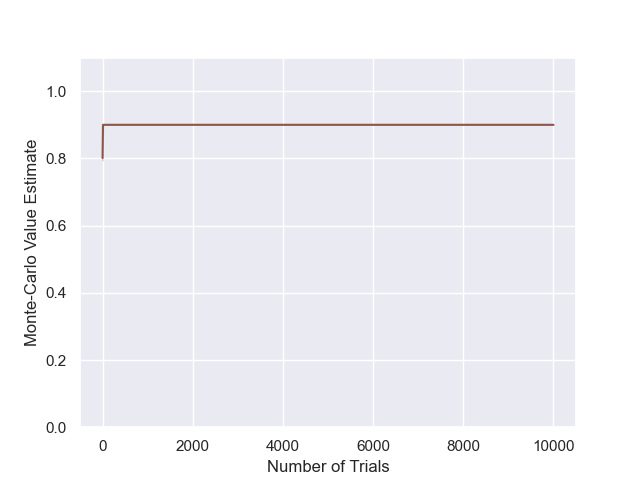}
                    \caption*{$\alpha=1,\epsilon=1$}
                \end{subfigure}
                \begin{subfigure}[b]{0.24\textwidth}
                    \centering
                    \includegraphics[width=\textwidth]{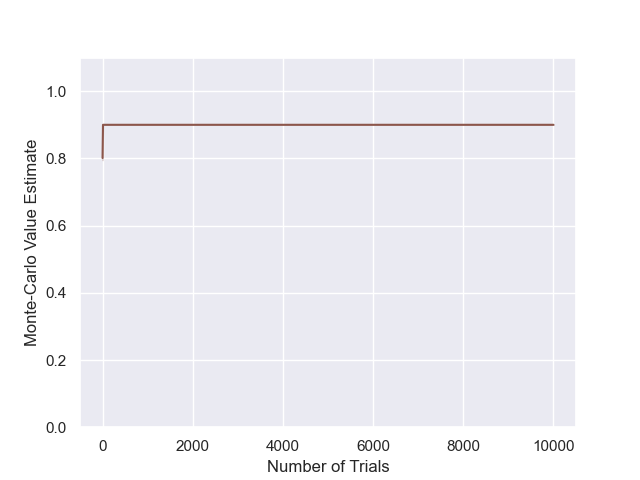}
                    \caption*{$\alpha=1,\epsilon=10$}
                \end{subfigure}
                
                \begin{subfigure}[b]{0.24\textwidth}
                    \centering
                    \includegraphics[width=\textwidth]{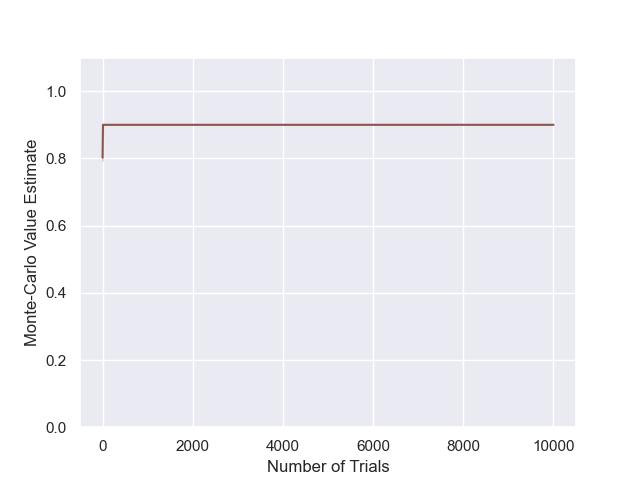}
                    \caption*{$\alpha=0.1,\epsilon=0.01$}
                \end{subfigure}
                \begin{subfigure}[b]{0.24\textwidth}
                    \centering
                    \includegraphics[width=\textwidth]{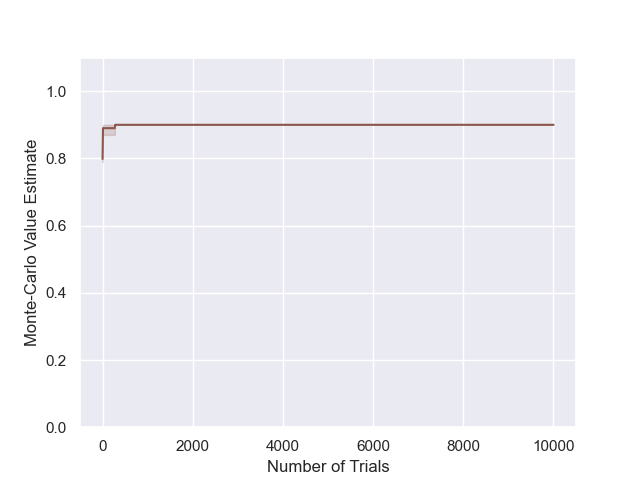}
                    \caption*{$\alpha=0.1,\epsilon=0.1$}
                \end{subfigure}
                \begin{subfigure}[b]{0.24\textwidth}
                    \centering
                    \includegraphics[width=\textwidth]{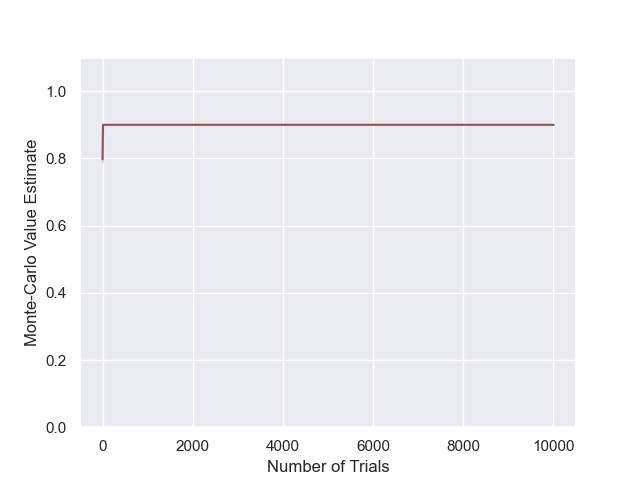}
                    \caption*{$\alpha=0.1,\epsilon=1$}
                \end{subfigure}
                \begin{subfigure}[b]{0.24\textwidth}
                    \centering
                    \includegraphics[width=\textwidth]{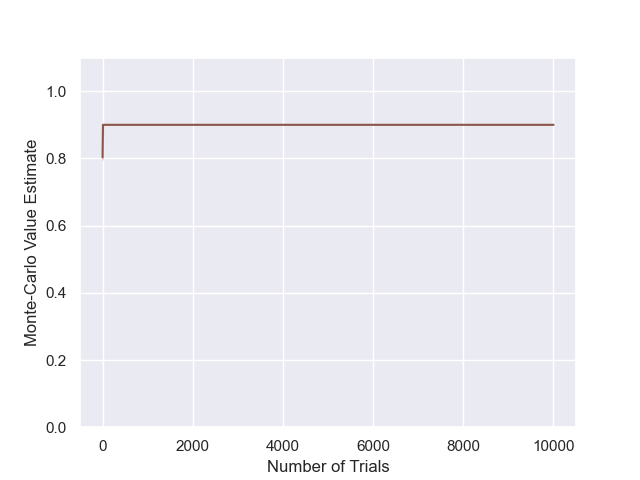}
                    \caption*{$\alpha=0.01,\epsilon=10$}
                \end{subfigure}
                
                \begin{subfigure}[b]{0.24\textwidth}
                    \centering
                    \includegraphics[width=\textwidth]{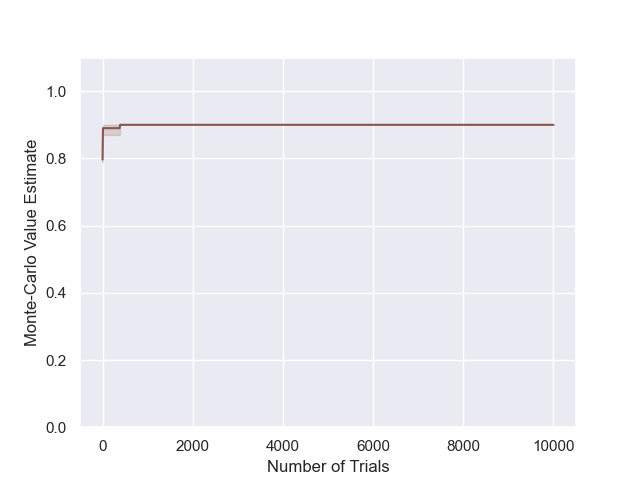}
                    \caption*{$\alpha=0.01,\epsilon=0.01$}
                \end{subfigure}
                \begin{subfigure}[b]{0.24\textwidth}
                    \centering
                    \includegraphics[width=\textwidth]{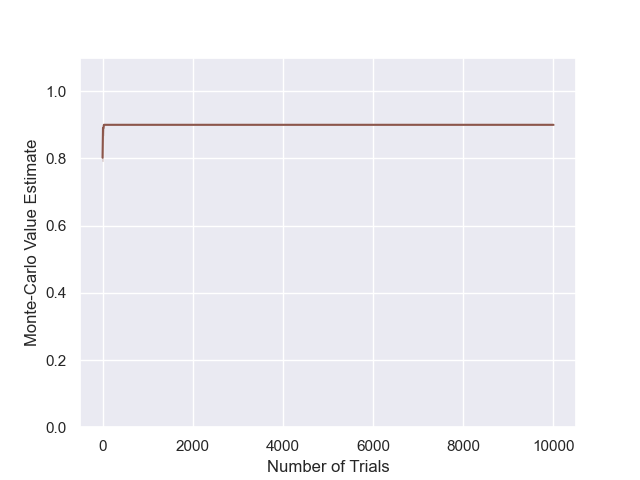}
                    \caption*{$\alpha=0.01,\epsilon=0.1$}
                \end{subfigure}
                \begin{subfigure}[b]{0.24\textwidth}
                    \centering
                    \includegraphics[width=\textwidth]{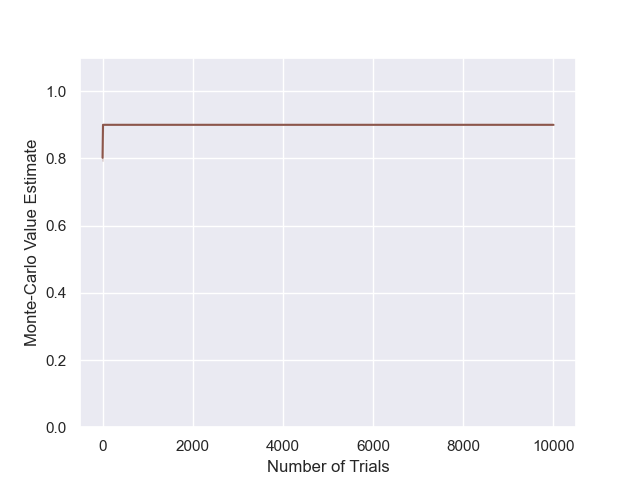}
                    \caption*{$\alpha=0.01,\epsilon=1$}
                \end{subfigure}
                \begin{subfigure}[b]{0.24\textwidth}
                    \centering
                    \includegraphics[width=\textwidth]{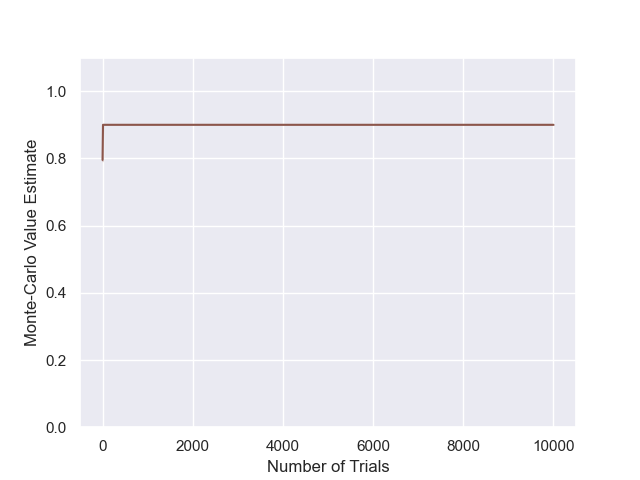}
                    \caption*{$\alpha=0.01,\epsilon=10$}
                \end{subfigure}
                
                \caption{Results for RENTS on the modified 10-chain ($D=10$, $R_f=0.5$), for varying temperatures and exploration parameters.}
                \label{fig:rents_10chain_half_hps}
            \end{figure}

            \begin{figure}
                \centering
                
                \begin{subfigure}[b]{0.24\textwidth}
                    \centering
                    \includegraphics[width=\textwidth]{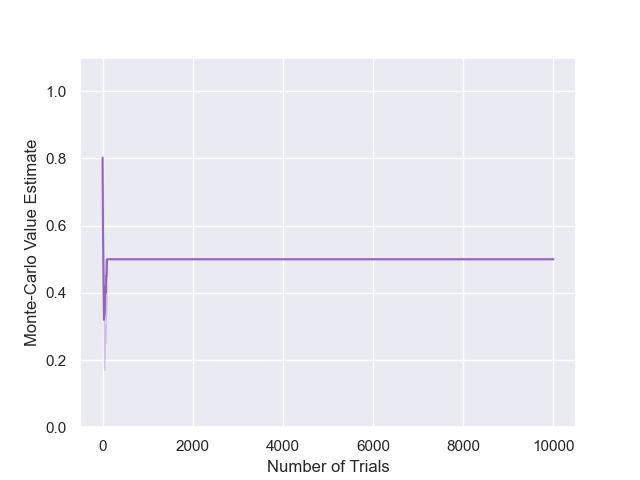}
                    \caption*{$\alpha=10,\epsilon=0.01$}
                \end{subfigure}
                \begin{subfigure}[b]{0.24\textwidth}
                    \centering
                    \includegraphics[width=\textwidth]{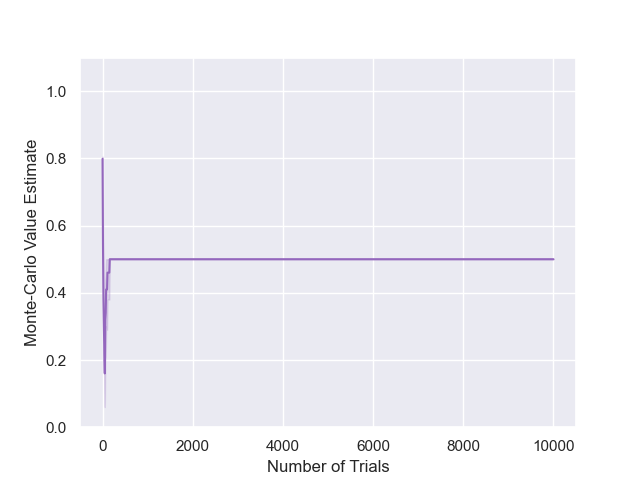}
                    \caption*{$\alpha=10,\epsilon=0.1$}
                \end{subfigure}
                \begin{subfigure}[b]{0.24\textwidth}
                    \centering
                    \includegraphics[width=\textwidth]{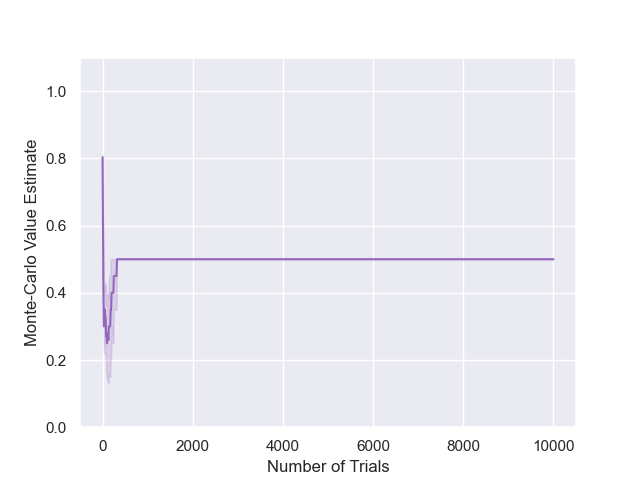}
                    \caption*{$\alpha=10,\epsilon=1$}
                \end{subfigure}
                \begin{subfigure}[b]{0.24\textwidth}
                    \centering
                    \includegraphics[width=\textwidth]{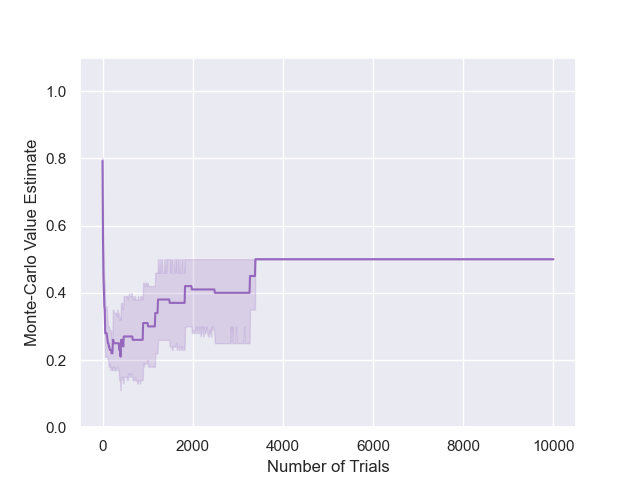}
                    \caption*{$\alpha=10,\epsilon=10$}
                \end{subfigure}
                
                \begin{subfigure}[b]{0.24\textwidth}
                    \centering
                    \includegraphics[width=\textwidth]{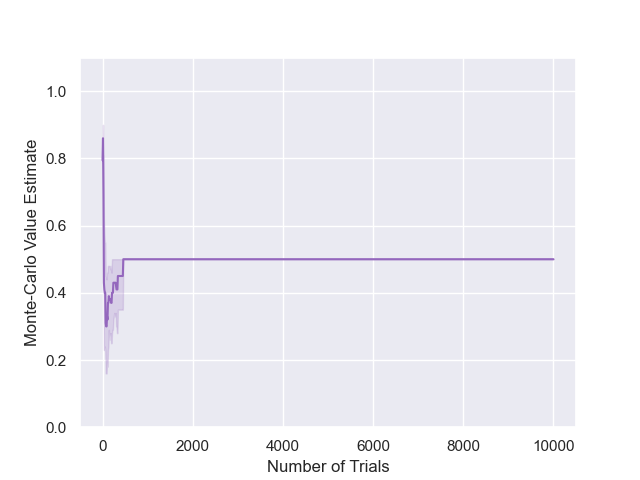}
                    \caption*{$\alpha=1,\epsilon=0.01$}
                \end{subfigure}
                \begin{subfigure}[b]{0.24\textwidth}
                    \centering
                    \includegraphics[width=\textwidth]{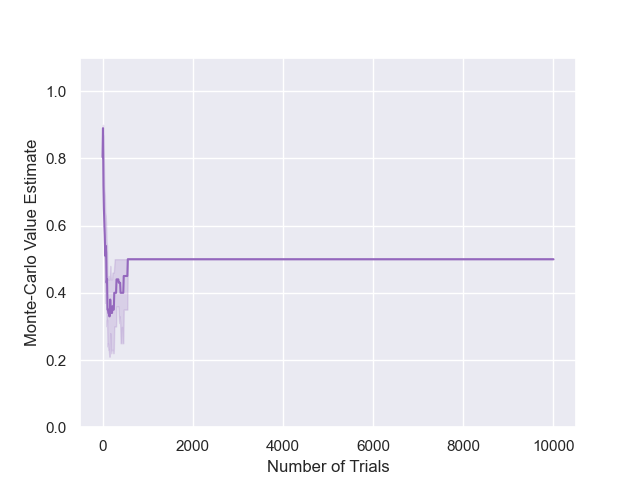}
                    \caption*{$\alpha=1,\epsilon=0.1$}
                \end{subfigure}
                \begin{subfigure}[b]{0.24\textwidth}
                    \centering
                    \includegraphics[width=\textwidth]{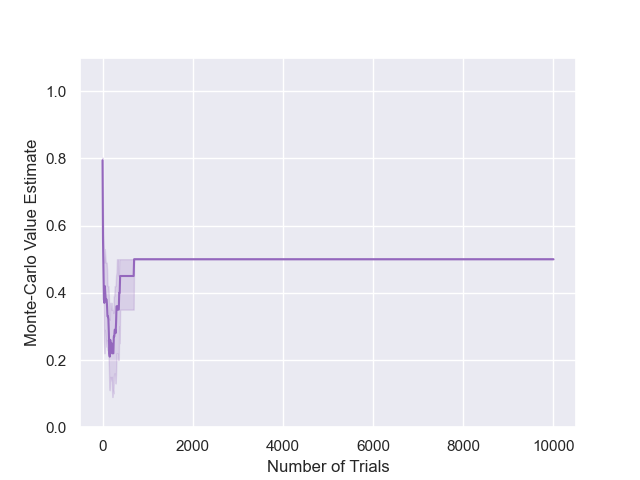}
                    \caption*{$\alpha=1,\epsilon=1$}
                \end{subfigure}
                \begin{subfigure}[b]{0.24\textwidth}
                    \centering
                    \includegraphics[width=\textwidth]{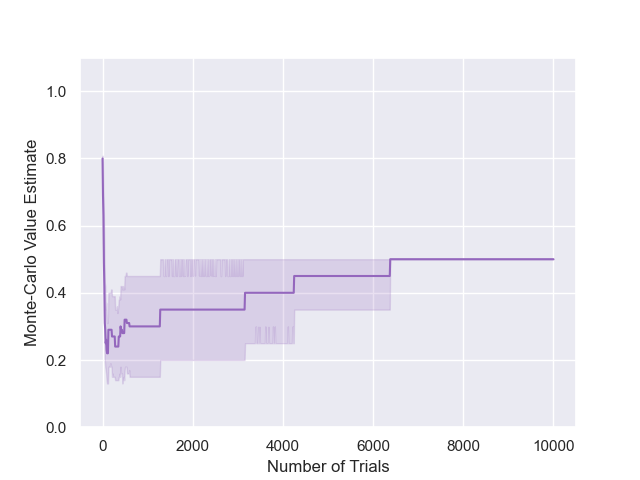}
                    \caption*{$\alpha=1,\epsilon=10$}
                \end{subfigure}
                
                \begin{subfigure}[b]{0.24\textwidth}
                    \centering
                    \includegraphics[width=\textwidth]{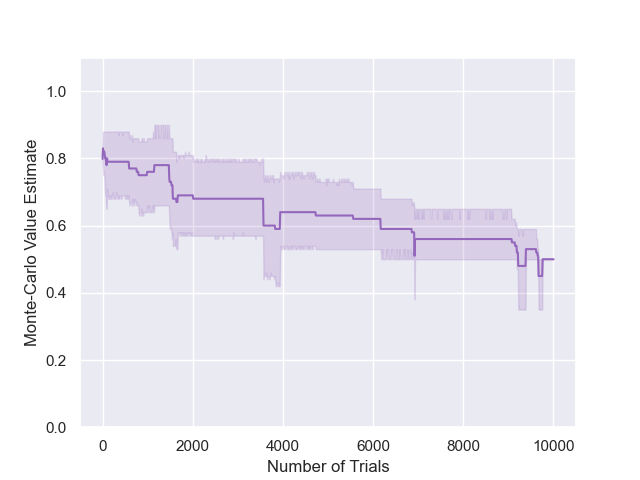}
                    \caption*{$\alpha=0.7,\epsilon=0.01$}
                \end{subfigure}
                \begin{subfigure}[b]{0.24\textwidth}
                    \centering
                    \includegraphics[width=\textwidth]{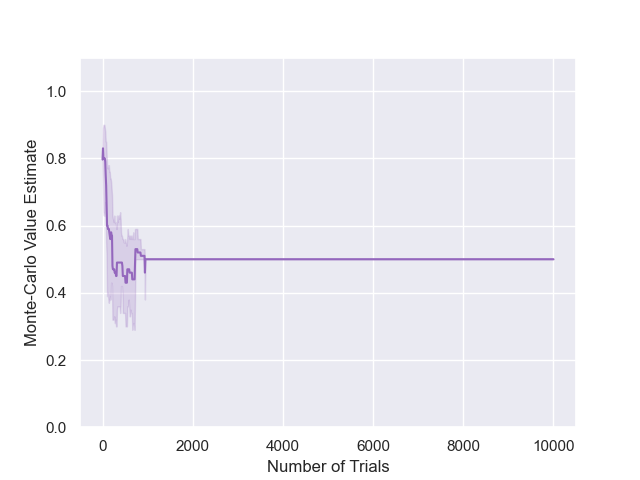}
                    \caption*{$\alpha=0.7,\epsilon=0.1$}
                \end{subfigure}
                \begin{subfigure}[b]{0.24\textwidth}
                    \centering
                    \includegraphics[width=\textwidth]{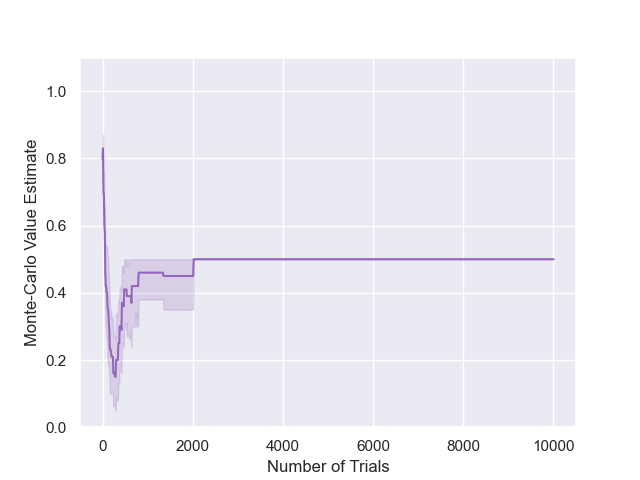}
                    \caption*{$\alpha=0.7,\epsilon=1$}
                \end{subfigure}
                \begin{subfigure}[b]{0.24\textwidth}
                    \centering
                    \includegraphics[width=\textwidth]{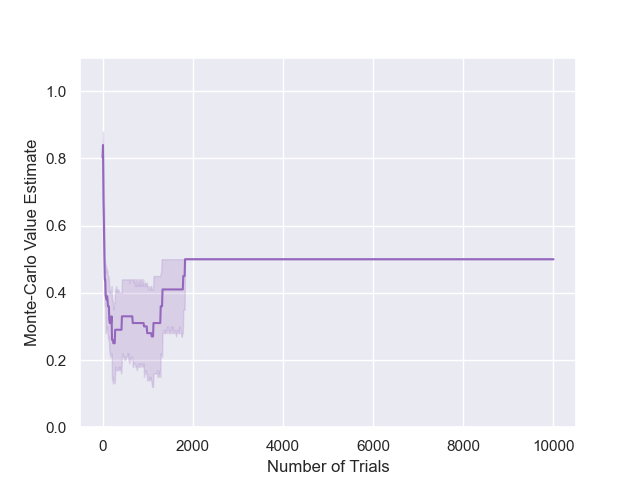}
                    \caption*{$\alpha=0.7,\epsilon=10$}
                \end{subfigure}
                
                \begin{subfigure}[b]{0.24\textwidth}
                    \centering
                    \includegraphics[width=\textwidth]{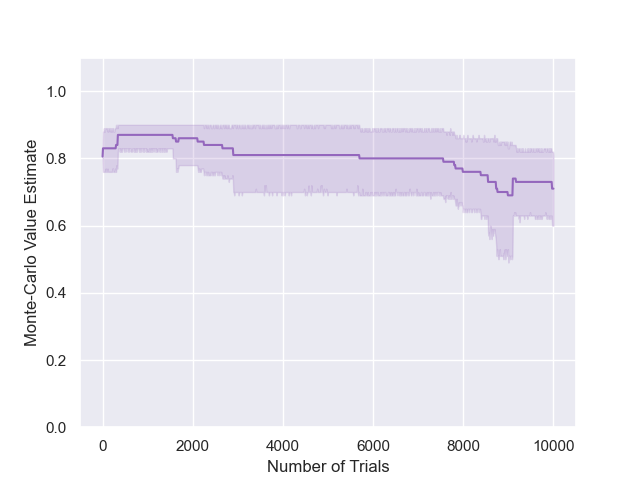}
                    \caption*{$\alpha=0.5,\epsilon=0.01$}
                \end{subfigure}
                \begin{subfigure}[b]{0.24\textwidth}
                    \centering
                    \includegraphics[width=\textwidth]{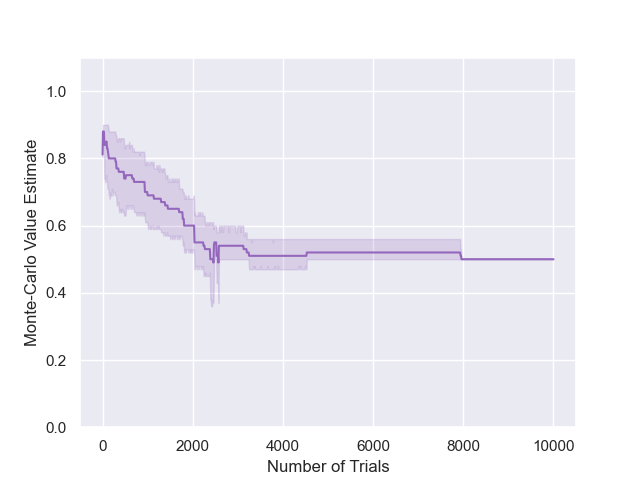}
                    \caption*{$\alpha=0.5,\epsilon=0.1$}
                \end{subfigure}
                \begin{subfigure}[b]{0.24\textwidth}
                    \centering
                    \includegraphics[width=\textwidth]{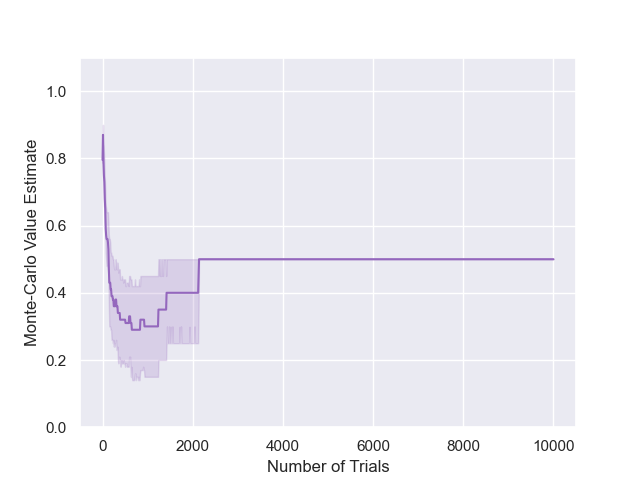}
                    \caption*{$\alpha=0.5,\epsilon=1$}
                \end{subfigure}
                \begin{subfigure}[b]{0.24\textwidth}
                    \centering
                    \includegraphics[width=\textwidth]{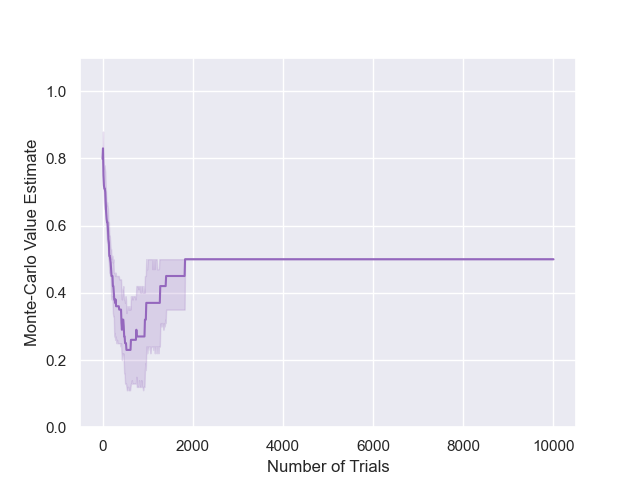}
                    \caption*{$\alpha=0.5,\epsilon=10$}
                \end{subfigure}
                
                \begin{subfigure}[b]{0.24\textwidth}
                    \centering
                    \includegraphics[width=\textwidth]{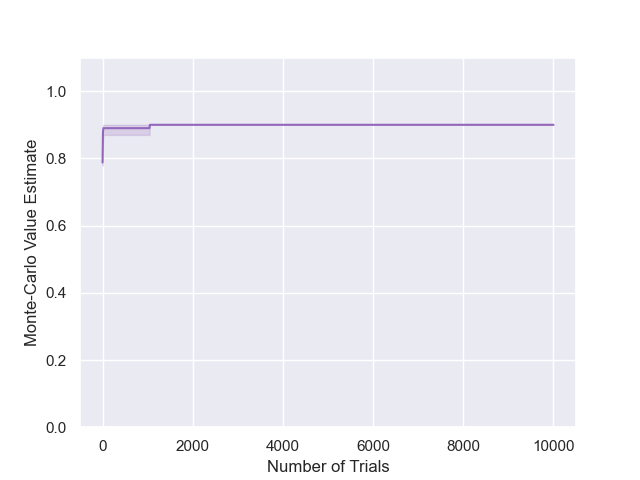}
                    \caption*{$\alpha=0.3,\epsilon=0.01$}
                \end{subfigure}
                \begin{subfigure}[b]{0.24\textwidth}
                    \centering
                    \includegraphics[width=\textwidth]{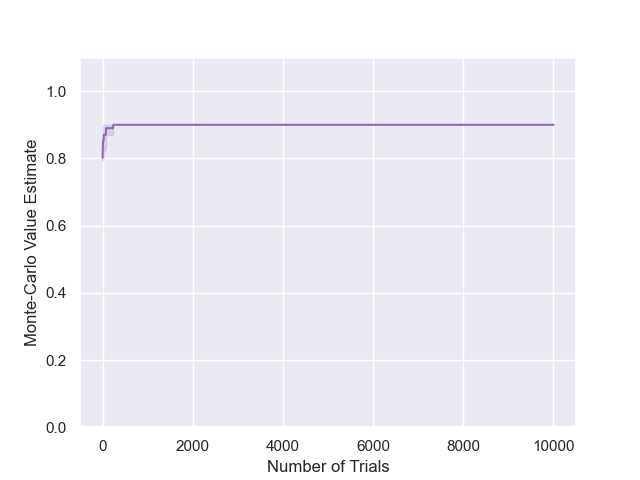}
                    \caption*{$\alpha=0.3,\epsilon=0.1$}
                \end{subfigure}
                \begin{subfigure}[b]{0.24\textwidth}
                    \centering
                    \includegraphics[width=\textwidth]{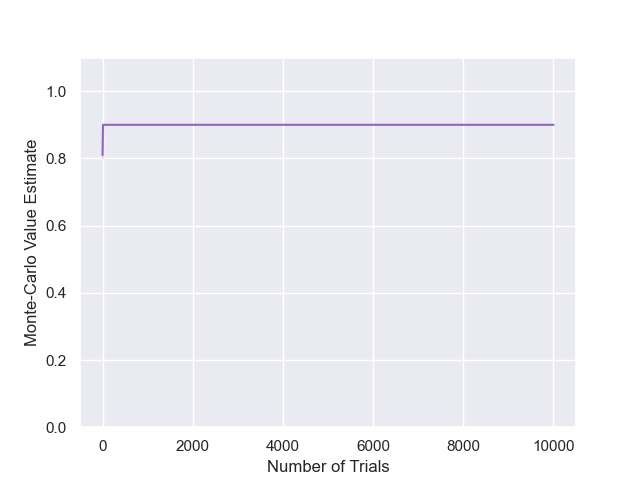}
                    \caption*{$\alpha=0.3,\epsilon=1$}
                \end{subfigure}
                \begin{subfigure}[b]{0.24\textwidth}
                    \centering
                    \includegraphics[width=\textwidth]{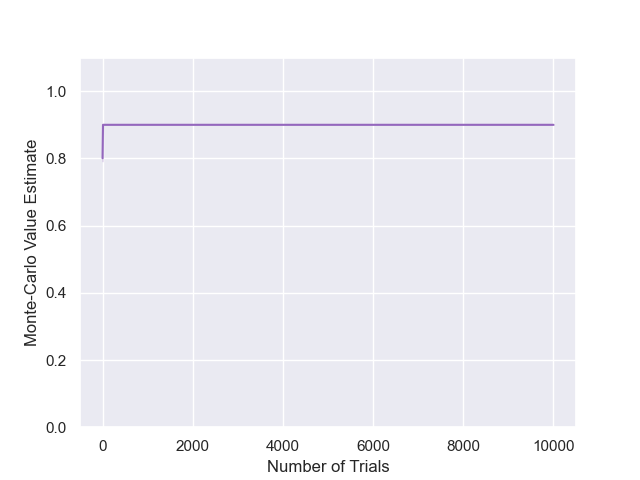}
                    \caption*{$\alpha=0.3,\epsilon=10$}
                \end{subfigure}
                
                \begin{subfigure}[b]{0.24\textwidth}
                    \centering
                    \includegraphics[width=\textwidth]{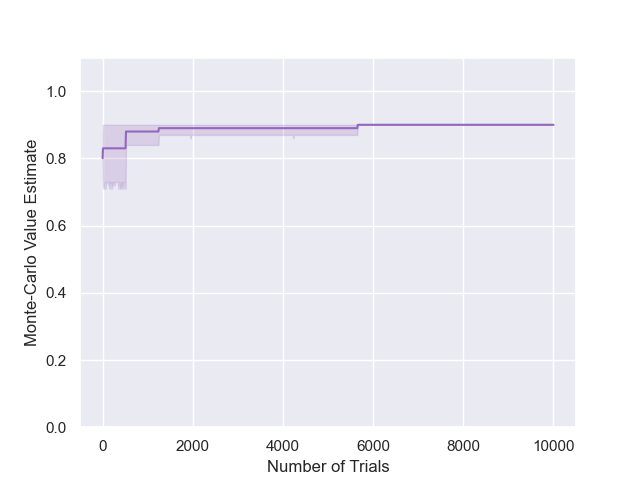}
                    \caption*{$\alpha=0.1,\epsilon=0.01$}
                \end{subfigure}
                \begin{subfigure}[b]{0.24\textwidth}
                    \centering
                    \includegraphics[width=\textwidth]{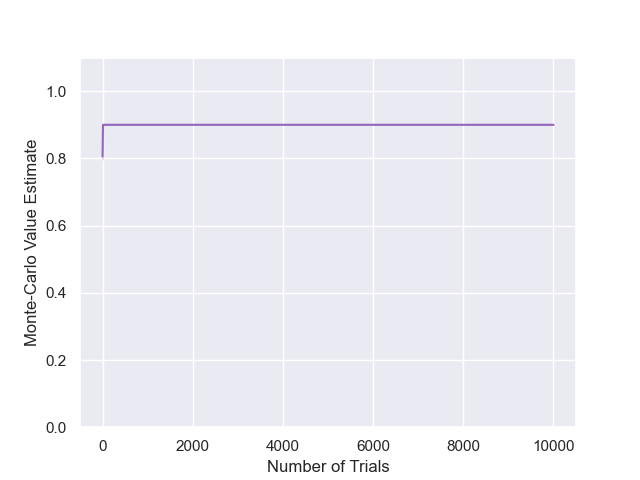}
                    \caption*{$\alpha=0.1,\epsilon=0.1$}
                \end{subfigure}
                \begin{subfigure}[b]{0.24\textwidth}
                    \centering
                    \includegraphics[width=\textwidth]{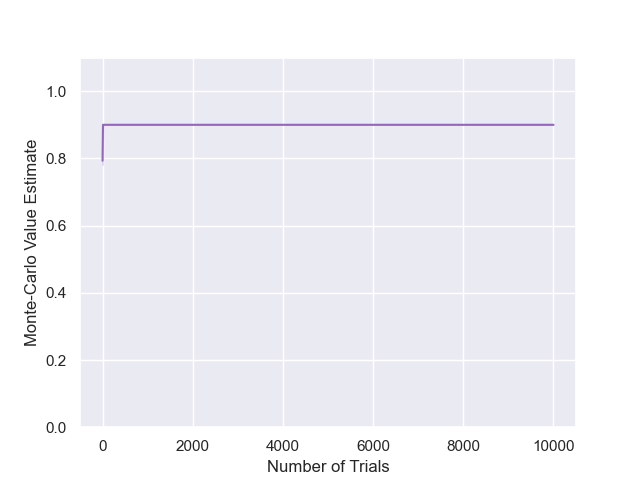}
                    \caption*{$\alpha=0.1,\epsilon=1$}
                \end{subfigure}
                \begin{subfigure}[b]{0.24\textwidth}
                    \centering
                    \includegraphics[width=\textwidth]{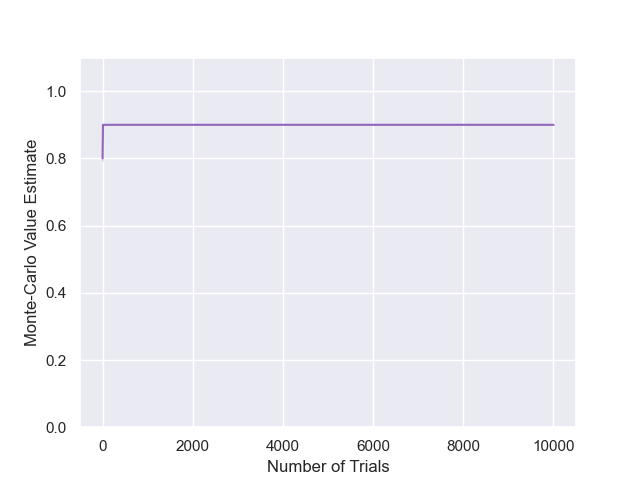}
                    \caption*{$\alpha=0.1,\epsilon=10$}
                \end{subfigure}
                
                \begin{subfigure}[b]{0.24\textwidth}
                    \centering
                    \includegraphics[width=\textwidth]{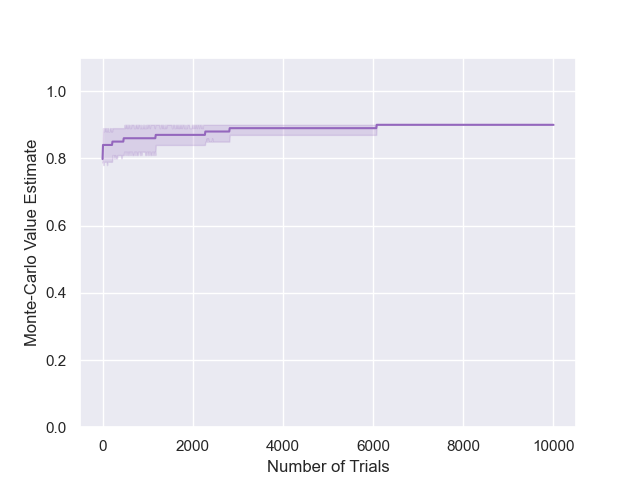}
                    \caption*{$\alpha=0.01,\epsilon=0.01$}
                \end{subfigure}
                \begin{subfigure}[b]{0.24\textwidth}
                    \centering
                    \includegraphics[width=\textwidth]{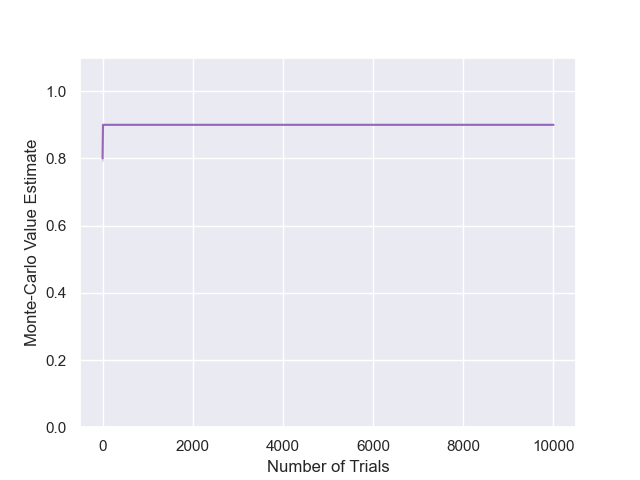}
                    \caption*{$\alpha=0.01,\epsilon=0.1$}
                \end{subfigure}
                \begin{subfigure}[b]{0.24\textwidth}
                    \centering
                    \includegraphics[width=\textwidth]{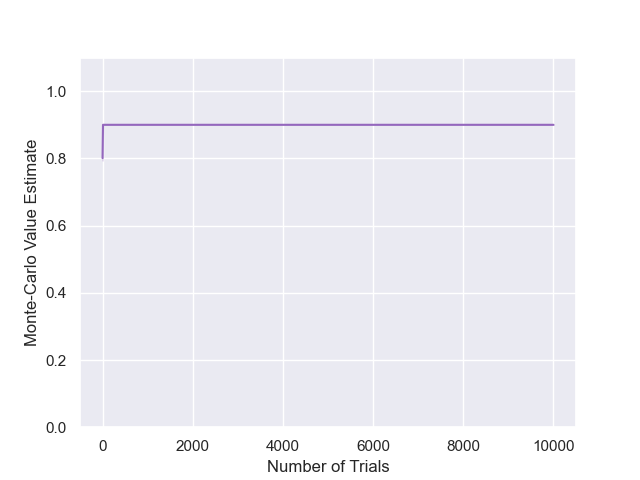}
                    \caption*{$\alpha=0.01,\epsilon=1$}
                \end{subfigure}
                \begin{subfigure}[b]{0.24\textwidth}
                    \centering
                    \includegraphics[width=\textwidth]{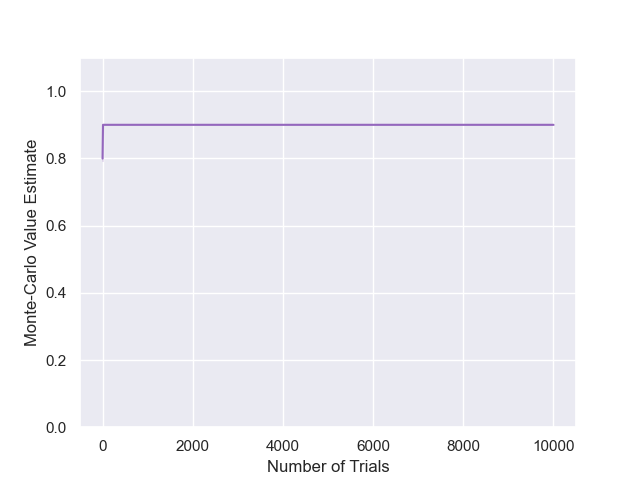}
                    \caption*{$\alpha=0.01,\epsilon=10$}
                \end{subfigure}
                
                \caption{Results for TENTS on the modified 10-chain ($D=10$, $R_f=0.5$), for varying temperatures and exploration parameters.}
                \label{fig:tents_10chain_half_hps}
            \end{figure}

            \begin{figure}
                \centering
                
                \begin{subfigure}[b]{0.24\textwidth}
                    \centering
                    \includegraphics[width=\textwidth]{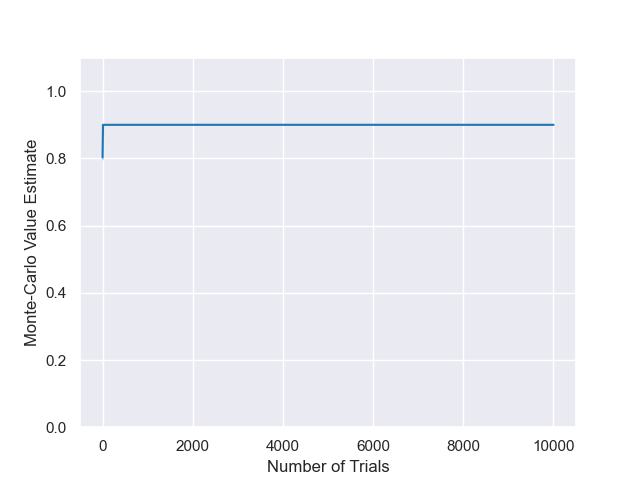}
                    \caption*{$\alpha=1,\epsilon=0.01$}
                \end{subfigure}
                \begin{subfigure}[b]{0.24\textwidth}
                    \centering
                    \includegraphics[width=\textwidth]{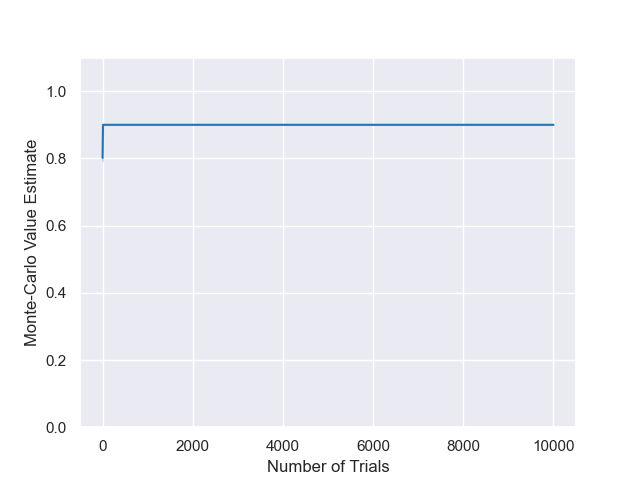}
                    \caption*{$\alpha=1,\epsilon=0.1$}
                \end{subfigure}
                \begin{subfigure}[b]{0.24\textwidth}
                    \centering
                    \includegraphics[width=\textwidth]{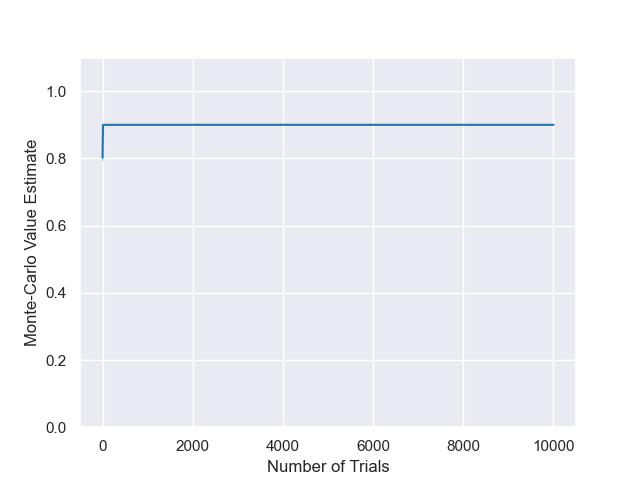}
                    \caption*{$\alpha=1,\epsilon=1$}
                \end{subfigure}
                \begin{subfigure}[b]{0.24\textwidth}
                    \centering
                    \includegraphics[width=\textwidth]{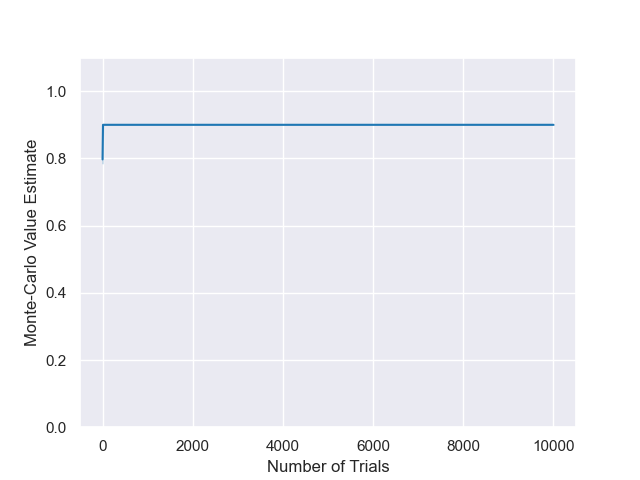}
                    \caption*{$\alpha=1,\epsilon=10$}
                \end{subfigure}
                
                \begin{subfigure}[b]{0.24\textwidth}
                    \centering
                    \includegraphics[width=\textwidth]{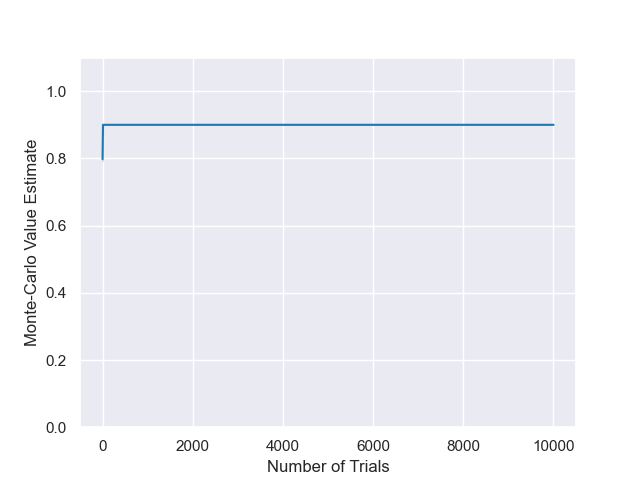}
                    \caption*{$\alpha=0.5,\epsilon=0.01$}
                \end{subfigure}
                \begin{subfigure}[b]{0.24\textwidth}
                    \centering
                    \includegraphics[width=\textwidth]{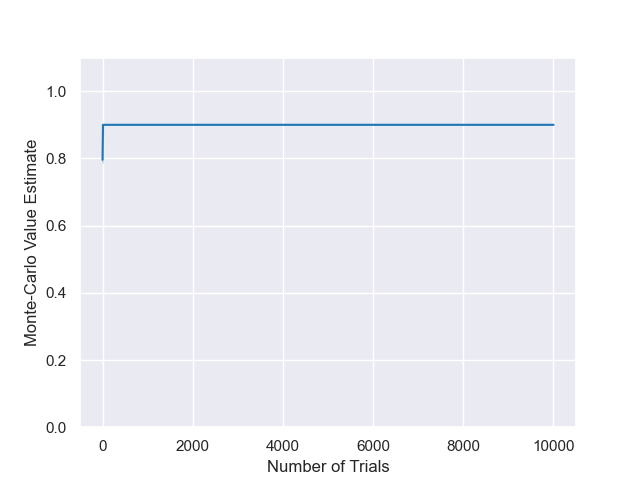}
                    \caption*{$\alpha=0.5,\epsilon=0.1$}
                \end{subfigure}
                \begin{subfigure}[b]{0.24\textwidth}
                    \centering
                    \includegraphics[width=\textwidth]{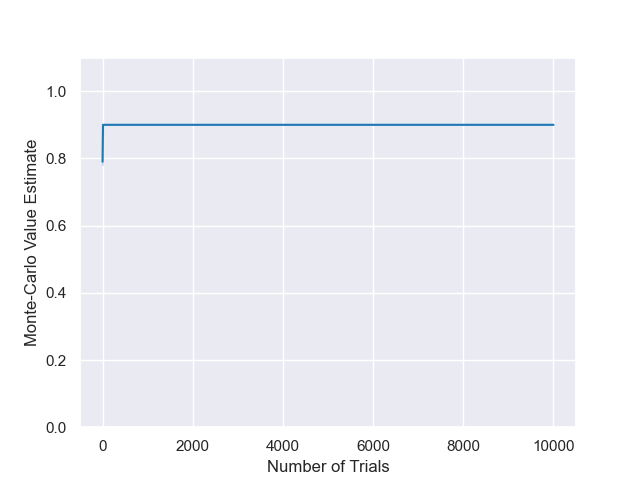}
                    \caption*{$\alpha=0.5,\epsilon=1$}
                \end{subfigure}
                \begin{subfigure}[b]{0.24\textwidth}
                    \centering
                    \includegraphics[width=\textwidth]{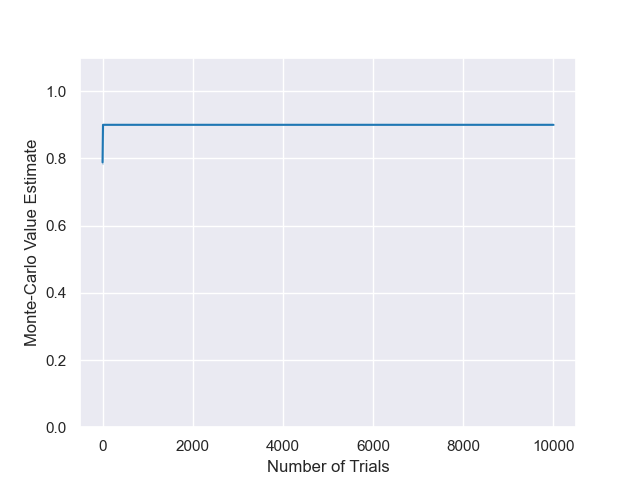}
                    \caption*{$\alpha=0.5,\epsilon=10$}
                \end{subfigure}
                
                \begin{subfigure}[b]{0.24\textwidth}
                    \centering
                    \includegraphics[width=\textwidth]{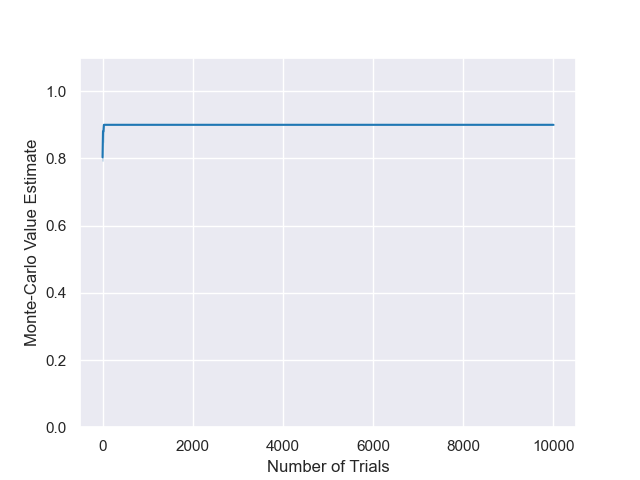}
                    \caption*{$\alpha=0.2,\epsilon=0.01$}
                \end{subfigure}
                \begin{subfigure}[b]{0.24\textwidth}
                    \centering
                    \includegraphics[width=\textwidth]{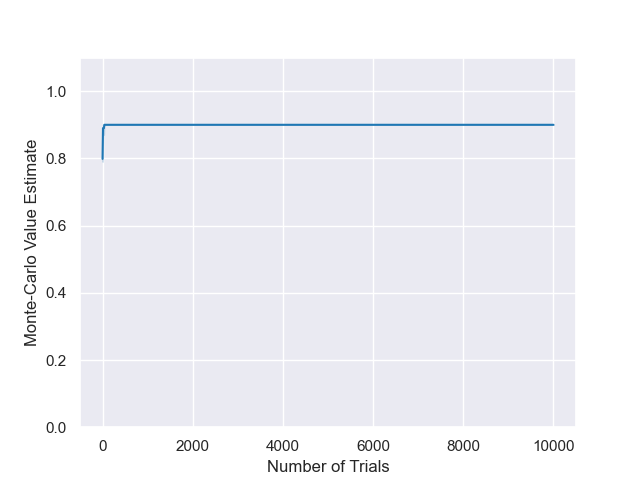}
                    \caption*{$\alpha=0.2,\epsilon=0.1$}
                \end{subfigure}
                \begin{subfigure}[b]{0.24\textwidth}
                    \centering
                    \includegraphics[width=\textwidth]{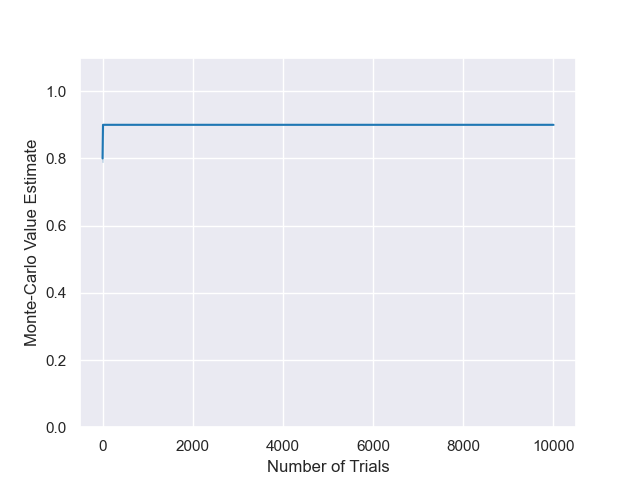}
                    \caption*{$\alpha=0.2,\epsilon=1$}
                \end{subfigure}
                \begin{subfigure}[b]{0.24\textwidth}
                    \centering
                    \includegraphics[width=\textwidth]{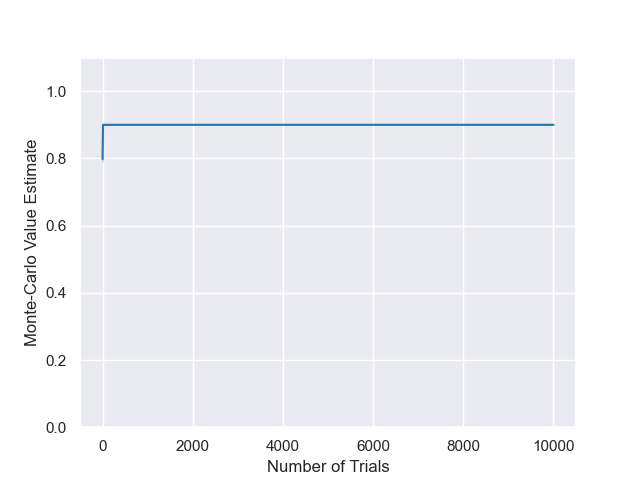}
                    \caption*{$\alpha=0.2,\epsilon=10$}
                \end{subfigure}
                
                \begin{subfigure}[b]{0.24\textwidth}
                    \centering
                    \includegraphics[width=\textwidth]{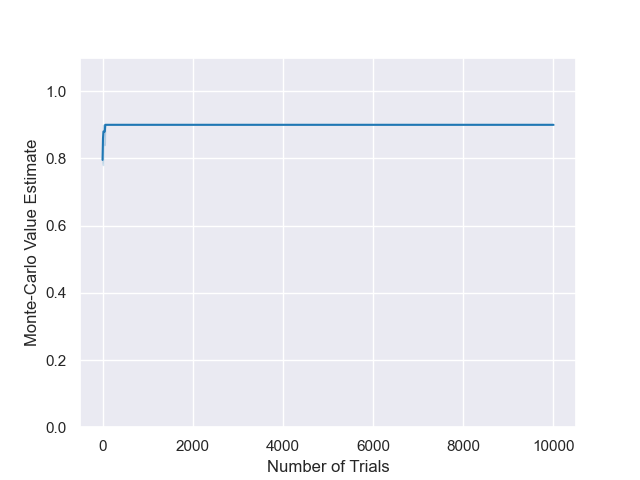}
                    \caption*{$\alpha=0.15,\epsilon=0.01$}
                \end{subfigure}
                \begin{subfigure}[b]{0.24\textwidth}
                    \centering
                    \includegraphics[width=\textwidth]{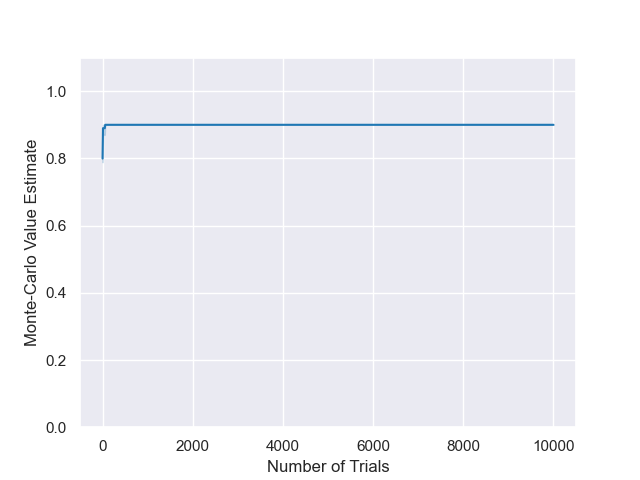}
                    \caption*{$\alpha=0.15,\epsilon=0.1$}
                \end{subfigure}
                \begin{subfigure}[b]{0.24\textwidth}
                    \centering
                    \includegraphics[width=\textwidth]{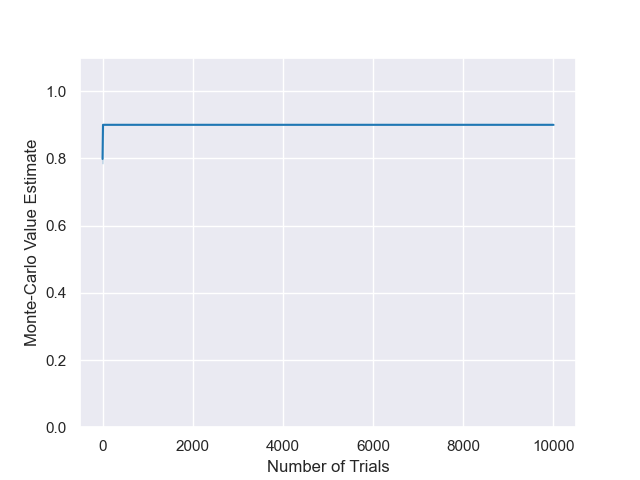}
                    \caption*{$\alpha=0.15,\epsilon=1$}
                \end{subfigure}
                \begin{subfigure}[b]{0.24\textwidth}
                    \centering
                    \includegraphics[width=\textwidth]{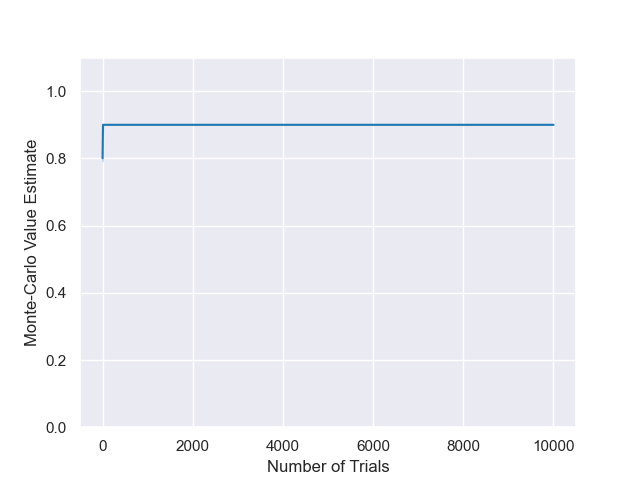}
                    \caption*{$\alpha=0.15,\epsilon=10$}
                \end{subfigure}
                
                \begin{subfigure}[b]{0.24\textwidth}
                    \centering
                    \includegraphics[width=\textwidth]{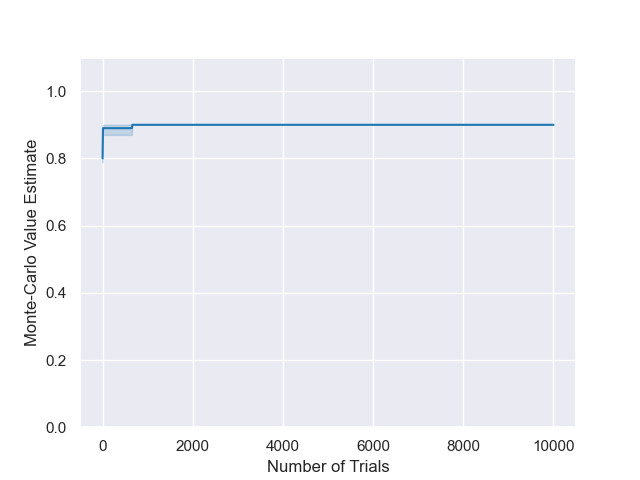}
                    \caption*{$\alpha=0.1,\epsilon=0.01$}
                \end{subfigure}
                \begin{subfigure}[b]{0.24\textwidth}
                    \centering
                    \includegraphics[width=\textwidth]{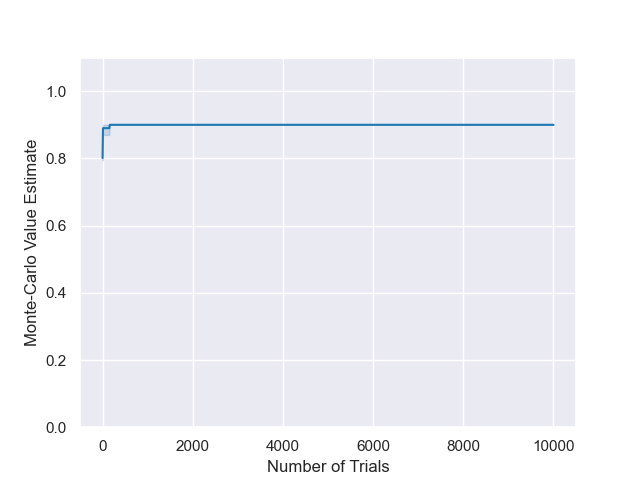}
                    \caption*{$\alpha=0.1,\epsilon=0.1$}
                \end{subfigure}
                \begin{subfigure}[b]{0.24\textwidth}
                    \centering
                    \includegraphics[width=\textwidth]{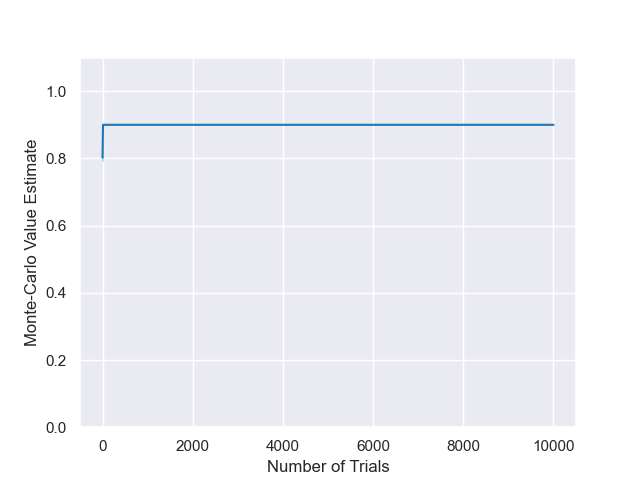}
                    \caption*{$\alpha=0.1,\epsilon=1$}
                \end{subfigure}
                \begin{subfigure}[b]{0.24\textwidth}
                    \centering
                    \includegraphics[width=\textwidth]{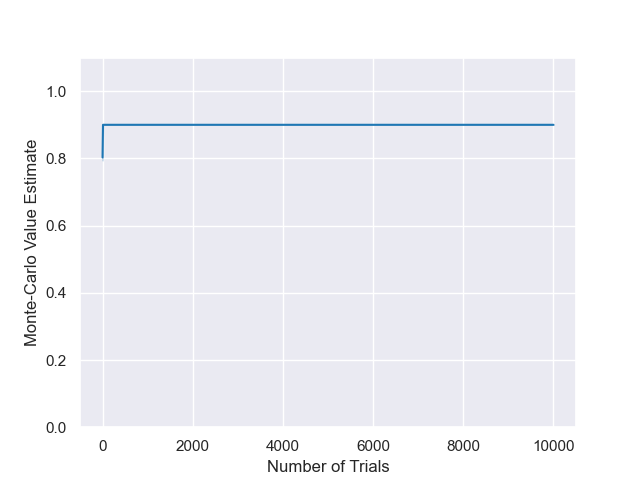}
                    \caption*{$\alpha=0.1,\epsilon=10$}
                \end{subfigure}
                
                \begin{subfigure}[b]{0.24\textwidth}
                    \centering
                    \includegraphics[width=\textwidth]{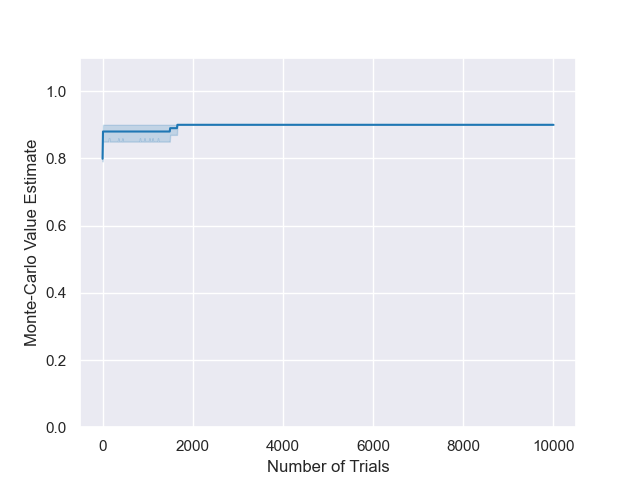}
                    \caption*{$\alpha=0.05,\epsilon=0.01$}
                \end{subfigure}
                \begin{subfigure}[b]{0.24\textwidth}
                    \centering
                    \includegraphics[width=\textwidth]{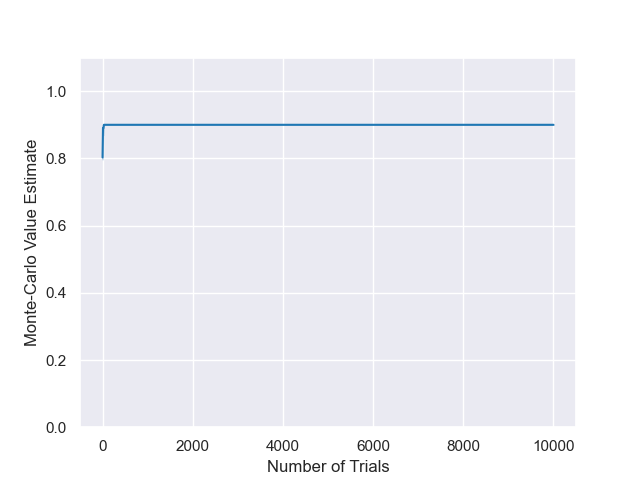}
                    \caption*{$\alpha=0.05,\epsilon=0.1$}
                \end{subfigure}
                \begin{subfigure}[b]{0.24\textwidth}
                    \centering
                    \includegraphics[width=\textwidth]{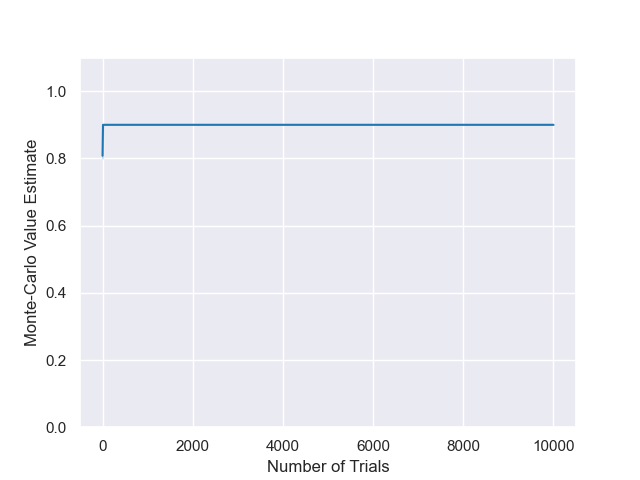}
                    \caption*{$\alpha=0.05,\epsilon=1$}
                \end{subfigure}
                \begin{subfigure}[b]{0.24\textwidth}
                    \centering
                    \includegraphics[width=\textwidth]{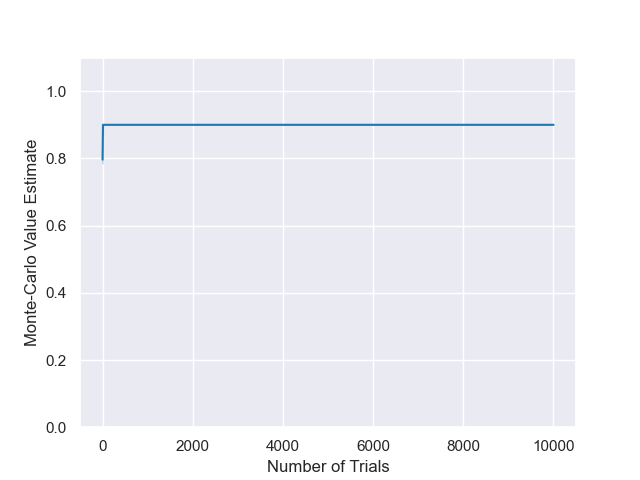}
                    \caption*{$\alpha=0.05,\epsilon=10$}
                \end{subfigure}
                
                \begin{subfigure}[b]{0.24\textwidth}
                    \centering
                    \includegraphics[width=\textwidth]{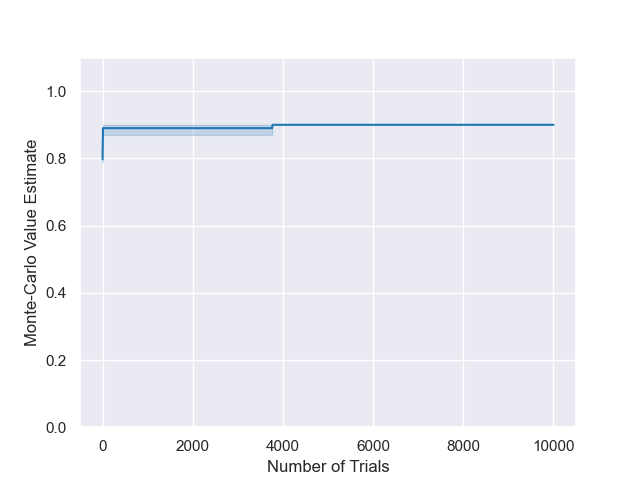}
                    \caption*{$\alpha=0.01,\epsilon=0.01$}
                \end{subfigure}
                \begin{subfigure}[b]{0.24\textwidth}
                    \centering
                    \includegraphics[width=\textwidth]{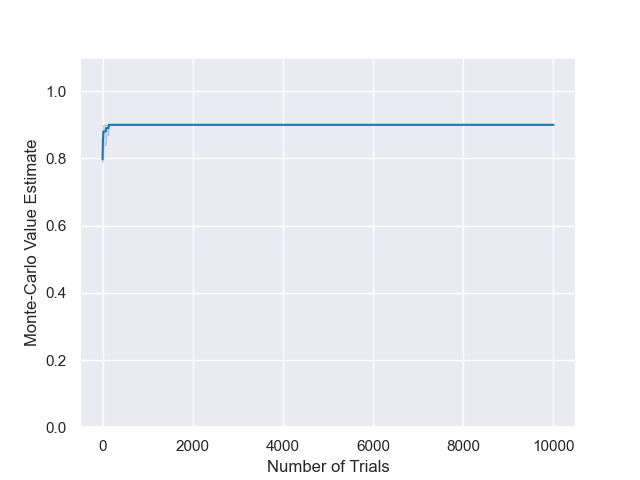}
                    \caption*{$\alpha=0.01,\epsilon=0.1$}
                \end{subfigure}
                \begin{subfigure}[b]{0.24\textwidth}
                    \centering
                    \includegraphics[width=\textwidth]{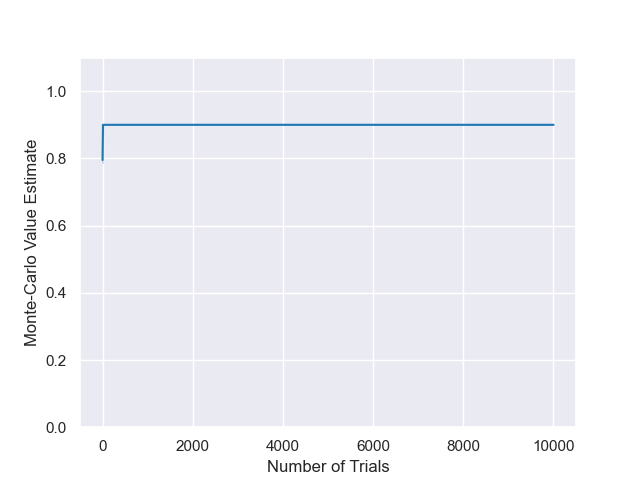}
                    \caption*{$\alpha=0.01,\epsilon=1$}
                \end{subfigure}
                \begin{subfigure}[b]{0.24\textwidth}
                    \centering
                    \includegraphics[width=\textwidth]{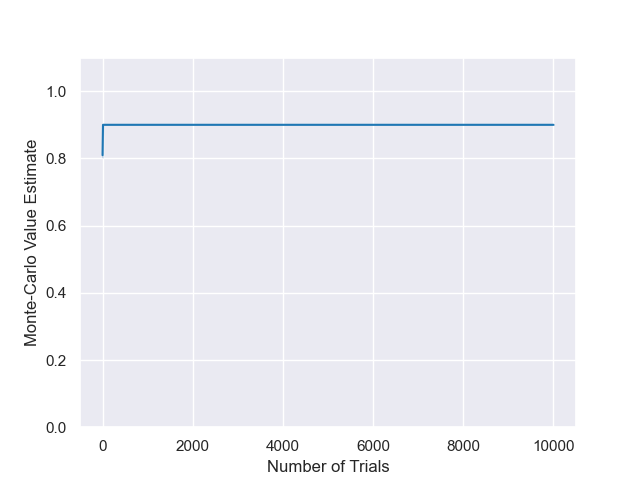}
                    \caption*{$\alpha=0.01,\epsilon=10$}
                \end{subfigure}
                
                \caption{Results for BTS on the modified 10-chain ($D=10$, $R_f=0.5$), for varying temperatures and exploration parameters.}
                \label{fig:bts_10chain_half_hps}
            \end{figure}

            \begin{figure}
                \centering
                
                \begin{subfigure}[b]{0.24\textwidth}
                    \centering
                    \includegraphics[width=\textwidth]{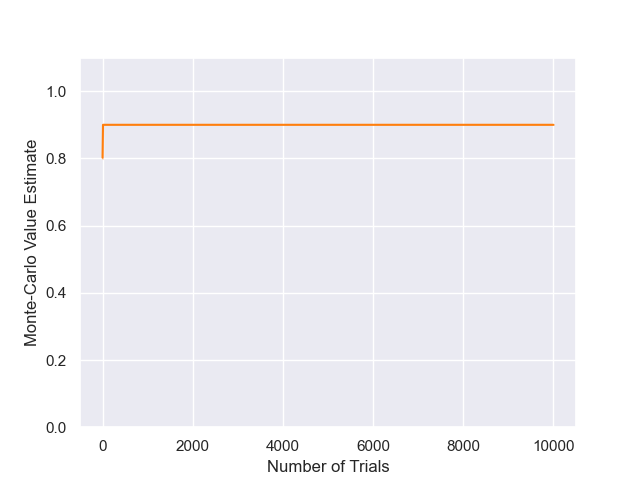}
                    \caption*{$\alpha=1,\epsilon=0.01$}
                \end{subfigure}
                \begin{subfigure}[b]{0.24\textwidth}
                    \centering
                    \includegraphics[width=\textwidth]{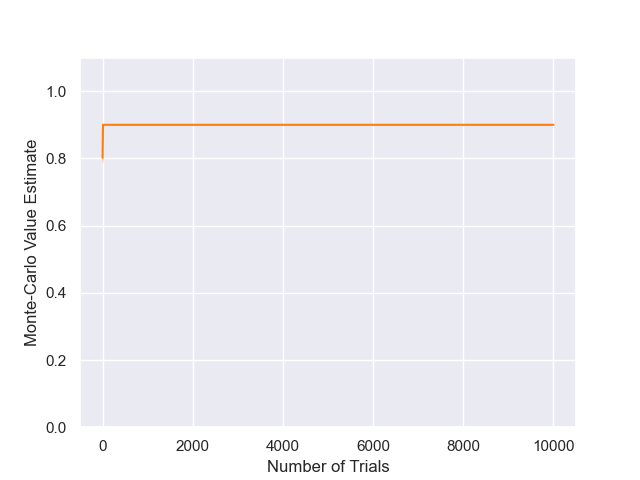}
                    \caption*{$\alpha=1,\epsilon=0.1$}
                \end{subfigure}
                \begin{subfigure}[b]{0.24\textwidth}
                    \centering
                    \includegraphics[width=\textwidth]{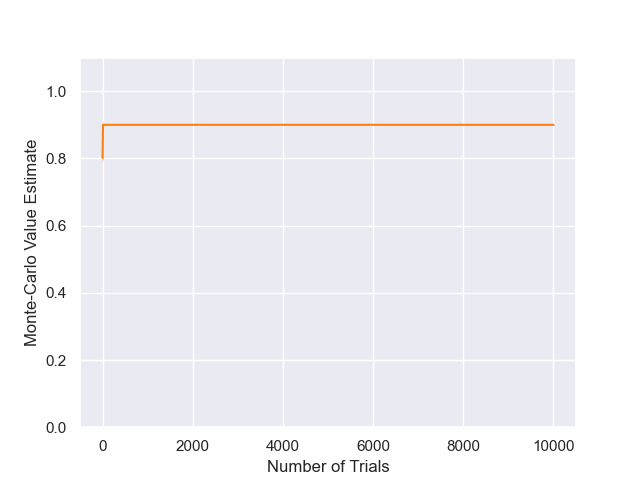}
                    \caption*{$\alpha=1,\epsilon=1$}
                \end{subfigure}
                \begin{subfigure}[b]{0.24\textwidth}
                    \centering
                    \includegraphics[width=\textwidth]{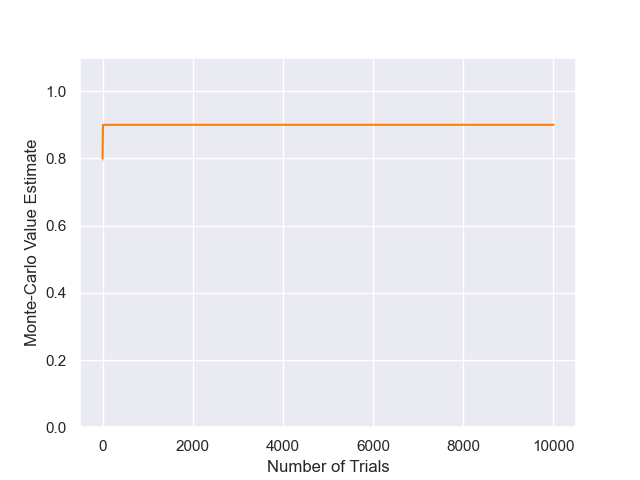}
                    \caption*{$\alpha=1,\epsilon=10$}
                \end{subfigure}
                
                \begin{subfigure}[b]{0.24\textwidth}
                    \centering
                    \includegraphics[width=\textwidth]{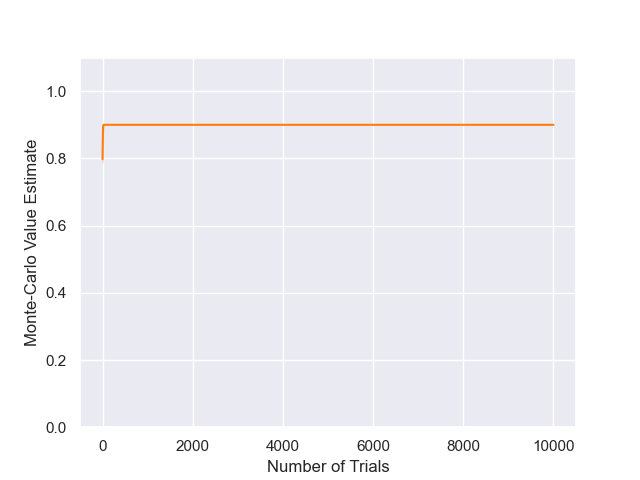}
                    \caption*{$\alpha=0.5,\epsilon=0.01$}
                \end{subfigure}
                \begin{subfigure}[b]{0.24\textwidth}
                    \centering
                    \includegraphics[width=\textwidth]{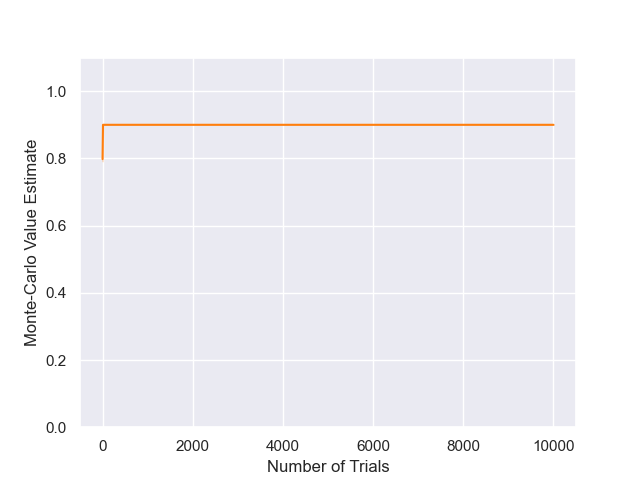}
                    \caption*{$\alpha=0.5,\epsilon=0.1$}
                \end{subfigure}
                \begin{subfigure}[b]{0.24\textwidth}
                    \centering
                    \includegraphics[width=\textwidth]{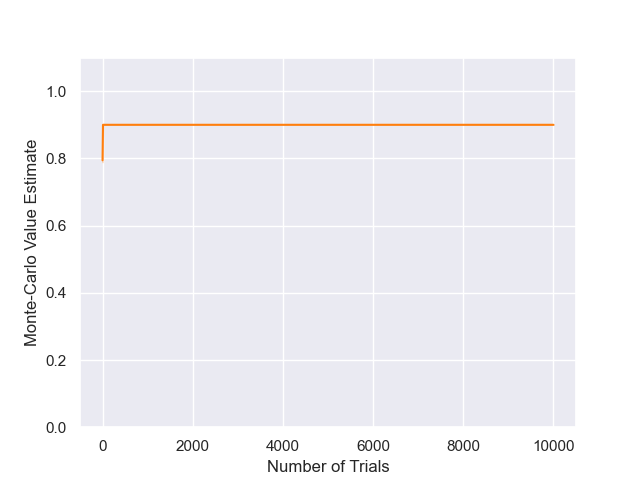}
                    \caption*{$\alpha=0.5,\epsilon=1$}
                \end{subfigure}
                \begin{subfigure}[b]{0.24\textwidth}
                    \centering
                    \includegraphics[width=\textwidth]{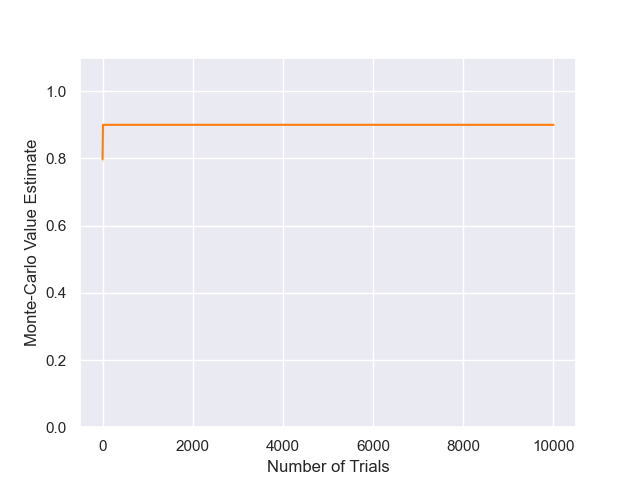}
                    \caption*{$\alpha=0.5,\epsilon=10$}
                \end{subfigure}
                
                \begin{subfigure}[b]{0.24\textwidth}
                    \centering
                    \includegraphics[width=\textwidth]{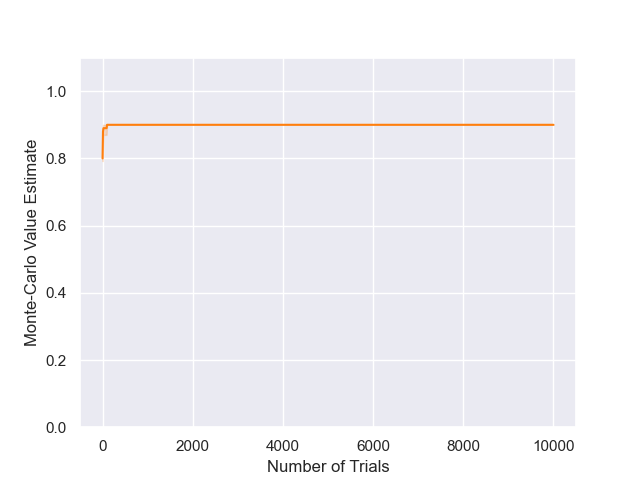}
                    \caption*{$\alpha=0.2,\epsilon=0.01$}
                \end{subfigure}
                \begin{subfigure}[b]{0.24\textwidth}
                    \centering
                    \includegraphics[width=\textwidth]{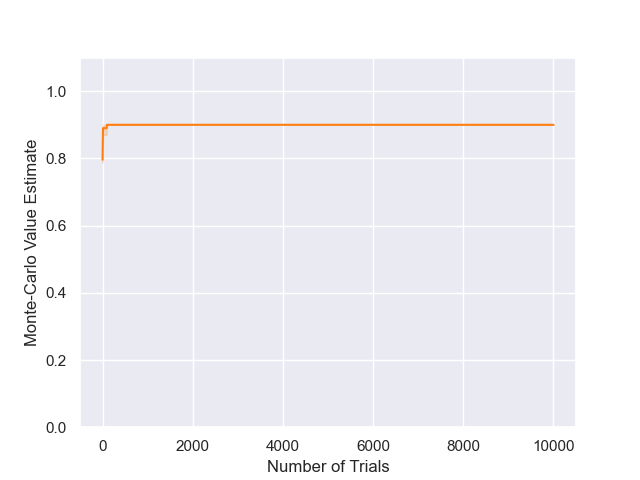}
                    \caption*{$\alpha=0.2,\epsilon=0.1$}
                \end{subfigure}
                \begin{subfigure}[b]{0.24\textwidth}
                    \centering
                    \includegraphics[width=\textwidth]{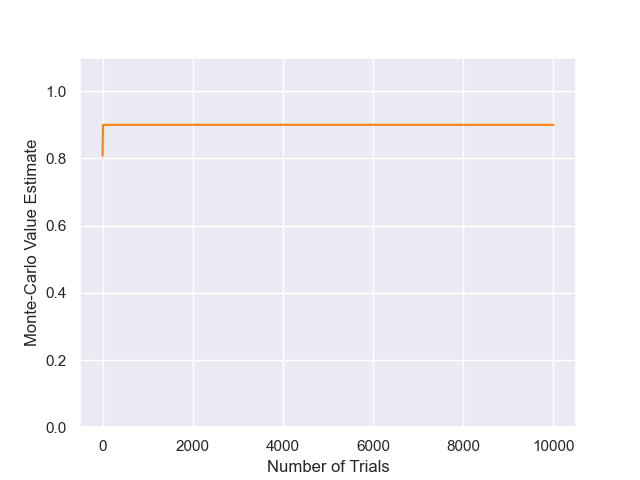}
                    \caption*{$\alpha=0.2,\epsilon=1$}
                \end{subfigure}
                \begin{subfigure}[b]{0.24\textwidth}
                    \centering
                    \includegraphics[width=\textwidth]{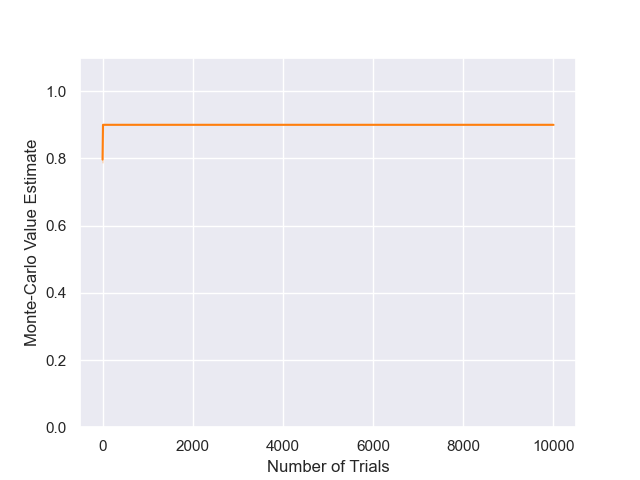}
                    \caption*{$\alpha=0.2,\epsilon=10$}
                \end{subfigure}
                
                \begin{subfigure}[b]{0.24\textwidth}
                    \centering
                    \includegraphics[width=\textwidth]{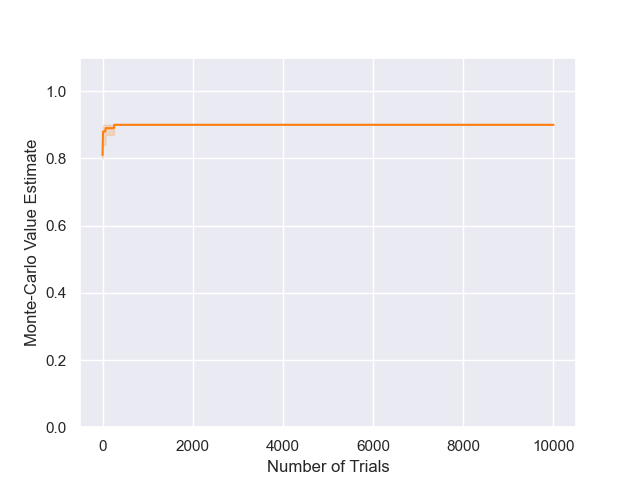}
                    \caption*{$\alpha=0.15,\epsilon=0.01$}
                \end{subfigure}
                \begin{subfigure}[b]{0.24\textwidth}
                    \centering
                    \includegraphics[width=\textwidth]{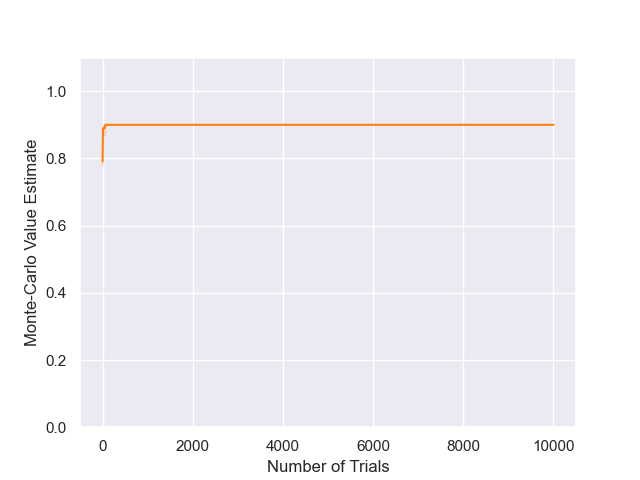}
                    \caption*{$\alpha=0.15,\epsilon=0.1$}
                \end{subfigure}
                \begin{subfigure}[b]{0.24\textwidth}
                    \centering
                    \includegraphics[width=\textwidth]{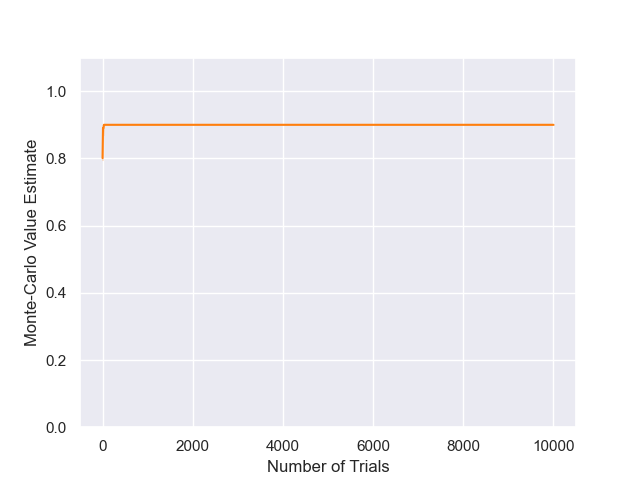}
                    \caption*{$\alpha=0.15,\epsilon=1$}
                \end{subfigure}
                \begin{subfigure}[b]{0.24\textwidth}
                    \centering
                    \includegraphics[width=\textwidth]{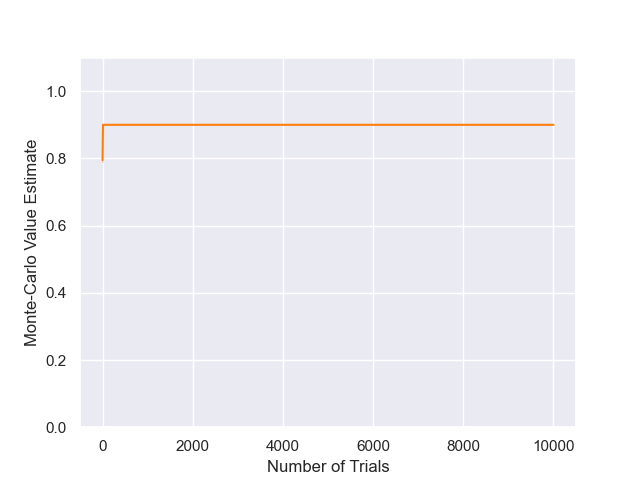}
                    \caption*{$\alpha=0.15,\epsilon=10$}
                \end{subfigure}
                
                \begin{subfigure}[b]{0.24\textwidth}
                    \centering
                    \includegraphics[width=\textwidth]{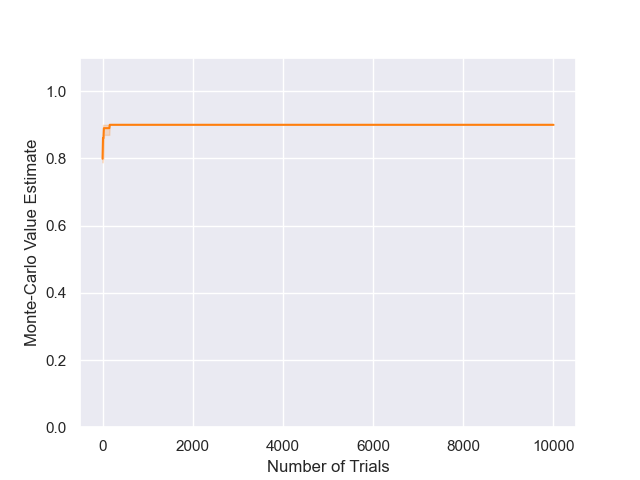}
                    \caption*{$\alpha=0.1,\epsilon=0.01$}
                \end{subfigure}
                \begin{subfigure}[b]{0.24\textwidth}
                    \centering
                    \includegraphics[width=\textwidth]{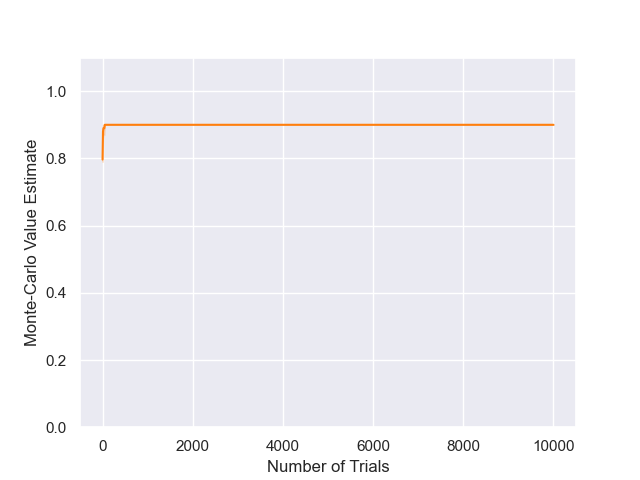}
                    \caption*{$\alpha=0.1,\epsilon=0.1$}
                \end{subfigure}
                \begin{subfigure}[b]{0.24\textwidth}
                    \centering
                    \includegraphics[width=\textwidth]{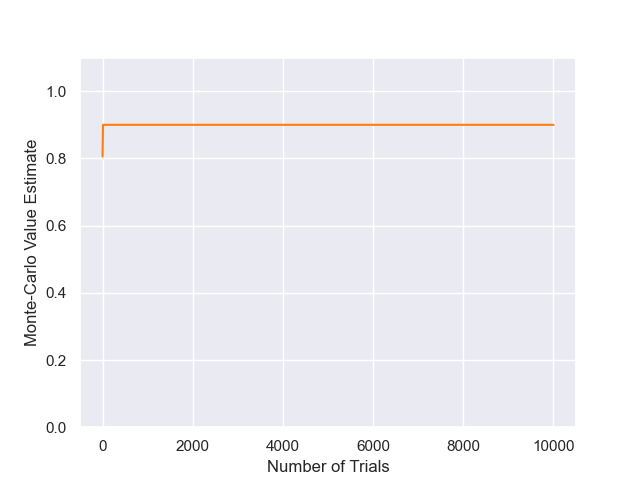}
                    \caption*{$\alpha=0.1,\epsilon=1$}
                \end{subfigure}
                \begin{subfigure}[b]{0.24\textwidth}
                    \centering
                    \includegraphics[width=\textwidth]{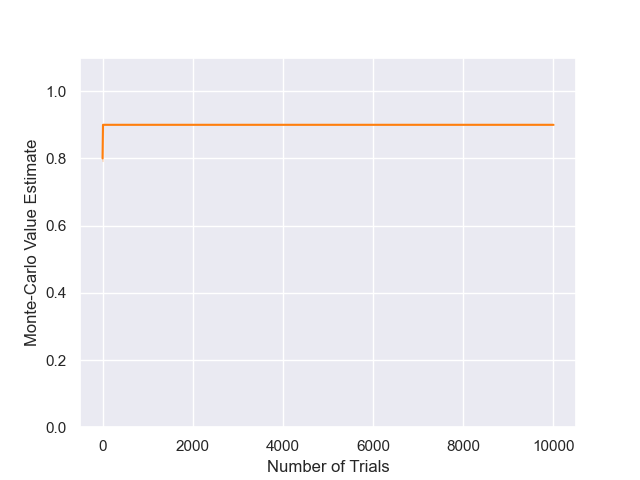}
                    \caption*{$\alpha=0.1,\epsilon=10$}
                \end{subfigure}
                
                \begin{subfigure}[b]{0.24\textwidth}
                    \centering
                    \includegraphics[width=\textwidth]{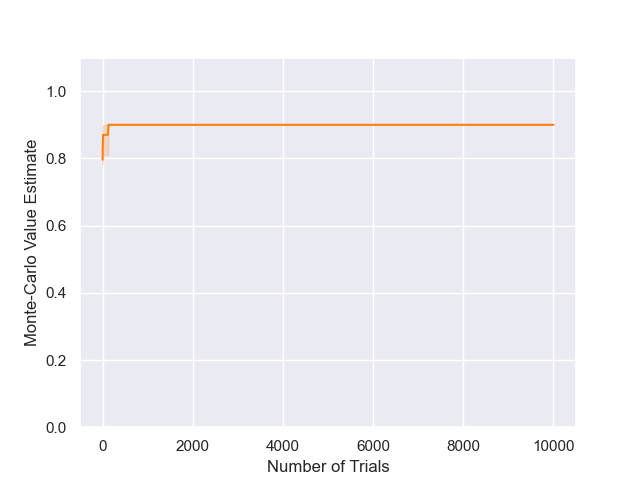}
                    \caption*{$\alpha=0.05,\epsilon=0.01$}
                \end{subfigure}
                \begin{subfigure}[b]{0.24\textwidth}
                    \centering
                    \includegraphics[width=\textwidth]{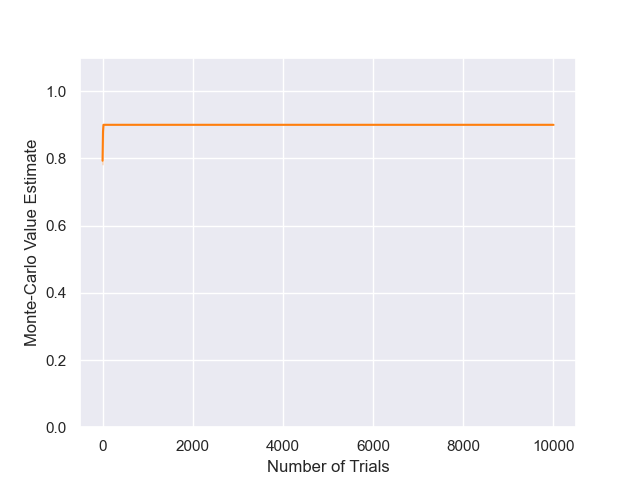}
                    \caption*{$\alpha=0.05,\epsilon=0.1$}
                \end{subfigure}
                \begin{subfigure}[b]{0.24\textwidth}
                    \centering
                    \includegraphics[width=\textwidth]{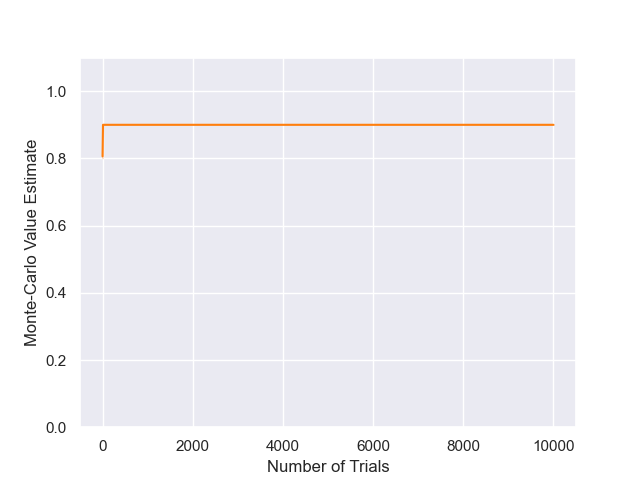}
                    \caption*{$\alpha=0.05,\epsilon=1$}
                \end{subfigure}
                \begin{subfigure}[b]{0.24\textwidth}
                    \centering
                    \includegraphics[width=\textwidth]{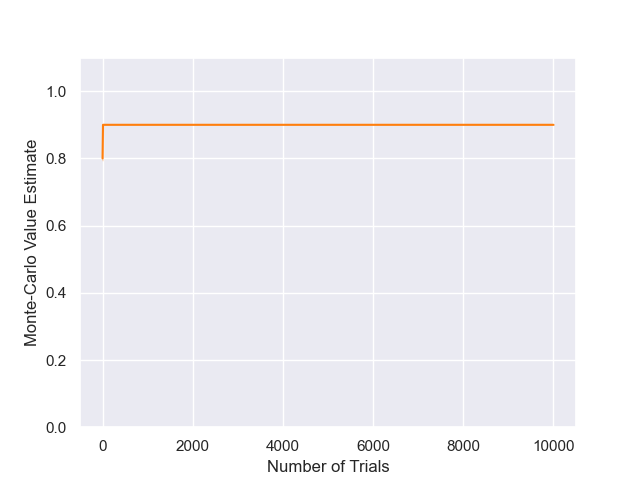}
                    \caption*{$\alpha=0.05,\epsilon=10$}
                \end{subfigure}
                
                \begin{subfigure}[b]{0.24\textwidth}
                    \centering
                    \includegraphics[width=\textwidth]{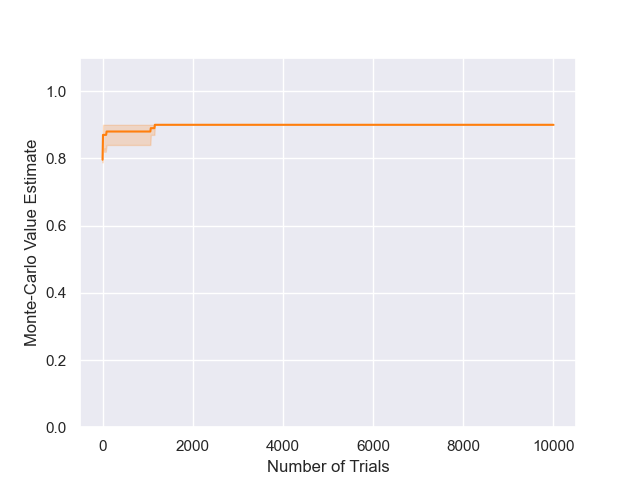}
                    \caption*{$\alpha=0.01,\epsilon=0.01$}
                \end{subfigure}
                \begin{subfigure}[b]{0.24\textwidth}
                    \centering
                    \includegraphics[width=\textwidth]{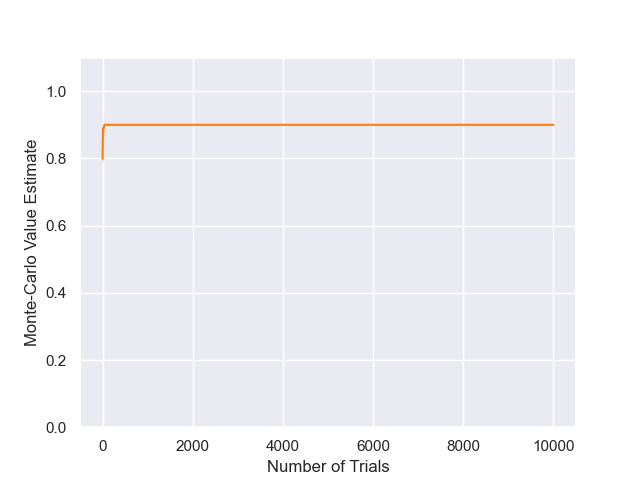}
                    \caption*{$\alpha=0.01,\epsilon=0.1$}
                \end{subfigure}
                \begin{subfigure}[b]{0.24\textwidth}
                    \centering
                    \includegraphics[width=\textwidth]{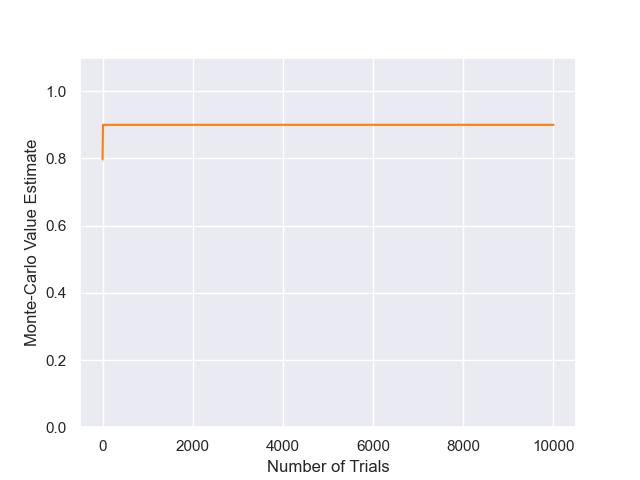}
                    \caption*{$\alpha=0.01,\epsilon=1$}
                \end{subfigure}
                \begin{subfigure}[b]{0.24\textwidth}
                    \centering
                    \includegraphics[width=\textwidth]{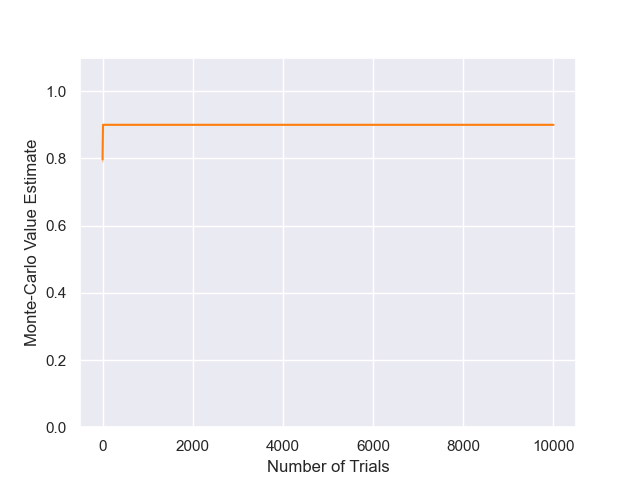}
                    \caption*{$\alpha=0.01,\epsilon=10$}
                \end{subfigure}
                
                \caption{Results for DENTS on the modified 10-chain ($D=10$, $R_f=0.5$), for varying temperatures and exploration parameters. The decay function was set to $\beta(m)=\alpha/\log(e+m)$.}
                \label{fig:dents_10chain_half_hps}
            \end{figure}

            \begin{figure}
                \centering
                
                \begin{subfigure}[b]{0.24\textwidth}
                    \centering
                    \includegraphics[width=\textwidth]{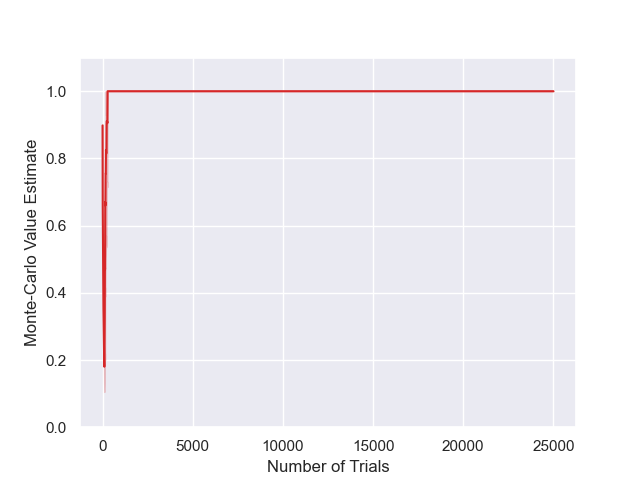}
                    \caption*{$\alpha=10,\epsilon=0.01$}
                \end{subfigure}
                \begin{subfigure}[b]{0.24\textwidth}
                    \centering
                    \includegraphics[width=\textwidth]{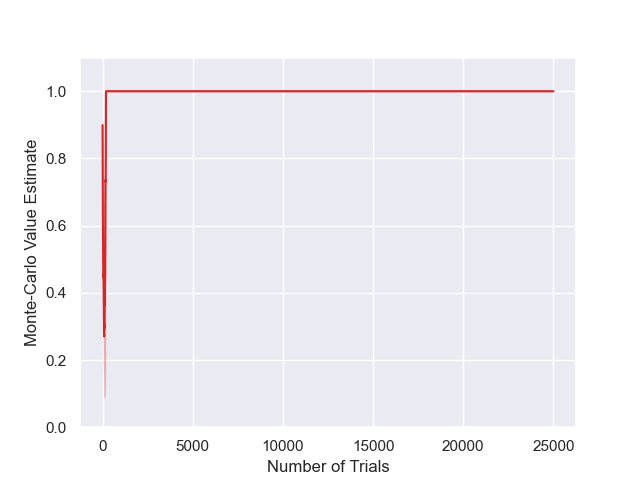}
                    \caption*{$\alpha=10,\epsilon=0.1$}
                \end{subfigure}
                \begin{subfigure}[b]{0.24\textwidth}
                    \centering
                    \includegraphics[width=\textwidth]{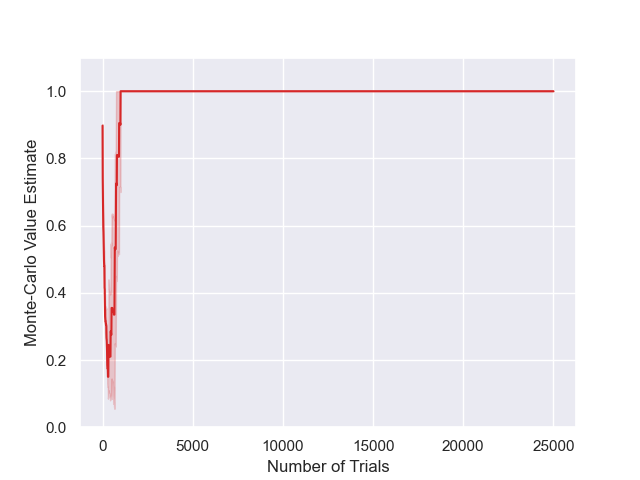}
                    \caption*{$\alpha=10,\epsilon=1$}
                \end{subfigure}
                \begin{subfigure}[b]{0.24\textwidth}
                    \centering
                    \includegraphics[width=\textwidth]{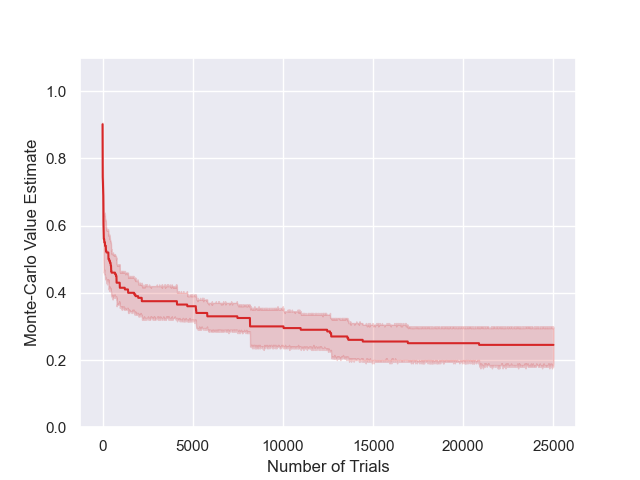}
                    \caption*{$\alpha=10,\epsilon=10$}
                \end{subfigure}
                
                \begin{subfigure}[b]{0.24\textwidth}
                    \centering
                    \includegraphics[width=\textwidth]{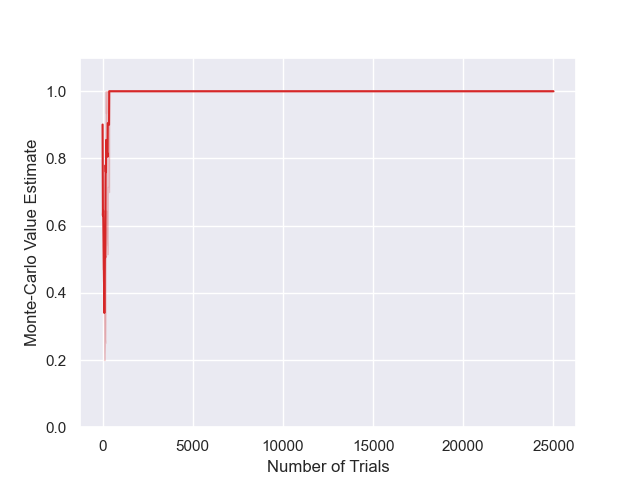}
                    \caption*{$\alpha=1,\epsilon=0.01$}
                \end{subfigure}
                \begin{subfigure}[b]{0.24\textwidth}
                    \centering
                    \includegraphics[width=\textwidth]{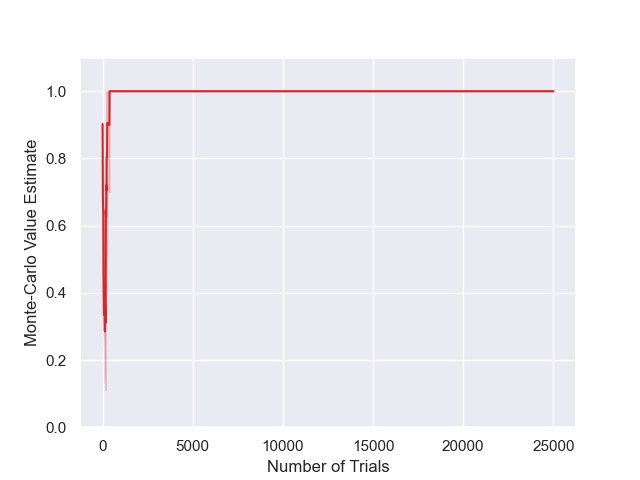}
                    \caption*{$\alpha=1,\epsilon=0.1$}
                \end{subfigure}
                \begin{subfigure}[b]{0.24\textwidth}
                    \centering
                    \includegraphics[width=\textwidth]{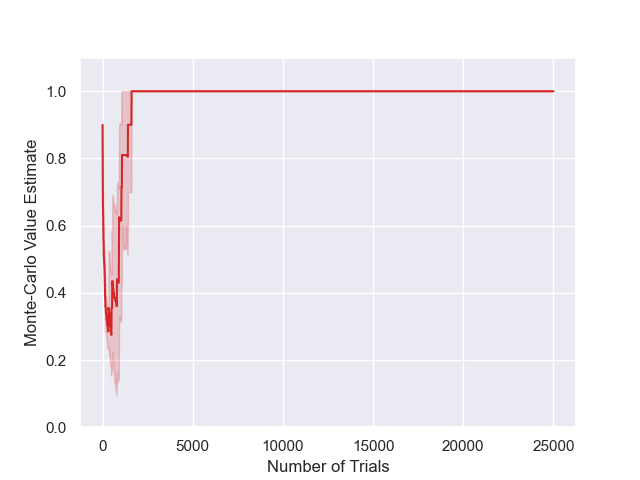}
                    \caption*{$\alpha=1,\epsilon=1$}
                \end{subfigure}
                \begin{subfigure}[b]{0.24\textwidth}
                    \centering
                    \includegraphics[width=\textwidth]{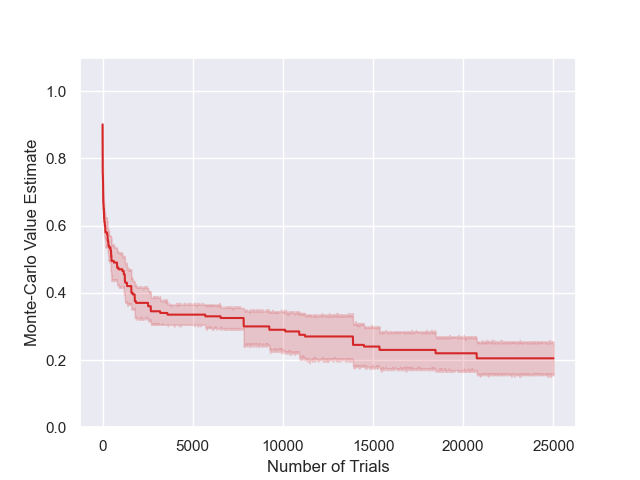}
                    \caption*{$\alpha=1,\epsilon=10$}
                \end{subfigure}
                
                \begin{subfigure}[b]{0.24\textwidth}
                    \centering
                    \includegraphics[width=\textwidth]{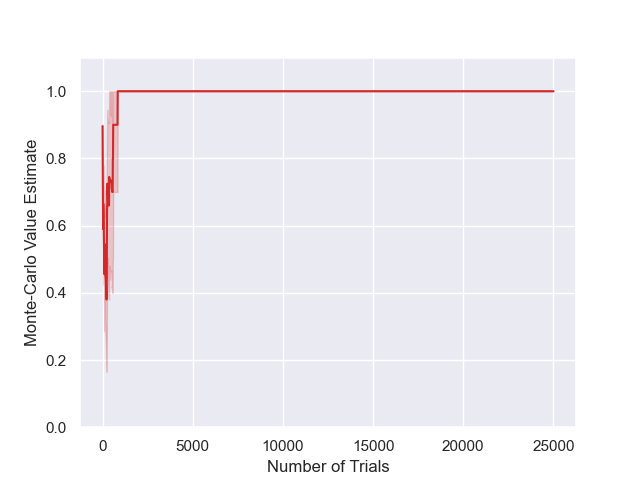}
                    \caption*{$\alpha=0.5,\epsilon=0.01$}
                \end{subfigure}
                \begin{subfigure}[b]{0.24\textwidth}
                    \centering
                    \includegraphics[width=\textwidth]{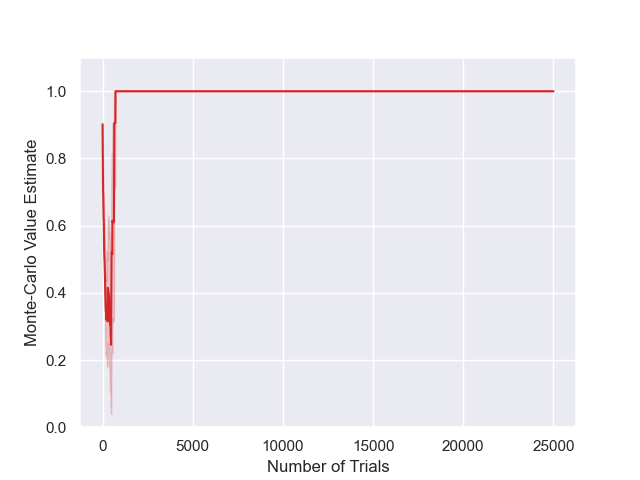}
                    \caption*{$\alpha=0.5,\epsilon=0.1$}
                \end{subfigure}
                \begin{subfigure}[b]{0.24\textwidth}
                    \centering
                    \includegraphics[width=\textwidth]{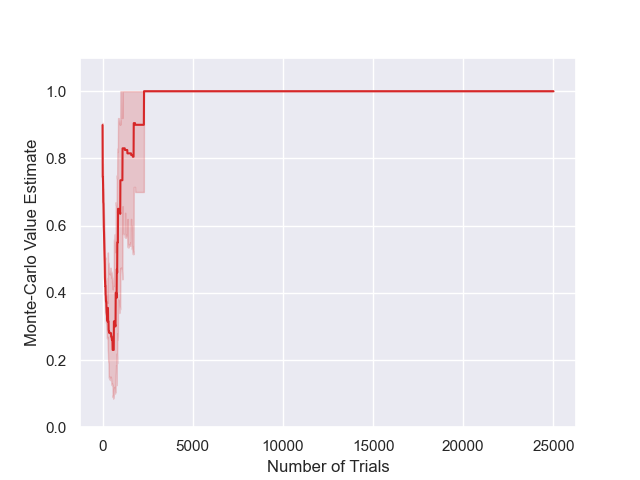}
                    \caption*{$\alpha=0.5,\epsilon=1$}
                \end{subfigure}
                \begin{subfigure}[b]{0.24\textwidth}
                    \centering
                    \includegraphics[width=\textwidth]{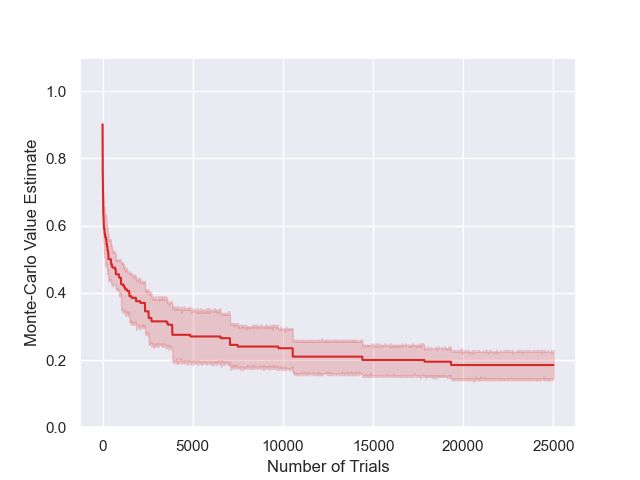}
                    \caption*{$\alpha=0.5,\epsilon=10$}
                \end{subfigure}
                
                \begin{subfigure}[b]{0.24\textwidth}
                    \centering
                    \includegraphics[width=\textwidth]{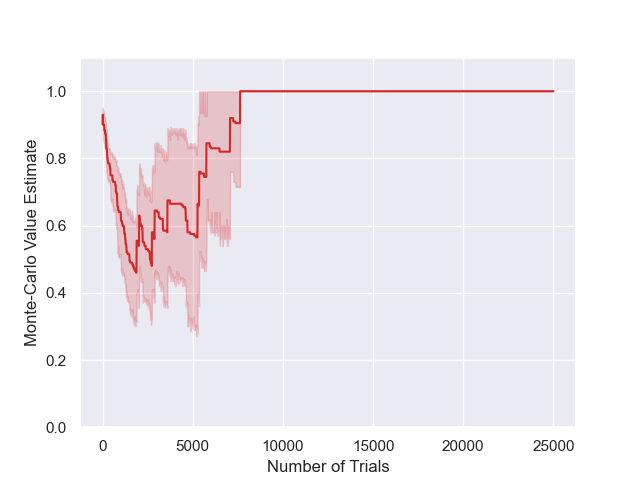}
                    \caption*{$\alpha=0.2,\epsilon=0.01$}
                \end{subfigure}
                \begin{subfigure}[b]{0.24\textwidth}
                    \centering
                    \includegraphics[width=\textwidth]{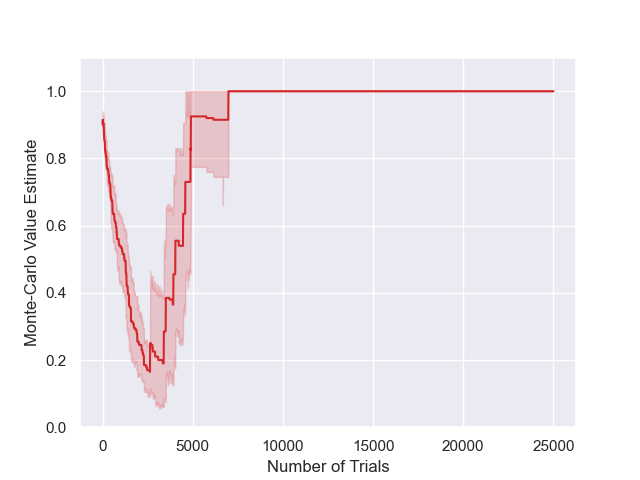}
                    \caption*{$\alpha=0.2,\epsilon=0.1$}
                \end{subfigure}
                \begin{subfigure}[b]{0.24\textwidth}
                    \centering
                    \includegraphics[width=\textwidth]{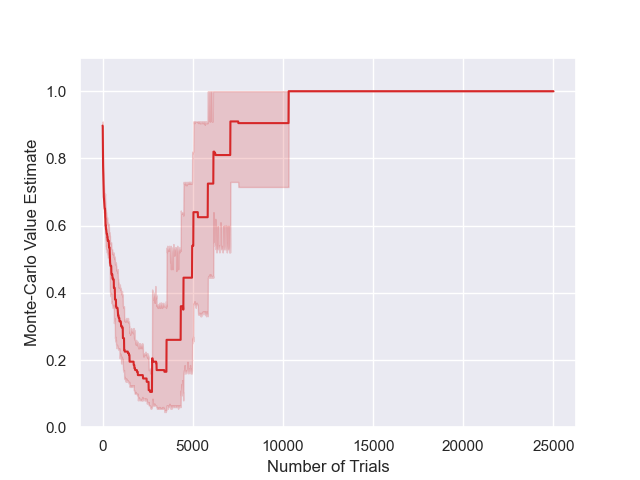}
                    \caption*{$\alpha=0.2,\epsilon=1$}
                \end{subfigure}
                \begin{subfigure}[b]{0.24\textwidth}
                    \centering
                    \includegraphics[width=\textwidth]{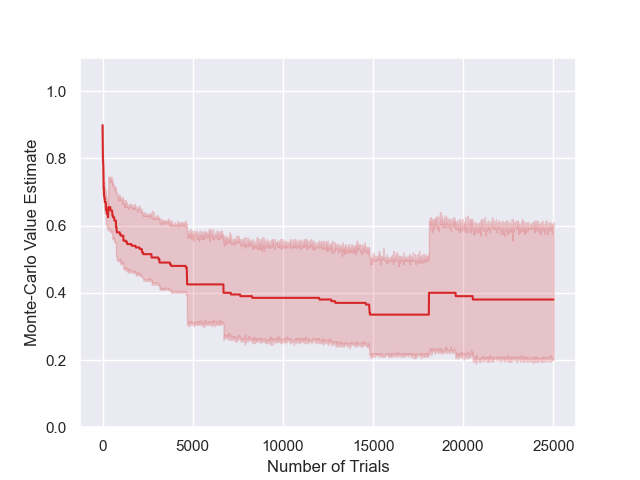}
                    \caption*{$\alpha=0.2,\epsilon=10$}
                \end{subfigure}
                
                \begin{subfigure}[b]{0.24\textwidth}
                    \centering
                    \includegraphics[width=\textwidth]{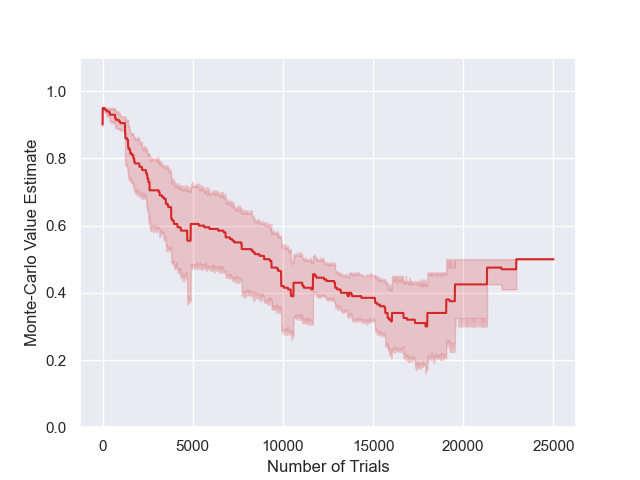}
                    \caption*{$\alpha=0.15,\epsilon=0.01$}
                \end{subfigure}
                \begin{subfigure}[b]{0.24\textwidth}
                    \centering
                    \includegraphics[width=\textwidth]{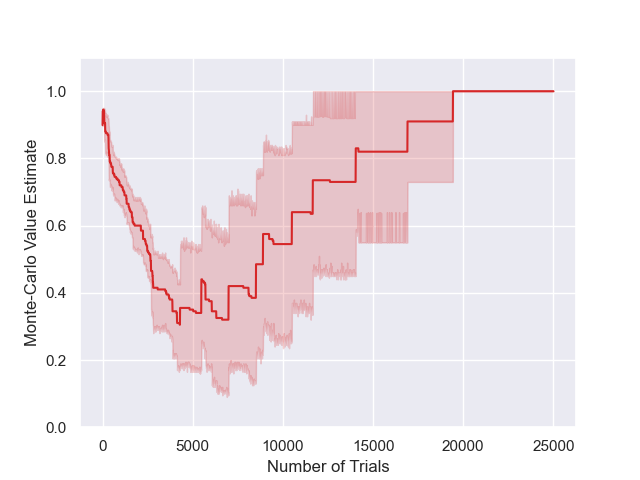}
                    \caption*{$\alpha=0.15,\epsilon=0.1$}
                \end{subfigure}
                \begin{subfigure}[b]{0.24\textwidth}
                    \centering
                    \includegraphics[width=\textwidth]{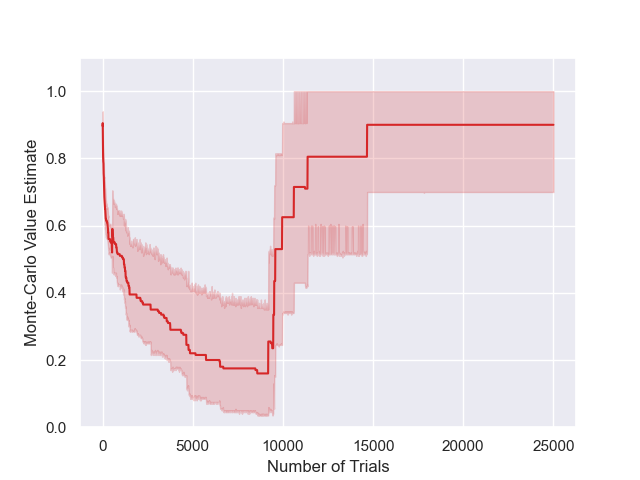}
                    \caption*{$\alpha=0.15,\epsilon=1$}
                \end{subfigure}
                \begin{subfigure}[b]{0.24\textwidth}
                    \centering
                    \includegraphics[width=\textwidth]{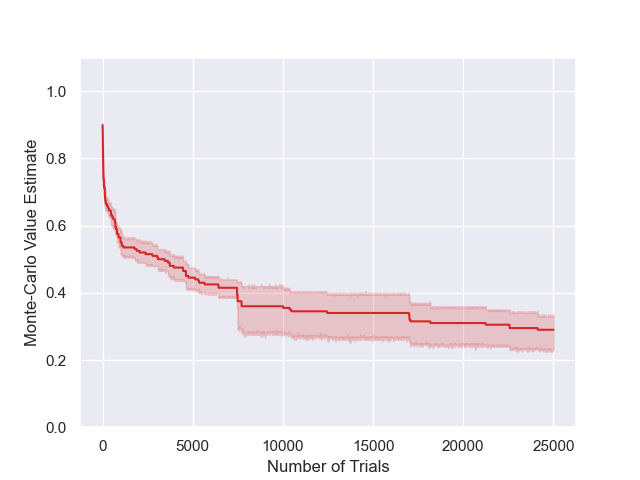}
                    \caption*{$\alpha=0.15,\epsilon=10$}
                \end{subfigure}
                
                \begin{subfigure}[b]{0.24\textwidth}
                    \centering
                    \includegraphics[width=\textwidth]{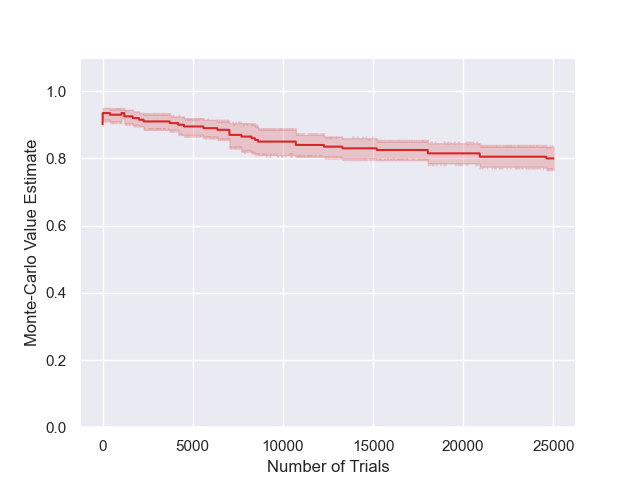}
                    \caption*{$\alpha=0.1,\epsilon=0.01$}
                \end{subfigure}
                \begin{subfigure}[b]{0.24\textwidth}
                    \centering
                    \includegraphics[width=\textwidth]{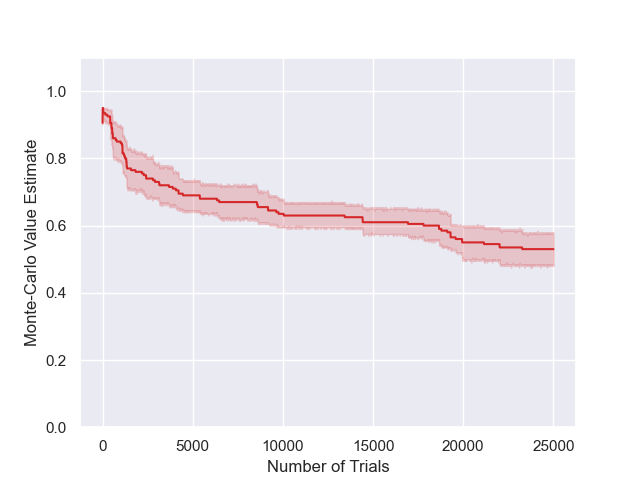}
                    \caption*{$\alpha=0.1,\epsilon=0.1$}
                \end{subfigure}
                \begin{subfigure}[b]{0.24\textwidth}
                    \centering
                    \includegraphics[width=\textwidth]{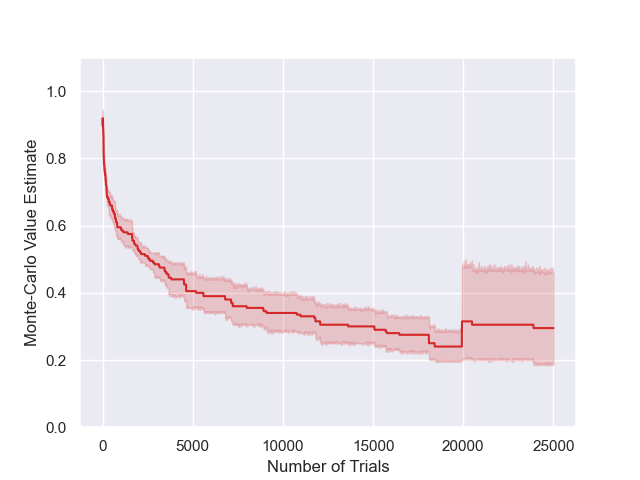}
                    \caption*{$\alpha=0.1,\epsilon=1$}
                \end{subfigure}
                \begin{subfigure}[b]{0.24\textwidth}
                    \centering
                    \includegraphics[width=\textwidth]{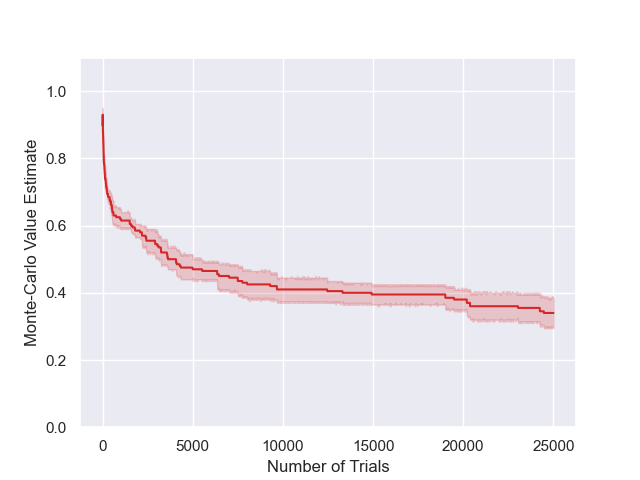}
                    \caption*{$\alpha=0.1,\epsilon=10$}
                \end{subfigure}
                
                \begin{subfigure}[b]{0.24\textwidth}
                    \centering
                    \includegraphics[width=\textwidth]{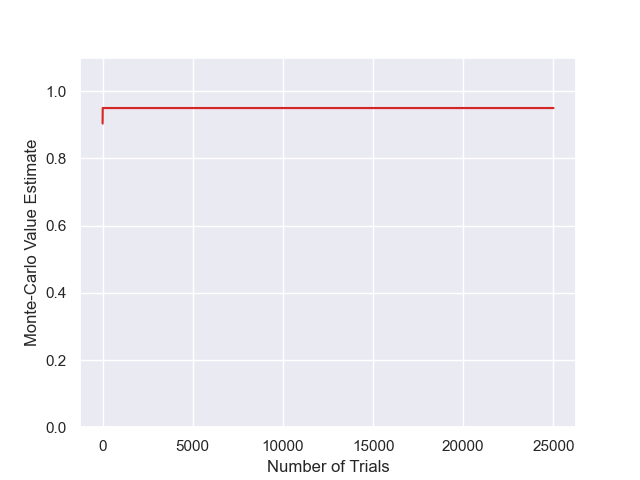}
                    \caption*{$\alpha=0.05,\epsilon=0.01$}
                \end{subfigure}
                \begin{subfigure}[b]{0.24\textwidth}
                    \centering
                    \includegraphics[width=\textwidth]{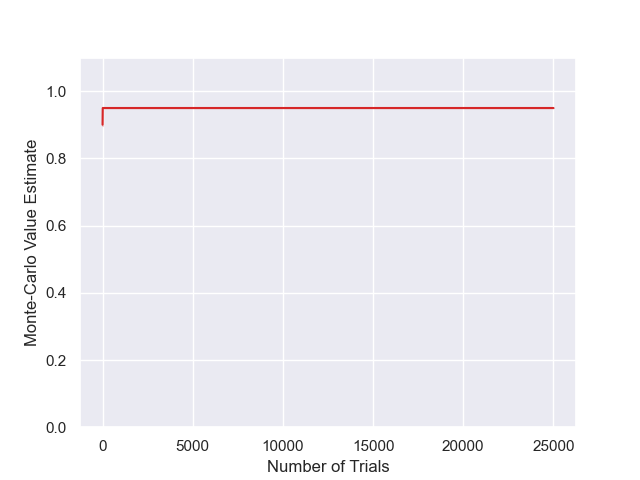}
                    \caption*{$\alpha=0.05,\epsilon=0.1$}
                \end{subfigure}
                \begin{subfigure}[b]{0.24\textwidth}
                    \centering
                    \includegraphics[width=\textwidth]{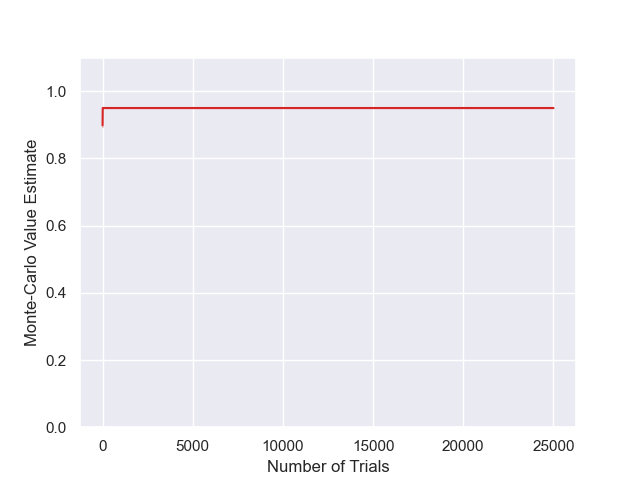}
                    \caption*{$\alpha=0.05,\epsilon=1$}
                \end{subfigure}
                \begin{subfigure}[b]{0.24\textwidth}
                    \centering
                    \includegraphics[width=\textwidth]{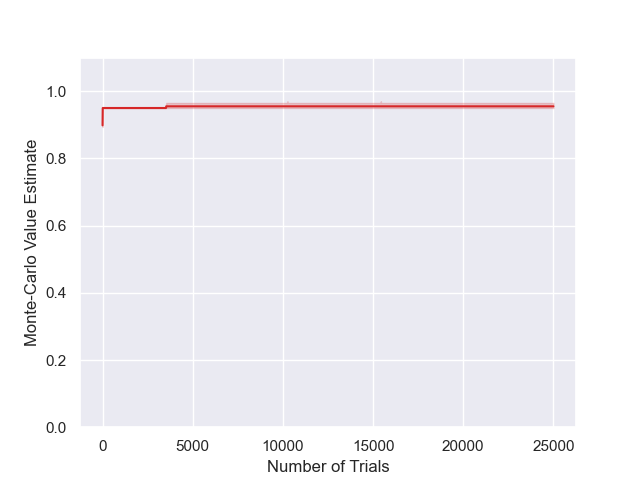}
                    \caption*{$\alpha=0.05,\epsilon=10$}
                \end{subfigure}
                
                \caption{Results for MENTS on the 20-chain ($D=20$, $R_f=1.0$), for varying temperatures and exploration parameters.}
                \label{fig:ments_20chain_hps}
            \end{figure}

            \begin{figure}
                \centering
                
                \begin{subfigure}[b]{0.24\textwidth}
                    \centering
                    \includegraphics[width=\textwidth]{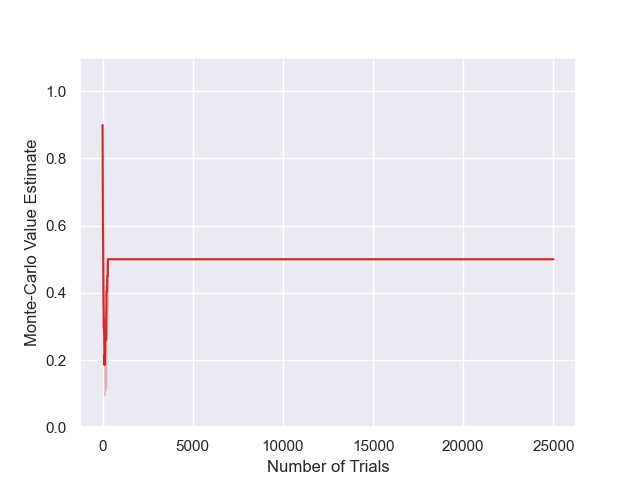}
                    \caption*{$\alpha=10,\epsilon=0.01$}
                \end{subfigure}
                \begin{subfigure}[b]{0.24\textwidth}
                    \centering
                    \includegraphics[width=\textwidth]{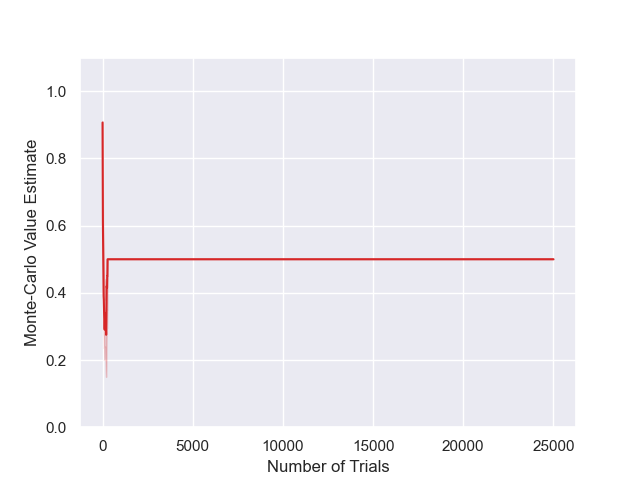}
                    \caption*{$\alpha=10,\epsilon=0.1$}
                \end{subfigure}
                \begin{subfigure}[b]{0.24\textwidth}
                    \centering
                    \includegraphics[width=\textwidth]{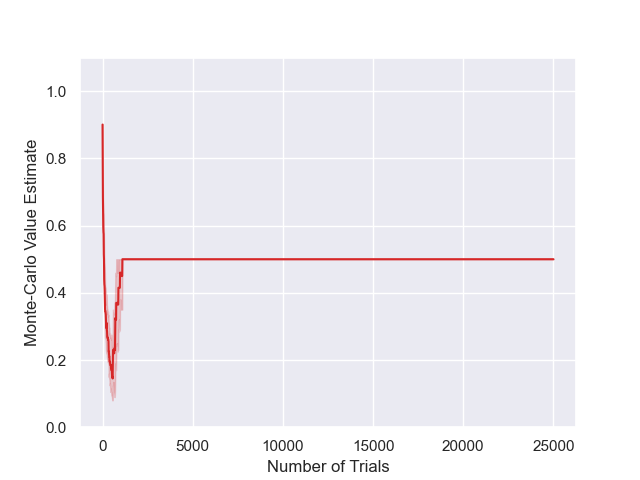}
                    \caption*{$\alpha=10,\epsilon=1$}
                \end{subfigure}
                \begin{subfigure}[b]{0.24\textwidth}
                    \centering
                    \includegraphics[width=\textwidth]{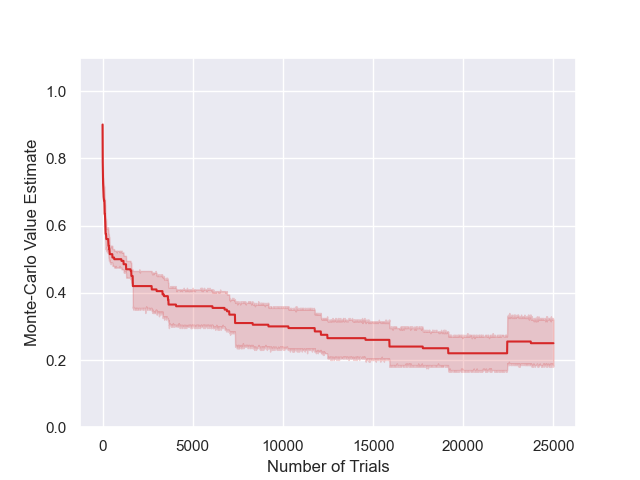}
                    \caption*{$\alpha=10,\epsilon=10$}
                \end{subfigure}
                
                \begin{subfigure}[b]{0.24\textwidth}
                    \centering
                    \includegraphics[width=\textwidth]{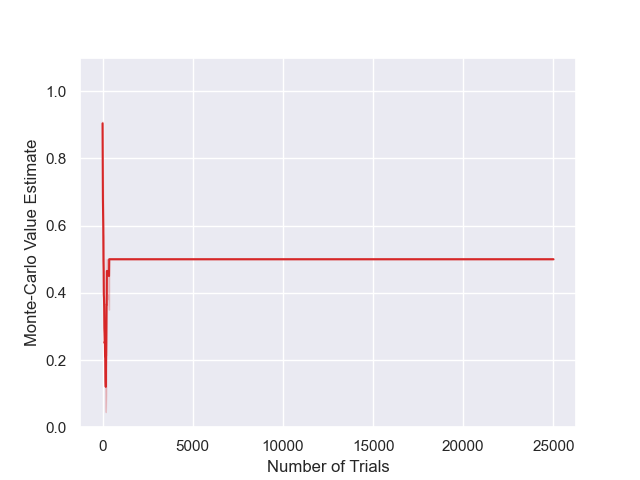}
                    \caption*{$\alpha=1,\epsilon=0.01$}
                \end{subfigure}
                \begin{subfigure}[b]{0.24\textwidth}
                    \centering
                    \includegraphics[width=\textwidth]{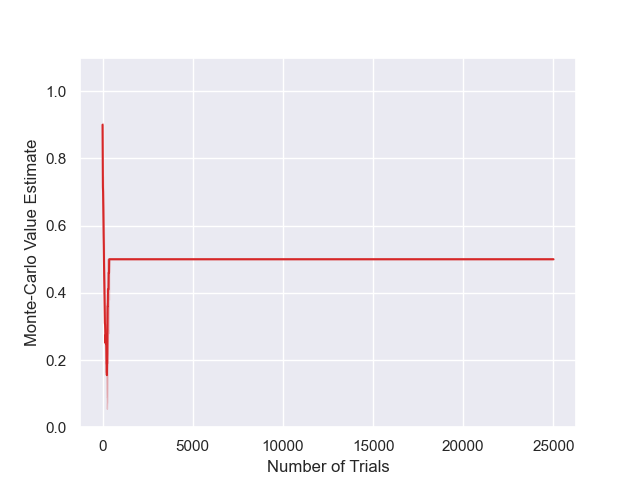}
                    \caption*{$\alpha=1,\epsilon=0.1$}
                \end{subfigure}
                \begin{subfigure}[b]{0.24\textwidth}
                    \centering
                    \includegraphics[width=\textwidth]{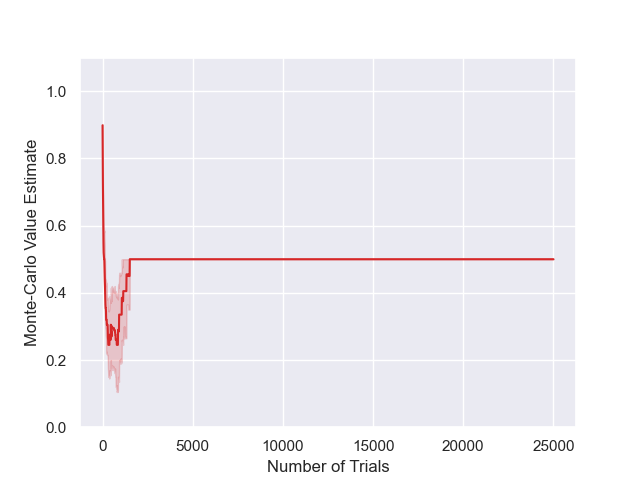}
                    \caption*{$\alpha=1,\epsilon=1$}
                \end{subfigure}
                \begin{subfigure}[b]{0.24\textwidth}
                    \centering
                    \includegraphics[width=\textwidth]{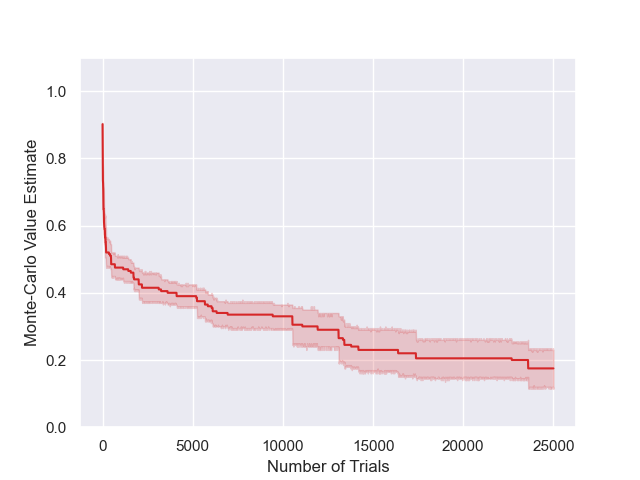}
                    \caption*{$\alpha=1,\epsilon=10$}
                \end{subfigure}
                
                \begin{subfigure}[b]{0.24\textwidth}
                    \centering
                    \includegraphics[width=\textwidth]{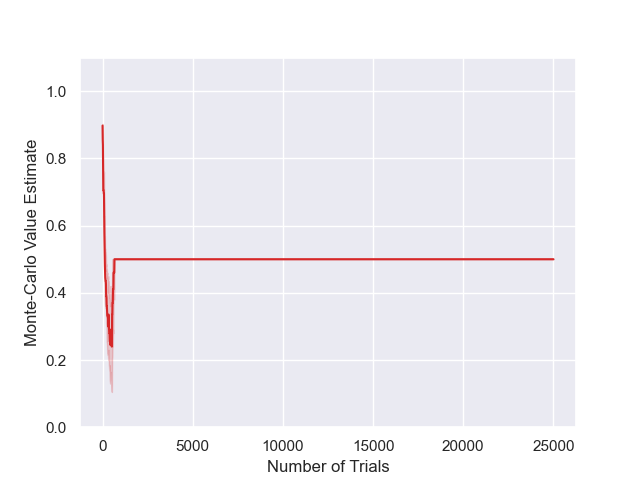}
                    \caption*{$\alpha=0.5,\epsilon=0.01$}
                \end{subfigure}
                \begin{subfigure}[b]{0.24\textwidth}
                    \centering
                    \includegraphics[width=\textwidth]{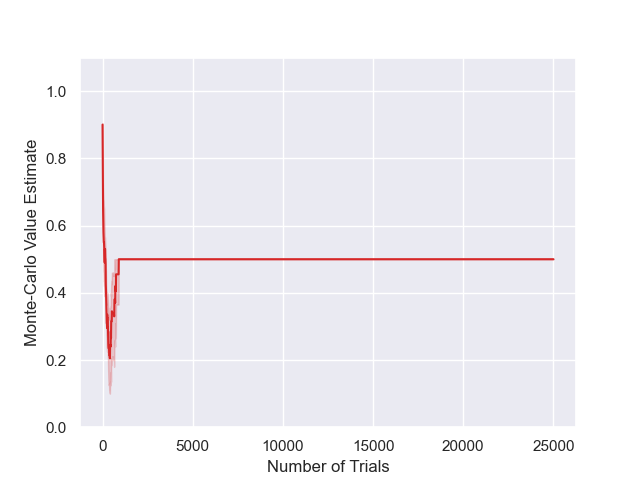}
                    \caption*{$\alpha=0.5,\epsilon=0.1$}
                \end{subfigure}
                \begin{subfigure}[b]{0.24\textwidth}
                    \centering
                    \includegraphics[width=\textwidth]{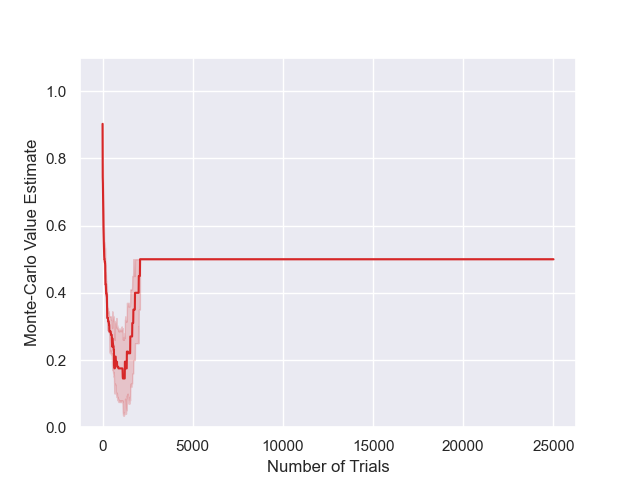}
                    \caption*{$\alpha=0.5,\epsilon=1$}
                \end{subfigure}
                \begin{subfigure}[b]{0.24\textwidth}
                    \centering
                    \includegraphics[width=\textwidth]{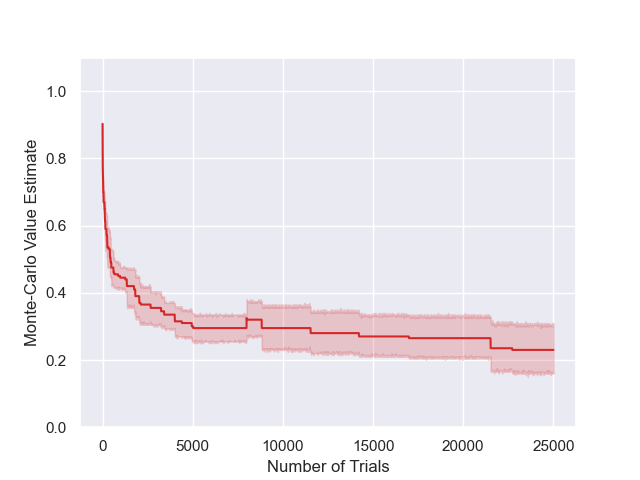}
                    \caption*{$\alpha=0.5,\epsilon=10$}
                \end{subfigure}
                
                \begin{subfigure}[b]{0.24\textwidth}
                    \centering
                    \includegraphics[width=\textwidth]{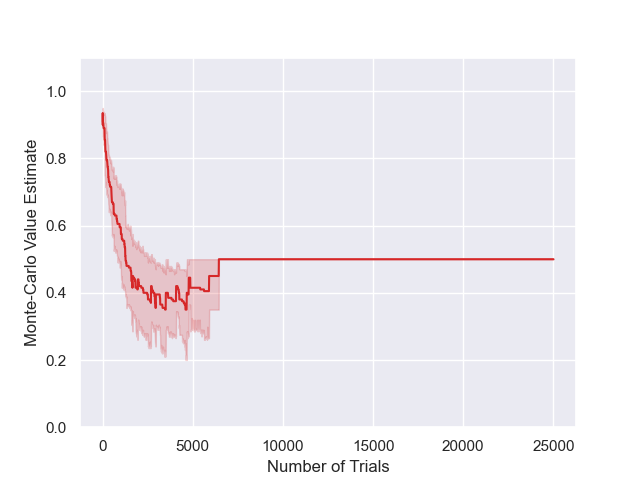}
                    \caption*{$\alpha=0.2,\epsilon=0.01$}
                \end{subfigure}
                \begin{subfigure}[b]{0.24\textwidth}
                    \centering
                    \includegraphics[width=\textwidth]{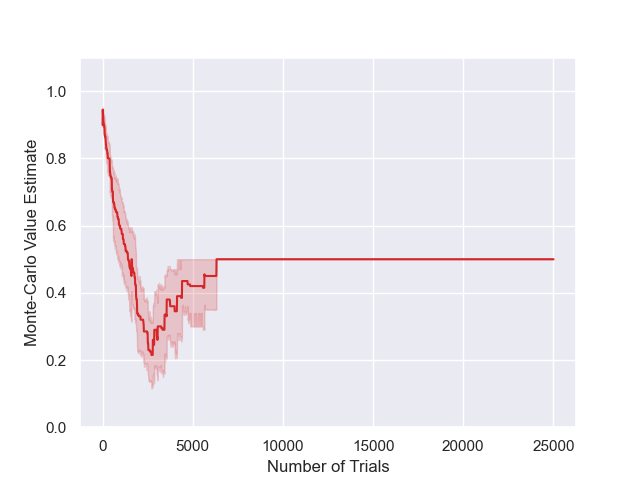}
                    \caption*{$\alpha=0.2,\epsilon=0.1$}
                \end{subfigure}
                \begin{subfigure}[b]{0.24\textwidth}
                    \centering
                    \includegraphics[width=\textwidth]{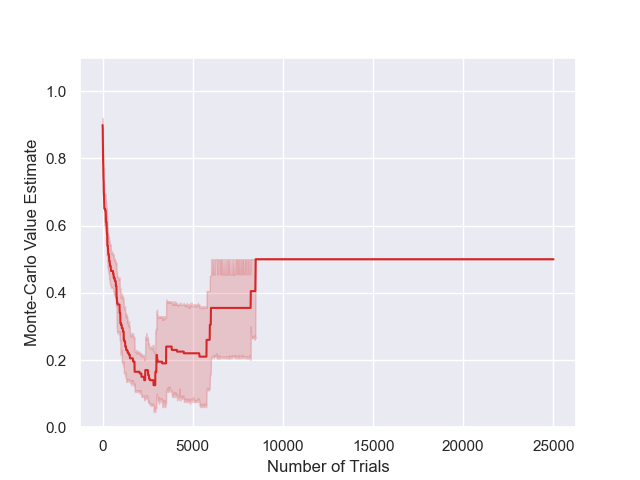}
                    \caption*{$\alpha=0.2,\epsilon=1$}
                \end{subfigure}
                \begin{subfigure}[b]{0.24\textwidth}
                    \centering
                    \includegraphics[width=\textwidth]{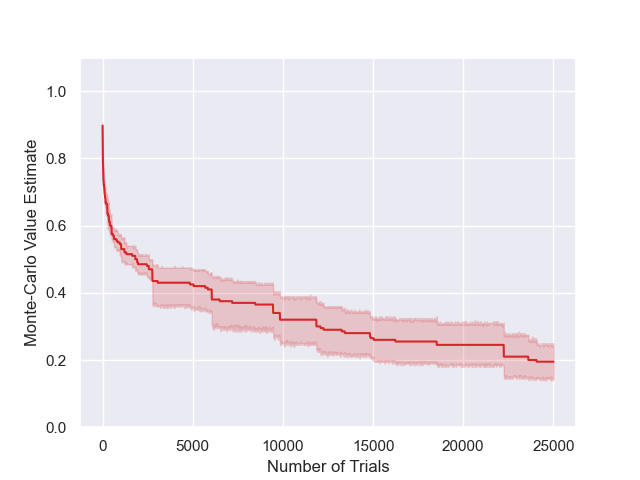}
                    \caption*{$\alpha=0.2,\epsilon=10$}
                \end{subfigure}
                
                \begin{subfigure}[b]{0.24\textwidth}
                    \centering
                    \includegraphics[width=\textwidth]{Figures/app/param_sens/dchain/005_20chain5_01_ments_epsilon=0.01,temp=0.15.png}
                    \caption*{$\alpha=0.15,\epsilon=0.01$}
                \end{subfigure}
                \begin{subfigure}[b]{0.24\textwidth}
                    \centering
                    \includegraphics[width=\textwidth]{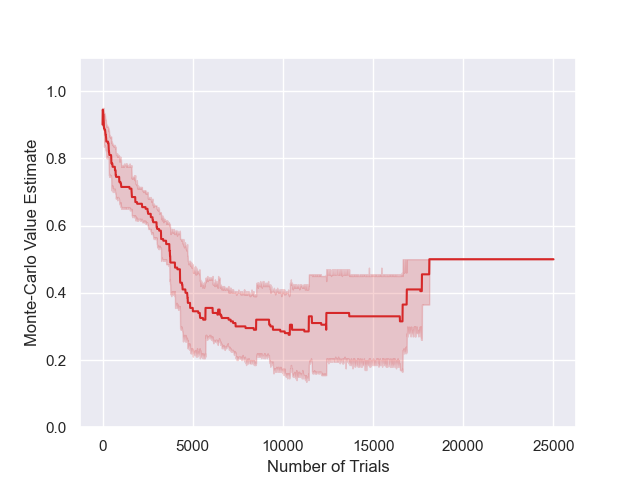}
                    \caption*{$\alpha=0.15,\epsilon=0.1$}
                \end{subfigure}
                \begin{subfigure}[b]{0.24\textwidth}
                    \centering
                    \includegraphics[width=\textwidth]{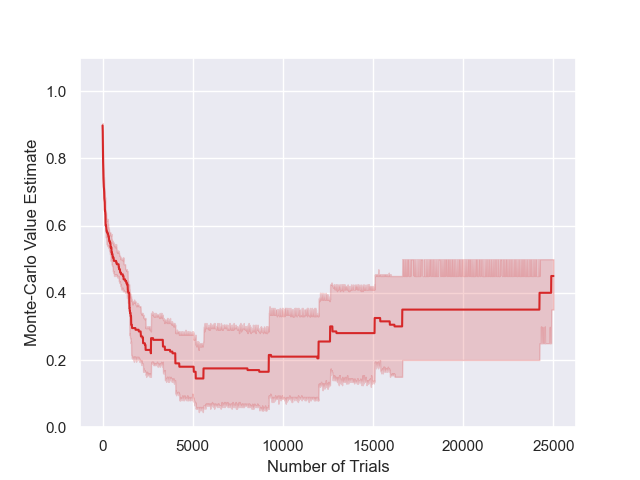}
                    \caption*{$\alpha=0.15,\epsilon=1$}
                \end{subfigure}
                \begin{subfigure}[b]{0.24\textwidth}
                    \centering
                    \includegraphics[width=\textwidth]{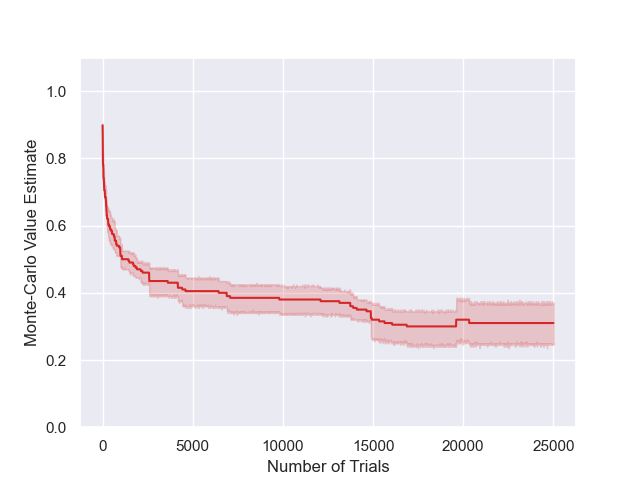}
                    \caption*{$\alpha=0.15,\epsilon=10$}
                \end{subfigure}
                
                \begin{subfigure}[b]{0.24\textwidth}
                    \centering
                    \includegraphics[width=\textwidth]{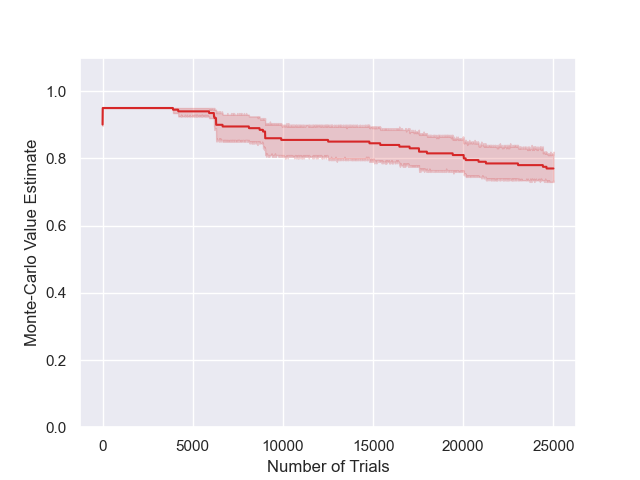}
                    \caption*{$\alpha=0.1,\epsilon=0.01$}
                \end{subfigure}
                \begin{subfigure}[b]{0.24\textwidth}
                    \centering
                    \includegraphics[width=\textwidth]{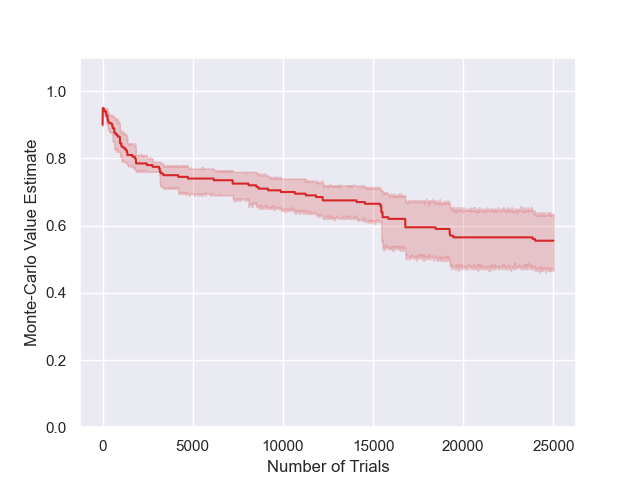}
                    \caption*{$\alpha=0.1,\epsilon=0.1$}
                \end{subfigure}
                \begin{subfigure}[b]{0.24\textwidth}
                    \centering
                    \includegraphics[width=\textwidth]{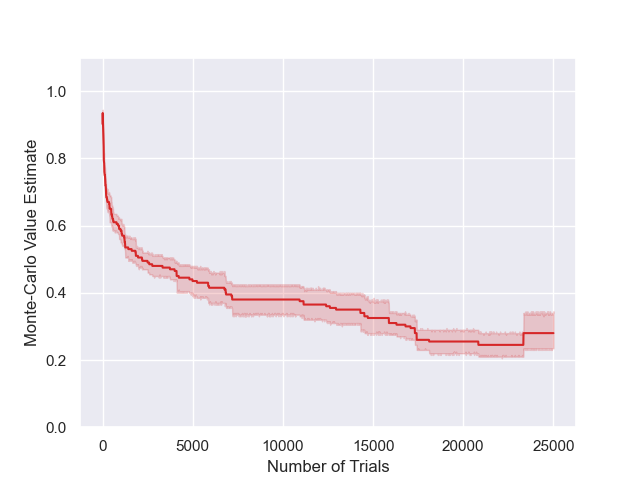}
                    \caption*{$\alpha=0.1,\epsilon=1$}
                \end{subfigure}
                \begin{subfigure}[b]{0.24\textwidth}
                    \centering
                    \includegraphics[width=\textwidth]{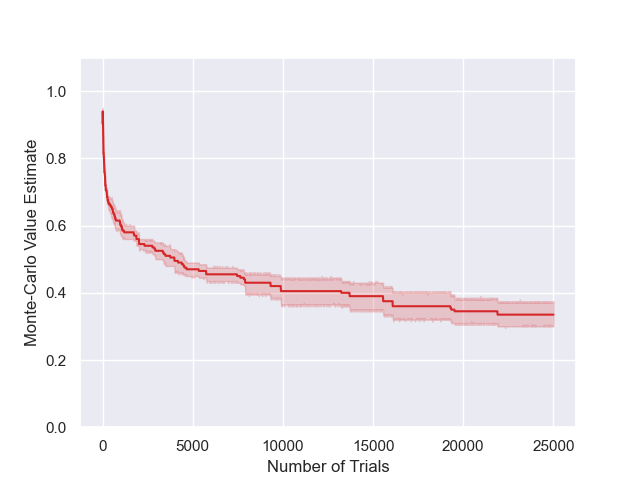}
                    \caption*{$\alpha=0.1,\epsilon=10$}
                \end{subfigure}
                
                \begin{subfigure}[b]{0.24\textwidth}
                    \centering
                    \includegraphics[width=\textwidth]{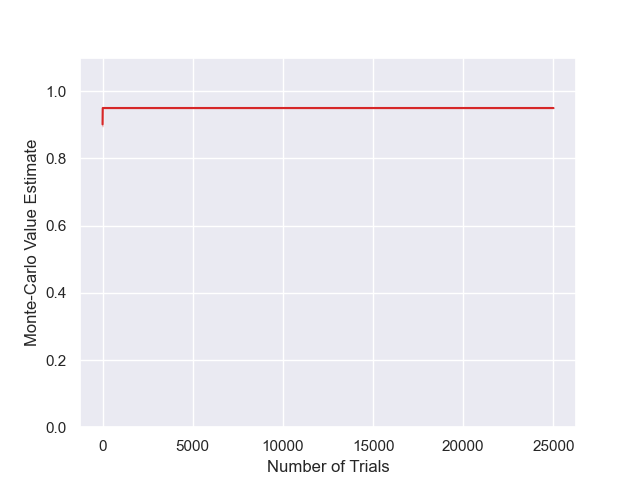}
                    \caption*{$\alpha=0.05,\epsilon=0.01$}
                \end{subfigure}
                \begin{subfigure}[b]{0.24\textwidth}
                    \centering
                    \includegraphics[width=\textwidth]{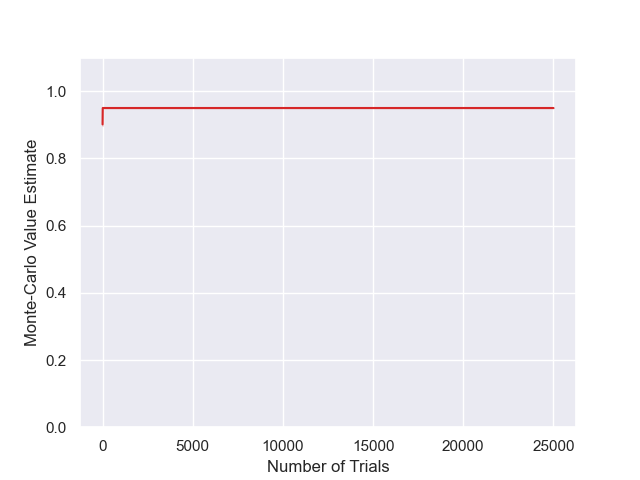}
                    \caption*{$\alpha=0.05,\epsilon=0.1$}
                \end{subfigure}
                \begin{subfigure}[b]{0.24\textwidth}
                    \centering
                    \includegraphics[width=\textwidth]{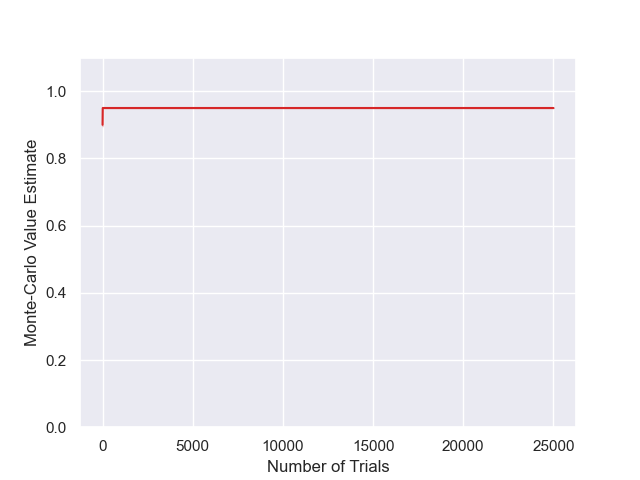}
                    \caption*{$\alpha=0.05,\epsilon=1$}
                \end{subfigure}
                \begin{subfigure}[b]{0.24\textwidth}
                    \centering
                    \includegraphics[width=\textwidth]{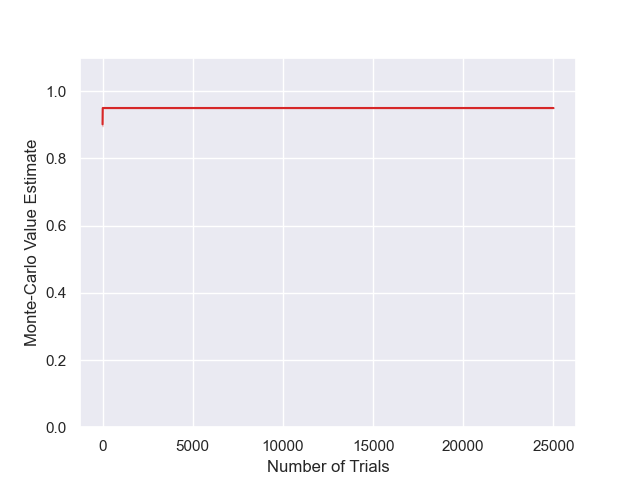}
                    \caption*{$\alpha=0.05,\epsilon=10$}
                \end{subfigure}
                
                \caption{Results for MENTS on the modified 20-chain ($D=20$, $R_f=0.5$), for varying temperatures and exploration parameters.}
                \label{fig:ments_20chain_half_hps}
            \end{figure}

            \begin{figure}
                \centering
                
                \begin{subfigure}[b]{0.24\textwidth}
                    \centering
                    \includegraphics[width=\textwidth]{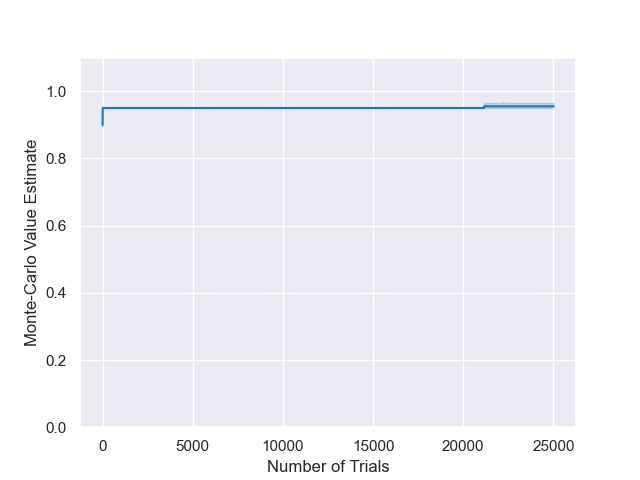}
                    \caption*{$\alpha=10,\epsilon=0.01$}
                \end{subfigure}
                \begin{subfigure}[b]{0.24\textwidth}
                    \centering
                    \includegraphics[width=\textwidth]{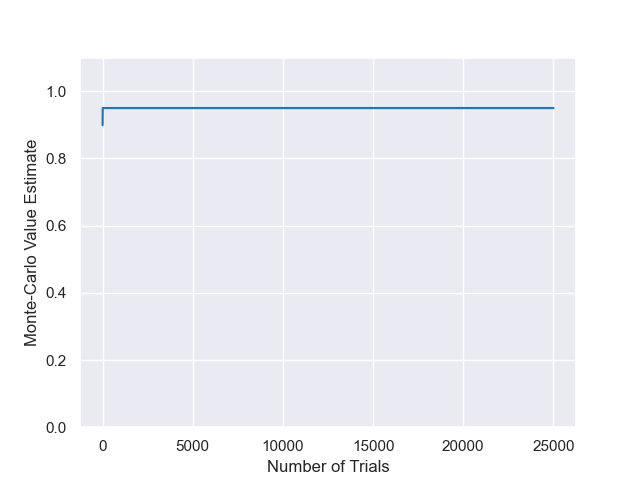}
                    \caption*{$\alpha=10,\epsilon=0.1$}
                \end{subfigure}
                \begin{subfigure}[b]{0.24\textwidth}
                    \centering
                    \includegraphics[width=\textwidth]{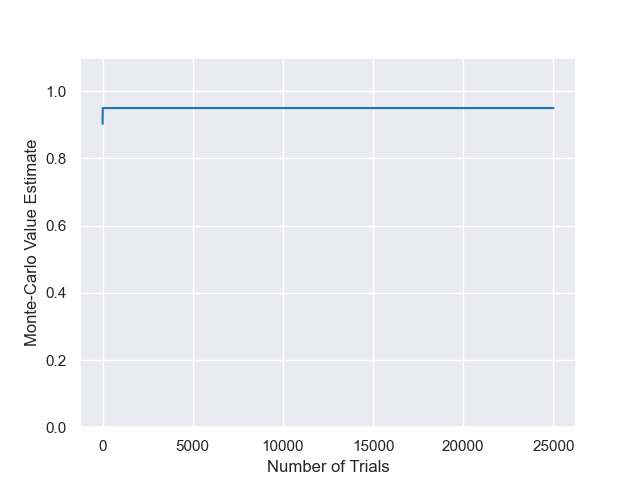}
                    \caption*{$\alpha=10,\epsilon=1$}
                \end{subfigure}
                \begin{subfigure}[b]{0.24\textwidth}
                    \centering
                    \includegraphics[width=\textwidth]{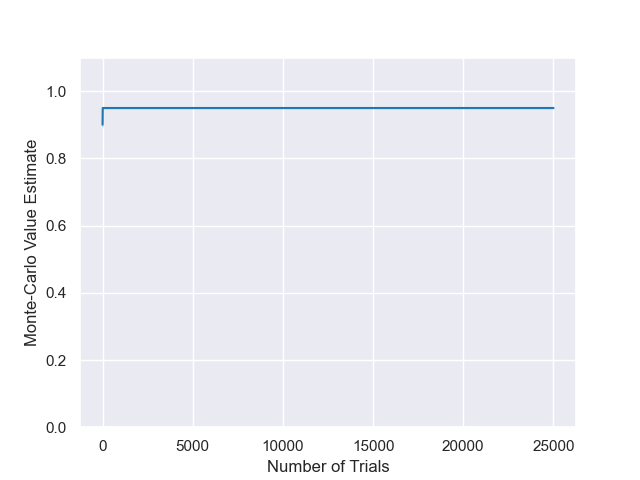}
                    \caption*{$\alpha=10,\epsilon=10$}
                \end{subfigure}
                
                \begin{subfigure}[b]{0.24\textwidth}
                    \centering
                    \includegraphics[width=\textwidth]{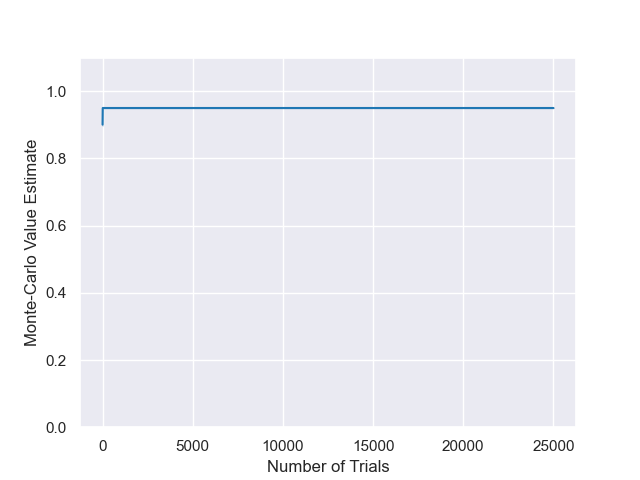}
                    \caption*{$\alpha=1,\epsilon=0.01$}
                \end{subfigure}
                \begin{subfigure}[b]{0.24\textwidth}
                    \centering
                    \includegraphics[width=\textwidth]{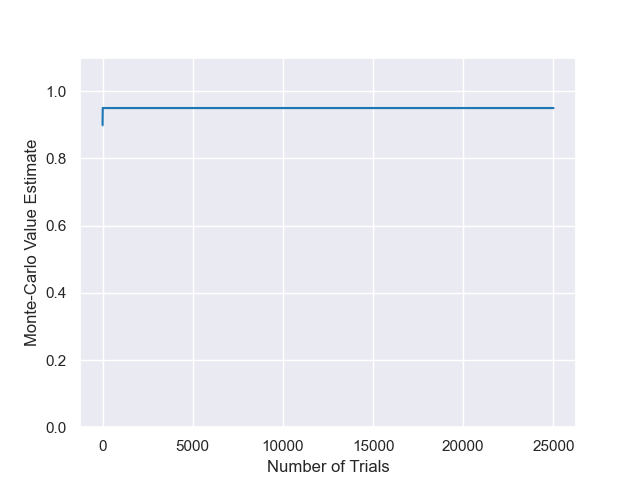}
                    \caption*{$\alpha=1,\epsilon=0.1$}
                \end{subfigure}
                \begin{subfigure}[b]{0.24\textwidth}
                    \centering
                    \includegraphics[width=\textwidth]{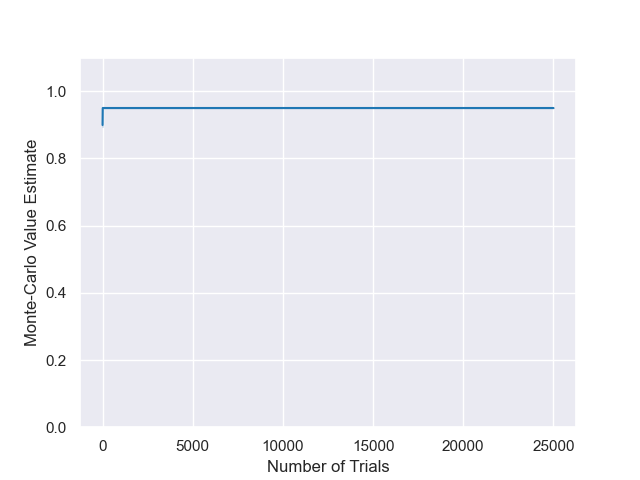}
                    \caption*{$\alpha=1,\epsilon=1$}
                \end{subfigure}
                \begin{subfigure}[b]{0.24\textwidth}
                    \centering
                    \includegraphics[width=\textwidth]{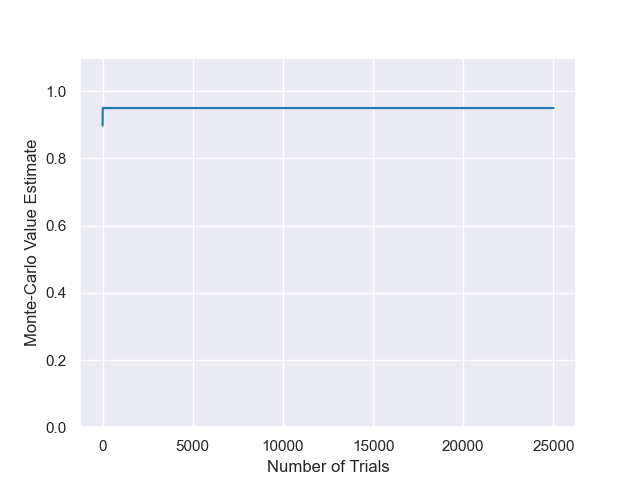}
                    \caption*{$\alpha=1,\epsilon=10$}
                \end{subfigure}
                
                \begin{subfigure}[b]{0.24\textwidth}
                    \centering
                    \includegraphics[width=\textwidth]{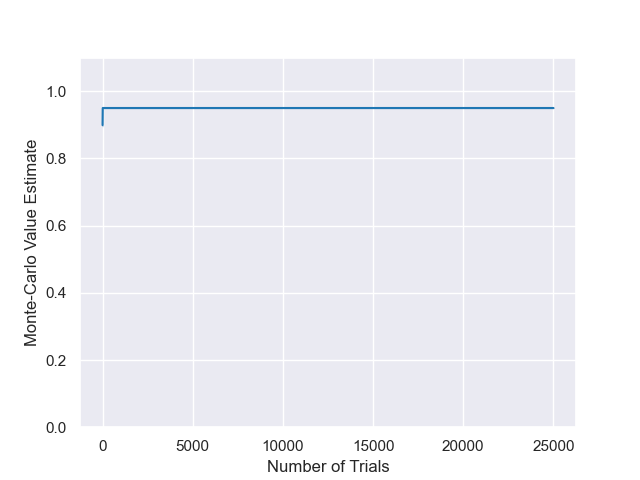}
                    \caption*{$\alpha=0.5,\epsilon=0.01$}
                \end{subfigure}
                \begin{subfigure}[b]{0.24\textwidth}
                    \centering
                    \includegraphics[width=\textwidth]{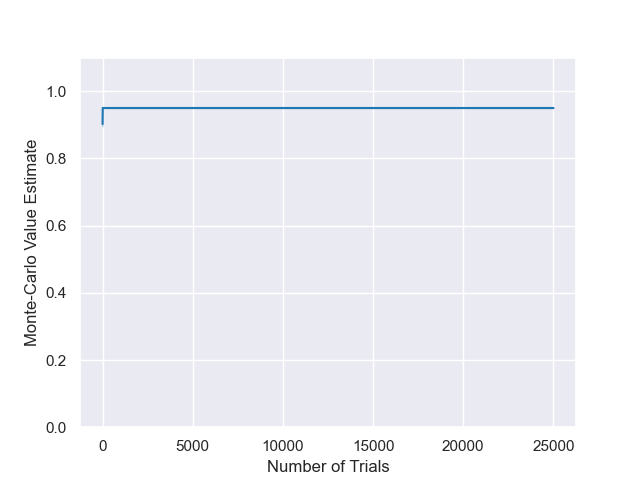}
                    \caption*{$\alpha=0.5,\epsilon=0.1$}
                \end{subfigure}
                \begin{subfigure}[b]{0.24\textwidth}
                    \centering
                    \includegraphics[width=\textwidth]{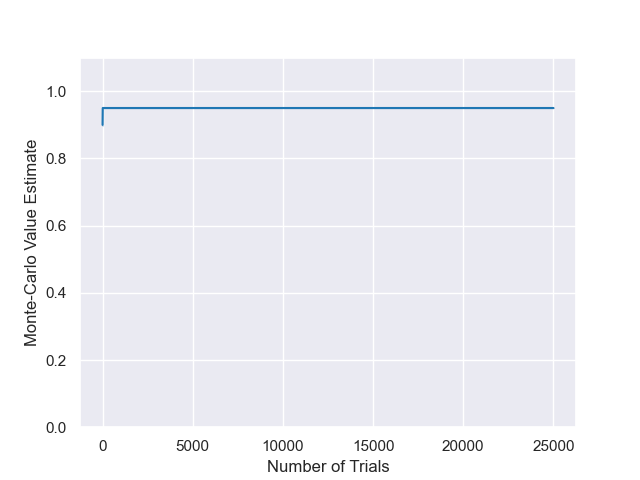}
                    \caption*{$\alpha=0.5,\epsilon=1$}
                \end{subfigure}
                \begin{subfigure}[b]{0.24\textwidth}
                    \centering
                    \includegraphics[width=\textwidth]{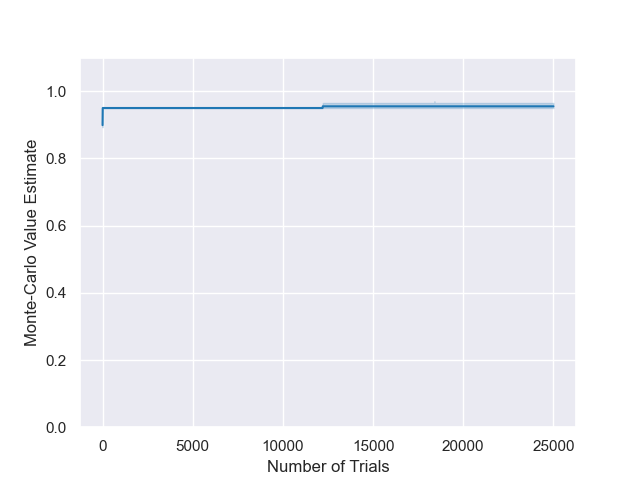}
                    \caption*{$\alpha=0.5,\epsilon=10$}
                \end{subfigure}
                
                \begin{subfigure}[b]{0.24\textwidth}
                    \centering
                    \includegraphics[width=\textwidth]{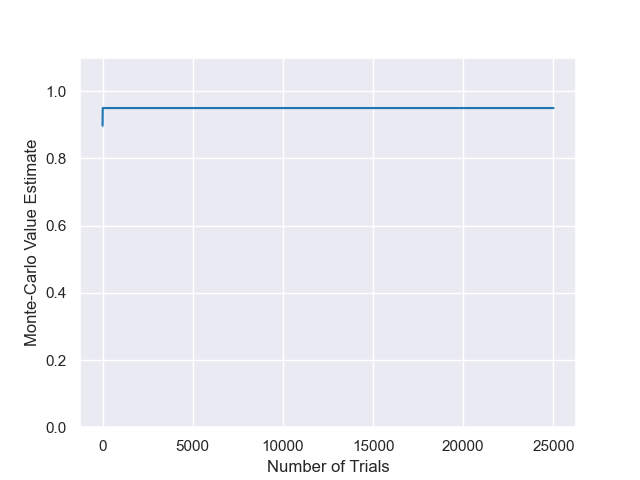}
                    \caption*{$\alpha=0.2,\epsilon=0.01$}
                \end{subfigure}
                \begin{subfigure}[b]{0.24\textwidth}
                    \centering
                    \includegraphics[width=\textwidth]{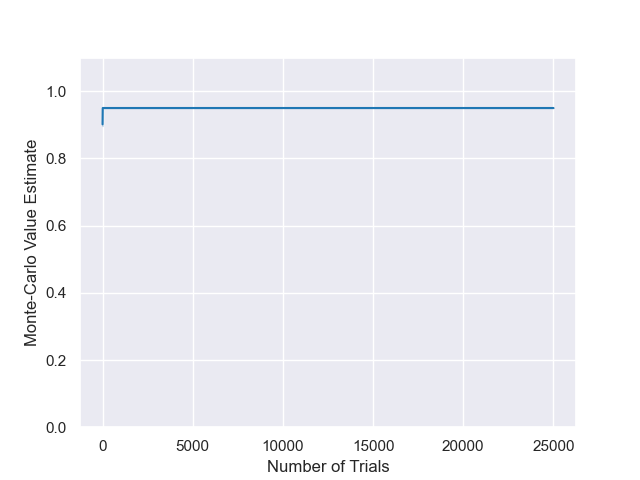}
                    \caption*{$\alpha=0.2,\epsilon=0.1$}
                \end{subfigure}
                \begin{subfigure}[b]{0.24\textwidth}
                    \centering
                    \includegraphics[width=\textwidth]{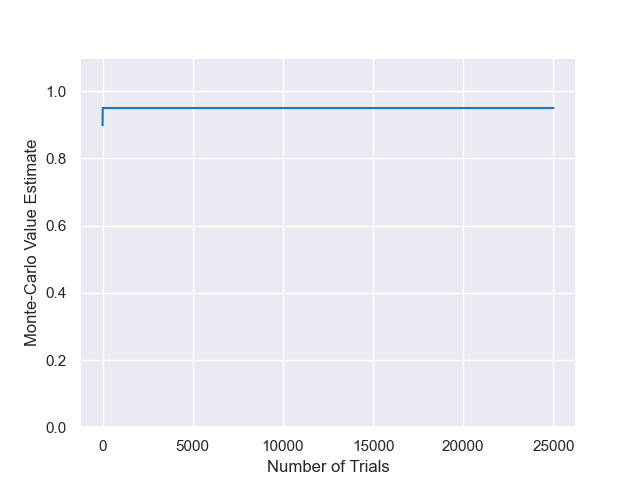}
                    \caption*{$\alpha=0.2,\epsilon=1$}
                \end{subfigure}
                \begin{subfigure}[b]{0.24\textwidth}
                    \centering
                    \includegraphics[width=\textwidth]{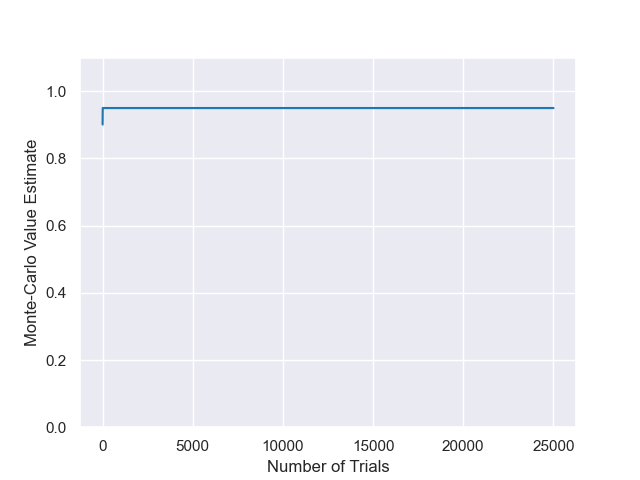}
                    \caption*{$\alpha=0.2,\epsilon=10$}
                \end{subfigure}
                
                \begin{subfigure}[b]{0.24\textwidth}
                    \centering
                    \includegraphics[width=\textwidth]{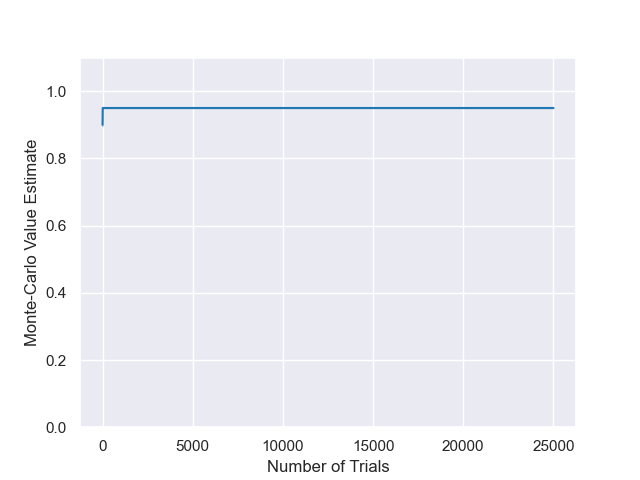}
                    \caption*{$\alpha=0.15,\epsilon=0.01$}
                \end{subfigure}
                \begin{subfigure}[b]{0.24\textwidth}
                    \centering
                    \includegraphics[width=\textwidth]{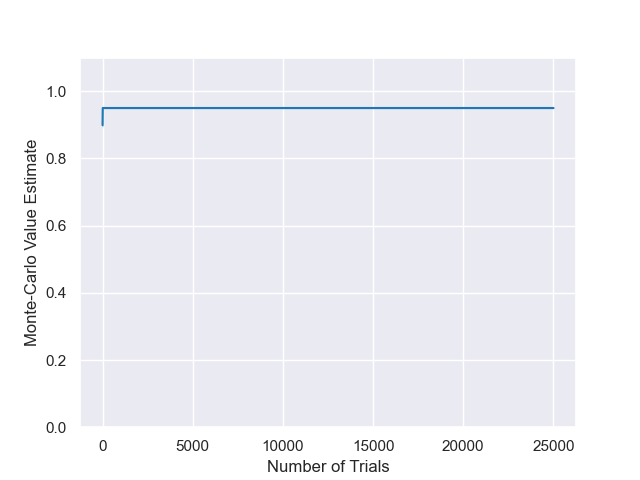}
                    \caption*{$\alpha=0.15,\epsilon=0.1$}
                \end{subfigure}
                \begin{subfigure}[b]{0.24\textwidth}
                    \centering
                    \includegraphics[width=\textwidth]{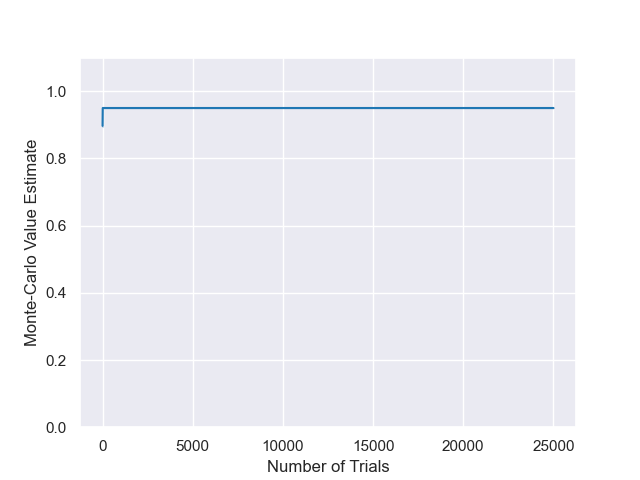}
                    \caption*{$\alpha=0.15,\epsilon=1$}
                \end{subfigure}
                \begin{subfigure}[b]{0.24\textwidth}
                    \centering
                    \includegraphics[width=\textwidth]{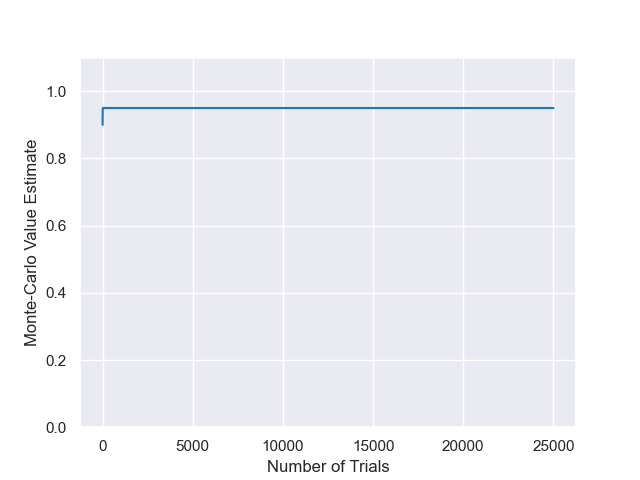}
                    \caption*{$\alpha=0.15,\epsilon=10$}
                \end{subfigure}
                
                \begin{subfigure}[b]{0.24\textwidth}
                    \centering
                    \includegraphics[width=\textwidth]{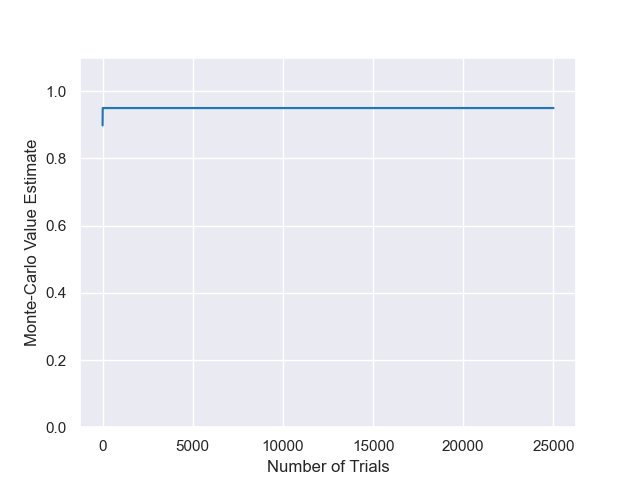}
                    \caption*{$\alpha=0.1,\epsilon=0.01$}
                \end{subfigure}
                \begin{subfigure}[b]{0.24\textwidth}
                    \centering
                    \includegraphics[width=\textwidth]{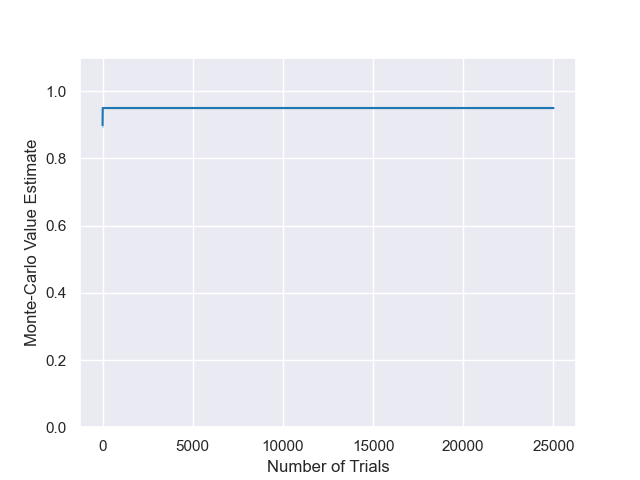}
                    \caption*{$\alpha=0.1,\epsilon=0.1$}
                \end{subfigure}
                \begin{subfigure}[b]{0.24\textwidth}
                    \centering
                    \includegraphics[width=\textwidth]{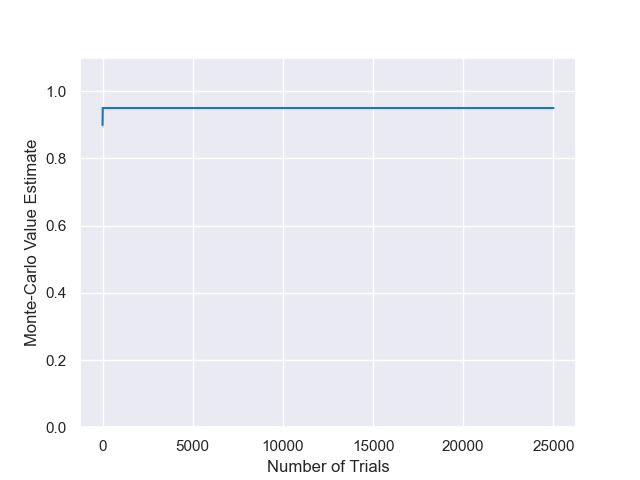}
                    \caption*{$\alpha=0.1,\epsilon=1$}
                \end{subfigure}
                \begin{subfigure}[b]{0.24\textwidth}
                    \centering
                    \includegraphics[width=\textwidth]{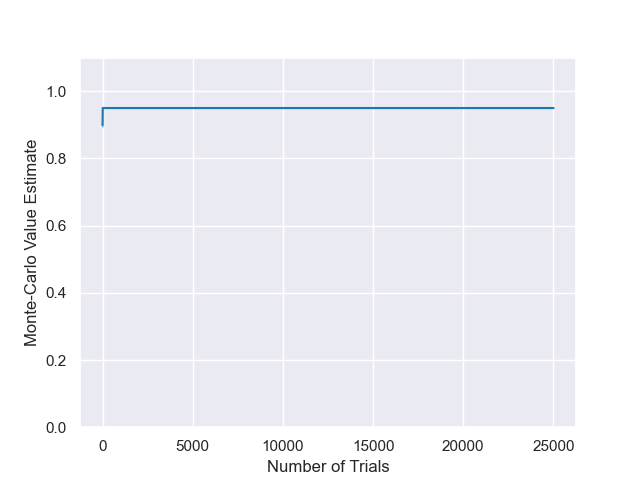}
                    \caption*{$\alpha=0.1,\epsilon=10$}
                \end{subfigure}
                
                \begin{subfigure}[b]{0.24\textwidth}
                    \centering
                    \includegraphics[width=\textwidth]{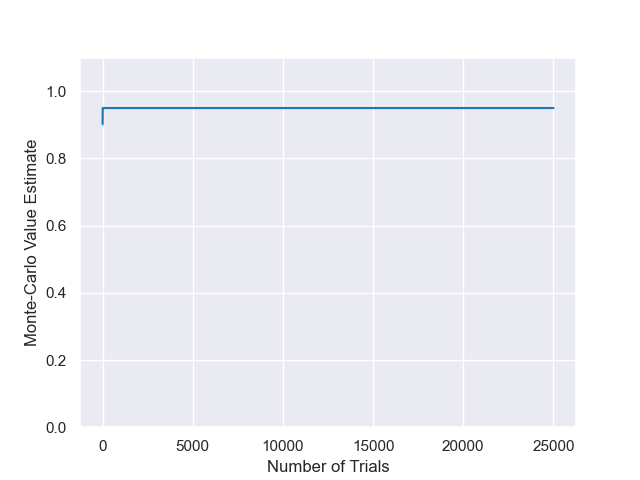}
                    \caption*{$\alpha=0.05,\epsilon=0.01$}
                \end{subfigure}
                \begin{subfigure}[b]{0.24\textwidth}
                    \centering
                    \includegraphics[width=\textwidth]{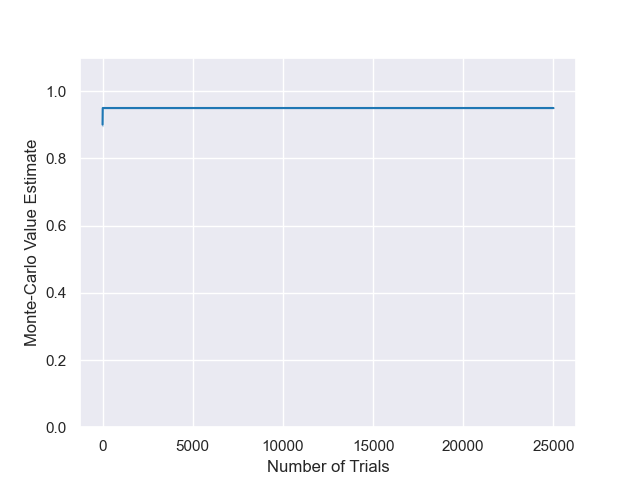}
                    \caption*{$\alpha=0.05,\epsilon=0.1$}
                \end{subfigure}
                \begin{subfigure}[b]{0.24\textwidth}
                    \centering
                    \includegraphics[width=\textwidth]{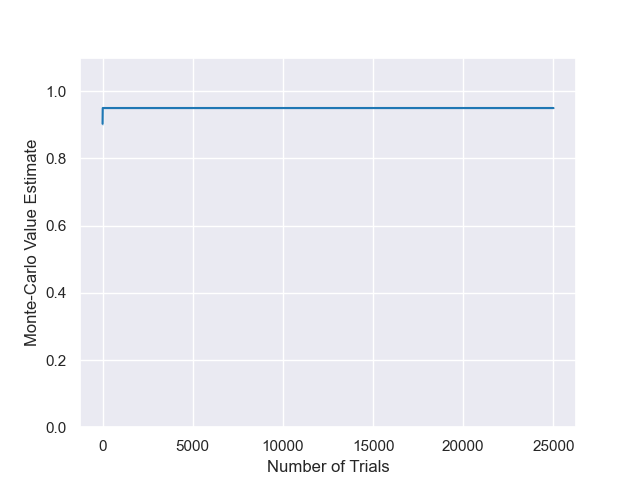}
                    \caption*{$\alpha=0.05,\epsilon=1$}
                \end{subfigure}
                \begin{subfigure}[b]{0.24\textwidth}
                    \centering
                    \includegraphics[width=\textwidth]{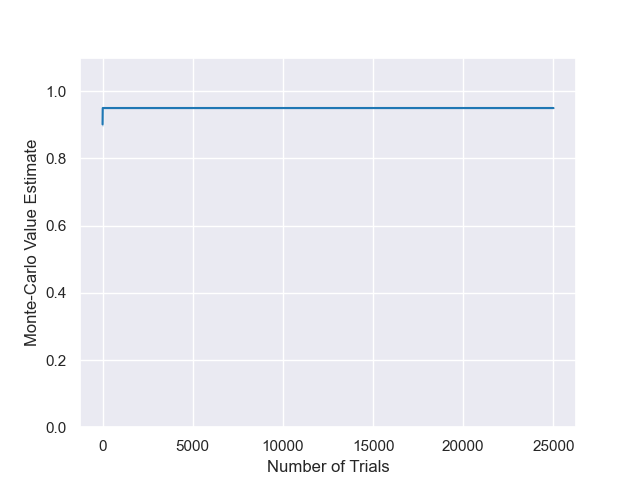}
                    \caption*{$\alpha=0.05,\epsilon=10$}
                \end{subfigure}
                
                \caption{Results for BTS on the 20-chain ($D=20$, $R_f=1.0$), for varying temperatures and exploration parameters.}
                \label{fig:bts_20chain_hps}
            \end{figure}

            \begin{figure}
                \centering
                
                \begin{subfigure}[b]{0.24\textwidth}
                    \centering
                    \includegraphics[width=\textwidth]{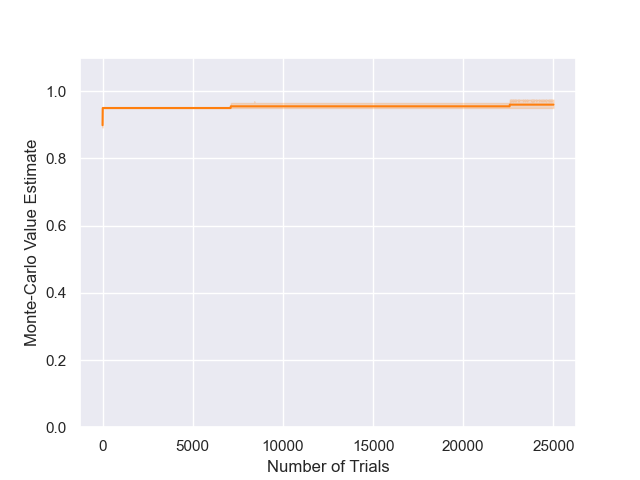}
                    \caption*{$\alpha=10,\epsilon=0.01$}
                \end{subfigure}
                \begin{subfigure}[b]{0.24\textwidth}
                    \centering
                    \includegraphics[width=\textwidth]{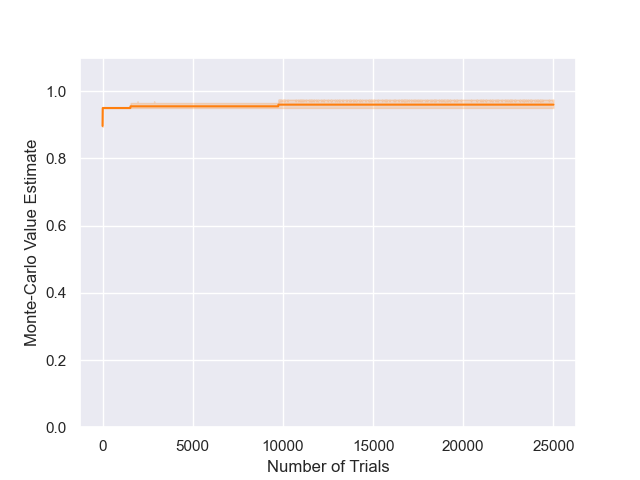}
                    \caption*{$\alpha=10,\epsilon=0.1$}
                \end{subfigure}
                \begin{subfigure}[b]{0.24\textwidth}
                    \centering
                    \includegraphics[width=\textwidth]{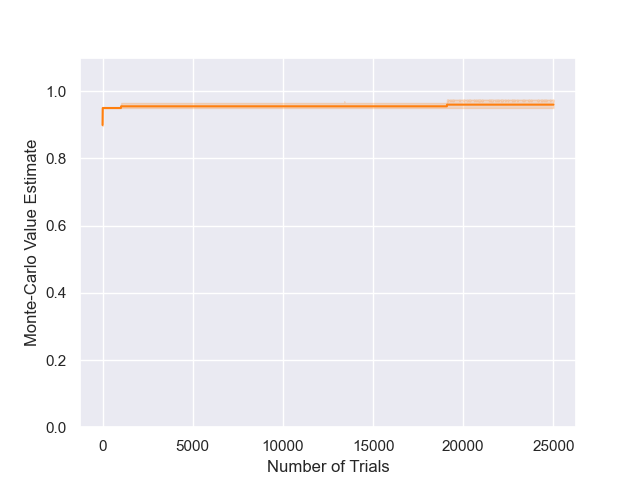}
                    \caption*{$\alpha=10,\epsilon=1$}
                \end{subfigure}
                \begin{subfigure}[b]{0.24\textwidth}
                    \centering
                    \includegraphics[width=\textwidth]{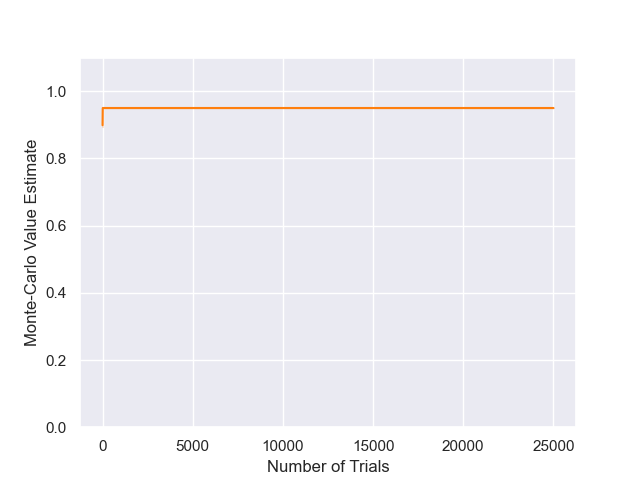}
                    \caption*{$\alpha=10,\epsilon=10$}
                \end{subfigure}
                
                \begin{subfigure}[b]{0.24\textwidth}
                    \centering
                    \includegraphics[width=\textwidth]{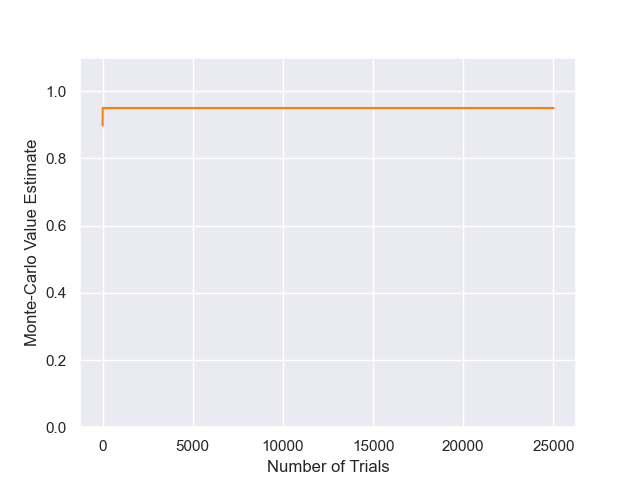}
                    \caption*{$\alpha=1,\epsilon=0.01$}
                \end{subfigure}
                \begin{subfigure}[b]{0.24\textwidth}
                    \centering
                    \includegraphics[width=\textwidth]{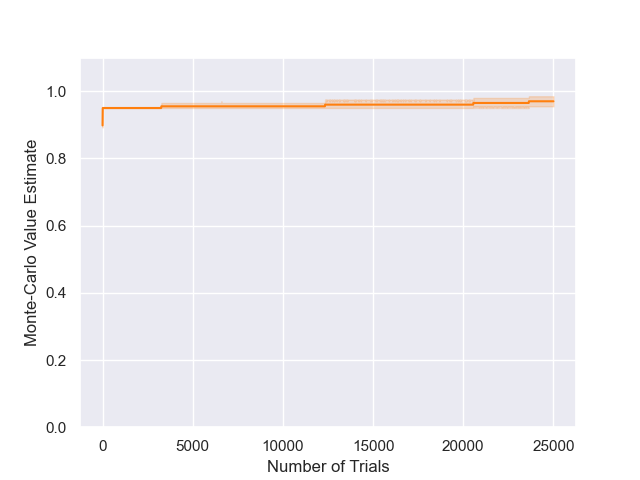}
                    \caption*{$\alpha=1,\epsilon=0.1$}
                \end{subfigure}
                \begin{subfigure}[b]{0.24\textwidth}
                    \centering
                    \includegraphics[width=\textwidth]{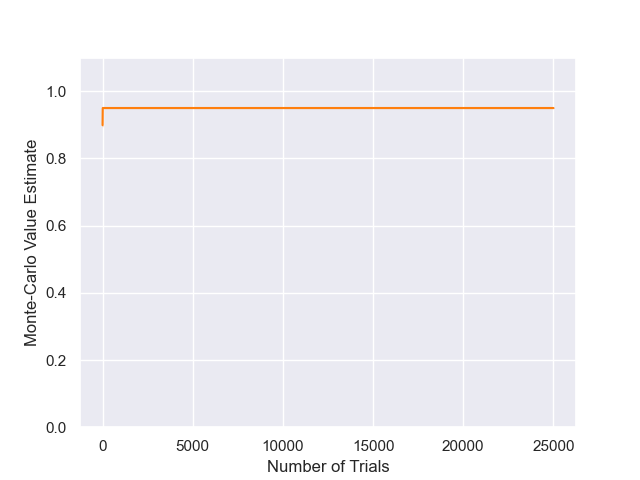}
                    \caption*{$\alpha=1,\epsilon=1$}
                \end{subfigure}
                \begin{subfigure}[b]{0.24\textwidth}
                    \centering
                    \includegraphics[width=\textwidth]{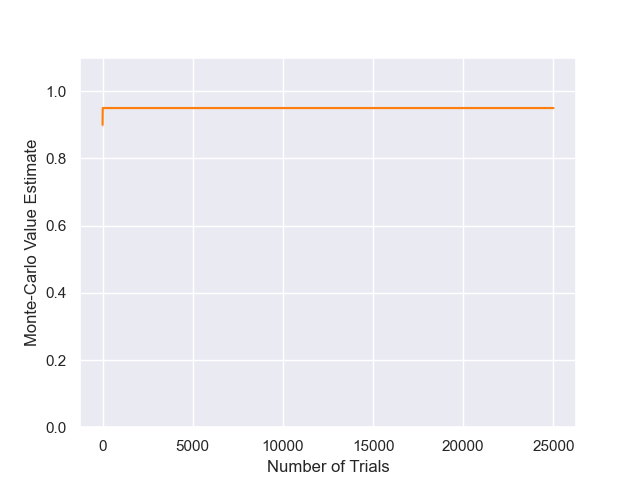}
                    \caption*{$\alpha=1,\epsilon=10$}
                \end{subfigure}
                
                \begin{subfigure}[b]{0.24\textwidth}
                    \centering
                    \includegraphics[width=\textwidth]{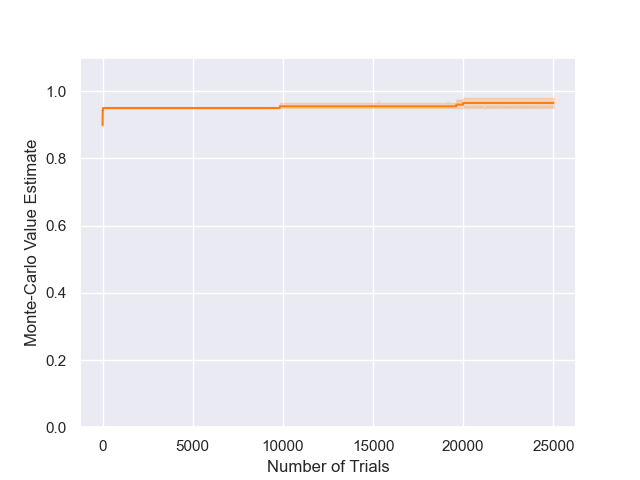}
                    \caption*{$\alpha=0.5,\epsilon=0.01$}
                \end{subfigure}
                \begin{subfigure}[b]{0.24\textwidth}
                    \centering
                    \includegraphics[width=\textwidth]{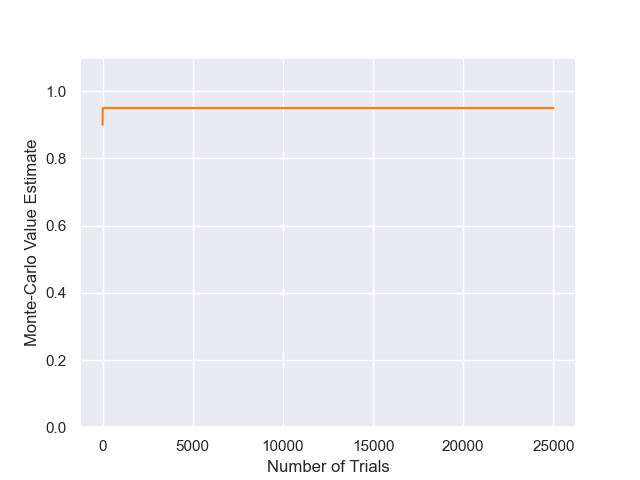}
                    \caption*{$\alpha=0.5,\epsilon=0.1$}
                \end{subfigure}
                \begin{subfigure}[b]{0.24\textwidth}
                    \centering
                    \includegraphics[width=\textwidth]{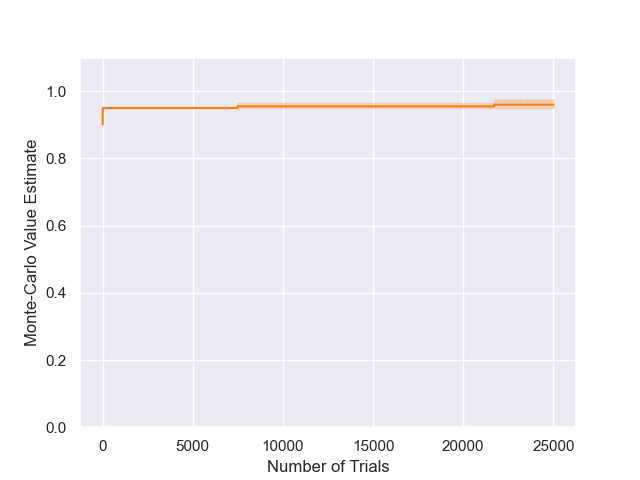}
                    \caption*{$\alpha=0.5,\epsilon=1$}
                \end{subfigure}
                \begin{subfigure}[b]{0.24\textwidth}
                    \centering
                    \includegraphics[width=\textwidth]{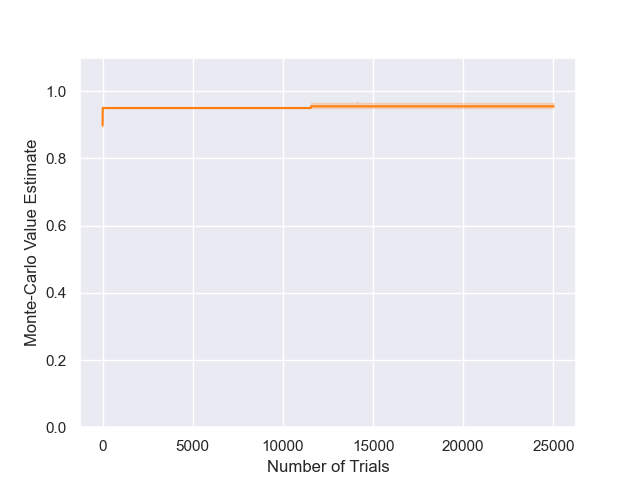}
                    \caption*{$\alpha=0.5,\epsilon=10$}
                \end{subfigure}
                
                \begin{subfigure}[b]{0.24\textwidth}
                    \centering
                    \includegraphics[width=\textwidth]{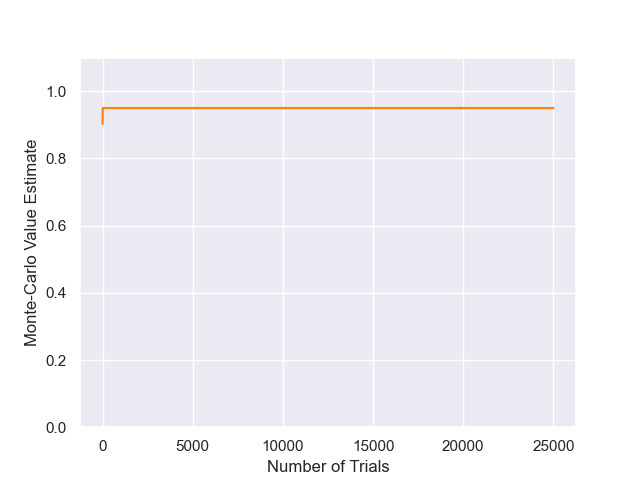}
                    \caption*{$\alpha=0.2,\epsilon=0.01$}
                \end{subfigure}
                \begin{subfigure}[b]{0.24\textwidth}
                    \centering
                    \includegraphics[width=\textwidth]{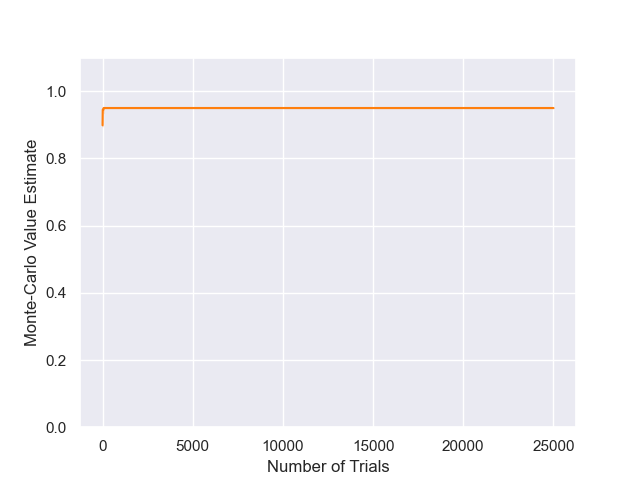}
                    \caption*{$\alpha=0.2,\epsilon=0.1$}
                \end{subfigure}
                \begin{subfigure}[b]{0.24\textwidth}
                    \centering
                    \includegraphics[width=\textwidth]{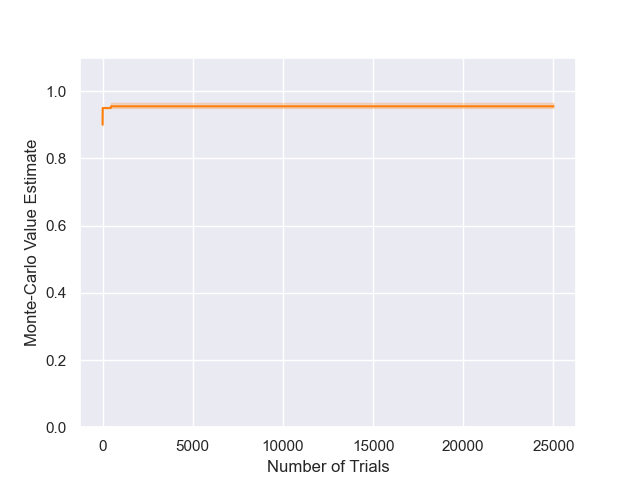}
                    \caption*{$\alpha=0.2,\epsilon=1$}
                \end{subfigure}
                \begin{subfigure}[b]{0.24\textwidth}
                    \centering
                    \includegraphics[width=\textwidth]{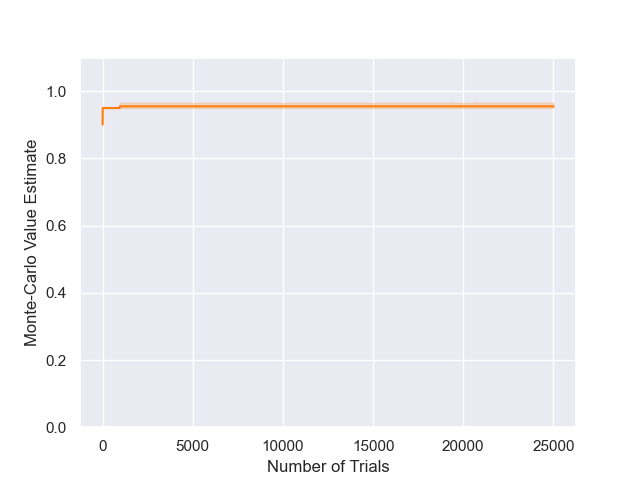}
                    \caption*{$\alpha=0.2,\epsilon=10$}
                \end{subfigure}
                
                \begin{subfigure}[b]{0.24\textwidth}
                    \centering
                    \includegraphics[width=\textwidth]{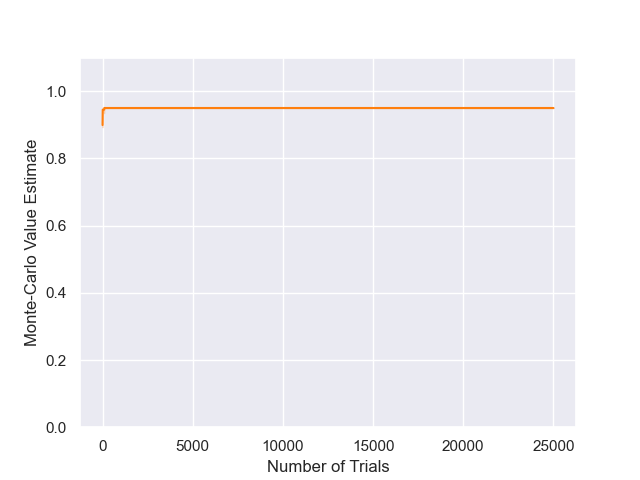}
                    \caption*{$\alpha=0.15,\epsilon=0.01$}
                \end{subfigure}
                \begin{subfigure}[b]{0.24\textwidth}
                    \centering
                    \includegraphics[width=\textwidth]{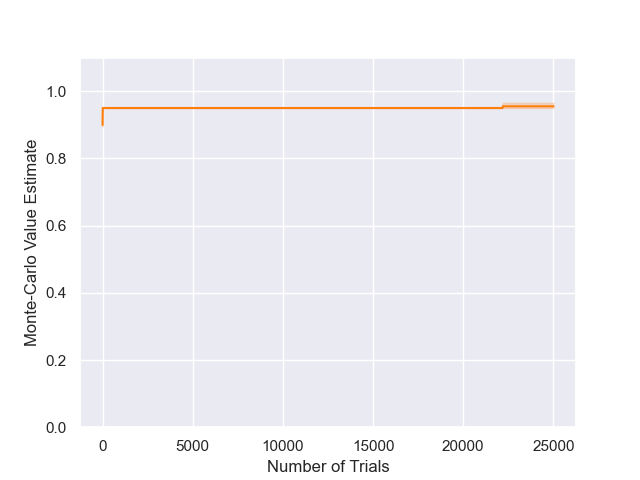}
                    \caption*{$\alpha=0.15,\epsilon=0.1$}
                \end{subfigure}
                \begin{subfigure}[b]{0.24\textwidth}
                    \centering
                    \includegraphics[width=\textwidth]{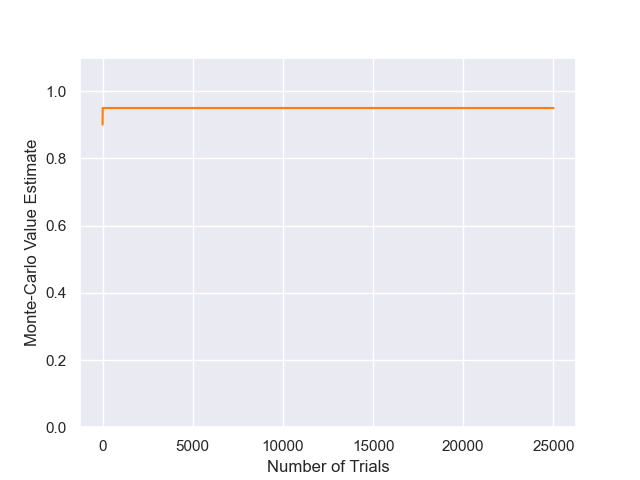}
                    \caption*{$\alpha=0.15,\epsilon=1$}
                \end{subfigure}
                \begin{subfigure}[b]{0.24\textwidth}
                    \centering
                    \includegraphics[width=\textwidth]{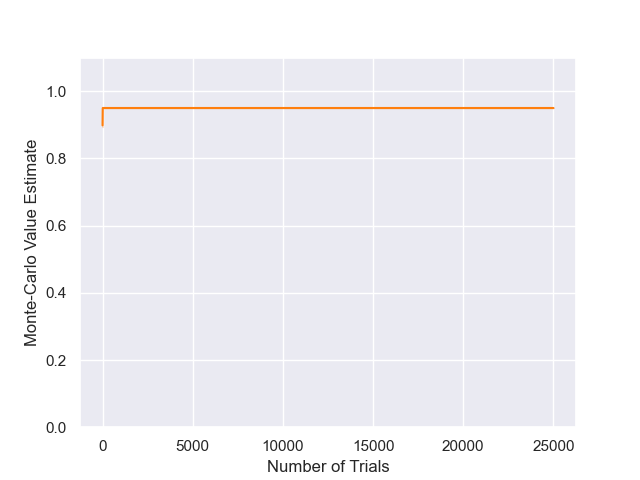}
                    \caption*{$\alpha=0.15,\epsilon=10$}
                \end{subfigure}
                
                \begin{subfigure}[b]{0.24\textwidth}
                    \centering
                    \includegraphics[width=\textwidth]{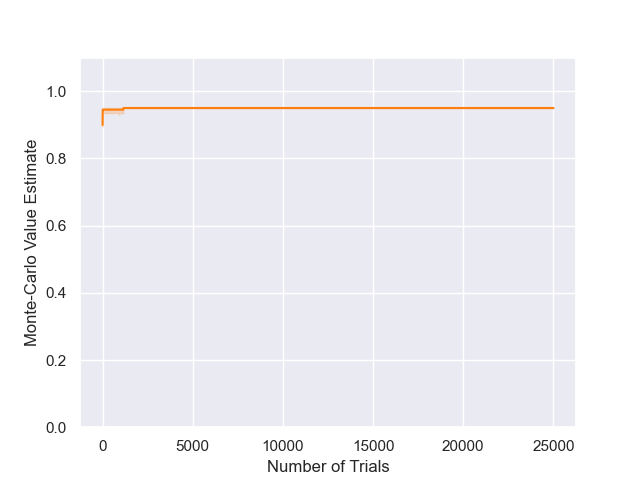}
                    \caption*{$\alpha=0.1,\epsilon=0.01$}
                \end{subfigure}
                \begin{subfigure}[b]{0.24\textwidth}
                    \centering
                    \includegraphics[width=\textwidth]{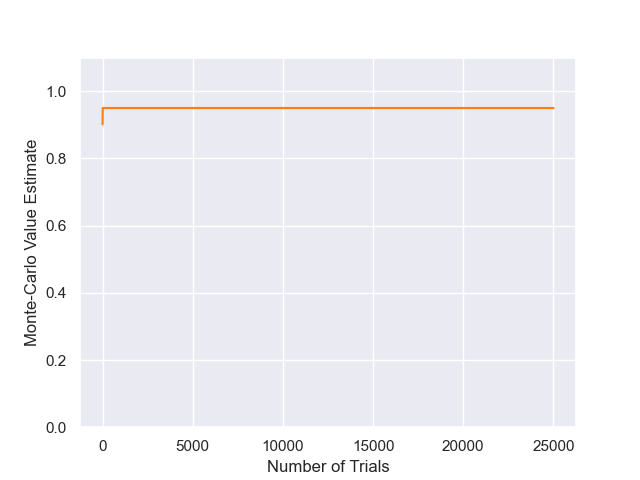}
                    \caption*{$\alpha=0.1,\epsilon=0.1$}
                \end{subfigure}
                \begin{subfigure}[b]{0.24\textwidth}
                    \centering
                    \includegraphics[width=\textwidth]{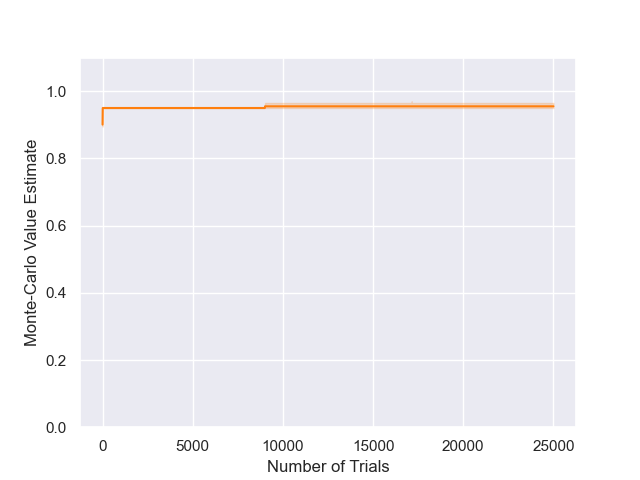}
                    \caption*{$\alpha=0.1,\epsilon=1$}
                \end{subfigure}
                \begin{subfigure}[b]{0.24\textwidth}
                    \centering
                    \includegraphics[width=\textwidth]{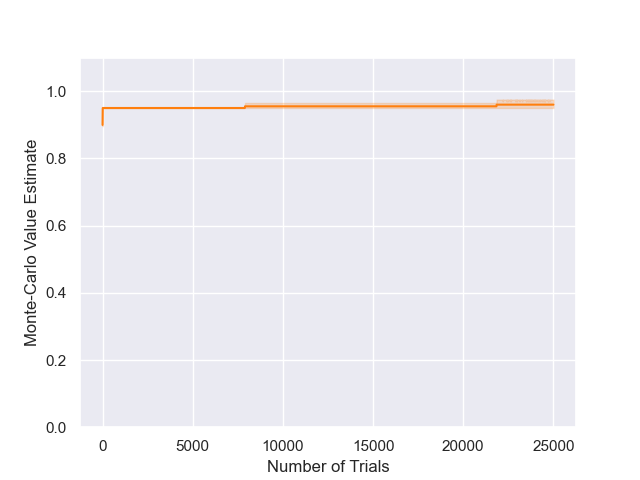}
                    \caption*{$\alpha=0.1,\epsilon=10$}
                \end{subfigure}
                
                \begin{subfigure}[b]{0.24\textwidth}
                    \centering
                    \includegraphics[width=\textwidth]{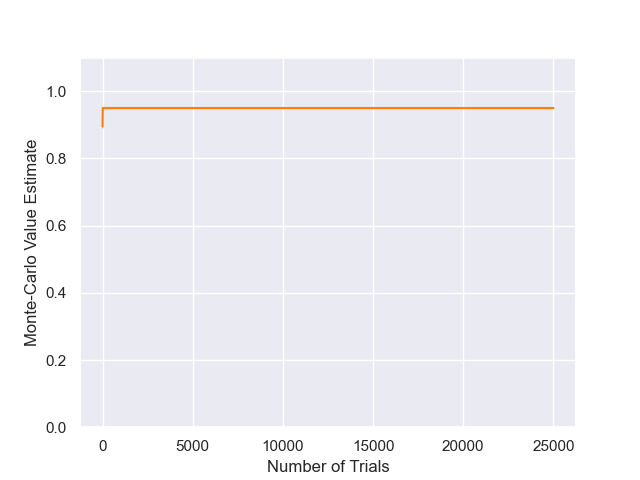}
                    \caption*{$\alpha=0.05,\epsilon=0.01$}
                \end{subfigure}
                \begin{subfigure}[b]{0.24\textwidth}
                    \centering
                    \includegraphics[width=\textwidth]{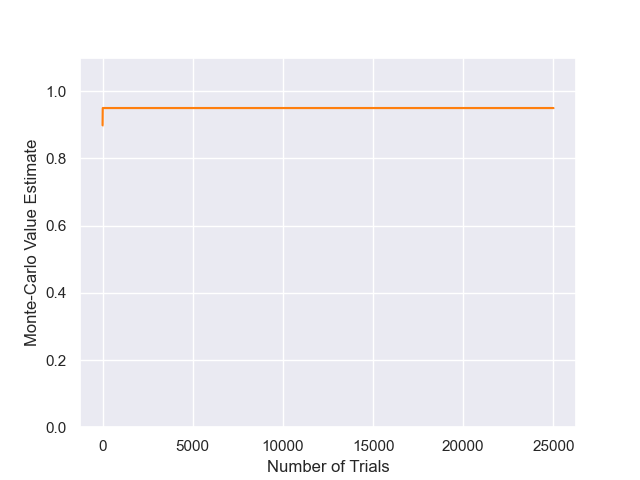}
                    \caption*{$\alpha=0.05,\epsilon=0.1$}
                \end{subfigure}
                \begin{subfigure}[b]{0.24\textwidth}
                    \centering
                    \includegraphics[width=\textwidth]{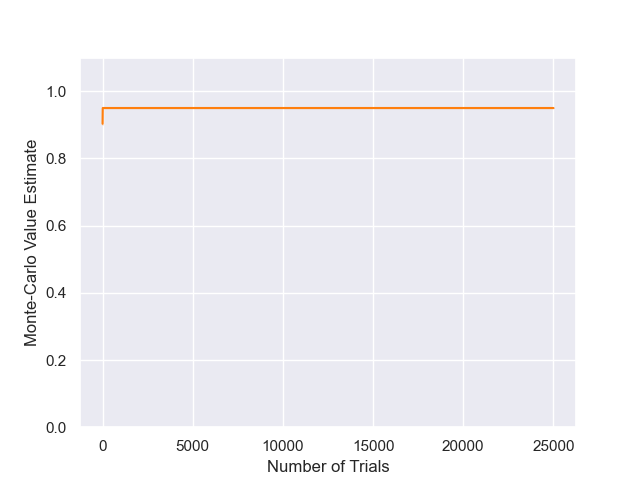}
                    \caption*{$\alpha=0.05,\epsilon=1$}
                \end{subfigure}
                \begin{subfigure}[b]{0.24\textwidth}
                    \centering
                    \includegraphics[width=\textwidth]{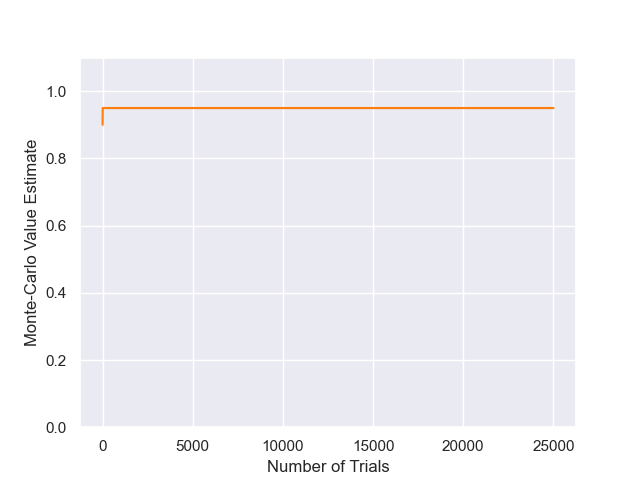}
                    \caption*{$\alpha=0.05,\epsilon=10$}
                \end{subfigure}
                
                \caption{Results for DENTS on the 20-chain ($D=20$, $R_f=1.0$), for varying temperatures and exploration parameters. The decay function was set to $\beta(m)=\alpha/\log(e+m)$.}
                \label{fig:dents_20chain_hps}
            \end{figure}

            \begin{figure}
                \centering
                
                \begin{subfigure}[b]{0.24\textwidth}
                    \centering
                    \includegraphics[width=\textwidth]{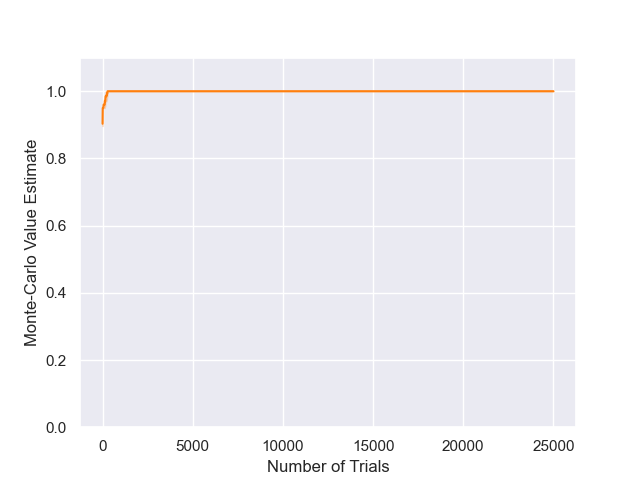}
                    \caption*{$\alpha=10,\epsilon=0.01$}
                \end{subfigure}
                \begin{subfigure}[b]{0.24\textwidth}
                    \centering
                    \includegraphics[width=\textwidth]{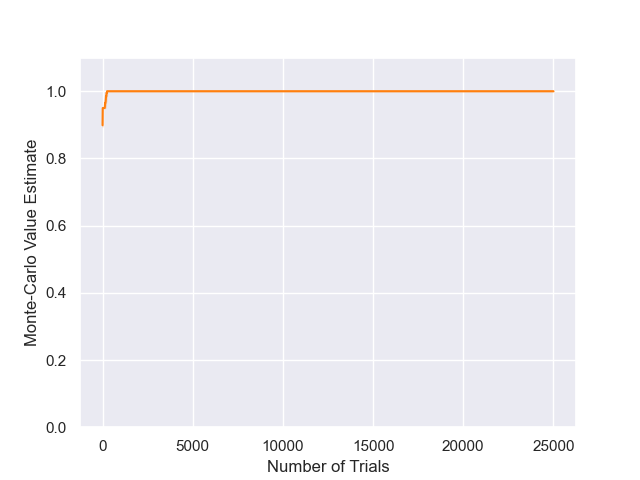}
                    \caption*{$\alpha=10,\epsilon=0.1$}
                \end{subfigure}
                \begin{subfigure}[b]{0.24\textwidth}
                    \centering
                    \includegraphics[width=\textwidth]{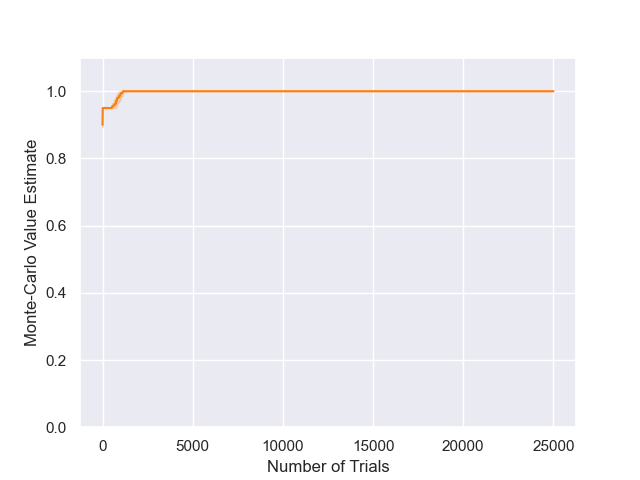}
                    \caption*{$\alpha=10,\epsilon=1$}
                \end{subfigure}
                \begin{subfigure}[b]{0.24\textwidth}
                    \centering
                    \includegraphics[width=\textwidth]{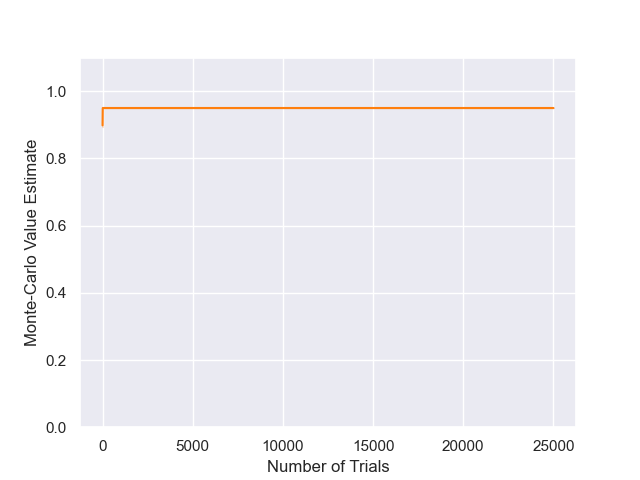}
                    \caption*{$\alpha=10,\epsilon=10$}
                \end{subfigure}
                
                \begin{subfigure}[b]{0.24\textwidth}
                    \centering
                    \includegraphics[width=\textwidth]{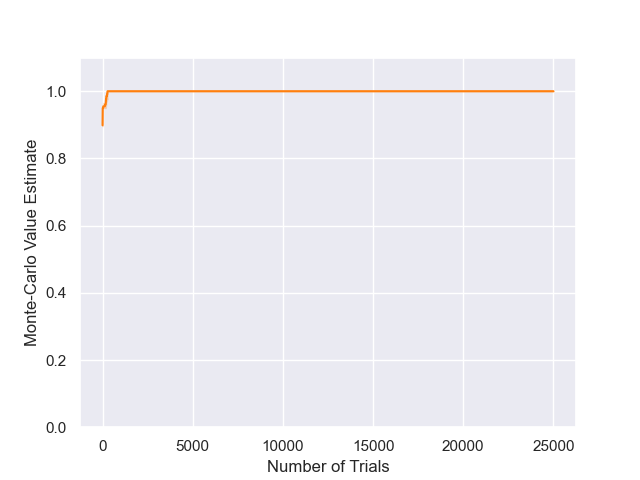}
                    \caption*{$\alpha=1,\epsilon=0.01$}
                \end{subfigure}
                \begin{subfigure}[b]{0.24\textwidth}
                    \centering
                    \includegraphics[width=\textwidth]{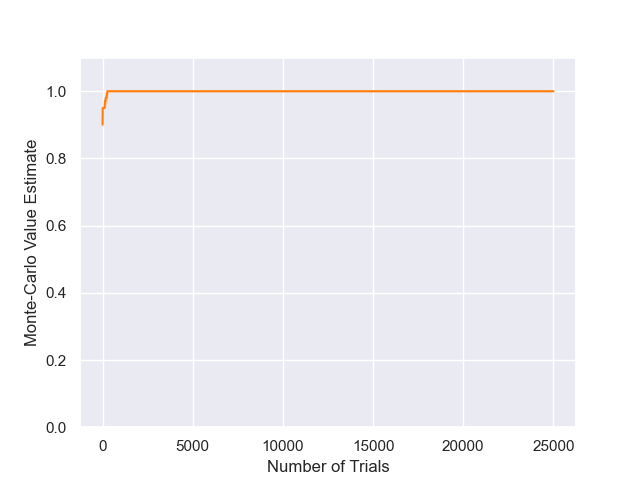}
                    \caption*{$\alpha=1,\epsilon=0.1$}
                \end{subfigure}
                \begin{subfigure}[b]{0.24\textwidth}
                    \centering
                    \includegraphics[width=\textwidth]{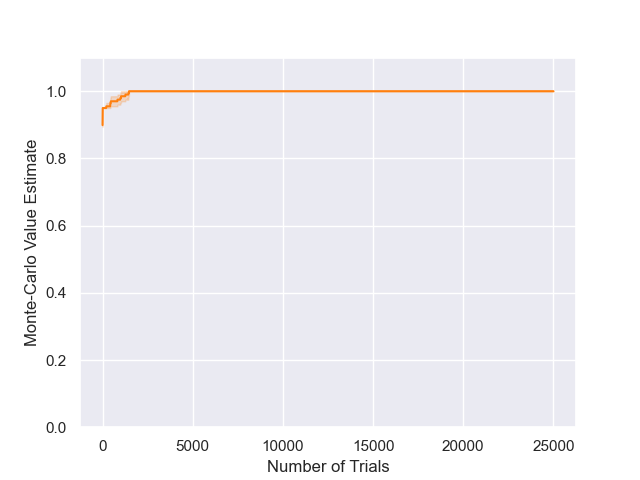}
                    \caption*{$\alpha=1,\epsilon=1$}
                \end{subfigure}
                \begin{subfigure}[b]{0.24\textwidth}
                    \centering
                    \includegraphics[width=\textwidth]{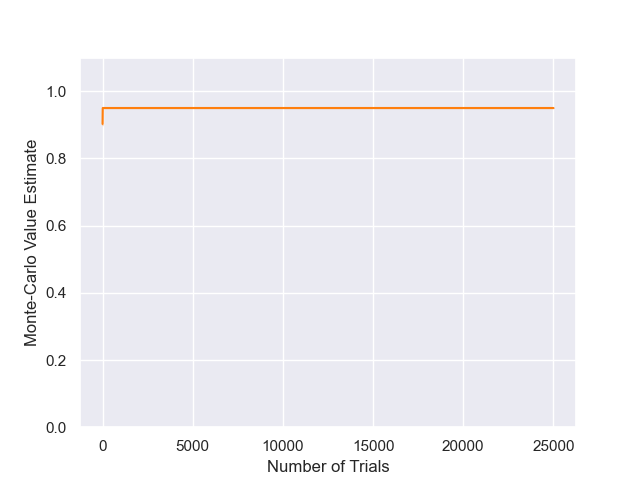}
                    \caption*{$\alpha=1,\epsilon=10$}
                \end{subfigure}
                
                \begin{subfigure}[b]{0.24\textwidth}
                    \centering
                    \includegraphics[width=\textwidth]{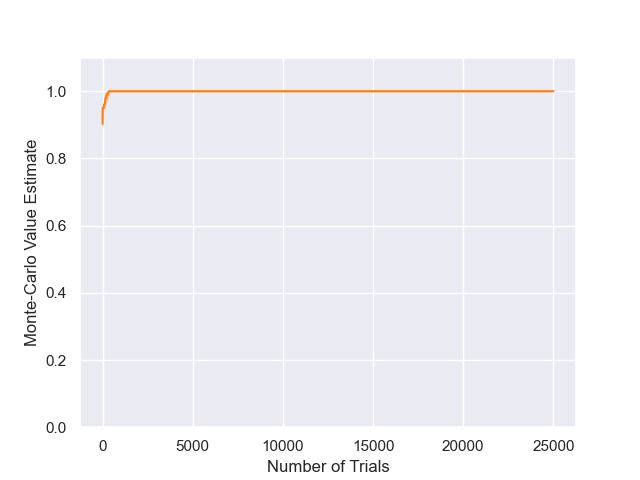}
                    \caption*{$\alpha=0.5,\epsilon=0.01$}
                \end{subfigure}
                \begin{subfigure}[b]{0.24\textwidth}
                    \centering
                    \includegraphics[width=\textwidth]{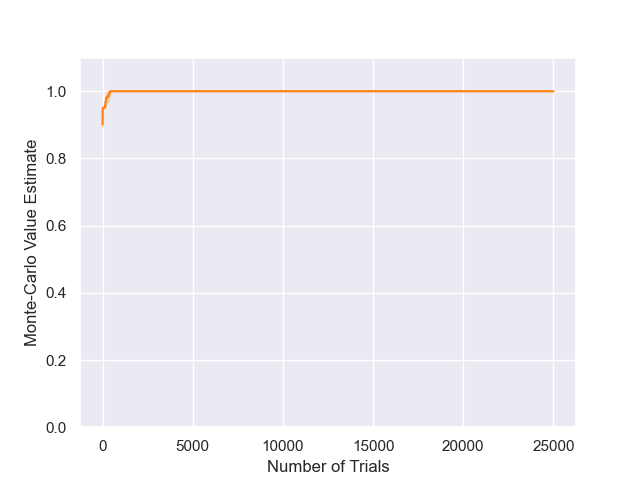}
                    \caption*{$\alpha=0.5,\epsilon=0.1$}
                \end{subfigure}
                \begin{subfigure}[b]{0.24\textwidth}
                    \centering
                    \includegraphics[width=\textwidth]{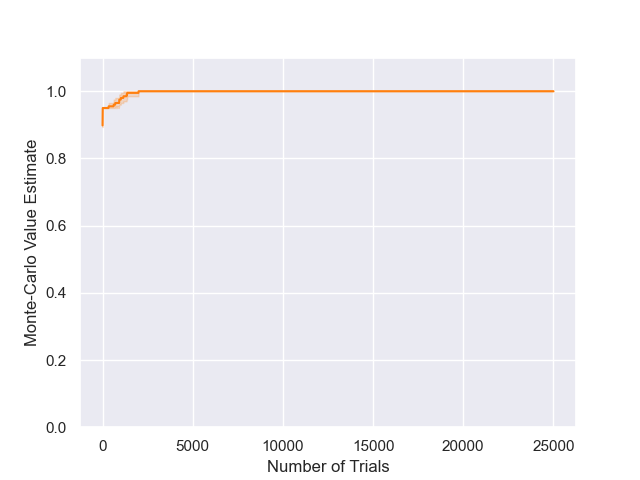}
                    \caption*{$\alpha=0.5,\epsilon=1$}
                \end{subfigure}
                \begin{subfigure}[b]{0.24\textwidth}
                    \centering
                    \includegraphics[width=\textwidth]{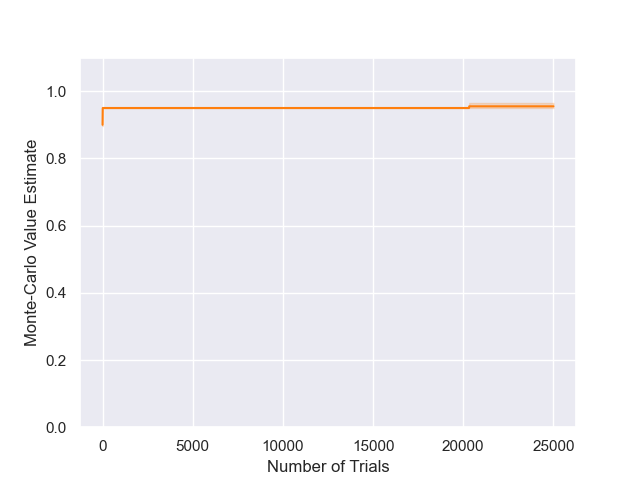}
                    \caption*{$\alpha=0.5,\epsilon=10$}
                \end{subfigure}
                
                \begin{subfigure}[b]{0.24\textwidth}
                    \centering
                    \includegraphics[width=\textwidth]{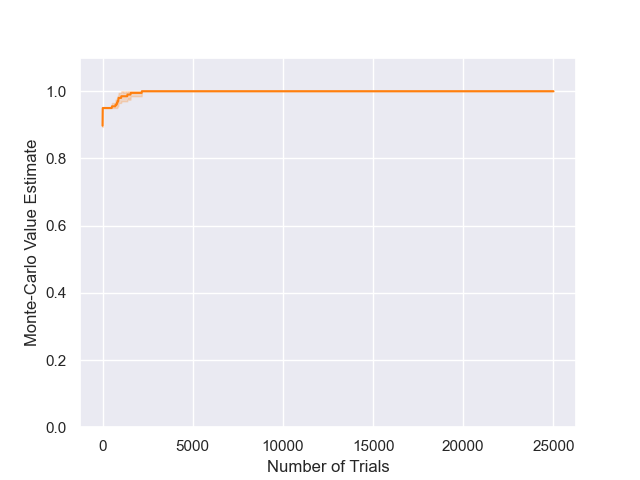}
                    \caption*{$\alpha=0.2,\epsilon=0.01$}
                \end{subfigure}
                \begin{subfigure}[b]{0.24\textwidth}
                    \centering
                    \includegraphics[width=\textwidth]{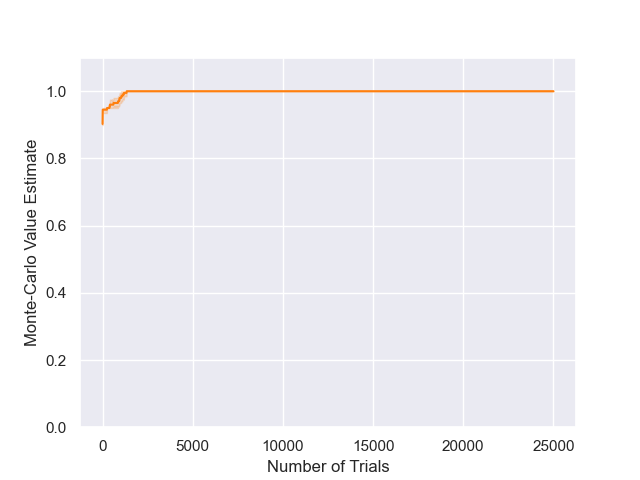}
                    \caption*{$\alpha=0.2,\epsilon=0.1$}
                \end{subfigure}
                \begin{subfigure}[b]{0.24\textwidth}
                    \centering
                    \includegraphics[width=\textwidth]{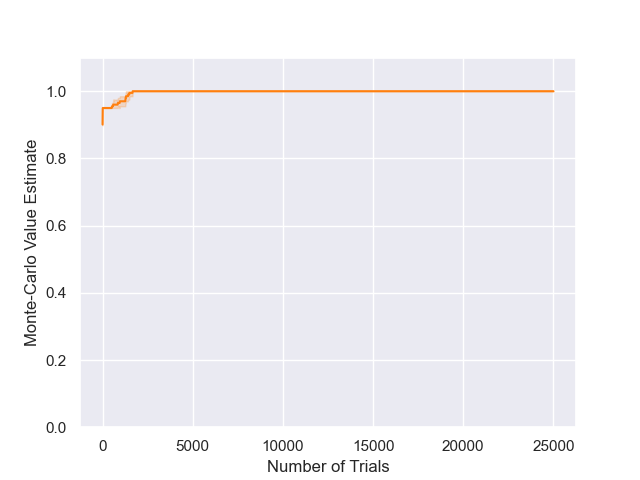}
                    \caption*{$\alpha=0.2,\epsilon=1$}
                \end{subfigure}
                \begin{subfigure}[b]{0.24\textwidth}
                    \centering
                    \includegraphics[width=\textwidth]{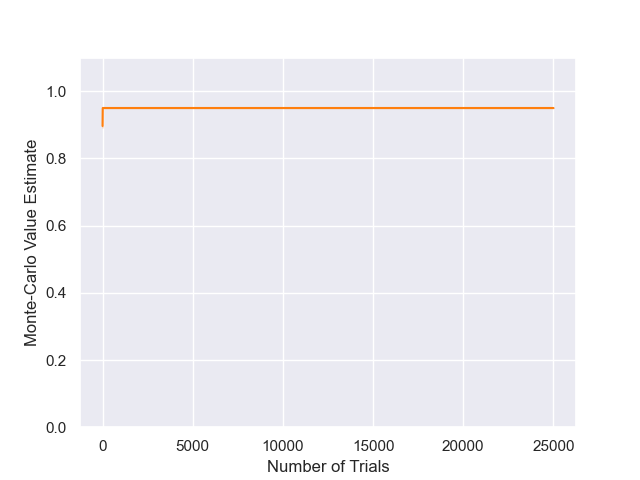}
                    \caption*{$\alpha=0.2,\epsilon=10$}
                \end{subfigure}
                
                \begin{subfigure}[b]{0.24\textwidth}
                    \centering
                    \includegraphics[width=\textwidth]{Figures/app/param_sens/dchain/005_20chain_01_dents_epsilon=0.01,temp=0.15.png}
                    \caption*{$\alpha=0.15,\epsilon=0.01$}
                \end{subfigure}
                \begin{subfigure}[b]{0.24\textwidth}
                    \centering
                    \includegraphics[width=\textwidth]{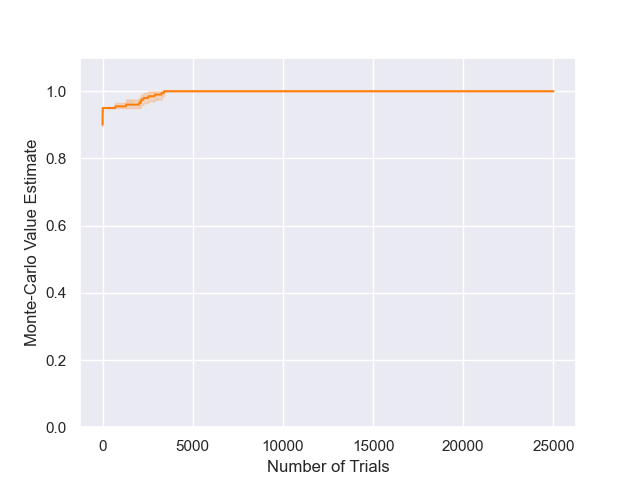}
                    \caption*{$\alpha=0.15,\epsilon=0.1$}
                \end{subfigure}
                \begin{subfigure}[b]{0.24\textwidth}
                    \centering
                    \includegraphics[width=\textwidth]{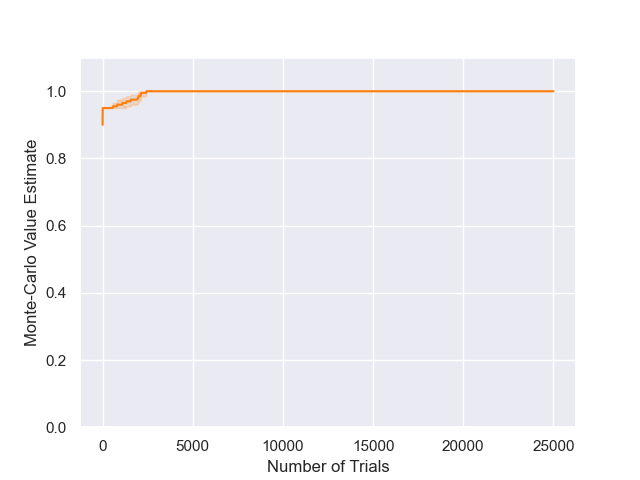}
                    \caption*{$\alpha=0.15,\epsilon=1$}
                \end{subfigure}
                \begin{subfigure}[b]{0.24\textwidth}
                    \centering
                    \includegraphics[width=\textwidth]{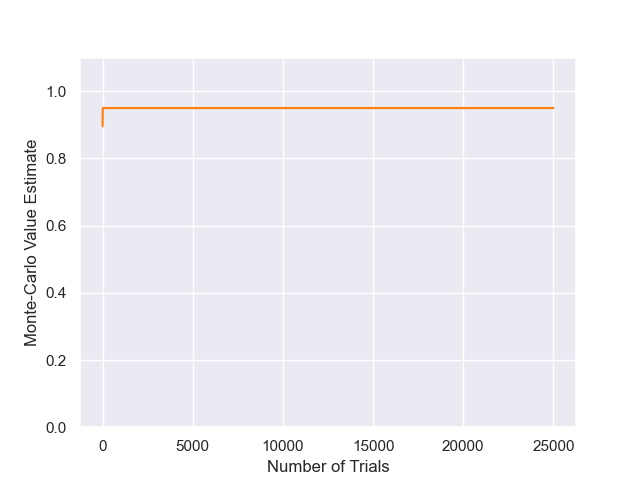}
                    \caption*{$\alpha=0.15,\epsilon=10$}
                \end{subfigure}
                
                \begin{subfigure}[b]{0.24\textwidth}
                    \centering
                    \includegraphics[width=\textwidth]{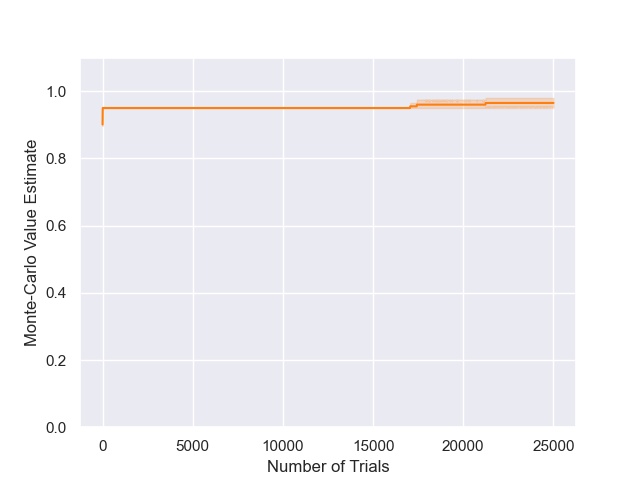}
                    \caption*{$\alpha=0.1,\epsilon=0.01$}
                \end{subfigure}
                \begin{subfigure}[b]{0.24\textwidth}
                    \centering
                    \includegraphics[width=\textwidth]{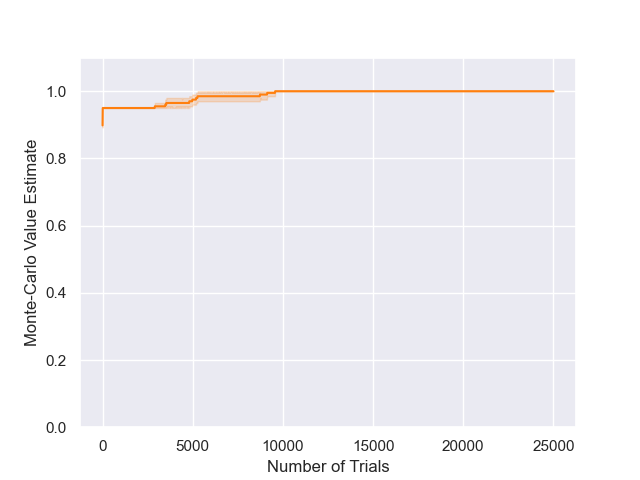}
                    \caption*{$\alpha=0.1,\epsilon=0.1$}
                \end{subfigure}
                \begin{subfigure}[b]{0.24\textwidth}
                    \centering
                    \includegraphics[width=\textwidth]{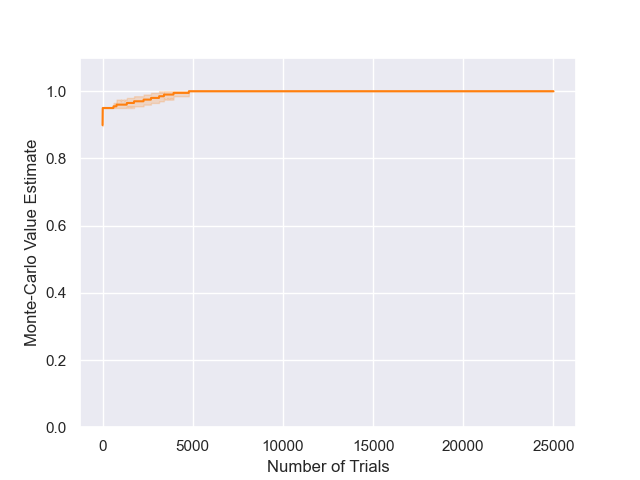}
                    \caption*{$\alpha=0.1,\epsilon=1$}
                \end{subfigure}
                \begin{subfigure}[b]{0.24\textwidth}
                    \centering
                    \includegraphics[width=\textwidth]{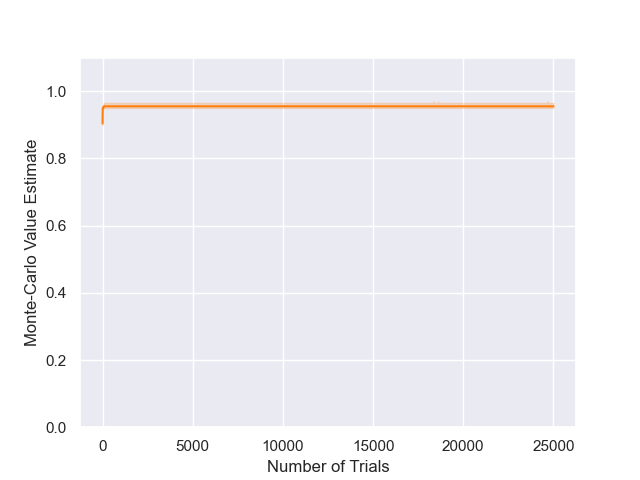}
                    \caption*{$\alpha=0.1,\epsilon=10$}
                \end{subfigure}
                
                \begin{subfigure}[b]{0.24\textwidth}
                    \centering
                    \includegraphics[width=\textwidth]{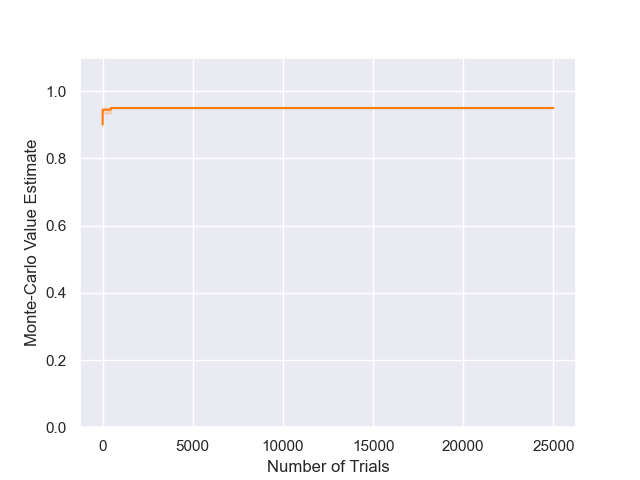}
                    \caption*{$\alpha=0.05,\epsilon=0.01$}
                \end{subfigure}
                \begin{subfigure}[b]{0.24\textwidth}
                    \centering
                    \includegraphics[width=\textwidth]{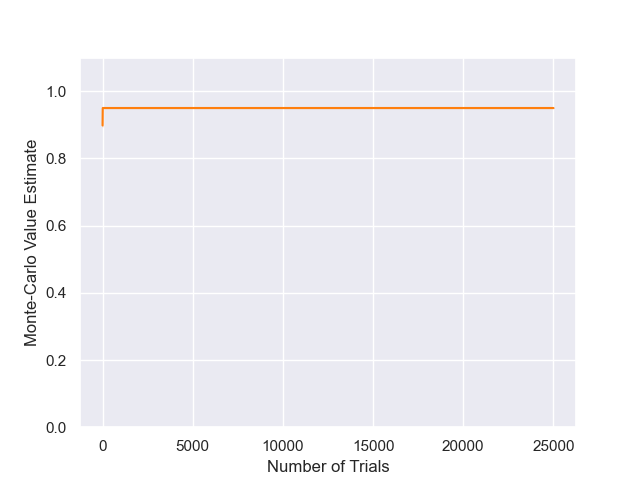}
                    \caption*{$\alpha=0.05,\epsilon=0.1$}
                \end{subfigure}
                \begin{subfigure}[b]{0.24\textwidth}
                    \centering
                    \includegraphics[width=\textwidth]{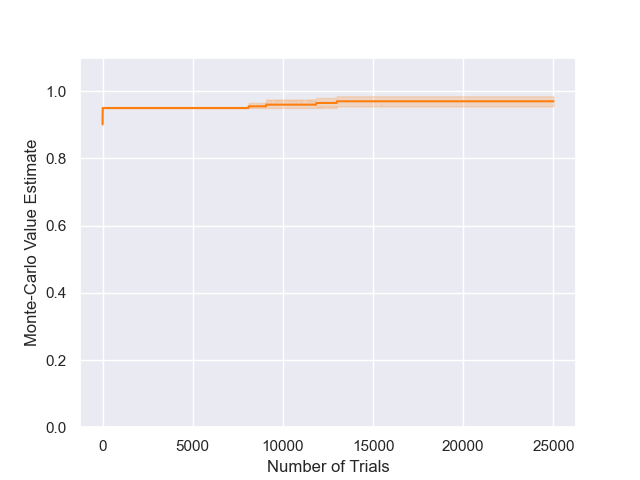}
                    \caption*{$\alpha=0.05,\epsilon=1$}
                \end{subfigure}
                \begin{subfigure}[b]{0.24\textwidth}
                    \centering
                    \includegraphics[width=\textwidth]{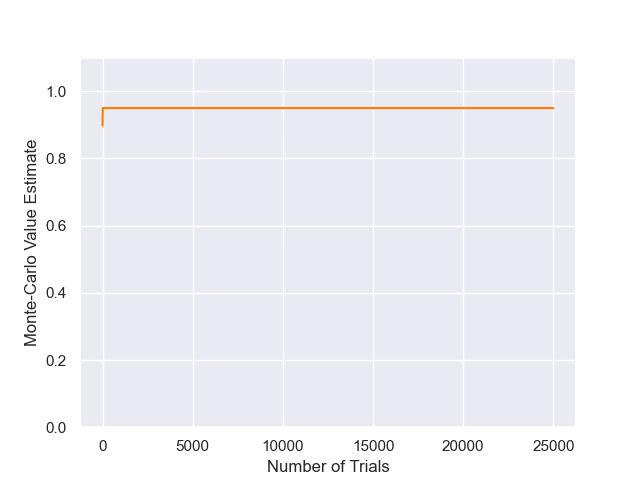}
                    \caption*{$\alpha=0.05,\epsilon=10$}
                \end{subfigure}
                
                \caption{Results for DENTS on the 20-chain ($D=20$, $R_f=1.0$). The decay function was set to $\beta(m)=\alpha$, so that DENTS mimics MENTS search policy.}
                \label{fig:dbments_20chain_hps}
            \end{figure}

            \begin{figure}
                \centering
                \includegraphics[width=0.5\textwidth]{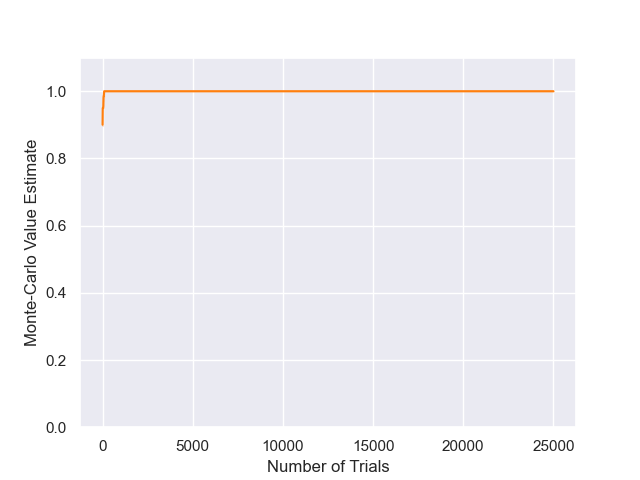}
                \caption{Results for DENTS on the 20-chain ($D=20$, $R_f=1.0$), with $\alpha=0.5$, $\beta(m)=10/\log(e+m)$ and $\epsilon=0.01$.}
                \label{fig:dents_20chain_tuned}
            \end{figure}

        \FloatBarrier
        \subsubsection{Parameter sensitivity in Frozen Lake} \label{app:param_sens_fl}
            We also ran MENTS, RENTS, TENTS, BTS and DENTS with a variety of temperatures on the 8x8 Frozen Lake environment given in Figure \ref{fig:fl8}. Again, we set $\beta(m)=\alpha/\log(e+m)$ for the decay function in DENTS, and we used an exploration parameter of $\epsilon=1$ for all of the algorithms. 

            In Figures \ref{fig:fl_param_sens_ments} and \ref{fig:fl_param_sens_tents} we can see that MENTS and TENTS take the scenic route to the goal state for medium temperatures, where the reward for reaching the goal is still significant, but they can obtain more entropy reward by wondering around the gridworld for a while first. For higher temperatures they completely ignore the goal state, opting to rather maximise policy entropy. Interestingly, RENTS in Figure \ref{fig:fl_param_sens_rents} fared better than MENTS and TENTS and never really ignored the goal state at the temperatures that we considered.

            In contrast, both BTS and DENTS were agnostic to the temperature parameter in this environment (with $\epsilon=1$) and were always able to find the goal state. We include the plots for BTS and DENTS in all of Figures \ref{fig:fl_param_sens_ments}, \ref{fig:fl_param_sens_tents} and \ref{fig:fl_param_sens_rents} for reference and as a comparison for MENTS, RENTS and TENTS.
        
            \FloatBarrier
            
            \begin{figure}
                \centering
                
                \begin{subfigure}[b]{0.32\textwidth}
                    \centering
                    \includegraphics[width=\textwidth]{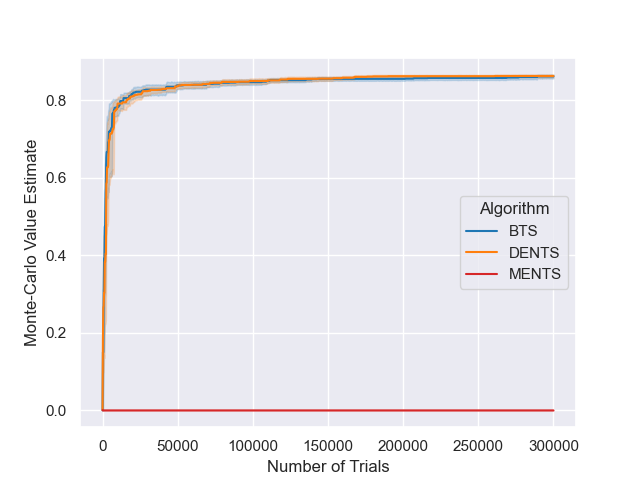}
                    \caption{$\alpha=1$}
                \end{subfigure}
                \begin{subfigure}[b]{0.32\textwidth}
                    \centering
                    \includegraphics[width=\textwidth]{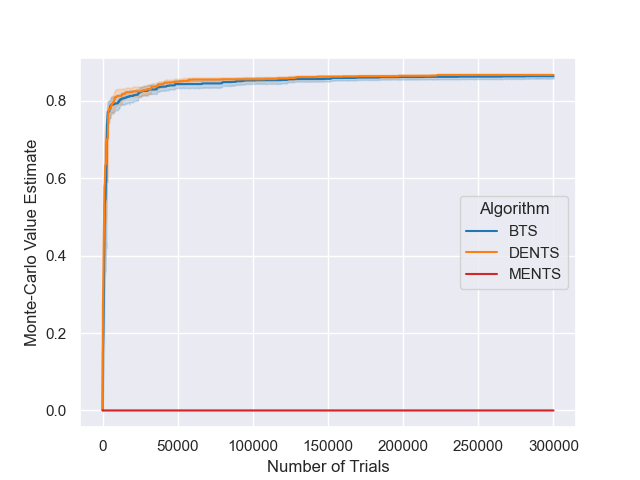}
                    \caption{$\alpha=0.5$}
                \end{subfigure}
                \begin{subfigure}[b]{0.32\textwidth}
                    \centering
                    \includegraphics[width=\textwidth]{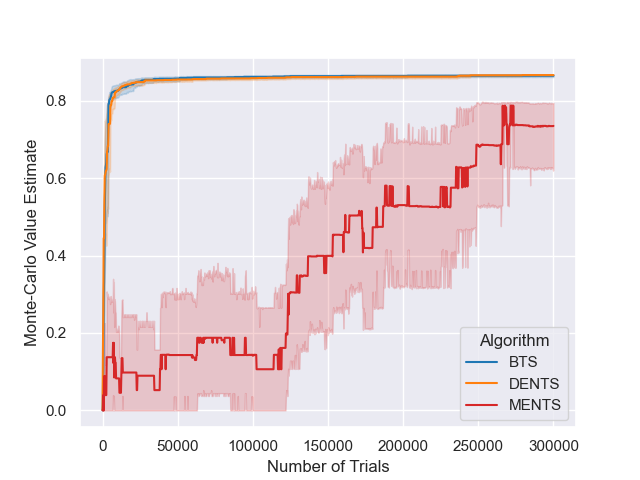}
                    \caption{$\alpha=0.1$}
                \end{subfigure}
                
                \begin{subfigure}[b]{0.32\textwidth}
                    \centering
                    \includegraphics[width=\textwidth]{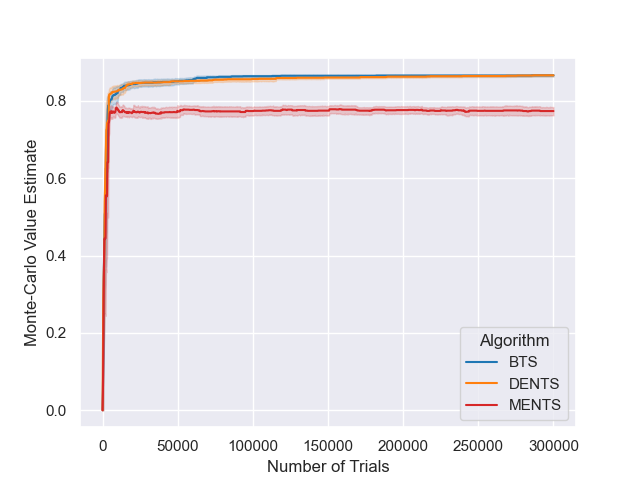}
                    \caption{$\alpha=0.05$}
                \end{subfigure}
                \begin{subfigure}[b]{0.32\textwidth}
                    \centering
                    \includegraphics[width=\textwidth]{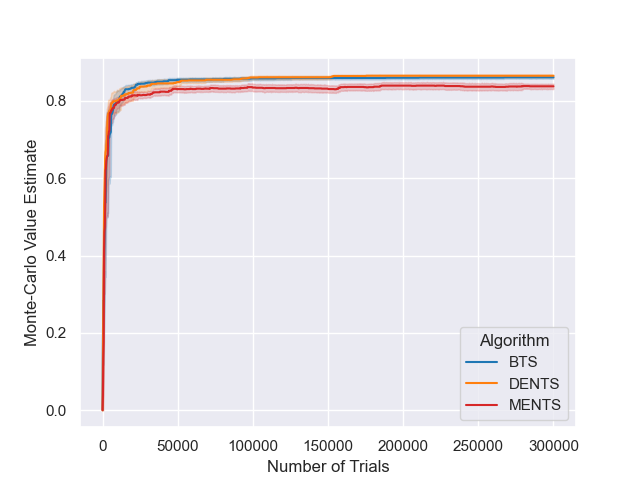}
                    \caption{$\alpha=0.01$}
                \end{subfigure}
                \begin{subfigure}[b]{0.32\textwidth}
                    \centering
                    \includegraphics[width=\textwidth]{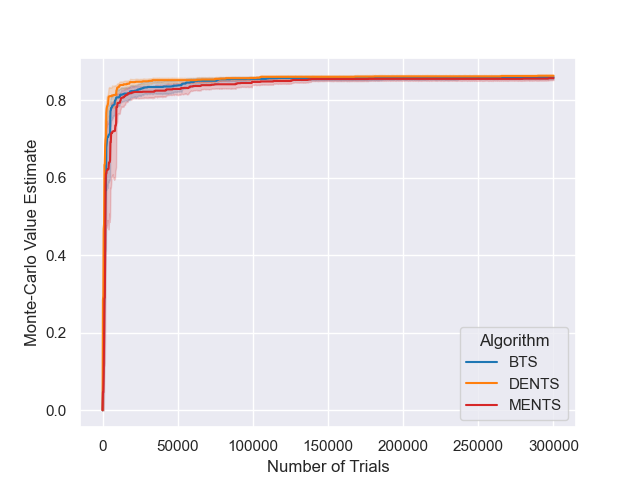}
                    \caption{$\alpha=0.005$}
                \end{subfigure}
                
                \caption{MENTS with a variety of temperatures on an 8x8 Frozen Lake environment. BTS and DENTS are included for reference.}
                \label{fig:fl_param_sens_ments}
            \end{figure}
            
            \begin{figure}
                \centering
                
                \begin{subfigure}[b]{0.32\textwidth}
                    \centering
                    \includegraphics[width=\textwidth]{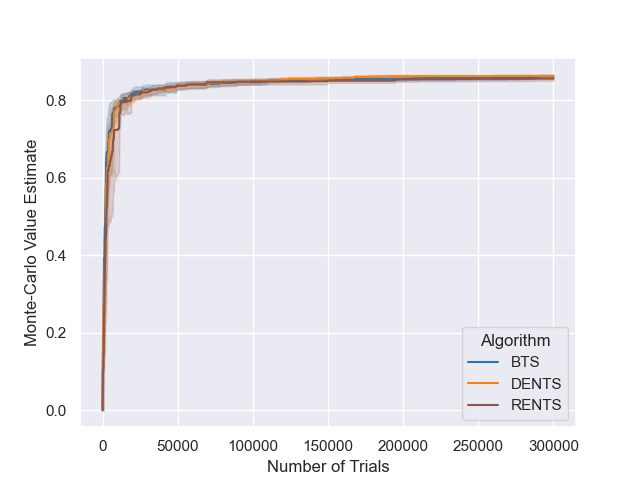}
                    \caption{$\alpha=1$}
                \end{subfigure}
                \begin{subfigure}[b]{0.32\textwidth}
                    \centering
                    \includegraphics[width=\textwidth]{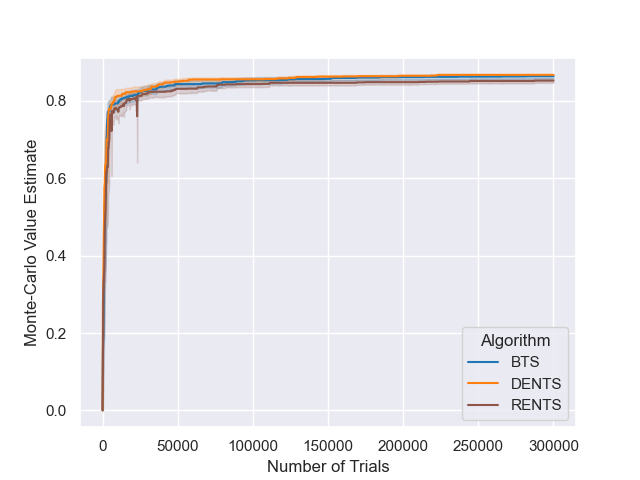}
                    \caption{$\alpha=0.5$}
                \end{subfigure}
                \begin{subfigure}[b]{0.32\textwidth}
                    \centering
                    \includegraphics[width=\textwidth]{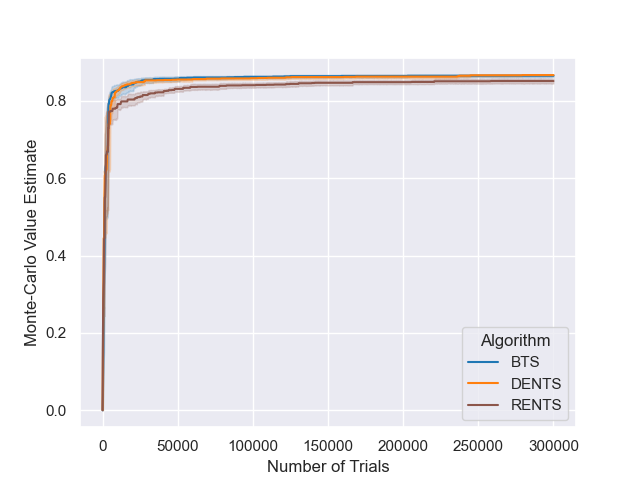}
                    \caption{$\alpha=0.1$}
                \end{subfigure}
                
                \begin{subfigure}[b]{0.32\textwidth}
                    \centering
                    \includegraphics[width=\textwidth]{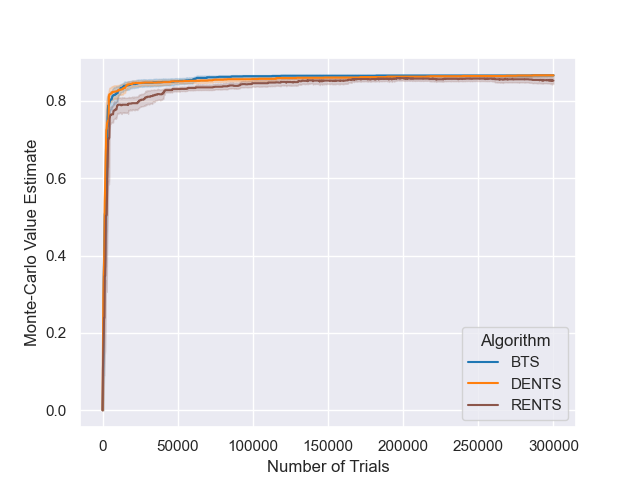}
                    \caption{$\alpha=0.05$}
                \end{subfigure}
                \begin{subfigure}[b]{0.32\textwidth}
                    \centering
                    \includegraphics[width=\textwidth]{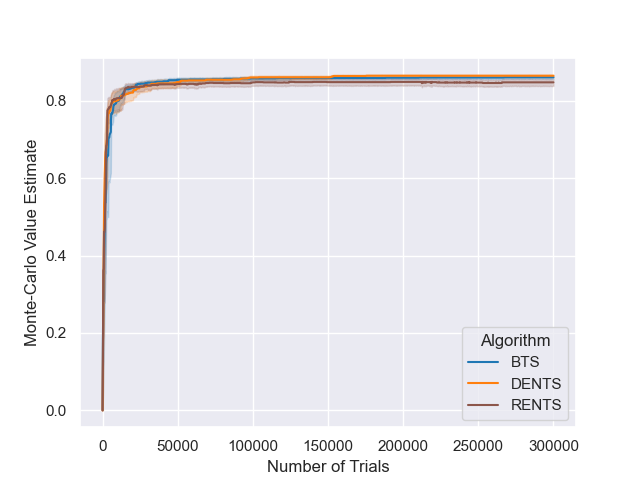}
                    \caption{$\alpha=0.01$}
                \end{subfigure}
                \begin{subfigure}[b]{0.32\textwidth}
                    \centering
                    \includegraphics[width=\textwidth]{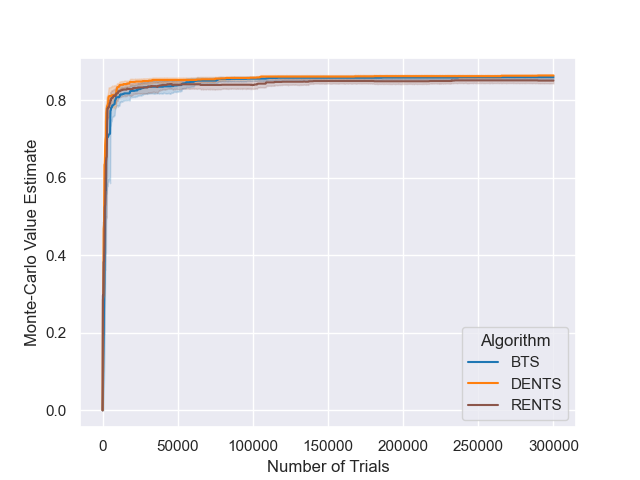}
                    \caption{$\alpha=0.005$}
                \end{subfigure}
                
                \caption{RENTS with a variety of temperatures on an 8x8 Frozen Lake environment. BTS and DENTS are included for reference.}
                \label{fig:fl_param_sens_rents}
            \end{figure}
            
            \begin{figure}
                \centering
                
                \begin{subfigure}[b]{0.32\textwidth}
                    \centering
                    \includegraphics[width=\textwidth]{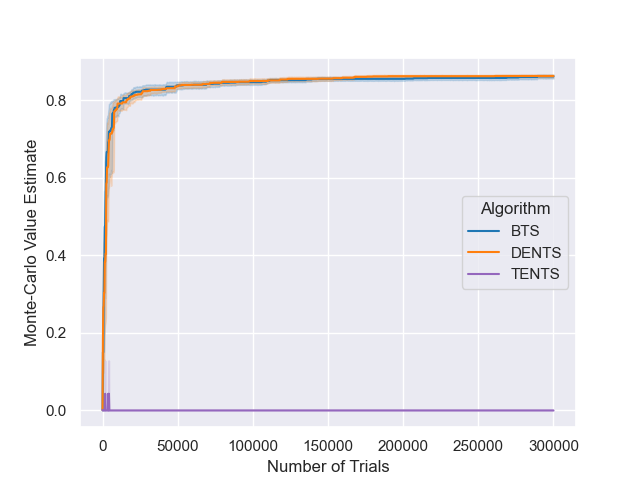}
                    \caption{$\alpha=1$}
                \end{subfigure}
                \begin{subfigure}[b]{0.32\textwidth}
                    \centering
                    \includegraphics[width=\textwidth]{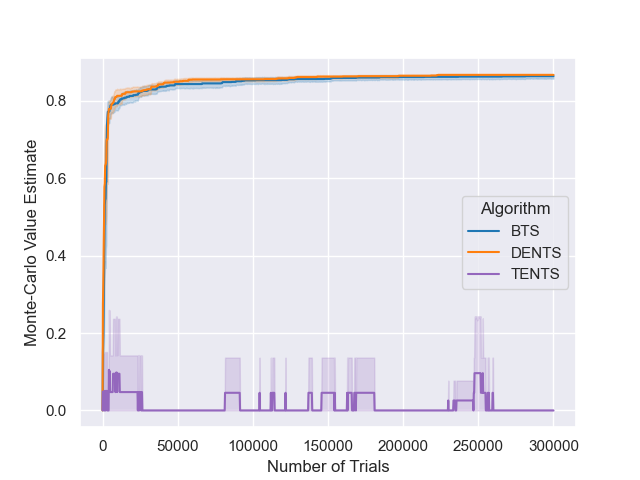}
                    \caption{$\alpha=0.5$}
                \end{subfigure}
                \begin{subfigure}[b]{0.32\textwidth}
                    \centering
                    \includegraphics[width=\textwidth]{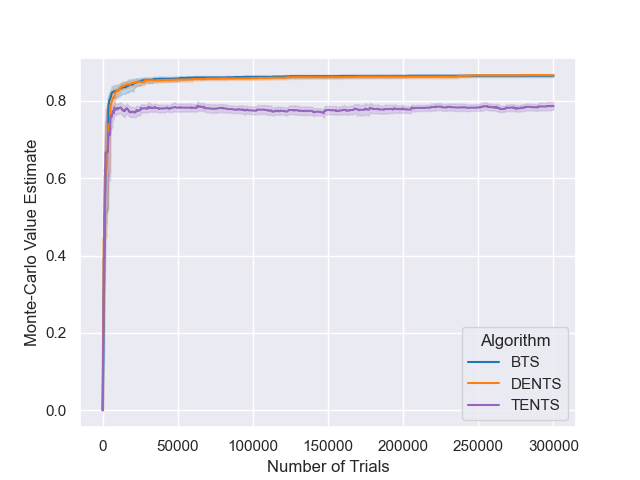}
                    \caption{$\alpha=0.1$}
                \end{subfigure}
                
                \begin{subfigure}[b]{0.32\textwidth}
                    \centering
                    \includegraphics[width=\textwidth]{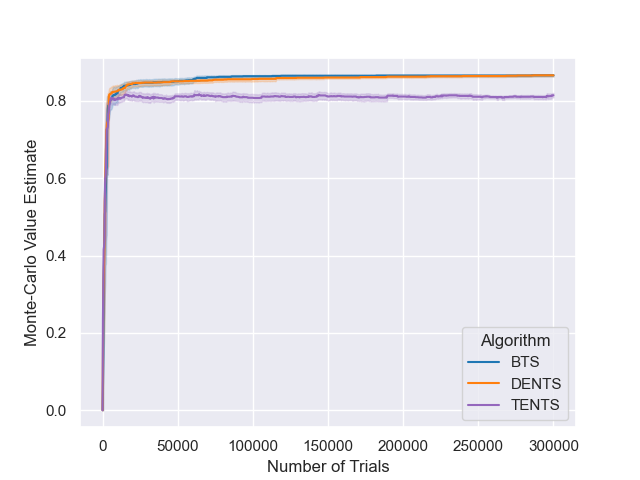}
                    \caption{$\alpha=0.05$}
                \end{subfigure}
                \begin{subfigure}[b]{0.32\textwidth}
                    \centering
                    \includegraphics[width=\textwidth]{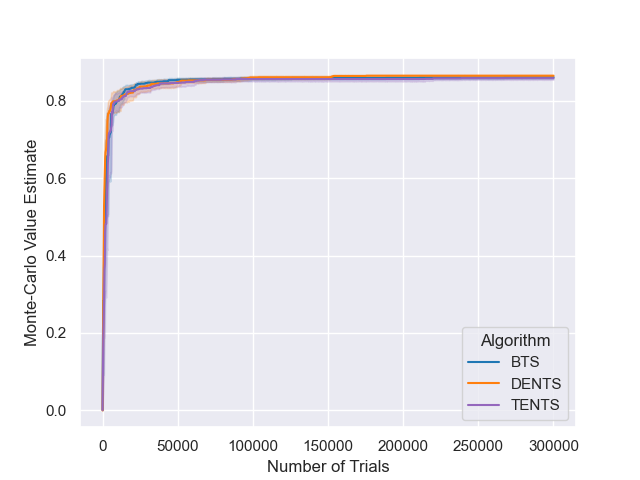}
                    \caption{$\alpha=0.01$}
                \end{subfigure}
                \begin{subfigure}[b]{0.32\textwidth}
                    \centering
                    \includegraphics[width=\textwidth]{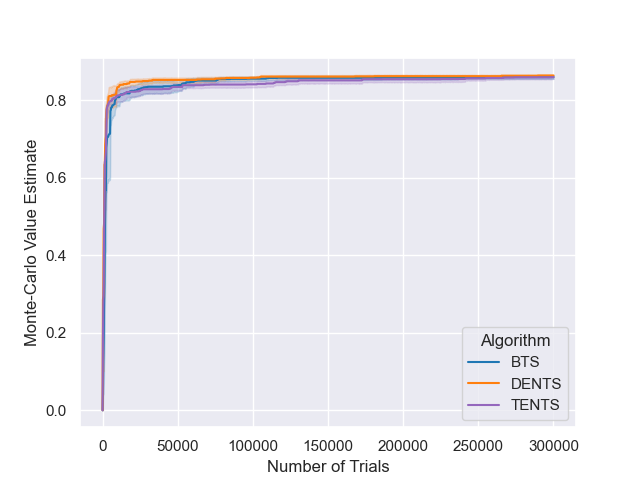}
                    \caption{$\alpha=0.005$}
                \end{subfigure}
                
                \caption{TENTS with a variety of temperatures on an 8x8 Frozen Lake environment. BTS and DENTS are included for reference.}
                \label{fig:fl_param_sens_tents}
            \end{figure}

            \FloatBarrier

    \subsection{Grid world hyper-parameter search and additional results} \label{app:hps}
        To select hyper-parameters for the experiments detailed in Section \ref{sec:gridworlds}, we performed a hyper-parameter search. The search was run on the 8x12 Frozen Lake environment from Figure \ref{fig:fl12}, and the results in Section \ref{sec:gridworlds} were run on the 8x12 Frozen Lake envrionment from Figure \ref{fig:fl12test}. For the Sailing problem, we performed the search using an initial wind direction of North, and the result in Section \ref{sec:results} used an initial wind direction of South-East.

        To avoid the search space from becoming too large, we set some parameters manually. A good rule of thumb for initial values is to assure that $Q^{\text{init}}_{\text{sft}}(s,a) < \Qss{s}{a}$ and $Q^{\text{init}}(s,a) < Q^*(s,a)$. Explicitly this means that an initial value of zero is \textit{not} a good choice for the Sailing problem, as rewards are negative (i.e. it has costs). In the Sailing environment, we actually set the initial values to $-200$, so that they were equal to the lowest possible return from a trial (the trial length was set to $50$, and an agent can incur a cost of at most $-4$ per timestep). To simplify the search space, we initially set the decay function in DENTS to $\beta(m)=\alpha/\log(e+m)$ and tune it after. 

        For the remaining parameters, we considered all combinations of the following values:
        \begin{itemize}
            \item \textit{UCT Bias}: \citeapp{prst}, 100.0, 10.0, 1.0, 0.1;
            \item \textit{MENTS exploration coefficient}: 2.0, 1.0, 0.3, 0.1, 0.03, 0.01;
            \item \textit{Temperature}: 100.0, 10.0, 1.0, 0.1, 0.01, 0.001;
            \item \textit{HMCTS UCT budget}: 100000, 30000, 10000, 3000, 1000, 300, 100, 30, 10.
        \end{itemize}
        
         Where a UCT bias of \citeapp{prst} refers to the adaptive bias introduced by Keller and Eyerich \citeapp{prst}, and `temperature' refers to the relevant temperature for the algorithm (i.e. search temperature in BTS and DENTS, and the temperature for Shannon/Relative/Tsallis entropy in MENTS/RENTS/DENTS). After that search, for DENTS we considered the decay functions of the form $\beta(m)=\beta_{\text{init}}/\log(e+m)$ and considered the following values:
        \begin{itemize}
        		\item \textit{DENTS initial entropy temperature} ($\beta_{\text{init}}$): 100.0, 10.0, 1.0, 0.1.
        \end{itemize}
        The final set of hyperparameters is given in Tables \ref{table:hyper_fl} and \ref{table:hyper_s}, which were used in the gridworld experiments in Section \ref{sec:gridworlds}. Not included in the tables: \citeapp{prst} was selected for the UCT bias in both Frozen Lake and the Sailing problem.
    
        \begin{table}[]
            \centering
            \begin{tabular}{r|ccc} 
                Algorithm   & Exploration Parameter ($\epsilon$)    & Temperature ($\alpha$)    & Initial Values ($Q^{\text{init}}$, $Q^{\text{init}}_{\text{sft}}$)    \\
                \hline
                MENTS       & 1.0                                   & 0.001                     & 0                                                                     \\
                RENTS       & 2.0                                   & 0.001                     & 0                                                                     \\
                TENTS       & 1.0                                   & 0.001                     & 0                                                                     \\
                BTS         & 2.0                                   & 0.1                       & 0                                                                     \\
                DENTS       & 1.0                                   & 0.1                       & 0                                                                     \\
            \end{tabular}
            \caption[Final hyperparameters used for Frozen Lake in Section \ref{sec:gridworlds}]{Final hyperparameters used for Frozen Lake in Section \ref{sec:gridworlds}. Not included in the table: \citeapp{prst} was selected for the bias in UCT, \citeapp{prst} was selected for the bias and 3000 for the UCT budget in HMCTS and $\beta_{\text{init}}=1.0$ was selected as the initial entropy temperature for DENTS. \label{table:hyper_fl}}
        \end{table}
    
        \begin{table}[]
            \centering
            \begin{tabular}{r|ccc} 
                Algorithm   & Exploration Parameter ($\epsilon$)    & Temperature ($\alpha$)    & Initial Values ($Q^{\text{init}}$, $Q^{\text{init}}_{\text{sft}}$)    \\
                \hline
                MENTS       & 1.0                                   & 10.0                      & -200                                                                  \\
                RENTS       & 1.0                                   & 10.0                      & -200                                                                  \\
                TENTS       & 2.0                                   & 0.1                       & -200                                                                  \\
                BTS         & 1.0                                   & 10.0                      & -200                                                                  \\
                DENTS       & 1.0                                   & 10.0                      & -200                                                                  \\
            \end{tabular}
            \caption[Final hyperparameters used for the Sailing Problem in Section \ref{sec:gridworlds}]{Final hyperparameters used for the Sailing Problem in Section \ref{sec:gridworlds}. Not included in the table: \citeapp{prst} was selected for the bias in UCT, \citeapp{prst} was selected for the bias and 30 for the UCT budget in HMCTS and $\beta_{\text{init}}=10.0$ was selected as the initial entropy temperature for DENTS. \label{table:hyper_s}}
        \end{table}

	\subsection{DENTS with a constant $\beta$} \label{app:dents_mimic_ments}
		To empirically demonstrate that DENTS search policy can mimic the search policy of MENTS, we ran DENTS with $\alpha=1.0,\beta(m)=\alpha$ on the 10-chain environment, and compared it to MENTS with $\alpha=1.0$. We also ran DENTS with $\alpha=0.001, \beta(m)=\alpha$ in the Frozen Lake environment, and compared it to MENTS with $\alpha=0.001$ which is what was selected in the hpyerparameter search (Appendix \ref{app:hps}). Results are given in Figure \ref{fig:dbments}. Note that in the 10-chain, only MENTS has a dip in performance initially, which is due to the two algorithms using different recommendation policies.

        \begin{figure*}
            \centering
            \begin{subfigure}[b]{0.49\textwidth}
                \centering
                \includegraphics[width=\textwidth]{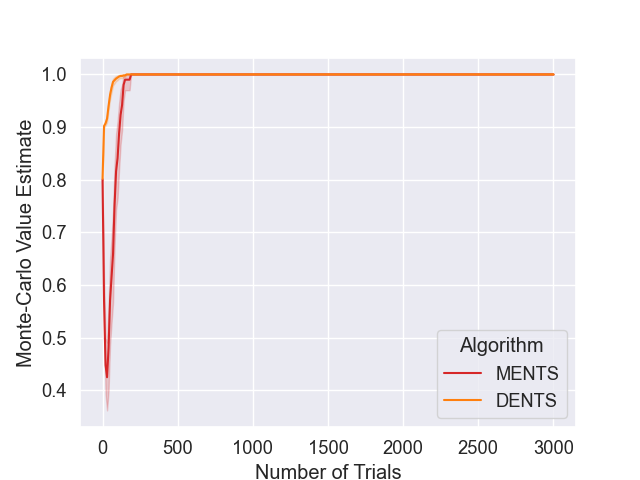}
                \caption{10-chain.}
            \end{subfigure}
            \begin{subfigure}[b]{0.49\textwidth}
                \centering
                \includegraphics[width=\textwidth]{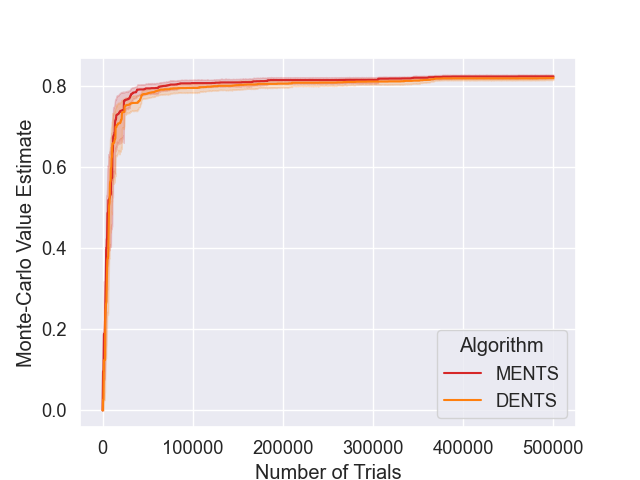}
                \caption{Frozen Lake.}
            \end{subfigure}
            \caption{Comparing DENTS with MENTS, by setting $\beta_{\text{DENTS}}(m)=\alpha_{\text{MENTS}}, \alpha_{\text{DENTS}}=\alpha_{\text{MENTS}}$, where $\alpha_{\text{MENTS}}$ is the temperature used for MENTS, and $\alpha_{\text{DENTS}},\beta_{\text{DENTS}}$ are the temperatures used by DENTS. }
            \label{fig:dbments}
        \end{figure*}

    \subsection{Additional Go details, results and discussion} \label{app:go_results}
    		Recall from Appendix \ref{app:adapt_for_games} that we initialised values with the neural networks as $Q^{\text{init}}(s,a)=\log \tilde{\pi}(a|s)+B$ and $V^{\text{init}}(s)=\tilde{V}(s)$, where $B$ is a constant (adapted from Xiao \etal \cite{xiao2019maximum}). For these experiments we set a value of $B=\frac{-1}{|\mathcal{A}|}\sum_{a\in\mathcal{A}} \log\tilde{\pi}(a|s)$. Initialising the values in such a way tends to lead to the values of $Q^{\text{init}}(s,a)$ being in the range $[-20,20]$, to account for this, we scaled the results of the game to $100$ and $-100$, which means that any parameters selected are about $100$ times larger than they would be if we had not have used this scaling.
    		
    		In these experiments we used a recommendation policy that recommends the action that was sampled the most, i.e. $\psi(s)=\max_a N(s,a)$, as this tends to be more robust to any noise from neural network outputs.
    		
    		To select each parameter we ran a round robin tournament where we varied the value of one parameter. The agent that won the most games was used to select the parameter moving forward, and in the case of a tie we used the agent which won the most games. If the winning agent had the largest or smallest value, then we ran another tournament adjusting the values accordingly.
    		
    		We tuned all of these algorithms using the game of 9x9 Go with a komi of 6.5, and giving each algorithm 2.5 seconds per move. For our final results we used the same parameters on 19x19 Go with a komi of 7.5, giving each algorithm 5.0 seconds per move.

        \subsubsection{Parameter selection for Go and supplimentary results} \label{app:go_hps}

            In this section we work through the process used to select parameters for the Go round robin tournament used in Section \ref{sec:go}. Predominantly parameters were chosen by playing out games of Go between agents. The parameters used for PUCT were copied from Kata Go and Alpha Go Zero \citeapp{katago,silver2017mastering}. 
            
            Initially we set the values of $\epsilon,\epsilon_{\tilde{\lambda}}$ to $0.03,1.0$. In Tables \ref{tab:w001} and \ref{tab:x001} we give results of tuning the search temperature for BTS and AR-BTS. For AR-BTS we tried different values of $\alpha_{\text{init}}$ in $\alpha(m)=\alpha_{\text{init}}/\sqrt{m}$.
            
            \begin{table*}[]
            \centering
                \begin{tabular}{l|cccccc}
                    \textbf{Black \textbackslash White}     & 3  & 1   & 0.3   & 0.1    & 0.03    \\ 
                    \hline
                                            3            & -     & 3-12  & 4-11  & 3-12  & 1-14  \\
                                            1             & 14-1  & -     & 8-7  & 9-6  & 7-8  \\
                                            0.3             & 14-1   & 8-7   &   -   & 7-8  & 11-4  \\
                                            0.1              & 15-0  & 11-4  & 9-6  & -     & 9-6  \\
                                            0.03              & 14-1  & 8-7  & 4-11  & 8-7  &   -   \\    
                \end{tabular}
                \caption{Results for round robin to select the temperature parameter $\alpha$ for BTS. The value of 0.1 won all four of its matches so was selected. \label{tab:w001}}
            \end{table*}
            
            \begin{table*}[]
            \centering
                \begin{tabular}{l|cccccc}
                    \textbf{Black \textbackslash White}     & 3  & 1   & 0.3   & 0.1    & 0.03    \\ 
                    \hline
                                            3            & -     	& 7-8  		& 9-6  		&  12-3 		& 12-3  		\\
                                            1            & 15-0 & - & 12-3 & 15-0 & 15-0   		\\
                                            0.3          & 13-2 & 11-4 & - & 14-1 & 15-0  		\\
                                            0.1          & 14-1 & 10-5 & 11-4 & - & 15-0   		\\
                                            0.03         & 13-2 & 5-10 & 8-7 & 13-2 &   -   	\\    
                \end{tabular}
                \caption{Results for round robin to select the temperature parameter $\alpha_{\text{init}}$ for AR-BTS. The value of 1.0 won all four of its matches so was selected. \label{tab:x001}}
            \end{table*}

            We then tuned the weighting of the prior policy $\epsilon_{\tilde{\lambda}}$. In Tables \ref{tab:w020} and \ref{tab:x020} we give results of tuning the prior policy weight for BTS and AR-BTS.
            
            \begin{table*}[]
            \centering
                \begin{tabular}{l|cccccc}
                    \textbf{Black \textbackslash White}     & 10  & 5   & 3   & 2    & 1    \\ 
                    \hline
                                            10            & - & 5-10 & 3-12 & 1-14 &  3-12  		\\
                                            5            & 14-1 & - & 12-3 & 11-4 & 11-4   		\\
                                            3          & 15-0 & 15-0 & - & 12-3 & 12-3   		\\
                                            2          & 15-0 & 12-3 & 12-3 & - & 10-5   		\\
                                            1         & 15-0 & 14-1 & 12-3 & 9-6 &   -   	\\    
                \end{tabular}
                \caption{Results for round robin to select the weighting of the prior policy $\epsilon_{\tilde{\lambda}}$ for BTS. The value of 2.0 won the most matches so was selected. \label{tab:w020}}
            \end{table*}
            
            \begin{table*}[]
            \centering
                \begin{tabular}{l|cccccc}
                    \textbf{Black \textbackslash White}     & 3  & 2   & 1   & 0.75    & 0.5    \\ 
                    \hline
                                            3            & - & 14-1 & 9-6 & 13-2 & 14-1   		\\
                                            2            & 12-3 & - & 12-3 & 15-0 & 10-5   		\\
                                            1          & 12-3 & 13-2 & - & 11-4 & 13-2   		\\
                                            0.75          & 11-4 & 10-5 & 12-3 & - & 14-1   		\\
                                            0.5         & 11-4 & 13-2 & 9-6 & 7-8 &   -   	\\   
                \end{tabular}
                \caption{Results for round robin to select the weighting of the prior policy $\epsilon_{\tilde{\lambda}}$ for AR-BTS. The value of 1.0 won the most matches so was selected. \label{tab:x020}}
            \end{table*}

            Following that, we then tuned MENTS exploration parameter $\epsilon$. In Tables \ref{tab:w030} and \ref{tab:x030} we give results of tuning the exploration parameter for BTS and AR-BTS. Although we selected the lowest value we tried here, we note that with such a value that a random action would have been sampled very few times, so the result of the hyperparameter selection was essentially that $\epsilon$ should be as low as possible.
            
            \begin{table*}[]
            \centering
                \begin{tabular}{l|cccccc}
                    \textbf{Black \textbackslash White}     & 0.1  & 0.03   & 0.01   & 0.003    & 0.001    \\ 
                    \hline
                                            0.1            & - & 12-3 & 10-5 & 8-7 & 7-8   		\\
                                            0.03            & 9-6 & - & 8-7 & 13-2 & 12-3   		\\
                                            0.01          & 11-4 & 9-6 & - & 7-8 & 12-3   		\\
                                            0.003          & 13-2 & 9-6 & 11-4 & - & 11-4   		\\
                                            0.001         & 12-3 & 11-4 & 8-7 & 9-6 &   -   	\\    
                \end{tabular}
                \caption{Results for round robin to select the exploration parameter $\epsilon$ for BTS. The value of 0.003 won the most matches so was selected. \label{tab:w030}}
            \end{table*}
            
            \begin{table*}[]
            \centering
                \begin{tabular}{l|cccccc}
                    \textbf{Black \textbackslash White}     & 0.1  & 0.03   & 0.01   & 0.003    & 0.001    \\ 
                    \hline
                                            0.1            & - & 12-3 & 14-1 & 12-3 & 13-2   		\\
                                            0.03            & 15-0 & - & 11-4 & 14-1 & 14-1   		\\
                                            0.01          & 15-0 & 11-4 & - & 13-2 & 13-2   		\\
                                            0.003          & 15-0 & 14-0 & 14-1 & - & 13-2   		\\
                                            0.001         & 14-1 & 14-1 & 14-1 & 15-0 &   -   	\\    
                \end{tabular}
                \caption{Results for round robin to select the exploration parameter $\epsilon$ for BTS. The value of 0.001 won the most matches so was selected. \label{tab:x030}}
            \end{table*}

            Then we tuned the (fixed and constant) search temperatures for MENTS, AR-MENTS, RENTS, AR-RENTS, TENTS and AR-TENTS in tables \ref{tab:w040}, \ref{tab:x040},\ref{tab:w050}, \ref{tab:x050},\ref{tab:w060} and \ref{tab:x060}.

            \begin{table*}[]
            \centering
                \begin{tabular}{l|cccccc}
                    \textbf{Black \textbackslash White}     & 0.3  & 0.1   & 0.03   & 0.01    & 0.003    \\ 
                    \textbf{Black \textbackslash White}     & 3  & 1   & 0.3   & 0.1    & 0.03    \\ 
                    \hline
                                            3            & -     	&   	9-6	&  9-6 		&   	11-4	&  10-5 		\\
                                            1            & 11-4  		& -     	&  11-4		&  9-6 		& 10-5  		\\
                                            0.3          &  9-6  	&  7-8  	&   -   &  11-4 		&  6-9 		\\
                                            0.1          &  4-11 		&   4-11		&   0-15		& -     	&  4-11 		\\
                                            0.03         &   	2-13	&  2-13 		&  0-15 		&   0-15		&   -   	\\      
                \end{tabular}
                \caption{Results for round robin to select the temperature parameter $\alpha$ for MENTS. The value of 1.0 won the most matches so was selected. \label{tab:w040}}
            \end{table*}
            
            \begin{table*}[]
            \centering
                \begin{tabular}{l|cccccc}
                    \textbf{Black \textbackslash White}     & 3  & 1   & 0.3   & 0.1    & 0.03    \\ 
                    \hline
                                            3            & -     	&  2-13 		& 1-14  		& 3-12  		& 5-10  		\\
                                            1            &  14-1 		& -     	& 11-4 		& 13-2  		& 15-0  		\\
                                            0.3          &  15-0  	&  12-3  	&   -   &  12-3 		& 14-1  		\\
                                            0.1          &   15-0		&  9-6 		&  9-6 		& -     	&  15-0 		\\
                                            0.03         &   15-0		&  8-7 		&  9-6 		& 14-1  		&   -   	\\    
                \end{tabular}
                \caption{Results for round robin to select the temperature parameter $\alpha$ for AR-MENTS. The value of 0.3 won the most matches so was selected. \label{tab:x040}}
            \end{table*}
            
            \begin{table*}[]
            \centering
                \begin{tabular}{l|cccccc}
                    \textbf{Black \textbackslash White}     & 3  & 1   & 0.3   & 0.1    & 0.03    \\ 
                    \hline
                                            3            & -     	&   	7-8	&  9-6 		&  10-5 		&  9-6 		\\
                                            1            &  13-2 		& -     	& 12-3 		& 11-4  		&  12-3 		\\
                                            0.3          &   10-5 	& 11-4   	&   -   &  10-5 		&  5-10 		\\
                                            0.1          &   3-12		&  2-13 		& 1-14  		& -     	&  3-12 		\\
                                            0.03         &   3-12		&  0-15 		&  0-15 		&  0-15 		&   -   	\\    
                \end{tabular}
                \caption{Results for round robin to select the temperature parameter $\alpha$ for RENTS. The value of 1.0 won the most matches so was selected. \label{tab:w050}}
            \end{table*}
            
            \begin{table*}[]
            \centering
                \begin{tabular}{l|cccccc}
                    \textbf{Black \textbackslash White}     & 3  & 1   & 0.3   & 0.1    & 0.03    \\ 
                    \hline
                                            3            & -     	&  4-11 		&  2-13 		& 9-6  		& 3-12  		\\
                                            1            &  15-0 		& -     	&  12-3		& 13-2  		&   13-2		\\
                                            0.3          &  15-0  	&  13-2  	&   -   &  14-1 		&  15-0 		\\
                                            0.1          &  15-0 		&   10-5		&   13-2		& -     	&  15-0 		\\
                                            0.03         &  13-2 		&  13-2 		&  12-3 		&   	14-1	&   -   	\\    
                \end{tabular}
                \caption{Results for round robin to select the temperature parameter $\alpha$ for AR-RENTS. The value of 0.3 won the most matches so was selected. \label{tab:x050}}
            \end{table*}
            
            \begin{table*}[]
            \centering
                \begin{tabular}{l|cccccc}
                    \textbf{Black \textbackslash White}     & 300  & 100   & 30   & 10    & 3    \\ 
                    \hline
                                            300            & -     	&  3-12 		&  5-10 		&  7-8 		& 9-6  		\\
                                            100            &  6-9 		& -     	&  9-6		&  12-3 		& 12-3  		\\
                                            30          & 8-7   	&  7-8  	&   -   &  9-6 		&  11-4 		\\
                                            10          &  0-15 		&  2-13 		& 2-13  		& -     	& 6-9  		\\
                                            3         &  1-14 		& 1-14  		& 2-13  		&  5-10 		&   -   	\\    
                \end{tabular}
                \caption{Results for round robin to select the temperature parameter $\alpha$ for TENTS. The value of 1.0 won the most matches so was selected. \label{tab:w060}}
            \end{table*}
            
            \begin{table*}[]
            \centering
                \begin{tabular}{l|cccccc}
                    \textbf{Black \textbackslash White}     & 30  & 10   & 3   & 1    & 0.3    \\ 
                    \hline
                                            30            & -     	&  14-1 		& 15-0  		& 14-1  		& 14-1  		\\
                                            10            &  15-0 		& -     	&  13-2		& 15-0  		&  15-0 		\\
                                            3          &  15-0  	&   14-1 	&   -   &  14-1 		&  15-0 		\\
                                            1          &   	15-0	&  15-0 		&  14-1 		& -     	& 15-0  		\\
                                            0.3         &   15-0		&   	15-0	&   14-1		& 15-0  		&   -   	\\    
                \end{tabular}
                \caption{Results for round robin to select the temperature parameter $\alpha$ for AR-TENTS. The value of 3.0 won the most matches so was selected. \label{tab:x060}}
            \end{table*}

            Finally, we considered entropy temperatures of the form $\beta(m)=\beta_{\text{init}}/\sqrt{m}$ for DENTS and AR-DENTS, and tuned the value of $\beta_{\text{init}}$, in Tables \ref{tab:w070} and \ref{tab:x070} for DENTS and AR-DENTS respectively.

            \begin{table*}[]
            \centering
                \begin{tabular}{l|cccccc}
                    \textbf{Black \textbackslash White}     & 0.1  & 0.03   & 0.01   & 0.003    & 0.001    \\ 
                    \hline
                                            0.1            & - & 10-5 & 12-3 & 8-7 & 11-4  		\\
                                            0.1            & 13-2 & - & 11-4 & 13-2  & 10-5   		\\
                                            0.1            & 12-3 & 11-4 & - & 10-5 & 11-4  		\\
                                            0.1            & 7-8 & 13-2 & 13-2 & - & 10-5  		\\
                                            0.1            & 13-2 & 13-2 & 11-4 & 11-4 &  - 		\\
                \end{tabular}
                \caption{Results for round robin to select the initial entropy temperature $\beta_{\text{init}}$ for BTS. The value of 0.3 won the most matches so was selected. \label{tab:w070}}
            \end{table*}
            
            \begin{table*}[]
            \centering
                \begin{tabular}{l|cccccc}
                    \textbf{Black \textbackslash White}     & 3  & 1   & 0.3   & 0.1    & 0.03    \\ 
                    \hline
                                            3            & - & 11-4 & 12-3 & 12-3 & 14-1  		\\
                                            1            & 13-2 & - & 11-4 & 13-2 & 13-2  		\\
                                            0.3            & 14-1 & 13-2 & - & 13-2 & 14-1  		\\
                                            0.1            & 12-3 & 12-3 & 12-3 & - & 11-4  		\\
                                            0.03            & 12-3 & 13-2 & 13-2 & 14-1 &  - 		\\
                \end{tabular}
                \caption{Results for round robin to select the initial entropy temperature $\beta_{\text{init}}$ for BTS. The value of 0.3 won the most matches so was selected. \label{tab:x070}}
            \end{table*}

            After tuning all of the algorithms, we compared each algorithm to their AR version in table \ref{tab:y010}, and the AR versions universally outperformed their counterparts. We used all of the selected parameters in 19x19 Go to run our final experiments given in Table \ref{table:go_results}.

            \begin{table*}[]
            \centering
                \begin{tabular}{l|cccccc}
                    \textbf{Black \textbackslash White}     & MENTS     & AR-MENTS  & BTS       & AR-BTS    & DENTS     & AR-DENTS  \\ 
                    \hline
                                            MENTS           &   -       & 0-25      &           &           &           &           \\
                                            AR-MENTS        & 25-0      &   -       &           &           &           &           \\
                                            BTS             &           &           &   -       & 16-9      &           &           \\
                                            AR-BTS          &           &           & 22-3      &   -       &           &           \\
                                            DENTS           &           &           &           &           &   -       & 21-4      \\
                                            AR-DENTS        &           &           &           &           & 22-3      & -         \\      
                    \hline
                    \textbf{Black \textbackslash White}     & RENTS     & AR-RENTS  & TENTS     & AR-TENTS  &           &           \\ 
                    \hline
                                            RENTS           &   -       & 2-23      &           &           &           &           \\
                                            AR-RENTS        & 24-1      &   -       &           &           &           &           \\
                                            TENTS           &           &           &   -       & 8-17      &           &           \\
                                            AR-TENTS        &           &           & 23-2      & -         &           &           \\         
                \end{tabular}
                \caption{Results for the matches of each algorithm against its AR version.\label{tab:y10}}
            \end{table*}

            Finally, we also ran AR-BTS against Kata Go directly limiting each algorithm to 1600 trials. KataGo beat AR-BTS by 31-19, confirming that the aditional exploration is outweighed by the information contained in the neural networks, and in Go the Boltzmann search algorithms gain their advantage via the Alias method and being able to run more trials quickly.

            \FloatBarrier

    \FloatBarrier
\newpage
\clearpage
\section{Proofs} \label{app:proofs}

    \FloatBarrier

    This proof section is structured as follows:
    \begin{enumerate}
        \item First, in Section \ref{app:mcts_process} we revisit MCTS as a stochastic process, defining some additional notation that was not useful in the main body of the paper, but will be for the following proofs;
        \item Second, in Section \ref{app:preliminaries} we introduce preliminary results, that will be useful building blocks for proofs in later Theorems;
        \item Third, in Section \ref{app:soft_learning_results} we show some general results about soft values that will also be useful later;
        \item Fourth, in Section \ref{app:simple_regret_results} simple regret is then revisited, and we show that any bounds on the simple regret of a policy are equivalent to showing bounds on the simple regret of an action;
        \item Fifth, in Section \ref{app:q_result} we show in a general way, that if a value function admits a concentration inequality, then the corresponding Q-value function admits a similar concentration inequality;
        \item Sixth, in Section \ref{app:ments_results} we show concentration inequalities for MENTS about the optimal soft values, and give bounds on the simple regret of MENTS, provided the temperature parameter is sufficiently small;
        \item Seventh, in Sections \ref{appsec:dents_proofs} and \ref{appsec:bts_proofs} we also provide concentration inequalities around the optimal standard values for BTS and DENTS, and give simple regret bounds, irrespective of the temperature paramters;
        \item Finally, in Section \ref{sec:ar_proofs} we consider results that are relavant for the algorithms using average returns from Section \ref{app:average_returns}.
    \end{enumerate}

    \subsection{Revisiting the MCTS stochastic process} \label{app:mcts_process}
        In this subsection we recall the relevant details of the \textit{MCTS stochastic process}, and introduce additional notation that will be useful in the proofs that were not necessary in the main paper. We will also recall the details for the UCT, MENTS, BTS and DENTS processes. We will use the phrase \textit{any Boltzmann MCTS process} to refer to any one of the MENTS, BTS and DENTS processes. In particular, we index the (Q-)values with the number visits to the relevant node.
        
        An MCTS process begins with a search tree $\mathcal{T}^0=\{s_0\}$ that contains only the root node/initial state. Let $\tau^n=(s_0^n,a_0^n,...,s_{h-1}^n,a_{h-1}^n,s_{h}^n)$ be a random variable for the $n$th trajectory or trial, where $s_0^n,...,s_{h-1}^n\in\mathcal{T}^{n-1}$ and $s_h^n\not\in\mathcal{T}^{n-1}$ (or $h=H$), and where actions are selected or sampled using $\pi^n$ the search policy for the $n$th trial. A new node is added to the search tree $\mathcal{T}^n=\{s_h^n\}\cup\mathcal{T}^{n-1}$. (Q-)Value estimates are kept at each node in the tree, and new values are initialised using the functions $V^{\text{init}}(s)$ and $Q^{\text{init}}(s,a)$, which are typically implemented as a constant value, using a \textit{rollout policy} or using a \textit{neural network}. Finally, the (Q-)value estimates are updated in a backup phase.
        
        In the following, it will be useful to write $N(s)$ the number of times state~$s$ was visited, and $N(s,a)$ the number of times action~$a$ was selected from state~$s$ as a sum of indicator random variables. Let $T(s_t)$ (and $T(s_t,a_t)$) be the set of trajectory indices that $s_t$ was visited on (and action $a_t$ selected), that is:
        \begin{align}
            T(s_t) &= \{i | s^i_t = s_t \} \\
            T(s_t,a_t) &= \{i | s^i_t = s_t, a^i_t = a_t \}.
        \end{align}
        This allows the counts $N(s_t),$ $N(s_t,a_t)$ and $N(s_{t+1})$ (with $s_{t+1}\in\succc{s_t}{a_t}$) to be written as sums of indicator random variables:
        \begin{align}
            N(s_t) = \sum_{i=1}^n \one[s^i_t=s_t] = |T(s_t)|, \\
            N(s_t,a_t) = \sum_{i=1}^n \one[s^i_t=s_t,a^i_t=a_t] = |T(s_t,a_t)|, \\ 
            N(s_t,a_t) = \sum_{i\in T(s_t)} \one[a^i_t = a_t], \label{appeq:nsa_sum} \\
            N(s_{t+1}) = \sum_{i\in T(s_t,a_t)} \one[s^i_{t+1} = s_{t+1}]. \label{appeq:ns_sum}
        \end{align}
        Additionally, we make the assumption that for any two states $s,s'\in\cl{S}$ that $s=s'$ if and only if the trajectories leading to them are identical. This assumption is purely to simplify notation, so that nodes in the tree have a one-to-one correspondence with states (or state-action pairs).

        \paragraph{The UCT process.} 
        	The UCT search policy can be defined as:
        \begin{align}
             \pi_{\textnormal{UCT}}^n(s_t) &= \max_{a\in\cl{A}} \text{UCB}^n(s_t,a), \\
             \text{UCB}^n(s_t, a) &= 
            		\begin{cases}
            			\infty & \text{if } N(s_t,a)=0 \\
            			\bar{Q}^{N(s_t,a)}(s_t, a)+c \sqrt{\frac{\log N(s_t)}{N(s_t, a)}} & \text{if } N(s_t,a)>0
            		\end{cases}
        \end{align}
        \noindent where, after $n$ trials, $\bar{Q}^{N(s,a)}(s,a)$ is the empirical estimate of the value at node $(s,a)$, where action~$a$ has been selected $N(s, a)$ from state~$s$. The backup consists of updating empirical estimates for $t=h-1,...,0$:
        \begin{align}
            \bar{V}^{N(s_t)+1}(s_t) &= \bar{V}^{N(s_t)}(s_t) + \frac{\bar{R}(t) - \bar{V}^{N(s_t)}(s_t)}{N(s_t) + 1}, \label{appeq:uct_vbar} \\
            \bar{Q}^{N(s_t,a_t)+1}(s_t, a_t) &= \bar{Q}^{N(s_t,a_t)}(s_t, a_t) + \frac{\bar{R}(t) - \bar{Q}^{N(s_t,a_t)}(s_t, a_t)}{N(s_t, a_t) + 1}, \label{appeq:uct_qbar}
        \end{align}
        \noindent where $\bar{R}(t) = V^{\text{init}}(s_h) + \sum_{i=t}^{h-1} R(s_i,a_i)$, and values are initialised as $\bar{V}^1(s)=V^{\text{init}}(s)$ and $\bar{Q}^0(s,a)=0$.

        \paragraph{The MENTS process.} 
        The policy for the MENTS process at state $s$ on the $n$th trial is:
        \begin{align}
            \pi^{n}_{\text{MENTS}}(a|s) = (1-\lambda_s)\exp\left(\left(\Qst{s}{a}{N(s,a)}-\Vst{s}{N(s)}\right)/\alpha\right) + \frac{\lambda_s}{|\cl{A}|}, \label{appeq:ments_soft_policy}
        \end{align}
        where $\lambda_s=\min(1,\epsilon/\log(e+N(s)))$, $\epsilon \in (0,\infty)$. The estimated soft values are computed using the backups for $t=h-1,...,0$:
        \begin{align}
            \Vst{s_t}{N(s_t)} &= \alpha \log \sum_{a\in\cl{A}} \exp \left(\Qst{s_t}{a}{N(s_t,a)}/\alpha \right), \label{appeq:soft_v_backup} \\
            \Qst{s_t}{a_t}{N(s_t,a_t)} &= R(s_t,a_t) + \sum_{s'\in\succc{s}{a}} \left( \frac{N(s')}{N(s_t,a_t)} \Vst{s'}{N(s')} \right). \label{appeq:soft_q_backup}
        \end{align}
        The soft (Q-)values are initialised as $\Vst{s}{1}=V^{\text{init}}(s)$ and $\Qst{s}{a}{0}=Q^{\text{init}}_{\text{sft}}(s)$ (typically $Q^{\text{init}}_{\text{sft}}(s)=0$).

        \paragraph{The DENTS process.}
        Let $\beta : \bb{R}\rightarrow [0,\infty)$ be a bounded function. The policy for the DENTS process at state $s$ on the $n$th trial is:
        \begin{align}
            \pi^{n}_{\text{DENTS}}(a|s) = (1-\lambda_s)\rho^n(a|s) + \frac{\lambda_s}{|\cl{A}|}, \label{appeq:dents_search_policy} 
        \end{align}
        where $\lambda_s=\min(1,\epsilon/\log(e+N(s)))$, $\epsilon \in (0,\infty)$ and where $\rho$ is given by:
        \begin{align}
            \rho^n_{\text{DENTS}}(a|s) \propto \exp\left(\frac{1}{\alpha}\left(\Qt{s}{a}{N(s,a)} + \beta(N(s))\cl{H}_Q^{N(s,a)}(s,a) \right)\right).
        \end{align}
        The Bellman values and entropy values are computed using the backups for $t=h-1,...,0$:
        \begin{align}
            \Qt{s_t}{a_t}{N(s_t,a_t)} &= R(s_t,a_t) + \sum_{s' \in \succc{s_t}{a_t}} \left( \frac{N(s')}{N(s_t,a_t)} \Vt{s'}{N(s')} \right), \label{appeq:dp_q_backup} \\ 
            \Vt{s_t}{N(s_t)} &=\max_{a\in\cl{A}} \Qt{s_t}{a}{N(s_t,a)}, \label{appeq:dp_v_backup} \\
            \cl{H}_Q^{N(s_t,a_t)}(s_t,a_t) &= \sum_{s'\in \succc{s_t}{a_t}} \frac{N(s')}{N(s_t,a_t)} \cl{H}_V^{N(s')}(s'), \\
            \cl{H}_V^{N(s_t)}(s_t) &= \cl{H}(\pi^n_{\text{DENTS}}(\cdot | s_t)) + \sum_{a\in\cl{A}} \pi^n_{\text{DENTS}}(a_t|s_t)\cl{H}_Q^{N(s_t,a_t)}(s_t,a_t), 
        \end{align}

        The (Q-)values are initialised as $\Vt{s}{1}=V^{\text{init}}(s)$ and $\Qt{s}{a}{0}=Q^{\text{init}}_{\text{sft}}(s)$ (typically $Q^{\text{init}}_{\text{sft}}(s)=0$). The entropy values are initialised as $\cl{H}_Q^{0}(s,a)=0$ and $\cl{H}_V^{1}(s) = \cl{H}(\pi^n_{\text{DENTS}}(\cdot | s_t))$, where the node for $s$ is created on the $n$th trial.

        Because we need to reason about $\pi^n_{\text{DENTS}}$, and subsequently $\rho^n_{\text{DENTS}}$, we give the exact form of $\rho^n_{\text{DENTS}}$:
        \begin{align}
            \rho^n_{\text{DENTS}}(a|s) &= \frac{\exp\left(\frac{1}{\alpha}\left(\Qt{s}{a}{N(s,a)} + \beta(N(s))\cl{H}_Q^{N(s,a)}(s,a) \right)\right)}{
                                \sum_{a'\in\cl{A}} \exp\left(\frac{1}{\alpha}\left(\Qt{s}{a'}{N(s,a')} + \beta(N(s))\cl{H}_Q^{N(s,a')}(s,a') \right)\right)  }.  \label{appeq:full_rho}
        \end{align}
        Let $V_\rho^{N(s)}(s)$ be defined as the value:
        \begin{align}
            V_\rho^{N(s)}(s) = \alpha \log \left[ \sum_{a'\in\cl{A}} \exp\left(\frac{1}{\alpha}\left(\Qt{s}{a'}{N(s,a')} + \beta(N(s))\cl{H}_Q^{N(s,a')}(s,a') \right)\right) \right], \label{appeq:v_rho}
        \end{align} 
        and notice that the value of $\exp(V_\rho^{N(s)}(s)/\alpha)$ is equal to the denominator in Equation (\ref{appeq:full_rho}), and so by rearranging we can write $\rho^n_{\text{DENTS}}$ as 
        \begin{align}
            \rho^n_{\text{DENTS}}(a|s) &= \exp\left(\frac{1}{\alpha}\left(\Qt{s}{a}{N(s,a)} + \beta(N(s))\cl{H}_Q^{N(s,a)}(s,a) - V_\rho^{N(s)}(s) \right)\right), \label{appeq:rho_concise}
        \end{align}
        and subsequently, we can write 
        \begin{align}
            \pi^{n}_{\text{DENTS}}(a|s) = (1-\lambda_s)\exp\left(\frac{1}{\alpha}\left(\Qt{s}{a}{N(s,a)} + \beta(N(s))\cl{H}_Q^{N(s,a)}(s,a) - V_\rho^{N(s)}(s) \right)\right) + \frac{\lambda_s}{|\cl{A}|}. \label{appeq:dents_search_policy_exact} 
        \end{align}

        \paragraph{The BTS process.}
        The BTS process is a special case of the DENTS process when $\beta(m)=0$ for all $m\in\bb{N}$.

    \subsection{Preliminaries} \label{app:preliminaries}
    
        Now we will provide lemmas are useful to avoid having to repeat the same argument in multiple proofs. In the following it will be useful to know that any action at any node in any MCTS process, for any number of trials, there is a minimum positive probability that it is chosen. 
        \begin{lemma} \label{lem:min_prob}
             Consider any Boltzmann MCTS process. There exists some $\pi^{\textnormal{min}}>0$, such that for any state $s_t\in\cl{S}$, for all $a_t\in\cl{A}$ and any number of trials $n\in\bb{N}$ we have $\pi^n(a_t|s_t)>\pi^{\textnormal{min}}$.
        \end{lemma}
        \begin{proofoutline}
            Firstly, we consider the case of the MENTS process. Define the $Q^{\min}$ function as follows:
            \begin{align}
                Q^{\min}(s_t,a_t) = \min_{s_{t+1},a_{t+1},...,s_H,a_H} \sum_{i=t}^H \min(0, Q^{\text{init}}_{\text{sft}}(s_i), R(s_i,a_i)).
            \end{align}
            
            And define the $V^{\max}$ function as:
            \begin{align}
                V^{\max}(s_t) &= \alpha \log \sum_{a\in\cl{A}} \exp (Q^{\max}(s_t,a)/\alpha), \label{appeq:vmax} \\
                Q^{\max}(s_t,a_t) &= R(s_t,a_t)+\max_{s_{t+1}\in\succc{s_t}{a_t}} V^{\max}(s_{t+1}).
            \end{align}
            
            Via induction, it is possible to show that $Q^{\min}(s_t,a_t)\leq \Qst{s_t}{a_t}{N(s_t,a_t)}$ and $V^{\max}(s_t)\geq \Vst{s_t}{N(s_t)}$. Now, define $\pi^{\min}$ as:
            \begin{align}
                \pi^{\min} = \inf_{\lambda\in[0,1]} \min_{(s,a)\in\cl{S}\times\cl{A}} (1-\lambda) \exp\left(\left(Q^{\min}(s,a)-V^{\max}(s)\right)/\alpha\right) + \frac{\lambda}{|\cl{A}|}. \label{eq:pi_min}
            \end{align}
            
            Because the value of an exponential is positive, as is $1/|\cl{A}|$, it follows that $\pi^{\min}>0.$ Recall the MENTS policy (Equation (\ref{appeq:ments_soft_policy})). By the monotonicity of the exponential function, it follows that for any $s_t\in\cl{S}, a_t\in\cl{A}, n\in\bb{N}$:
            \begin{align}
                \pi^{n}(a_t|s_t) 
                    =& (1-\lambda_{s_t})\exp\left(\left(\Qst{s_t}{a_t}{N(s_t,a_t)}-\Vst{s_t}{N(s_t)}\right)/\alpha\right) 
                        + \frac{\lambda_{s_t}}{|\cl{A}|} \\
                    \geq& (1-\lambda_{s_t})\exp\left(\left(Q^{\min}(s_t,a_t)-V^{\max}(s_t)\right)/\alpha\right) 
                        + \frac{\lambda_{s_t}}{|\cl{A}|} \\
                    \geq& \pi^{\min}.
            \end{align}

            Now we consider the DENTS process.

            For the DENTS process we can use similar reasoning, but need to update the definition of $V^{\max}$ from Equation (\ref{appeq:vmax}) to:
            \begin{align}
                V^{\max}(s_t) &= \alpha \log \sum_{a\in\cl{A}} \exp \left(\left(Q^{\max}(s_t,a) + H\beta_{\max}\log|\cl{A}| \right)/\alpha\right), 
            \end{align}
            where $\beta_{\max}=\sup_{x\in\bb{R}}\beta(x)$. Similarly to the MENTS process case, we can show that $Q^{\min}(s_t,a_t)\leq \Qt{s_t}{a_t}{N(s_t,a_t)}$ and $V^{\max}(s_t)\geq V_{\rho}^{N(s_t)}$, which implicitly uses that $0 \leq \cl{H}_Q^{N(s_t,a_t)}(s_t,a_t) \leq (H-t)\log|\cl{A}| \leq H\log|\cl{A}|$, which can be shown with an inductive argument, and using well-known properties of entropy. Defining $\pi^{\min}$ in the same way as Equation (\ref{eq:pi_min}), with the updated definition of $V^{\max}$. Recalling the DENTS policy (Equation (\ref{appeq:dents_search_policy_exact})), and using similar reasoning to before, we have:
            \begin{align}
                \pi^{n}(a_t|s_t) 
                    =& (1-\lambda_s)\exp\left(\frac{1}{\alpha}\left(\Qt{s}{a}{N(s,a)} + \beta(N(s))\cl{H}_Q^{N(s,a)}(s,a) - V_\rho^{N(s)}(s) \right)\right) + \frac{\lambda_s}{|\cl{A}|} \\
                    \geq& (1-\lambda_s)\exp\left(\frac{1}{\alpha}\left(\Qt{s}{a}{N(s,a)} - V_\rho^{N(s)}(s) \right)\right) + \frac{\lambda_s}{|\cl{A}|} \\
                    \geq& (1-\lambda_{s_t})\exp\left(\left(Q^{\min}(s_t,a_t)-V^{\max}(s_t)\right)/\alpha\right) 
                        + \frac{\lambda_{s_t}}{|\cl{A}|} \\
                    \geq& \pi^{\min}.
            \end{align}
        \end{proofoutline}

        It will also be useful in the following that the union of exponentially unlikely events is also exponentially unlikely, and that the intersection of exponentially likely events is exponentially likely:
        \begin{lemma} \label{lem:union_bound}
            Let $A_1,...,A_{\ell}$ be some events that satisfy for $1\leq i \leq \ell$ the inequality $\Pr(\lnot A_i) \leq C_i\exp(-k_i)$ then:
            \begin{align}
                \Pr\left(\bigcup_{i=1}^\ell \lnot A_i\right) &\leq C\exp(-k), \label{eq:comb_exp_like} \\
                \Pr\left(\bigcap_{i=1}^\ell A_i\right) = 1-\Pr\left(\bigcup_{i=1}^\ell \lnot A_i\right) &\geq 1-C\exp(-k), \label{eq:exp_likely}
            \end{align}
            where $C=\sum_{i=1}^\ell C_i$ and $k = \min_i k_i$.
        \end{lemma}
        
        \begin{proofoutline}
            Lemma \ref{lem:union_bound} is a consequence of the union bound, rearranging and simplifying. Inequality (\ref{eq:exp_likely}) is a consequence from Inequality (\ref{eq:comb_exp_like}) by negating the events.
        \end{proofoutline}

        Additionally, Hoeffding's inequality will be useful to bound the difference between a sum of indicator random variables and its expectation. 
        \begin{theorem} \label{thrm:hoeffding}
            Let $\{X_i\}_{i=1}^m$ be indicator random variables (i.e. $X_i\in\{0,1\}$), and $S_m=\sum_{i=1}^m X_i$. Then Hoeffding's inequality for indicator random variables states for any $\varepsilon > 0$ that:
            \begin{align}
                \Pr(|S_m - \bb{E}S_m|>\varepsilon)\leq 2\exp\left(-\frac{2\varepsilon^2}{m}\right).
            \end{align}
        \end{theorem}
        \begin{proof}
            This is a specific case of Hoeffding's inequality. See \citeapp{hoeffding1994probability} for proof.
        \end{proof}

        It will also be convenient to be able to `translate' bounds that depend on some $N(s,a)$ to a corresponding bound on $N(s)$:
        \begin{lemma} \label{lem:sa_to_s}
            Consider any Boltzmann MCTS process. Suppose that every action $a_t$ has some minimum probability $\eta$ of being chosen from some state $s_t$ (irrespective of the number of trials), i.e. $\Pr(a^i_t=a_t|s^i_t=s_t)>\eta$. And suppose for some $C',k'>0$ that some event $E$ admits a bound:
            \begin{align}
                \Pr(E) \leq C'\exp(-k'N(s_t,a_t)). \label{eq:sa_to_s_assume_bound}
            \end{align}
            
            Then, there exists $C,k>0$ such that:
            \begin{align}
                \Pr(E) \leq C\exp(-k N(s_t)). 
            \end{align}
        \end{lemma}
        
        \begin{proof}
            Recall from Equation (\ref{appeq:nsa_sum}) that $N(s_t,a_t)=\sum_{i\in T(s_t)} \one[a^i_t=a_t] = \sum_{i\in T(s_t)} \one[a^i_t=a_t | s^i_t=s_t]$. By taking expectations and using the assumed $\Pr(a^i_t=a_t|s^i_t=s_t)\geq\eta$ we have $\bb{E}N(s_t,a_t) \geq \eta N(s_t)$ (and more specifically as a consequence $\bb{E}N(s_t,a_t) - \eta N(s_t)/2 \geq \eta N(s_t)/2$). The probability of $N(s_t,a_t)$ being below a multiplicative ratio of $N(s_t)$ is bounded as follows:
            \begin{align}
                \Pr\left(N(s_t,a_t) < \frac{1}{2}\eta N(s_t)\right) 
                    &\leq \Pr\left(N(s_t,a_t) < \bb{E}N(s_t,a_T) - \frac{1}{2}\eta N(s_t)\right) \\
                    &= \Pr\left(\bb{E}N(s_t,a_t) - N(s_t,a_t) > \frac{1}{2}\eta N(s_t)\right) \label{local:seven} \\
                    &\leq \Pr\left(|\bb{E}N(s_t,a_t) - N(s_t,a_t)| > \frac{1}{2}\eta N(s)\right) \label{local:six} \\
                    &\leq 2\exp\left(-\frac{1}{2}\eta^2N(s_t)\right). \label{eq:bound_nsa}
            \end{align}
            
            The first line follows from $\bb{E}N(s_t,a_t) - \eta N(s_t)/2 \geq \eta N(s_t)/2$, the second line is a rearrangement, the third line comes from the inequality in (\ref{local:seven}) implying the inequality in (\ref{local:six}), and the final line uses Theorem \ref{thrm:hoeffding}, a Hoeffding bound for the sum of indicator random variables. 
            
            Finally, the bound using $N(s_t,a_t)$ can be converted into one depending on $N(s_t)$ as follows:
            \begin{align}
                \Pr\left( E \right)
                    = & \Pr\left( E \bigg| N(s_t,a_t) \geq \frac{1}{2}\eta N(s_t)\right) 
                        \Pr\left(N(s_t,a_t) \geq \frac{1}{2}\eta N(s_t)\right) \\
                      & + \Pr\left( E \bigg| N(s_t,a_t) < \frac{1}{2}\eta N(s_t)\right) 
                        \Pr\left(N(s_t,a_t) < \frac{1}{2}\eta N(s_t)\right) \\
                    \leq & \Pr\left( E \bigg| N(s_t,a_t) \geq \frac{1}{2}\eta N(s_t)\right) 
                        \cdot 1 \\
                      & + 1 \cdot \Pr\left(N(s_t,a_t) < \frac{1}{2}\eta N(s_t)\right) \\
                    \leq & C'\exp\left(-\frac{1}{2}k'\eta N(s_t) \right) + 
                        2\exp\left(-\frac{1}{2}\eta^2N(s_t)\right) \label{local:eight} \\
                    \leq & C\exp(-k N(s_t)),
            \end{align}
            where (\ref{local:eight}) uses the assumed Inequality (\ref{eq:sa_to_s_assume_bound}) with the condition $N(s_t,a_t) \geq \eta N(s_t)/2$, and also uses the bound from (\ref{eq:bound_nsa}). On the last line there is an appropriate $C,k>0$, similar to Lemma \ref{lem:union_bound}.
        \end{proof}

        Similar to Lemma \ref{lem:sa_to_s}, if we have some $s'\in\succc{s}{a}$, it will also be convenient to be able to translate bounds that depend on $N(s')$ into bounds that depend on $N(s,a)$:
        \begin{lemma} \label{lem:s_to_sa}
            Consider any Boltzmann MCTS process. For some state action pair $(s_t,a_t)$, and for some $s'_{t+1}\in\succc{s_t}{a_t}$, suppose for some $C',k'>0$ that some event $E$ admits a bound:
            \begin{align}
                \Pr(E) \leq C'\exp(-k'N(s'_t)).
            \end{align}
            Then, there exists $C,k>0$ such that:
            \begin{align}
                \Pr(E) \leq C\exp(-k N(s_t,a_t)).
            \end{align}
        \end{lemma}
        \begin{proofoutline}
            Proof is similar to Lemma \ref{lem:sa_to_s}. Instead of having a minimum probability of selecting an action $\eta$, replace it with the corresponding probability from the transition distribution $p(s_{t+1}|s_t,a_t)$. Then swapping any $N(s_t)$ with $N(s_t,a_t)$, any $N(s_t,a_t)$ with $N(s'_t)$, and using $N(s_{t+1}) = \sum_{i\in T(s_t,a_t)} \one[s^i_{t+1} = s_{t+1}]$ (Equation (\ref{appeq:ns_sum})) as the sum of indicator random variables will give the result.
        \end{proofoutline}

        
        

    \subsection{Soft learning} \label{app:soft_learning_results}

        For this subsection we will temporarily reintroduce the time-step parameter into our value functions to simplify other notation. We will show two results about soft Q-values: that the optimal standard Q-value is less than the optimal soft Q-value, and, that given a sufficiently small temperature, the optimal soft Q-values will preserve any \textit{strict} ordering over actions given by the optimal standard Q-values.

        For some policy $\pi$, the definition of $V^{\pi}_{\text{sft}}$ (Equation (\ref{eq:vsft_def})) can be re-arranged, to give a relation between the soft Q-value, the standard Q-value and the entropy of the policy:
        \begin{align}
            Q^{\pi}_{\text{sft}}(s,a,t) &= Q^{\pi}(s,a,t) 
                + \alpha \bb{E}_{\pi}\left[
                    \sum_{i=t+1}^H \cl{H}\left(\pi(\cdot|s_i)\right) \Bigg| s_t=s, a_t=a
                    \right], \label{eq:soft_standard_rel} \\
                &= Q^{\pi}(s,a,t) 
                + \alpha\cl{H}_{t+1}(\pi|s_t=s,a_t=a),
        \end{align}
        where $\cl{H}_{t+1}$ is used as a shorthand for the entropy term. By using this relation, it can be shown that the optimal soft Q-value will always be at least as large as the optimal standard Q-value:
        \begin{lemma} \label{lem:soft_geq_standard}
            $Q^*(s,a,t) \leq Q_{\textnormal{sft}}^*(s,a,t).$
        \end{lemma}
        \begin{proof}
            Taking a maximum over policies in Equation (\ref{eq:soft_standard_rel}), and considering that $\pi^*$, the optimal standard policy, is one of the possible policies considered in the maximisation, gives the result:
            \begin{align}
                Q_{\text{sft}}^*(s,a,t) &= \max_\pi \left(Q^{\pi}(s,a,t) + \alpha\cl{H}_{t+1}(\pi|s_t=s,a_t=a)\right) \\
                    &\geq Q^{\pi^*}(s,a,t) + \alpha\cl{H}_{t+1}(\pi^*|s_t=s,a_t=a) \\
                    &\geq Q^*(s,a,t).
            \end{align} 
            
            Noting that the entropy function is non-negative function.
        \end{proof}

        The optimal soft and standard values can be `tied together' by picking a very low temperature. Let $\delta(s,t)$ be the set of actions that have different optimal standard Q-values, that is $\delta(s,t)=\{(a,a')|Q^*(s,a,t)\neq Q^*(s,a',t)\}$. Now define $\Delta_{\cl{M}}$ as follows:
        \begin{align}
            \Delta_{s,t} &= \min_{(a,a')\in\delta(s,t)} |Q^*(s,a,t)-Q^*(s,a',t)|, \\
            \Delta_{\cl{M}} &= \min_{s,t} \Delta_{s,t}. \label{eq:delta}
        \end{align}
    
        Note in particular, for some $(a,a')\in\delta(s,t)$ that the definition of $\Delta_{\cl{M}}$ implies that if $Q^*(s,a,t)<Q^*(s,a',t)$ then 
        \begin{align}
            Q^*(s,a,t)+\Delta_{\cl{M}}\leq Q^*(s,a',t). \label{appeq:delta_diff}
        \end{align}

        \begin{lemma} \label{lem:soft_standard_consistent_order}
            If $\alpha < \Delta_{\cl{M}} / H\log |\cl{A}|$, then for all $t=1,...,H$, for all $s\in\cl{S}$ and for all $(a,a')\in\delta(s,t)$ we have $Q_{\textnormal{sft}}^*(s,a,t)<Q_{\textnormal{sft}}^*(s,a',t)$ iff $Q^*(s,a,t) < Q^*(s,a',t)$.
        \end{lemma}
        \begin{proof}
            $(\Leftarrow)$ First consider that the optimal soft Q-value is less than or equal to the optimal standard Q-value and maximum possible entropy:
            \begin{align}
                Q_{\textnormal{sft}}^*(s,a,t) &= \max_\pi \left(Q^{\pi}(s,a,t) + \alpha\cl{H}_{t+1}(\pi)\right) \\
                    &\leq \max_\pi Q^{\pi}(s,a,t) + \max_\pi \alpha\cl{H}_{t+1}(\pi) \label{appeq:softdiv} \\
                    &= Q^*(s,a,t) + \alpha (H-t)\log |\cl{A}| \\
                    &\leq Q^*(s,a,t) + \alpha H\log |\cl{A}|.
            \end{align}
            
            By using the assumed $\alpha < \Delta_{\cl{M}} / H\log |\cl{A}|$, using $Q^*(s,a,t)+\Delta_{\cl{M}}\leq Q^*(s,a',t)$ and that $Q^*(s,a',t)\leq \Qss{s}{a',t}$  from Lemma \ref{lem:soft_geq_standard} we get:
            \begin{align}
                Q_{\text{sft}}^*(s,a,t) &\leq Q^*(s,a,t) + \alpha H\log |\cl{A}| \\
                    &< Q^*(s,a,t) + \Delta_{\cl{M}} \\
                    &\leq Q^*(s,a',t) \\
                    &\leq Q_{\text{sft}}^*(s,a',t).
            \end{align}
            
            ($\Rightarrow$) To show that given the assumptions that $Q_{\text{sft}}^*(s,a,t)<Q_{\text{sft}}^*(s,a',t) \Rightarrow Q^*(s,a,t) < Q^*(s,a',t)$ we will show the contrapostive instead, which is $Q^*(s,a,t) \geq Q^*(s,a',t) \Rightarrow Q_{\text{sft}}^*(s,a,t)\geq Q_{\text{sft}}^*(s,a',t)$. Given that it is assumed that $(a,a')\in\delta(s,t)$, the following implications hold:
            \begin{align}
                Q^*(s,a,t) &\geq Q^*(s,a',t) \\
                \Rightarrow Q^*(s,a,t) &> Q^*(s,a',t) \\
                \Rightarrow Q_{\text{sft}}^*(s,a,t) &> Q_{\text{sft}}^*(s,a',t) \label{eq:reuse_proof} \\
                \Rightarrow Q_{\text{sft}}^*(s,a,t) &\geq Q_{\text{sft}}^*(s,a',t),
            \end{align}
            where the first implication uses that $(a,a')\in\delta(s)$, the second reuses the ($\Leftarrow$) proof, and the final implication holds generally. 
        \end{proof}

    \subsection{Simple regret} \label{app:simple_regret_results}
    
        In this section we will revisit \textit{simple regret} in more detail. Recall that our definition of simple regret is:
        \begin{align} 
            \reg(s_t,\psi^n) &= V^*(s_t) - V^{\psi^n}(s_t), 
        \end{align}
        where $\psi^n$ is the policy recommended by a forecaster after $n$ rounds or trials. This definition is different to what has been considered by the tree search literature so far \citeapp{mcts_simple_regret,brue1}. However, we feel that this definition is a more natural extension to the simple regret considered for multi-armed bandit problems \citeapp{simple_regret_short,simple_regret_long}, as it stems from a more general MDP planning problem that does not necessarily have to be solved using a tree search (for example, consider if we want to reason about a non-tree search algorithm's simple regret, that only outputs full policies). Moreover, we can reconcile the difference by defining the \textit{simple regret of an action} (or \textit{immediate simple regret}) as:
        \begin{align}
            \reg_I(s_t,\psi^n) =& V^*(s) - \bb{E}_{a_t\sim\psi^n(s_t)}[Q^*(s_t,a_t)],
        \end{align}
        and we show that any asymptotic upper bounds provided on the two definitions are equivalent up to a multiplicative factor.

        \begin{lemma} \label{lem:imm_simple_regret}
            $\bb{E}\reg(s_t,\psi^n)=O(f(n))$ for all $s_t\in\cl{S}$ iff $\bb{E}\reg_I(s_t,\psi^n)=O(f(n))$ for all $s_t\in\cl{S}$.
        \end{lemma}
        \begin{proof}
            Firstly, notice that the simple regret can be written recursively in terms of the immediate simple regret:
            \begin{align}
                \reg(s_t,\psi^n) =& V^*(s_t) - V^{\psi^n}(s_t) \\
                    =& V^*(s_t) - \bb{E}_{a_t\sim\psi^n(s)}[Q^{\psi^n}(s_t,a_t)] \\
                    =& V^*(s_t) - \bb{E}_{a_t\sim\psi^n(s)}[R(s_t,a_t) +    
                        \bb{E}_{s_{t+1}\sim\Pr(\cdot|s_t,a_t)}[V^{\psi^n}(s_{t+1})]] \\
                    =& V^*(s_t) - \bb{E}_{a_t\sim\psi^n(s)}\big[Q^*(s,a) - 
                        \bb{E}_{s_{t+1}\sim\Pr(\cdot|s_t,a_t)}[V^*(s_{t+1})]  \notag \\
                        &+ \bb{E}_{s_{t+1}\sim\Pr(\cdot|s_t,a_t)}[V^{\psi^n}(s_{t+1})]\big] \label{eq:simple_reg} \\
                    =& V^*(s_t) - \bb{E}_{a_t\sim\psi^n(s_t)}[Q^*(s_t,a_t)] \notag \\
                        &+ \bb{E}_{a_t\sim\psi_n(s_t),s_{t+1}\sim\Pr(\cdot|s_t,a_t)}[
                            V^*(s_{t+1}) - V^{\psi^n}(s_{t+1})] \\
                    =& \reg_I(s,\psi^n) + 
                        \bb{E}_{a_t\sim\psi^n(s_t),s'\sim\Pr(\cdot|s_t,a_t)}[\reg(s_{t+1},\psi^n)], \label{eq:regret_relation}
            \end{align}
            where in line (\ref{eq:simple_reg}) we used $R(s,a) = Q^*(s,a) - \bb{E}_{s'\sim\Pr(\cdot|s,a)}[V^*(s')]$, a rearrangement of the Bellman optimality equation for $Q^*(s,a)$.
            
            This shows that if $\bb{E}\reg(s_t,\psi^n)=O(f(n))$ then $\bb{E}\reg_I(s_t,\psi^n)\leq \bb{E}\reg(s_t,\psi^n) = O(f(n))$. Now suppose that $\bb{E}\reg_I(s_t,\psi^n)=O(f(n))$ and assume an inductive hypothesis that $\bb{E}\reg(s_{t+1},\psi^n)=O(f(n))$ for all $s_{t+1}\in\suc{s_t}$, then:
            \begin{align}
                \bb{E}\reg(s_t,\psi^n) = \bb{E}\left[ O(f(n)) + \bb{E}_{a_t\sim\psi^n(s_t),s_{t+1}\sim\Pr(\cdot|s_t,a_t)}[O(f(n))] \right] = O(f(n)),
            \end{align}
            where the outer expectation is with respect $\psi^n$ (as the relevant MENTS process is a stochastic one).
        \end{proof}

        It is useful to specialise Lemma \ref{lem:imm_simple_regret} specifically for the form of bounds used later. I.e. the simple regret at $s_t$ admits a regret bound exponential in $N(s_t)$, if and only if, the immediate simple regret admits a bound exponential in $N(s_t)$.
        
        \begin{corollary} \label{cor:imm_to_full_simple_regret}
            Consider any Boltzmann MCTS process. There exists $C_1,k_1>0$ such that $\bb{E}\reg(s_t,\psi^n) \leq C_1\exp(-k_1 N(s_t))$ iff there exists $C_2,k_2>0$ such that $\bb{E}\reg_I(s_t,\psi^n) \leq C_2\exp(-k_2 N(s_t))$.
        \end{corollary}
        \begin{proofoutline}
            The proof follows similarly to Lemma \ref{lem:imm_simple_regret}. The additional nuance is that we need to apply Lemmas \ref{lem:sa_to_s} and \ref{lem:s_to_sa}. Note that the assumption of a minimum action probability in Lemma \ref{lem:sa_to_s} is satisfied, because there is a minimum positive probability of selecting an action in any Boltzmann MCTS process (Lemma \ref{lem:min_prob}). The inductive hypothesis for $s_{t+1}\in\suc{s_t}$ would give a bound with respect to $N(s_{t+1})$, and the lemmas are required to `translate' the bound into one with respect to $N(s_t)$.
        \end{proofoutline}

    \subsection{General Q-value convergence result} \label{app:q_result}

        Recall the (soft) Q-value backups used by Boltzmann MCTS processes (Equations (\ref{appeq:soft_q_backup}) and (\ref{appeq:dp_q_backup})). Considering that the backups for MENTS and DENTS processes are of similar form, we will show that generally, backups of that form converge (exponentially), given that the values at any child nodes also converge (exponentially). However, towards showing this, we first need to consider the concentration of the empirical transition distribution around the true transition distribution.

        \begin{theorem} \label{thrm:dkw_inequality}
            Let $\{X_i\}_{i=1}^m$ be random variables drawn from a probability distribution with a cumulative distribution function of $F$. Let the empirical cumulative distribution function be $F_m(x)=\frac{1}{m} \sum_{i=1}^m \one[X_i < x]$. Then the Dvoretzky-Kiefer-Wolfowitz inequality is:
            \begin{align}
                \Pr\left(\sup_x |F_m(x)-F(x)| > \varepsilon\right) \leq 2\exp\left(-2m\varepsilon^2\right).
            \end{align}
        \end{theorem}
        \begin{proof}
            See \citeapp{dvoretzky1956asymptotic}.
        \end{proof}

        The Dvoretzky-Kiefer-Wolfowitz inequality is of interest because it allows us to tightly bound the empirical transition probability $N(s_{t+1})/N(s_t,a_t)$ with the true transition probability $p(s_{t+1}|s_t,a_t)$. 
        \begin{corollary} \label{cor:bound_transition_distribution}
            Consider any Boltzmann MCTS process. For all $(s_t,a_t)\in\cl{S}\times\cl{A}$ and for all $\varepsilon >0$ we have:
            \begin{align}
                \Pr\left(\max_{s_{t+1}\in\succc{s_t}{a_t}}\left| \frac{N(s_{t+1})}{N(s_t,a_t)} - p(s_{t+1}|s_t,a_t) \right| > \varepsilon \right) \leq 2 \exp\left(-\frac{1}{2}\varepsilon^2 N(s_t,a_t) \right).
            \end{align}
        \end{corollary}
        \begin{proofoutline}
            Recall Equation (\ref{appeq:ns_sum}) that allows $N(s_{t+1})$ to be written as a sum of $N(s_t,a_t)$ indicator random variables: $N(s_{t+1}) = \sum_{i\in T(s_t,a_t)} \one[s^i_{t+1} = s_{t+1}]$. Consider mapping the empirical transition distribution, with $N(s_t,a_t)=|T(s_t,a_t)|$ samples, to an appropriate catagorical distribution. Applying Theorem \ref{thrm:dkw_inequality} with $\varepsilon/2$ gives the result. In the constructed cumulative distribution function we need to account for an error of $\varepsilon/2$ before any $s_{t+1}$, and another error of $\varepsilon/2$ after the same $s_{t+1}$, to account for an error of $\varepsilon$ in the probability mass function (see Figure \ref{fig:dkw_diag}).
        \end{proofoutline}
        
        \begin{figure}
            \centering
            \includegraphics[scale=0.4]{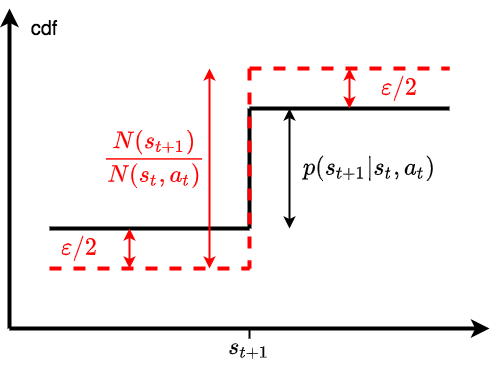}
            \caption{Example cdfs for the transition distribution (black solid line) and empirical transition distribution (red dashed line) around the successor state $s_{t+1}$. An error of $\varepsilon/2$ between the empirical and true cdfs is shown, and the probability mass for $s_{t+1}$ for the empirical and true transition distributions are shown to indicate how the constructed distribution gives Corollary \ref{cor:bound_transition_distribution}.}
            \label{fig:dkw_diag}
        \end{figure}

        Now that Corollary \ref{cor:bound_transition_distribution} can be used to bound  the empirical transition distribution to the true transition distribution, we can provide a general purpose concentration inequality for Q-values to use in later proofs.
        
        \begin{lemma} \label{lem:stochastic_step}
            Consider any Boltzmann MCTS process, and some state action pair $(s_t,a_t)\in\cl{S}\times\cl{A}$. Let $\dot{V}^{N(s)}(s):\cl{S}\rightarrow \bb{R}$, $\dot{V}^*(s):\cl{S}\rightarrow \bb{R}$ be some estimated and optimal value functions respectively and suppose that for all $s_{t+1}\in\succc{s_t}{a_t}$ that:
            \begin{align}
                \Pr\left(\left| \dot{V}^{N(s_{t+1})}(s_{t+1}) - \dot{V}^*(s_{t+1}) \right| > \varepsilon_1 \right) \leq C_{s_{t+1}}\exp(-k_{s_{t+1}}\varepsilon_1^2 N(s')).
            \end{align}
            
            If we have optimal and estimated Q-values defined as follows: 
            \begin{align}
                \dot{Q}^*(s,a)&=R(s,a)+\bb{E}_{s'\sim p(\cdot|s,a)}\left[\dot{V}^*(s')\right], \\
                \dot{Q}^{N(s,a)}(s,a)&=R(s,a)+\sum_{s'\in\succc{s}{a}} \left[\frac{N(s')}{N(s,a)} \dot{V}^{N(s')}(s')\right].
            \end{align}
            
            Then for some $C,k>0$, and for all $\varepsilon>0$:
            \begin{align}
                \Pr\left(\left| \dot{Q}^{N(s,a)}(s,a) - \dot{Q}^*(s,a) \right| > \varepsilon \right) \leq C\exp(-k\varepsilon^2 N(s,a)).
            \end{align}
        \end{lemma}
        
        \begin{proof}
            By the assumed bounds, Lemma \ref{lem:union_bound} and Lemma \ref{lem:s_to_sa}, there is some $C_1,k_1>0$, such that for any $\varepsilon_1 >0$:
            \begin{align}
                \Pr\left(\forall s_{t+1}. \left|\dot{V}^{N(s_{t+1})}(s_{t+1})-\dot{V}^*(s_{t+1}) \right| \leq \varepsilon_1 \right) > 1-C_1\exp(-k_1\varepsilon_1^2 N(s_t,a_t)), \label{local:one}
            \end{align}
            and recall using Corollary \ref{cor:bound_transition_distribution} that for any $p(s_{t+1}|s_t,a_t) > \varepsilon_2>0$ we must have:
            \begin{align}
                \Pr\left(\max_{s_{t+1}\in\succc{s_t}{a_t}} 
                    \left| \frac{N(s_{t+1})}{N(s_t,a_t)} - p(s_t,a_t,s_{t+1}) \right| 
                    \leq \varepsilon_2 \right) 
                        > 1 - 2 \exp\left(-\frac{1}{2}\varepsilon_2^2 N(s_t,a_t) \right). \label{local:two}
            \end{align}
            
            If the events in Inequalities (\ref{local:one}) and (\ref{local:two}) hold, then $\dot{V}^*(s_{t+1})- \varepsilon_1 \leq \dot{V}^{N_n(s_{t+1})}(s_{t+1}) \leq \dot{V}^*(s_{t+1})+ \varepsilon_1$ and $p(s_{t+1}|s_t,a_t) - \varepsilon_2 \leq N(s_{t+1})/N(s_t,a_t) \leq p(s_{t+1}|s_t,a_t) + \varepsilon_2$. The upper bounds on $\dot{V}^{N_n(s_{t+1})}(s_{t+1})$ and $N(s_{t+1})/N(s_t,a_t)$ can be used to obtain an upper bound on $\dot{Q}^{N(s_t,a_t)}$: 
            \begin{align}
                &\dot{Q}^{N(s_t,a_t)}(s_t,a_t) \\
                    =& R(s_t,a_t) 
                        + \sum_{s_{t+1}\in\succc{s_t}{a_t}}\frac{N_n(s_{t+1})}{N_n(s_t,a_t)} 
                            \dot{V}^{N_n(s_{t+1})}(s_{t+1}) \\
                    \leq& R(s_t,a_t) 
                        + \sum_{s_{t+1}\in\succc{s_t}{a_t}}(p(s_{t+1}|s_t,a_t)+\varepsilon_2)
                            (\dot{V}^*(s')+\varepsilon_1) \\
                    =& R(s_t,a_t) 
                        + \bb{E}_{s_{t+1}\sim p(\cdot|s_t,a_t)}[\dot{V}^*(s_{t+1})] 
                        + \varepsilon_1 
                        + \varepsilon_2\sum_{s_{t+1}\in\succc{s_t}{a_t}} \dot{V}^*(s_{t+1}) 
                        + \varepsilon_1 \varepsilon_2 \\
                    \leq& R(s_t,a_t) 
                        + \bb{E}_{s_{t+1}\sim p(\cdot|s_t,a_t)}[\dot{V}^*(s_{t+1})] 
                        + \varepsilon_1\left|\sum_{s_{t+1}\in\succc{s_t}{a_t}} \dot{V}^*(s_{t+1})\right| 
                        + \varepsilon_2 + \varepsilon_1 \varepsilon_2 \\
                    =& \dot{Q}^*(s_t,a_t) 
                        + \varepsilon_1\left|\sum_{s_{t+1}\in\succc{s_t}{a_t}} \dot{V}^*(s_{t+1})\right| 
                        + \varepsilon_2 
                        + \varepsilon_1 \varepsilon_2.
            \end{align}
            
            Following the same reasoning but using the lower bounds on $\dot{V}^{N_n(s_{t+1})}(s_{t+1})$ and $N(s_{t+1})/N(s_t,a_t)$ gives:
            \begin{align}
                \dot{Q}^{N(s_t,a_t)}(s_t,a_t) 
                    &\geq \dot{Q}^*(s_t,a_t) 
                        - \varepsilon_1\left|\sum_{s_{t+1}\in\succc{s_t}{a_t}} \dot{V}^*(s_{t+1})\right| 
                        - \varepsilon_2 
                        + \varepsilon_1 \varepsilon_2.
            \end{align}
            
            For any $\varepsilon >0$, setting $\varepsilon_1 = \min(\varepsilon/3, \varepsilon/3|\sum_{s_{t+1}} \dot{V}^*(s_{t+1})|)$ and $\varepsilon_2 = \min( \varepsilon/3, p(s_{t+1}|s_t,a_t))$ then gives:
            \begin{align}
                \dot{Q}^*(s_t,a_t) - \varepsilon \leq \dot{Q}^{N(s_t,a_t)}(s_t,a_t) \leq \dot{Q}^*(s_t,a_t) + \varepsilon.
            \end{align}
            
            Using Lemma \ref{lem:union_bound} liberally, there is some $C_2,C_3,k_2,k_3>0$ such that:
            \begin{align}
                &\Pr\left(\left| \dot{Q}^{N(s_t,a_t)}(s_t,a_t) - \dot{Q}^*(s_t,a_t) \right| \leq \varepsilon \right) \\
                    >& \left(1-C_1\exp(-k_1\varepsilon_1^2 N(s_t,a_t))\right) \left(1-2\exp\left(-\frac{1}{2}\varepsilon_2^2 N(s_t,a_t) \right)\right) \\
                    =& \left(1-C_2\exp(-k_2 \varepsilon^2 N(s_t,a_t) \right) \cdot \left(1-C_3\exp(-k_3 \varepsilon^2 N(s_t,a_t)\right) \\
                    =& 1 -C_2\exp(-k_2 \varepsilon^2 N(s_t,a_t)) -C_3\exp(-k_3 \varepsilon^2 N(s_t,a_t)) \notag \\
                        &+ C_2C_3\exp(-(k_2+k_3)\varepsilon^2 N(s_t,a_t) ) \\
                    >& 1 -C_2\exp(-k_2 \varepsilon^2 N(s_t,a_t)) -C_3\exp(-k_3 \varepsilon^2 N(s_t,a_t)),
            \end{align}
            and then by negating, with $C=C_1+C_2$ and $k=\min(k_1,k_2)$ the result follows:
            \begin{align}
                    \Pr\left(\left| \dot{Q}^{N(s_t,a_t)}(s_t,a_t) - \dot{Q}^*(s_t,a_t) \right| > \varepsilon \right) \leq C\exp(-k\varepsilon^2 N(s_t,a_t)).
            \end{align}
        \end{proof}

    \subsection{MENTS results} \label{app:ments_results}
    
        In this subsection we provide results related to MENTS. We begin by providing concentration inequalities for the estimated soft values at each node, around the optimal soft value. Note that this result is similar to some of the results provided in \citeapp{xiao2019maximum}, however to keep this paper more self contained we still provide a proof in our particular flavour. After, the concentration inequalities are then used to show bounds on the simple regret of MENTS, given constraints on the temperature.
        
        To prove the concentration inequality around the optimal soft values, start by showing an inductive step.

        \begin{lemma} \label{lem:ments_val_induction_step}
            Consider a MENTS process. Let $s_t\in\cl{S}$, with $1\leq t \leq H$. If for all $s_{t+1}\in\suc{s_t}$ there is some $C_{s_{t+1}},k_{s_{t+1}}>0$ for any $\varepsilon_{s_{t+1}}>0$:
            \begin{align}
                \Pr\left(\left| \Vst{s_{t+1}}{N(s_{t+1})} - \Vss{s_{t+1}} \right| > \varepsilon_{s_{t+1}} \right) 
                    &\leq C_{s_{t+1}}\exp\left( -k_{s_{t+1}}\varepsilon_{s_{t+1}}^2 N(s_{t+1}) \right), 
            \end{align}
            then there is some $C,k>0$, for any $\varepsilon>0$:
            \begin{align}
                \Pr\left(\left| \Vst{s_{t}}{N(s_{t})} - \Vss{s_{t}} \right| > \varepsilon \right) 
                    &\leq C\exp\left( -k\varepsilon^2 N(s_{t}) \right).
            \end{align}
        \end{lemma}
        
        \begin{proof}
            Given the assumptions and by Lemmas \ref{lem:stochastic_step}, \ref{lem:sa_to_s} and \ref{lem:union_bound}, there is some $C,k>0$ such that for any $\varepsilon>0$:
            \begin{align}
                \Pr\left(\forall a_t\in\cl{A}. \left|\Qst{s_t}{a_t}{N(s_t,a_t)}-\Qss{s_t}{a_t}\right|\leq \varepsilon\right) &> 1-C \exp(-k \varepsilon^2 N(s_t)).
            \end{align}
            
            So with probability at least $1-C \exp(-k \varepsilon^2 N(s_t))$, for any $a_t$, the following holds:
            \begin{align}
                \Qss{s_t}{a_t} - \varepsilon \leq \Qst{s_t}{a_t}{N(s_t,a_t)} \leq \Qss{s_t}{a_t} + \varepsilon. 
            \end{align}
            
            Using the upper bound on $\Qst{s_t}{a_t}{N(s_t,a_t)}$ in the soft backup equation for $\Vst{s_t}{N(s_t)}$ (Equation (\ref{appeq:soft_v_backup}) gives:
            \begin{align}
                \Vst{s_t}{N(s_t)} &= \alpha \log \sum_{a\in\cl{A}} \exp\left(\frac{\Qst{s_t}{a}{N(s_t,a_t)}}{\alpha}\right) \\
                    &\leq \alpha \log \sum_{a\in\cl{A}} \exp\left(\frac{\Qss{s_t}{a}+\varepsilon}{\alpha}\right) \\
                    &= \alpha \log \sum_{a\in\cl{A}} \exp\left(\frac{\Qss{s_t}{a}}{\alpha}\right) + \varepsilon \\
                    &= \Vss{s_t} + \varepsilon,
            \end{align}
            noting that the \textit{softmax} function monotonically increases in its arguments. Then with similar reasoning using the lower bound on $\Qst{s_t}{a_t}{N(s_t,a_t)}$ gives:
            \begin{align}
                \Vst{s_t}{N(s_t)} \geq \Vss{s_t} - \varepsilon,
            \end{align}
            and hence:
            \begin{align}
                |\Vst{s_t}{N(s_t)}-\Vss{s_t}| \leq \varepsilon.
            \end{align}
            
            This therefore shows:
            \begin{align}
                \Pr\left(\left|\Vst{s_t}{N(s_t)}-\Vss{s_t}\right| \leq \varepsilon_1\right) > 1-C \exp(-k \varepsilon^2 N(s_t)),
            \end{align}
            and negating probabilities gives the result.
        \end{proof}

        Then completing the induction gives the concentration inequalities desired for any state that MENTS might visit.
        \begin{theorem} \label{thrm:ments_val_converge}
            Consider a MENTS process, let $s_t\in\cl{S}$ then there is some $C,k>0$ for any $\varepsilon>0$:
            \begin{align}
                \Pr\left(\left| \Vst{s_{t}}{N(s_{t})} - \Vss{s_{t}} \right| > \varepsilon \right) 
                    &\leq C\exp\left( -k\varepsilon^2 N(s_{t}) \right).
            \end{align}
            Moreover, at the root node $s_0$ we have:
            \begin{align}
                \Pr\left(\left| \Vst{s_{0}}{N(s_{0})} - \Vss{s_{0}} \right| > \varepsilon \right) 
                    &\leq C\exp\left( -k\varepsilon^2 n \right). \label{local:three}
            \end{align}
        \end{theorem}
        \begin{proof}
            Consider that the result holds for $t=H+1$, because $\Vst{s_{H+1}}{N(s_{H+1})}=\Vss{s_{H+1}}=0$. Therefore the result holds for any $t=1,...,H+1$ by induction using Lemma \ref{lem:ments_val_induction_step}. Noting that $N(s_0)=n$ gives (\ref{local:three}).
        \end{proof}

        Note that all of our concentration inequalities imply a convergence in probability. For example, we explicitly demonstrate this in Corollary \ref{cor:sftq_convg}.
        
        \begin{corollary} \label{cor:sftq_convg}
        		$\Qst{s_t}{a_t}{N(s_{t},a_t)}\rap \Qss{s_t}{a_t}$
        \end{corollary}
        \begin{proof}
        		This follows from the concentration inequalities shown in Lemma \ref{lem:stochastic_step} and Theorem \ref{thrm:ments_val_converge}. As the RHS of the bound tend to zero as $n\rightarrow\infty$, these concentration inequalities follow by taking the limit $n\rightarrow\infty$:
        		\begin{align}	
                \lim_{n\rightarrow\infty} \Pr\left(\left| \Qst{s_{t}}{a_t}{N(s_{t},a_t)} - \Qss{s_t}{a_t} \right| > \varepsilon \right) 
                    &\leq \lim_{n\rightarrow\infty}  C\exp\left( -k\varepsilon^2 N(s_{t}) \right) \\
                    &= 0.
        		\end{align}
        	\end{proof}

        Now we use the concentration inequalities to show that we are exponentially unlikely to recommend a suboptimal action with respect to the standard values (i.e. the simple regret tends to zero exponentially), provided the temperature parameter $\alpha$ is small enough.
        
        \begin{lemma} \label{lem:ments_imm_simple_regret}
            Consider a MENTS process with $\alpha<\Delta_{\cl{M}}/3H\log |\cl{A}|$. Let $s_t\in\cl{S}$, with $1\leq t \leq H$ then there is some $C',k'>0$ such that:
            \begin{align}
                \bb{E} \reg_I(s_{t},\psi^n_{\textnormal{MENTS}}) &\leq C'\exp\left( -k' N(s_{t}) \right).
            \end{align}
        \end{lemma}
        \begin{proof}
            Let $a^*$ be the optimal action with respect to the soft values, so $a^*=\argmax_{a\in\cl{A}} \Qss{s_t}{a}$. Then by Lemma \ref{lem:soft_standard_consistent_order} $a^*$ must also be the optimal action for the standard Q-value function $a^*=\argmax_{a\in\cl{A}} Q^*(s_t,a)$. By Theorem \ref{thrm:ments_val_converge} and Lemmas \ref{lem:stochastic_step} and \ref{lem:sa_to_s} there exists $C_1,k_1>0$ such that for all $\varepsilon_1>0$:
            \begin{align}
                \Pr\left(\forall a_t\in\cl{A} \left|\Qst{s_t}{a_t}{N(s_t,a_t)}-\Qss{s_t}{a_t}\right|\leq \varepsilon_1\right) &> 1-C_1 \exp(-k_1 \varepsilon_1^2 N(s_t)).
            \end{align}
            
            Setting $\varepsilon_1=\Delta_{\cl{M}}/3H\log |\cl{A}|$ then gives with probability at least $1-C_1 \exp(-k_2N(s_t))$ (where $k_2=k_1(\Delta_{\cl{M}} /3H\log|\cl{A}|)^2$) that for all actions $a_t\in\cl{A}$:
            \begin{align}
                \Qss{s_t}{a_t} - \Delta_{\cl{M}}/3 \leq \Qst{s_t}{a_t}{N(s_t,a_t)} \leq \Qss{s_t}{a_t} + \Delta_{\cl{M}}/3.
            \end{align}
            
            And hence, with probability at least $1-C_1 \exp(-k_2N(s_t))$, for all $a\in\cl{A}-\{a^*\}$ we have:
            \begin{align}
                \Qst{s_t}{a}{N(s_t,a)}
                    \leq& \Qss{s_t}{a} + \Delta_{\cl{M}}/3 \\
                    \leq& Q^*(s_t,a) + \alpha H\log |\cl{A}| + \Delta_{\cl{M}}/3 \\
                    \leq& Q^*(s_t,a) + 2\Delta_{\cl{M}}/3 \\
                    \leq& Q^*(s_t,a^*) - \Delta_{\cl{M}}/3 \\
                    \leq& \Qss{s_t}{a^*} - \Delta_{\cl{M}}/3 \\
                    \leq& \Qst{s_t}{a^*}{N(s_t,a^*)}. \label{local:nine}
            \end{align}
            
            Where in the above, the first line holds from the upper bound on $\Qst{s_t}{a_t}{N(s_t,a_t)}$; The second holds from maximising the standard return and entropy portions of the soft value seperately (recall Inequality (\ref{appeq:softdiv}) in Lemma \ref{lem:soft_standard_consistent_order}); The third holds from the assumption on $\alpha$; The fourth holds from the definition of $\Delta_{\cl{M}}$ (also see Inequality (\ref{appeq:delta_diff})); The fifth holds from the optimal soft value being greater than the optimal standard value (Lemma \ref{lem:soft_geq_standard}); And the final line holds by using the lower bound on $\Qst{s_t}{a_t}{N(s_t,a_t)}$ given above with $a_t=a^*$. 
            
            Negating the probability that (\ref{local:nine}) holds gives:
            \begin{align}
                \Pr\left(\exists a_t\in\cl{A}-\{a^*\}.\left( \Qst{s_t}{a}{N(s_t,a)} > \Qst{s_t}{a^*}{N(s_t,a^*)}\right)\right)
                    \leq C_1 \exp(-k_2N(s_t)).
            \end{align}
            
            Finally, we can bound our expected immediate regret as follows:
            \begin{align}
                & \bb{E}\reg_I(s_t,\psi^n_{\text{MENTS}})  \\
                    =& \sum_{a\in\cl{A}-\{a^*\}} 
                        \left(V^*(s_t)-Q^*(s_t,a)\right) \Pr\left(\psi^n_{\text{MENTS}}(s_t) =a \right) \\
                    =& \sum_{a\in\cl{A}-\{a^*\}} 
                        \left(V^*(s_t)-Q^*(s_t,a)\right) \Pr\left(a = \argmax_{a'} \Qst{s_t}{a'}{N(s_t,a)} \right) \\
                    \leq& \sum_{a\in\cl{A}-\{a^*\}} 
                        \left(V^*(s_t)-Q^*(s_t,a)\right) \Pr\left(\Qst{s_t}{a}{N(s_t,a)} > \Qst{s_t}{a^*}{N(s_t,a^*)}  \right) \\
                    \leq&  \sum_{a\in\cl{A}-\{a^*\}} 
                        \left(V^*(s_t)-Q^*(s_t,a)\right) C_1 \exp(-k_2N(s_t)),
            \end{align}
            and setting $k'=k_2$ and $C'=C_1\sum_{a\in\cl{A}-\{a^*\}} \left(V^*(s_t)-Q^*(s_t,a)\right)$ gives the result.
        \end{proof}

        And finally, by using Lemma \ref{lem:ments_imm_simple_regret}, we can convert the bound on the immediate simple regret to a bound on the simple regret.
        
        \begin{theorem} \label{thrm:ments_simple_regret_converge}
            Consider a MENTS process with $\alpha<\Delta_{\cl{M}}/3H\log |\cl{A}|$. Let $s_t\in\cl{S}$, with $1\leq t \leq H$ then there is some $C',k'>0$ such that:
            \begin{align}
                \bb{E} \reg(s_{t},\psi^n_{\textnormal{MENTS}}) &\leq C'\exp\left( -k' N(s_{t}) \right).
            \end{align}
            Moreover, at the root node $s_0$ we have:
            \begin{align}
                \bb{E} \reg(s_{0},\psi^n_{\textnormal{MENTS}}) &\leq C'\exp\left( -k'n) \right).
            \end{align}
        \end{theorem}
        \begin{proof}
            This theorem holds as a consequence of Corollary \ref{cor:imm_to_full_simple_regret} and Lemma \ref{lem:ments_imm_simple_regret}, and noting that at the root node $N(s_0)=n$.
        \end{proof}

        We have now shown Theorem \ref{thrm:ments}:
        \begin{customthm}{3.2} 
            For any MDP $\cl{M}$, after running $n$ trials of the MENTS algorithm with $\alpha \leq \Delta_{\cl{M}}/3H\log|\cl{A}|$, there exists constants $C,k>0$ such that: $\bb{E}[\reg(s_0,\psi^n_{\textnormal{MENTS}})] \leq C\exp(-kn)$, where $\Delta_{\cl{M}}=\min \{Q^*(s,a,t)-Q^*(s,a',t)\vert Q^*(s,a,t) \neq Q^*(s,a',t),s\in\cl{S}, a,a'\in\cl{A},t\in\bb{N}\}$.
        \end{customthm}
        \begin{proof}
            This is part of Theorem \ref{thrm:ments_simple_regret_converge}.
        \end{proof}

        And to show Proposition \ref{prop:ments_bad} we will utilise Theorem \ref{thrm:ments_val_converge} and return to the modified 10-chain problem, with illustration repeated in Figure \ref{fig:dchain_illustration_tres} for reference.
        
        \begin{customprop}{3.1}
            There exists an MDP $\cl{M}$ and temperature $\alpha$ such that $\bb{E}[\reg(s_0,\psi^n_{\textnormal{MENTS}})] \not\to 0$ as $n\to\infty$. That is, MENTS is not consistent.
        \end{customprop}
        
        \begin{proof}
            We give a proof by construction. Recall the modified 10-chain problem, with $R_f=1/2$ in Figure \ref{fig:dchain_illustration_tres}, and consider a MENTS process (i.e. running MENTS) with a temperature $\alpha=1$ for $n$ trials. By considering the optimal soft Bellman equations (\ref{eq:v_soft_bellman}) and (\ref{eq:q_soft_bellman}), one can verify that $\Qss{1}{2}=0.9$ and $\Qss{1}{1}=\log\left(\exp(1/2)+\sum_{i=0}^8\exp(i/10)\right)\approx 2.74$. 
            
            Theorem \ref{thrm:ments_val_converge}, Lemma \ref{lem:stochastic_step} and Lemma \ref{lem:sa_to_s} implies that there is some $C,k>0$ for any $\varepsilon > 0$:
            \begin{align}
                \Pr\left(\left| \Qsp{1}{1}{N(1,1)} - \Qss{1}{1}\right| > \varepsilon \right) \leq C\exp(-k\varepsilon^2 N(1)) = C\exp(-k\varepsilon^2 n).
            \end{align}
            
            Letting $\varepsilon=1$ and using $\Qss{1}{1}>5/2$ gives:
            \begin{align}
                \Pr\left(\Qsp{1}{1}{N(1,1)} < 3/2 \right) 
                    \leq& \Pr\left(\left| \Qsp{1}{1}{N(1,1)} - \Qss{1}{1}\right| > 1 \right) \\
                    \leq& C\exp(-kn).
            \end{align}
            
            And hence:
            \begin{align}
                \Pr(\psi^n_{\text{MENTS}}(1)=1) > 1 - C\exp(-kn)
            \end{align}
            
            Consider that the best simple regret an agent can achieve after selecting action 1 from the starting state is $1/10$. Let $M=\log(2C)/k$, so that $C\exp(-kM)=1/2$. Then, for all $n>M$ we have $\Pr(\psi^n_{\text{MENTS}}(1)=1)> 1/2$, and hence:
            \begin{align}
                \bb{E}\reg(1,\psi^n_{\text{MENTS}}) &> \frac{1}{10} \cdot \Pr(\psi^n_{\text{MENTS}}(1)=1) \\
                    &> \frac{1}{20}.
            \end{align}
            Thus $\bb{E}r(1,\psi^n_{\text{MENTS}})\not\rightarrow 0$.
        \end{proof}
        
        \begin{figure}
            \centering
            \includegraphics[width=0.8\textwidth]{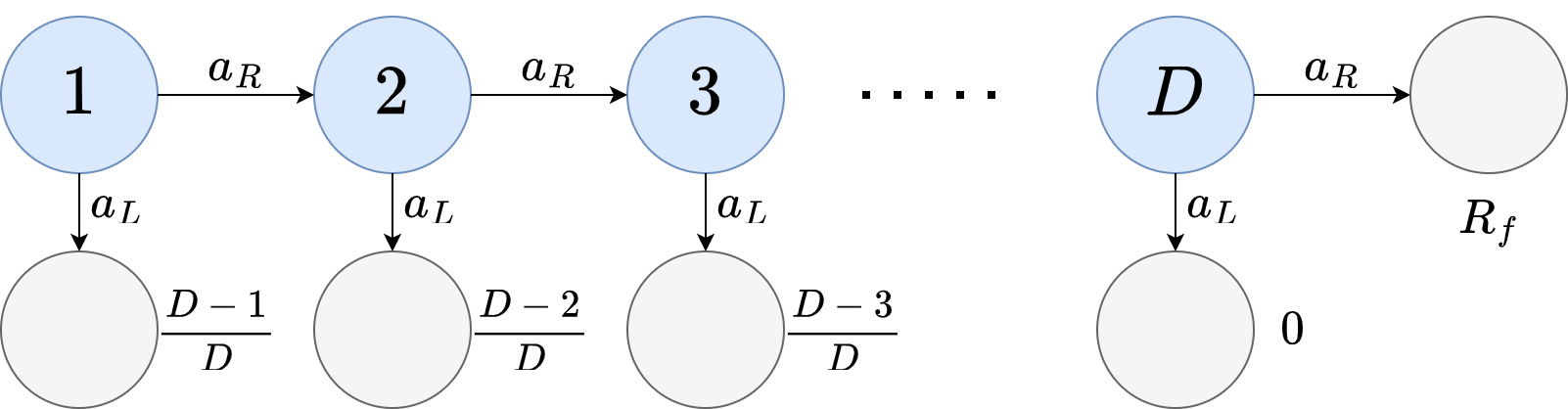}
            \caption{An illustration of the \textit{(modified) D-chain problem}, where 1 is the starting state, all transitions are deterministic and values next to states represents the reward for arriving in that state.}
            \label{fig:dchain_illustration_tres}
        \end{figure}

        Additionally, we can show that there is an MDP such that for any $\alpha$ MENTS will either not be consistent, or will take exponentially long in the size of the state space of the MDP. Below we state this formally and provide a proof outline in Theorem \ref{thrm:ments_bad_mdp}, which uses the MDP defined in Figure \ref{fig:adapted_chain}. 
        
        \begin{figure}
            \centering
            \includegraphics[width=0.8\textwidth]{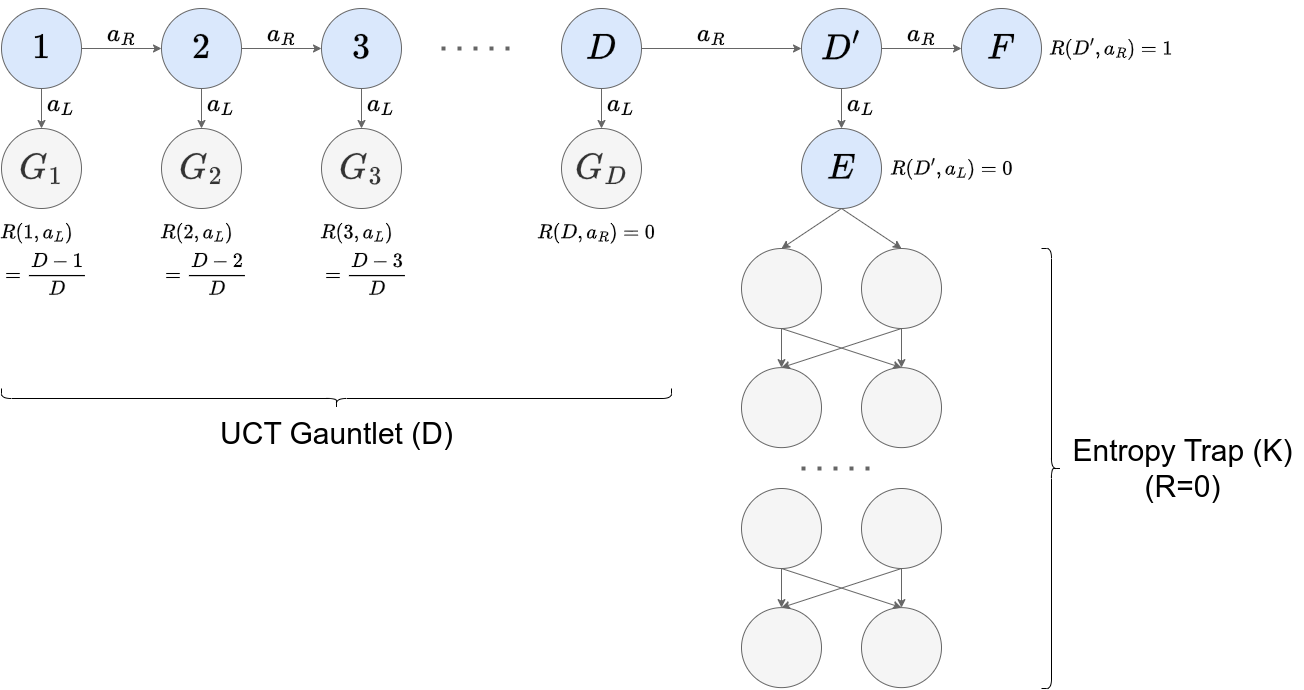}
            \caption{An illustration of the \textit{adapted-chain problem}, where 1 is the starting state, all transitions are deterministic and values next to states represents the reward for arriving in that state. The MDP is split into two sections, firstly, the UCT gauntlet refers to the D-chain part of the MDP, which is difficult for UCT algorithms, and secondly, the entropy trap, where the agent has to chose between states $E$ and $F$. The entropy trap consists of two chains of $K$ states, all of which give a reward of $0$ for visiting, but allows an agent to follow a policy with up to $\log(2)K$ entropy. So the optimal values for $E$ and $F$ are $\Vss{E}=\log(2)K$ and $\Vss{F}=1$ respectively.}
            \label{fig:adapted_chain}
        \end{figure}

        \begin{theorem} \label{thrm:ments_bad_mdp}
        		Consider a MENTS process with arbitrary temperature $\alpha$. There exists and MDP such that for any $\alpha$ that MENTS is either not consistent, or requires an exponential number of trials in the size of the state space. More precisely, either $\mathbb{E}\reg(s_0,\psi^n_{\text{MENTS}})\not\rightarrow 0$ or $\mathbb{E}\reg(1,\psi^n_{\text{MENTS}}) \geq c(1 - \frac{n}{k^{|S|}})$, which implies that $\mathbb{E}\reg(s_0,\psi^n_{MENTS})>0$ for $n<k^{|S|}$. 
        \end{theorem}
        
        \begin{proofoutline}
        		Proof by construction. Consider the adapted-chain MDP defined in Figure \ref{fig:adapted_chain}, which is parameterised by $D$ the length of the UCT gauntlet, and $K$ half the number of states in the entropy trap. To show our claim, we split the analysis into two cases.
        		
        		Case 1: $\alpha>\frac{1}{\log(2)K}$
        		
        		If $\alpha>\frac{1}{\log(2)K}$, the soft value of E is greater than one for any policy over the actions, and $\phi$ is a uniform policy (note that because there are no MDP rewards after $E$, $\phi$ is both the initial policy for MENTS and the optimal soft policy): 
	        \begin{align}
	        		\Vss{E} & = 0 + \alpha \cdot \mathcal{H}(\phi) \\
	        			& = \alpha \cdot \log(2)  K \\
	        			& > 1 \\
	        			& = \Vss{F}.
	        \end{align}
	        
	        Hence the optimal soft poilicy (which MENTS converges to) will recommend going to state $E$ and gathering $0$ reward. Hence in this case the simple regret will converge to $1$ as the optimal value is $V^*(1)=1$. That is $\mathbb{E}\reg(1,\psi^n_{\text{MENTS}}) \rightarrow 1 > 0$.
	        
	        Case 2: $\alpha\leq\frac{1}{\log(2)K}$
	        
	        In this case, we argue that with a low $\alpha$ that MENTS will only have a low probability of hitting state $D$, that is, it requires a lot of trials to garuntee that at least one trial has reached state $D$, which is necessary to get the reward of $1$ (and simple regret of $0$).
	        
	        First consider the composed and simplified soft backup on the adapted-chain problem for any $0<i<D$ to get:
	        \begin{align}
	        		\Vst{i-1}{N(i-1)} & = \alpha \log\left( \frac{1}{\alpha} \left( \frac{D-i+1}{D} + \Vst{i+1}{N(i+1)} \right)\right) \\
	        			& \leq \max\left(\frac{D-i+1}{D},\Vst{i}{N(i)}\right) + \alpha\log(2) \\
	        			& \leq \max\left(\frac{D-i+1}{D},\Vst{i}{N(i)}\right) + \frac{1}{K}
	        \end{align}
	        
	        where we used the property of log-sum-exp that $\alpha \log \sum_{i=1}^\ell \exp (x_i/\alpha) \leq \max_i (x_i) + \alpha \log(\ell)$. Assume that $K \geq D$ and $\Vst{D}{N(D)}=0$ (we will check these assumptions are valid later). Then using an induction hypothesis of $\Vst{i}{N(i)}\leq \frac{D-i+1}{D}$ we have:
	         \begin{align}
	         	\Vst{i-1}{N(i-1)} & \leq \max(\frac{D-i+1}{D},\Vst{i}{N(i)}) + \frac{\log(2)}{\log(K)} \\
	         		& = \frac{D-i+1}{D} + \frac{1}{K} \\
	         		& \leq \frac{D-i+1}{D} + \frac{1}{D} \\
	         		& = \frac{D-i+2}{D} = \frac{D-(i-1)+1}{D}
	         \end{align}
	         
	         We can then show that the probability the event $Y(n)$ that MENTS visits state $D$ in its $n$th trial is less than $2^{-D}$, given $\Vst{D}{N(D)}=0$. Because we have just shown that $\Vst{i}{N(i)}\leq \frac{D-i+1}{D} = \Vst{G_{i-1}}{N(G_{i-1})}$, using a symmetry argument we must have $\psi^n_{\text{MENTS}}(a_R|i) < \frac{1}{2}$, and as such, we must have $\Pr\left(Y(n) \middle| \Vst{D}{N(D)}=0\right) < \frac{1}{2^D}$. 
	         
	         Then let $Z(n)=\neg \bigcup_{j=1}^n Y(j)$ be the event that no trial of MENTS has visited state $D$ in any of the first $n$ trials. And note that $Z(n)$ implies that $\Vst{D}{N(D)}=0$:
	         \begin{align}
	         	\Pr(Z(n)) 
	         		&= \Pr\left(Z(n) 	              \cap \Vst{D}{N(D)}=0\right) 													\\
	         		&= \Pr\left(\neg Y(n) \cap Z(n-1) \cap \Vst{D}{N(D)} = 0 \right) 													\\
	         		&= \Pr\left(\neg Y(n) \middle|    Z(n-1) \cap \Vst{D}{N(D)} = 0 \right) 	\Pr\left( Z(n-1) \cap \Vst{D}{N(D)} = 0 \right) 	\\
	         		&\geq (1-2^{-D}) 													    \Pr\left( Z(n-1) \cap \Vst{D}{N(D)} = 0 \right) 	\\
	         		&= ... 																											\\
	         		&\geq (1-2^{-D})^n \Pr\left( Z(0) \cap \Vst{D}{N(D)}=0\right) \\
	         		&= (1-2^{-D})^n \\
	         		&\geq 1-n2^{-D},
	         \end{align}
	         where in the penultimate line we used $\Pr\left( Z(0) \cap \left(\Vst{D}{N(D)}=0\right)\right)=1$, as $Z(0)$ and $\Vst{D}{N(D)}=0$ are vacuously true at the start of running the algorithm, and in the final line we used Bernoulli's inequality.
	         
	         Informally, $Z(n)$ implies that $V^{\psi^n_{\text{MENTS}}}(1) \leq \frac{9}{10}$, as no trial has even reached $D$ to be able to reach the reward of $1$ from $F$. And hence the expected simple regret in this environment of MENTS can be bounded below as follows:
	         \begin{align}
	         	\mathbb{E}\reg(1,\psi^n_{\text{MENTS}}) & \geq \left(1-\frac{9}{10}\right) \Pr(Z(n)) \\
	         		& \geq \frac{1}{10} \left( 1-n2^{-D} \right).
	         \end{align}
	         
	         Finally, setting $K=D$, we must have that $D>\frac{|S|}{3}$ and so $\mathbb{E}\reg(1,\psi^n_{\text{MENTS}}) \geq 1 - \frac{n}{\sqrt[3]{2}^{|S|}}$, which is greater than $0$ for $n < \sqrt[3]{2}^{|S|}$. That is $c=0.1$ and $k=\sqrt[3]{2}$.
	        
        \end{proofoutline}

    \subsection{DENTS results} \label{appsec:dents_proofs}
    
        In this subsection we provide results related to DENTS. Similarly to the MENTS results, we begin by providing concentration inequalities for the estimated soft values at each node, around the optimal soft value. And after, the concentration inequalities are then used to show bounds on the simple regret of DB-MENTS, however, this time no constraints on the temperature are required.

        \begin{lemma} \label{lem:dents_val_induction_step}
            Consider a DENTS process. Let $s_t\in\cl{S}$, with $1\leq t \leq H$. If for all $s_{t+1}\in\suc{s_t}$ we have some $C_{s_{t+1}},k_{s_{t+1}}>0$ for any $\varepsilon_{s_{t+1}}>0$:
            \begin{align}
                \Pr\left(\left| \Vt{s_{t+1}}{N(s_{t+1})} - V^*(s_{t+1}) \right| > \varepsilon_{s_{t+1}} \right) 
                    &\leq C_{s_{t+1}}\exp\left( -k_{s_{t+1}}\varepsilon_{s_{t+1}}^2 N(s_{t+1}) \right), 
            \end{align}
            then there is some $C,k>0$, for any $\varepsilon>0$:
            \begin{align}
                \Pr\left(\left| \Vt{s_{t}}{N(s_{t})} - V^*(s_t) \right| > \varepsilon \right) 
                    &\leq C\exp\left( -k\varepsilon^2 N(s_{t}) \right).
            \end{align}
        \end{lemma}
        \begin{proof}
            Given the assumptions and by Lemmas \ref{lem:stochastic_step}, \ref{lem:union_bound} and \ref{lem:sa_to_s}, for some $C,k>0$ we have for any $\varepsilon_1>0$:
            \begin{align}
                \Pr\left(\forall a_t\in\cl{A}. \left|\Qt{s_t}{a_t}{N(s_t,a_t)}-Q^*(s_t,a_t)\right|\leq \varepsilon_1\right) &> 1-C \exp(-k \varepsilon_1^2 N(s_t)).
            \end{align}
            
            Let $\varepsilon >0$, and set $\varepsilon_1 = \min(\varepsilon,\Delta_{\cl{M}}/2)$. So with probability at least $1-C_1 \exp(-k_1 \varepsilon_1^2 N(s_t))$ we have for any $a_t$ that:
            \begin{align}
                Q^*(s_t,a_t) - \varepsilon_1 \leq \Qt{s_t}{a_t}{N(s_t,a_t)} &\leq Q^*(s_t,a_t) + \varepsilon_1. \label{local:ten}
            \end{align}
            
            Let $a^*=\max_{a\in\cl{A}} Q^*(s_t,a)$. Using $\varepsilon_1 \leq \Delta_{\cl{M}}/2$ in (\ref{local:ten}), then for any $a\in\cl{A}-\{ a^*\}$:
            \begin{align}
                \Qt{s_t}{a}{N(s_t,a)} \leq& Q^*(s_t,a) + \Delta_{\cl{M}}/2 \\
                    \leq& Q^*(s_t,a^*) - \Delta_{\cl{M}}/2 \\
                    \leq& \Qt{s_t}{a^*}{N(s_t,a^*)},
            \end{align}
            and hence $\argmax_{a\in\cl{A}} \Qt{s_t}{a}{N(s_t,a)} = a^*$. As a consequence:
            \begin{align}
                \Vt{s_t}{N(s_t)} &= \max_a \Qt{s_t}{a}{N(s_t,a)} = \Qt{s_t}{a^*}{N(s_t,a^*)}. \label{local:editing_one}
            \end{align}
            
            Then using (\ref{local:ten}) with $a_t=a^*$ (noting $V^*(s_t)=Q^*(s_t,a^*)$), using (\ref{local:editing_one}) and using $\varepsilon_1 \leq \varepsilon$ gives:
            \begin{align}
                V^*(s_t) - \varepsilon
                    \leq& V^*(s_t) - \varepsilon_1 \\
                    \leq& \Vt{s_t}{N(s_t)} \\
                    \leq& V^*(s_t) + \varepsilon_1 \\
                    \leq& V^*(s_t) + \varepsilon. 
            \end{align}
            
            Hence:
            \begin{align}
                \Pr\left(\left|\Vt{s_t}{N(s_t)}-V^*(s_t)\right| > \varepsilon\right) 
                    \leq C \exp(-k \varepsilon_1^2 N(s_t)) 
                    \leq C \exp(-k \varepsilon^2 N(s_t)),
            \end{align}
            which is the result.
        \end{proof}

        Similarly to the MENTS section, Lemma \ref{lem:dents_val_induction_step} provides an inductive step, which is used in Theorem \ref{thrm:dents_val_converge} to show concentration inequalities at all states that DENTS visits.
        
        \begin{theorem} \label{thrm:dents_val_converge}
            Consider a DENTS process, let $s_t\in\cl{S}$ then there is some $C,k>0$ for any $\varepsilon>0$:
            \begin{align}
                \Pr\left(\left| \Vt{s_{t}}{N(s_{t})} - V^*(s_{t}) \right| > \varepsilon \right) 
                    &\leq C\exp\left( -k\varepsilon^2 N(s_{t}) \right).
            \end{align}
            Moreover, at the root node $s_0$ we have:
            \begin{align}
                \Pr\left(\left| \Vt{s_{0}}{N(s_{0})} - V^*(s_0) \right| > \varepsilon \right) 
                    &\leq C\exp\left( -k\varepsilon^2 n \right). \label{local:four}
            \end{align}
        \end{theorem}
        \begin{proof}
            The result holds for $t=H+1$ because $\Vt{s_{H+1}}{N(s_{H+1})}=V^*(s_{H+1})=0$. Hence the result holds for all $t=1,...,H+1$ by induction using Lemma \ref{lem:dents_val_induction_step}. Noting that $N(s_0)=n$ gives (\ref{local:four}).
        \end{proof}

        Again, the concentration inequalities are used to show that the simple regret tends to zero exponentially, and therefore that DENTS will be exponentially likely in the number of visits to recommend the optimal standard action at every node.
        
        \begin{lemma} \label{lem:dents_imm_simple_regret}
            Consider a DENTS process. Let $s_t\in\cl{S}$, with $1\leq t \leq H$ then there is some $C',k'>0$ such that:
            \begin{align}
                \bb{E} \reg_I(s_{t},\psi^n_{\textnormal{DENTS}}) &\leq C'\exp\left( -k' N(s_{t}) \right).
            \end{align}
        \end{lemma}
        \begin{proof}
            Let $a^*$ be the locally optimal standard action, so $a^*=\argmax_{a\in\cl{A}} Q^*(s_t,a)$. By Theorem \ref{thrm:dents_val_converge} and Lemmas \ref{lem:stochastic_step} and \ref{lem:sa_to_s} there exists $C_1,k_1>0$ such that for all $\varepsilon_1>0$:
            \begin{align}
                \Pr\left(\forall a_t\in\cl{A}. \left|\Qt{s_t}{a_t}{N(s_t,a_t)}-Q^*(s_t,a_t)\right|\leq \varepsilon_1\right) &> 1-C_1 \exp(-k_1 \varepsilon_1^2 N(s_t)).
            \end{align}
            
            Setting $\varepsilon_1=\Delta_{\cl{M}}/2$ then gives with probability $1-C_1 \exp(-k_2N(s_t))$ (where $k_2=k_1\Delta_{\cl{M}}/2$) for all actions $a\in\cl{A}$ that:
            \begin{align}
                Q^*(s_t,a) - \Delta_{\cl{M}}/2 \leq \Qt{s_t}{a}{N(s_t,a_t)} \leq Q^*(s_t,a) + \Delta_{\cl{M}}/2.
            \end{align}
            
            And hence, with probability $1-C_1 \exp(-k_2N(s_t))$, for all $a_t\in\cl{A}-\{a^*\}$ we have:
            \begin{align}
                \Qt{s_t}{a_t}{N(s_t,a)}
                    \leq& Q^*(s_t,a_t) + \Delta_{\cl{M}}/2 \\
                    \leq& Q^*(s_t,a^*) - \Delta_{\cl{M}}/2 \\
                    \leq& \Qt{s_t}{a^*}{N(s_t,a^*)}. \label{local:eleven}
            \end{align}
            
            Negating the probability of (\ref{local:eleven}) then gives:
            \begin{align}
                \Pr\left(\Qst{s_t}{a}{N(s_t,a)} > \Qst{s_t}{a^*}{N(s_t,a^*)}\right) \leq C_1 \exp(-k_2N(s_t))
            \end{align}
            
            Finally, the bound on the expected immediate regret follows:
            \begin{align}
                & \bb{E}\reg_I(s_t,\psi^n_{\text{DENTS}})  \\
                    =& \sum_{a\in\cl{A}-\{a^*\}} 
                        \left(V^*(s_t)-Q^*(s_t,a)\right) \Pr\left(\psi^n_{\text{DENTS}}(s_t) =a \right) \\
                    =& \sum_{a\in\cl{A}-\{a^*\}} 
                        \left(V^*(s_t)-Q^*(s_t,a)\right) \Pr\left(a = \argmax_{a'} \Qt{s_t}{a'}{N(s_t,a)} \right) \\
                    \leq& \sum_{a\in\cl{A}-\{a^*\}} 
                        \left(V^*(s_t)-Q^*(s_t,a)\right) \Pr\left(\Qt{s_t}{a}{N(s_t,a)} > \Qt{s_t}{a^*}{N(s_t,a^*)}  \right) \\
                    \leq&  \sum_{a\in\cl{A}-\{a^*\}} 
                        \left(V^*(s_t)-Q^*(s_t,a)\right) C_1 \exp(-k_2N(s_t)),
            \end{align}
            and setting $k'=k_2$ and $C'=C_1\sum_{a\in\cl{A}-\{a^*\}} \left(V^*(s_t)-Q^*(s_t,a)\right)$ gives the result.
        \end{proof}

        Again similarly to before, by using Lemma \ref{lem:dents_imm_simple_regret}, we can convert the bound on the immediate simple regret to a bound on the simple regret.
        
        \begin{theorem} \label{thrm:dents_simple_regret_converge}
            Consider a DENTS process. Let $s_t\in\cl{S}$, with $1\leq t \leq H$ then there is some $C',k'>0$ such that:
            \begin{align}
                \bb{E} \reg(s_{t},\psi^n_{\textnormal{DENTS}}) &\leq C'\exp\left( -k' N(s_{t}) \right).
            \end{align}
            
            Moreover, at the root node $s_0$:
            \begin{align}
                \bb{E} \reg(s_{0},\psi^n_{\textnormal{DENTS}}) &\leq C'\exp\left( -k'n) \right). \label{local:five}
            \end{align}
        \end{theorem}
        \begin{proof}
            Theorem holds as a consequence of Corollary \ref{cor:imm_to_full_simple_regret} and Lemma \ref{lem:dents_imm_simple_regret}. To arrive at (\ref{local:five}) note that $N(s_0)=n$.
        \end{proof}

        Finally, Theorem \ref{thrm:dents} follows from the results that we have just shown.

        Finally, we have shown Theorem \ref{thrm:dents}, as it is a subset of what we have already shown.
        \begin{customthm}{4.2}
            For any MDP $\cl{M}$, after running $n$ trials of the DENTS algorithm with a root node of $s_0$, if $\beta$ is bounded above and $\beta(m)\geq 0$ for all $m\in\bb{N}$, then there exists constants $C,k>0$ such that for all $\varepsilon>0$ we have 
            $\bb{E}[\reg(s_0,\psi^n_{\textnormal{DENTS}})] \leq C\exp(-kn)$, and also $\Vt{s_0}{N(s_0)} \rap V^*(s_0)$ as $n\rightarrow\infty$.
        \end{customthm}
        \begin{proof}
            The simple regret bound follows immediately from Theorem \ref{thrm:dents_simple_regret_converge}. Using Theorem \ref{thrm:dents_val_converge} we must have  $\Pr\left(\left| \Vt{s_{0}}{N(s_{0})} - V^*(s_0) \right| > \varepsilon \right) \leq C\exp\left( -k\varepsilon^2 n \right) \rightarrow 0$ as $n\rightarrow \infty$, and hence $\Vt{s_{0}}{N(s_{0})}$ converges in probability to $V^*(s_0)$.
        \end{proof}

    \subsection{BTS results} \label{appsec:bts_proofs}
    
        Recall that the BTS process is a special case of the DENTS process, where the decay function is set to $\beta(m)=0$. As such, all of the results for DENTS processes also hold for BTS processes, and specifically Theorem \ref{thrm:bts} must hold.

        \begin{customthm}{4.1}
            For any MDP $\cl{M}$, after running $n$ trials of the BTS algorithm with a root node of $s_0$, there exists constants $C,k>0$ such that for all $\varepsilon>0$ we have $\bb{E}[\reg(s_0,\psi^n_{\textnormal{BTS}})] \leq C\exp(-kn)$, and also $\Vt{s_0}{N(s_0)} \rap V^*(s_0)$ as $n\rightarrow\infty$.
        \end{customthm}
        \begin{proof}
            Follows from setting $\beta(m)=0$ and using Theorem \ref{thrm:dents}.
        \end{proof}

    \subsection{Results for using average returns in Boltzmann MCTS processes} \label{sec:ar_proofs}

        In this section we give informal proof outlines of our results for AR-BTS and AR-DENTS. To begin with we define the average return $\bar{V}^{N(s)}(s)$ for a decision node at $s$, and recall the definition of $\bar{Q}^{N(s,a)}(s, a)$:

        \begin{align}
            \bar{V}^{N(s_t)+1}(s_t) &= \bar{V}^{N(s_t)}(s_t) + \frac{\bar{R}(s_t) - \bar{V}^{N(s_t)}(s_t)}{N(s_t) + 1},  \label{appeq:ar_v} \\
            \bar{Q}^{N(s_t,a_t)+1}(s_t, a_t) &= \bar{Q}^{N(s_t,a_t)}(s_t, a_t) + \frac{\bar{R}(s_t,a_t) - \bar{Q}^{N(s_t,a_t)}(s_t, a_t)}{N(s_t, a_t) + 1},  \label{appeq:ar_q}
        \end{align}
        where $\bar{R}(s_t)=\sum_{i=t}^H R(s_i,a_i)$ and $\bar{R}(s_t, a_t)=\sum_{i=t}^H R(s_i,a_i)$. Note that these average return values also satisfy the equations:
        \begin{align}
            \bar{V}^{N(s_t)}(s_t) &= \sum_{a\in\cl{A}} \frac{N(s_t,a)}{N(s_t)} \bar{Q}^{N(s_t,a_t)}(s_t, a_t), \label{appeq:ar_v_rel} \\
            \bar{Q}^{N(s_t,a_t)}(s_t, a_t) &= R(s_t,a_t) + \sum_{a\in\cl{A}} \frac{N(s')}{N(s_t,a_t)} \bar{V}^{N(s')}(s'). \label{appeq:ar_q_rel}
        \end{align}

        Firstly, we show that using a non-decaying search temperature with AR-BTS is not guaranteed to recommend the optimal policy.
        \begin{customprop}{B.1}
            For any $\alpha_{\textnormal{fix}}>0$, there is an MDP $\cl{M}$ such that AR-BTS with $\alpha(m)=\alpha_{\textnormal{fix}}$ is not consistent: $\bb{E}[\reg(s_0,\psi^n_{\textnormal{AR-BTS}})] \not\to 0$ as $n\to\infty$. 
        \end{customprop}
        \begin{proofoutline}
            Consider the MDP given in Figure \ref{fig:ar_gen_mdp}. We can show inductively that (as $n\rightarrow\infty$) the value of $\bb{E}\bar{V}^{N(k)}(k)\leq 2E^{D-k+1} < 2$, where $E=e^2/(1+e^2)$ for $2\leq k \leq D$. The inductive step is as follows:
            \begin{align}
                \bb{E}\bar{V}^{N(k)}(k) =& \frac{\exp\left(\bar{V}^{N(k+1)}(k+1)\right)}{1+\exp\left(\bar{V}^{N(k+1)}(k+1)\right)} \bb{E}\bar{V}^{N(k+1)}(k+1) 
                    + \frac{1}{1+\exp\left(\bar{V}^{N(k+1)}(k+1)\right)} \cdot 0 \\
                    \leq& E \bb{E} \cdot \bar{V}^{N(k+1)}(k+1) \\
                    \leq& E \cdot 2E^{D-k} \\
                    =& 2E^{D-k+1}.
            \end{align}

            where we know that $\exp\left(\bar{V}^{N(k+1)}(k+1)\right)/\left(1+\exp\left(\bar{V}^{N(k+1)}(k+1)\right)\right) < E$, because the function $e^x/(1+e^x)$ is monotonically increasing in $x$, and $\bar{V}^{N(k+1)}(k+1) < 2$. Hence, by choosing an integer $D$ such that $D-1 \geq \log(1/3) / \log(E)$, we have $\bb{E}\bar{V}^{N(2)}(2)=2E^{D-1}\leq 2/3 < 1 = \bb{E}\bar{Q}^{N(1,a_2)}(1,a_2)$. 

            A full proof should show that AR-BTS does indeed converge to these expected values, possibly through concentration bounds. 
            \crtodo{Do the full proof, and show that AR-BTS converges to these value with non zero prob, hence the non zero simple regret.} 

            Hence, AR-BTS does not converge to a simple regret of zero, because the expected Q-values as $n\rightarrow\infty$ are $\bb{E}\bar{Q}^{N(1,a_2)}(1,a_2)=1$ and $\bb{E}\bar{Q}^{N(1,a_1)}(1,a_1)<2/3$, so AR-BTS would incorrectly recommend action $a_2$ from the root node. 
            \crtodo{Formally give the simple regret converges to something strictly greater than zero. Also can I even do that intersection to product equality? Doesn't that mean that they are independent. Are they independent?}
        \end{proofoutline}
        
        \begin{figure}
            \centering
            \includegraphics[width=0.7\textwidth]{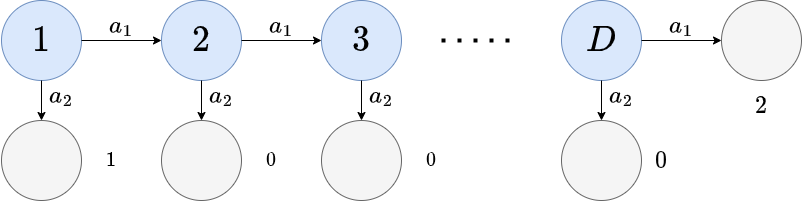}
            \caption{An MDP that AR-BTS will not converge to recommending the optimal policy on, for a large enough value of $D$.}
            \label{fig:ar_gen_mdp}
        \end{figure}

        Because the search temperature is now decayed, we need to show that the exploration term in the search policies leads to actions being sampled infinitely often.
        
        \begin{lemma} \label{lem:inf_often_action_select}
            Let $\rho^k$ be an arbitrary policy, and let a search policy be $\pi^k(a)=(1-\lambda(k))\rho^k(a) + \lambda(k) / |\cl{A}|$, with $\lambda(k)=\min(1,\epsilon/\log(e+k))$, where $\epsilon\in(0,\infty)$. Let $a^k\sim \pi^k$, and let $M(a)$ be the number of times action $a$ was sampled out of $m$ samples, i.e. $M(a)=\sum_{i=1}^m \one[a^i=a]$. Then for all $a\in\cl{A}$ we have $M(a)\rightarrow\infty$ as $m\rightarrow \infty$.
        \end{lemma}
        \begin{proofoutline}
            This lemma restated means that our search policies will select all actions infinitely often. To show this, we argue by contradiction, and suppose that there is some $b\in\cl{A}$ such that after $\ell$ samples that $b$ is never sampled again. The probability of this happening at the $m$th sample is then:
            \begin{align}
                \Pr\left(\bigcap_{i=\ell}^m a^i \neq b \right)
                    &= \prod_{i=\ell}^m \Pr(a^i \neq b) \\
                    &\leq \prod_{i=\ell}^m \left( 1 - \frac{\epsilon}{|\cl{A}|\log(e+i)} \right) \\
                    &\leq \left( 1 - \frac{\epsilon}{|\cl{A}|\log(e+m)} \right)^{m-\ell} \\
                    &\rightarrow 0,
            \end{align}
            as $m\rightarrow\infty$. Which is a contradiction. 
            \crtodo{Actually show the limit is zero using some convergence result. Also explain the maths a bit more. Also be more precise around the use of epsilon}
        \end{proofoutline}

        \begin{customthm}{B.2} 
            For any MDP $\cl{M}$, if $\alpha(m)\rightarrow 0$ as $m\rightarrow\infty$ then $\bb{E}[\reg(s_0,\psi^n_{\textnormal{AR-BTS}})]\rightarrow 0$ as $n\rightarrow\infty$, where $n$ is the number of trials.
        \end{customthm}
        \begin{proofoutline}
            We can argue that average returns converge in probability by induction similar to the previous proofs. Suppose that $\bar{Q}^{N(s_t,a_t)}(s_t, a_t) \rap Q^*(s_t,a_t)$ as $N(s_t,a_t)\rightarrow\infty$. From Lemma \ref{lem:inf_often_action_select}, we know that if $N(s_t)\rightarrow\infty$, then $N(s_t,a_t)\rightarrow\infty$.

            All that remains to show that $\bar{V}^{N(s_t)}(s_t) \rap V^*(s_t)$ as $N(s_t)\rightarrow\infty$ is that $\frac{N(s_t,a)}{N(s_t)}\rightarrow \one[a=a^*]$. If that is true, then by considering Equation (\ref{appeq:ar_v_rel}), we can see that in the limit $\bar{V}^{N(s_t)}(s_t)$ tends to $\bar{Q}^{N(s_t,a*)}(s_t,a*)$. 
            \crtodo{formally prove that the distribution and value converges. Will need some epsilons and use the equation referenced for the value. the distribution bit needs to show that eventually the samples are dominated by the greedy boltzmann policy, and not the exploration part}

            Intuitively we can see that $\frac{N(s_t,a)}{N(s_t)}\rightarrow \one[a=a^*]$ as the AR-BTS search policy converges to a greedy policy as $\alpha(m)\rightarrow 0$. 
            \crtodo{technically only the rho bit does, shrugs}
        \end{proofoutline}

        \begin{customthm}{B.3} 
            For any MDP $\cl{M}$, if $\alpha(m)\rightarrow 0$ and $\beta(m)\rightarrow 0$ as $m\rightarrow\infty$ then $\bb{E}[\reg(s_0,\psi^n_{\textnormal{AR-DENTS}})]\rightarrow 0$ as $n\rightarrow\infty$, where $n$ is the number of trials.
        \end{customthm}
        \begin{proofoutline}
            Proof is similar to the proof for Theorem \ref{thrm:ar_bts}.
        \end{proofoutline}

        As a final note, using BTS or DENTS with a decaying search temperature $\alpha(m)$ will lead to a consistent algorithm, although they would not admit an exponential regret bound. Proofs would be nearly identical to Theorems \ref{thrm:ar_bts} and \ref{thrm:ar_dents}.

    
{
\small

\bibliographystyleapp{plain}
\bibliographyapp{mybib}
}

\end{appendices}


\end{document}